\documentclass[runningheads]{llncs}

\usepackage{eccv}
\usepackage{eccvabbrv}
\usepackage{graphicx}
\usepackage{booktabs}

\usepackage[accsupp]{axessibility}  %
\usepackage{hyperref}

\usepackage{orcidlink}

\usepackage[numbers]{natbib}
\usepackage{amssymb}
\usepackage{amsmath}
\usepackage{multirow}
\usepackage{booktabs}
\usepackage{lscape}    %
\usepackage{geometry}  %
\usepackage[table]{xcolor}
\usepackage{xcolor}
\usepackage{xspace}
\usepackage{tablefootnote}
\definecolor{matplotlibblue}{RGB}{20, 90, 140}

\newtheorem{assumption}{Assumption}
\newtheorem{observation}{Observation}

\newcommand{\approach}{{\texttt{DAVIS}}}

\begin{document}

\title{$\approach$: OOD Detection via Dominant Activations and Variance for Increased Separation} 

\titlerunning{OOD Detection via $\approach$}

\author{Abid Hassan \orcidlink{0000-1111-2222-3333} \and
Tuan Ngo\orcidlink{1111-2222-3333-4444} \and
Saad Shafiq\orcidlink{2222--3333-4444-5555} \and
Nenad Medvidovic\orcidlink{2222--3333-4444-5555}}

\authorrunning{A.~Hassan et al.}

\institute{University of Southern California, Los Angeles CA \\
\email{\{mdskabid, tkngo, sshafiq, neno\}@usc.edu}}

\maketitle

\begin{abstract}

    Detecting out-of-distribution (OOD) inputs is a critical safeguard for deploying ML models in real-world settings. Most post-hoc detection methods operate on penultimate feature representations derived from global average pooling (GAP), a lossy operation that discards valuable distributional statistics from activation maps prior to pooling. We argue that these overlooked statistics, particularly channel-wise variance and dominant (maximum) activations, are highly discriminative for OOD detection. We introduce $\approach$, a simple and broadly applicable post-hoc technique that enriches feature representations by incorporating these informative statistics, directly addressing the information loss induced by GAP. The proposed method achieves strong performance across diverse architectures and models, significantly reducing the false positive rate (FPR95). Our analysis further reveals the underlying mechanism driving these improvements, providing a principled foundation for moving beyond GAP-based representations in OOD detection.  Our code is available at: https://github.com/epsilon-2007/DAVIS


\end{abstract}

\section{Introduction}
\label{sec: introduction}

    \looseness-1
    Safe deployment of ML models in the open world hinges on a critical capability: recognizing and handling inputs that fall outside their training distribution. When faced with such out-of-distribution (OOD) data, which consist of samples from novel contexts or unknown classes, a robust model should signal its uncertainty rather than making a confident and likely incorrect prediction~\cite{msp, msp_failed_01, msp_failed_02}. The ability to reliably detect OOD inputs is thus paramount for safety-critical systems, from medical diagnosis~\cite{medical_diagnosis_01, medical_diagnosis_02} to autonomous driving~\cite{autonomous}.

    \looseness-1
    To address this challenge, numerous studies have explored approaches for detecting OOD samples in deep learning. While some methods modify the model's training objective~\cite{training_time_10, cider, ssn}, a particularly prominent line of work focuses on \textit{post-hoc} methods, which do not require costly retraining. In the post-hoc setting, the majority of prior work has concentrated on estimating OOD uncertainty from the activation space of a neural network, for example by leveraging model outputs~\cite{msp, deep_ensembles, odin, energy, ReAct, DICE, ASH, SCALE}, feature representations~\cite{knn, maha_distance}, or gradient information~\cite{GradNorm}. Despite their differences, these approaches share a fundamental and largely unexamined limitation. They operate on feature vectors produced by \textit{global average pooling (GAP)}. Although effective for classification, GAP is inherently a lossy operation because it summarizes each channel's spatial activation map, which contains a rich distribution of spatial responses, into a single scalar value. In doing so, it permanently discards potentially informative cues about the spread and peak intensity of the activation distribution, which we argue are powerful signals for identifying anomalies.

    \looseness-1

        \begin{figure*}[t]
        \centering
        \includegraphics[width=0.675\textwidth]{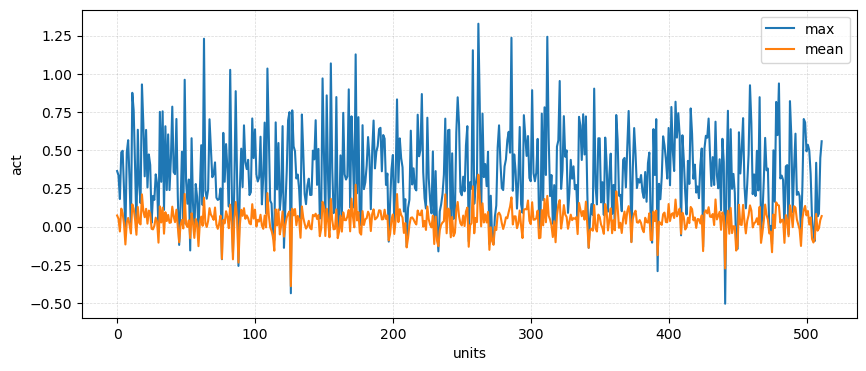}
        \vspace{-3mm}
        \caption{Dominant activations provide a stronger OOD signal than mean activations. The plot shows the average activation gap between ID (CIFAR-10) and OOD (Texture) samples for each penultimate-layer unit of a pre-trained ResNet-18. The gap from dominant (maximum) activations is consistently larger than that from mean activations.}
        \label{fig:raw_delta_plot}
        \vspace{-3mm}
    \end{figure*}
    
    \begin{figure*}[!ht]
        \centering
        \includegraphics[width=0.675\textwidth]{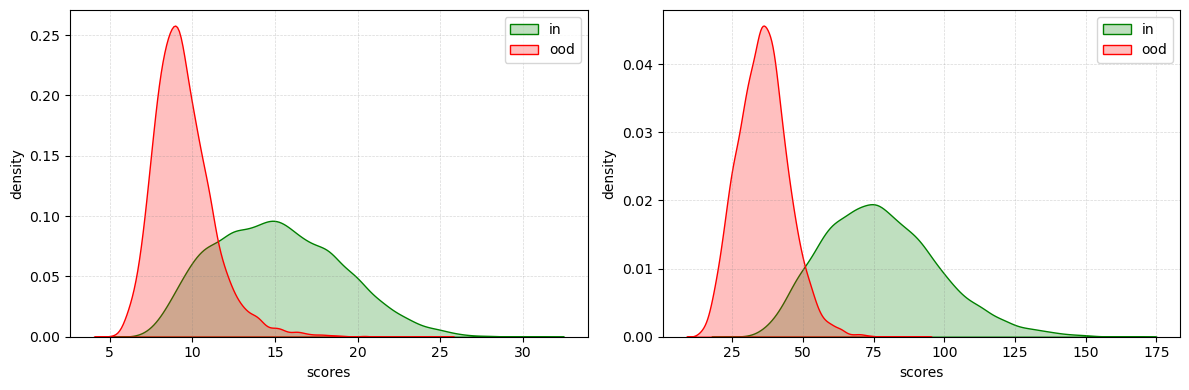}
        \vspace{-3mm}
        \caption{Using dominant activations improves OOD score separation using ResNet-18. \textit{Left:} Energy scores based on mean activations show substantial overlap between ID (CIFAR-10) and OOD (Texture) samples, resulting in poor separability. \textit{Right:} Using dominant activations shifts OOD scores away from ID scores, improving separation.}
        \label{fig:density_delta_plot}
        \vspace{-3mm}
    \end{figure*}

    \looseness-1
    Our work is motivated by the observation that pre-pooled activation distributions contain highly discriminative OOD signals that are suppressed by GAP. As shown in Figures~\ref{fig:raw_delta_plot} and~\ref{fig:density_delta_plot}, statistics such as dominant activation amplify ID-OOD separation, leading to more reliable OOD scoring. We reveal a structural mismatch between classification-oriented feature learning and the requirements of OOD detection: while GAP promotes spatial invariance beneficial for classification, it suppresses complementary activation statistics that encode distributional uncertainty. Based on this insight, we propose $\approach$ (\textit{Dominant Activations and Variance for Increased Separation}), a simple post-hoc plug-and-play method that enriches the penultimate feature representation by incorporating channel-wise maximum and variance statistics before downstream scoring. $\approach$ complements existing baselines by providing a more expressive feature representation. In summary, our contributions are as follows:

\begin{enumerate}
\vspace{-2mm}
    \item We identify information loss from global average pooling as a key yet overlooked weakness in post-hoc OOD detection. To address this, we propose $\approach$, a simple plug-and-play module that enriches feature representations by incorporating complementary pre-pooling statistics. 
    
    \item We evaluate $\approach$ on widely used benchmarks across diverse architectures, including Swin-B, ConvNeXt-B, ResNet-50, EfficientNet-B0, DenseNet-121, and MobileNet-V2. Compared to the previous best method on ResNet-50, ConvNeXt-B, and Swin-B,  $\approach$ achieves respective FPR95 reductions of 7.96\%,  16.53\%, and 25.09\% on a large-scale ImageNet benchmark.

    \item Additionally, we evaluate $\approach$ on a near-OOD benchmark using Swin-B, ConvNeXt-B, and ResNet-50, where it attains FPR95 reductions of 25.52\%, 12.01\%, and a modest 2.56\%, respectively. 
    
    \item We provide a statistical analysis, grounded in empirical evidence, that reveals the underlying mechanism of $\approach$ (Appendix~\ref{appendix: analysis}). This provides a principled justification for moving beyond the mean and leveraging richer distributional statistics for robust OOD detection.
\end{enumerate}

\section{Related Work}
\label{sec: related work}
 
\looseness-1 Existing work in OOD detection can be broadly grouped into three  paradigms: post-hoc methods that operate on pre-trained classifiers, training-time regularization techniques that modify the learning objective, and generative models that learn the ID density. Our work, $\approach$, belongs to the post-hoc paradigm. %
In this section, we review each of the three paradigms. 

\looseness-1
\textbf{Post-Hoc Methods.} Early OOD detection research primarily focused on designing scoring functions in the logit space, such as MSP~\cite{msp}, MaxLogit~\cite{max_logit}, and ODIN~\cite{odin}. Although initially effective, these methods are prone to overconfidence~\cite{msp_failed_01, msp_failed_02, msp}, motivating the development of more robust alternatives such as the Energy score~\cite{energy}. Beyond the logit space, approaches including GradNorm~\cite{GradNorm} and GradOrth~\cite{gradorth} leverage gradient information for uncertainty estimation. Other methods explore the activation space, employing distance-based techniques such as Mahalanobis distance~\cite{maha_distance} and kNN~\cite{knn, ssn}, while approaches like ViM~\cite{ViM,CombOOD} combine information from both logit and feature spaces.

More recently, a line of work has focused on directly modifying the penultimate-layer feature representation to enhance separability. Methods such as ReAct~\cite{ReAct}, DICE~\cite{DICE}, ASH~\cite{ASH}, and SCALE~\cite{SCALE} improve detection by rectifying, sparsifying, or pruning activations. Similarly, BATS~\cite{BATS}, LAPS~\cite{laps}, CORES~\cite{CORES}, NN-Guide~\cite{NNGuide}, and CADRef~\cite{CADRef} exploit class-aware feature distributions and manifold geometry.  More recent methods, including fDBD~\cite{fdbd} and NCI~\cite{NCI}, focus on improving computational efficiency.

However, all these methods operate on feature vectors produced by GAP, which discards potentially discriminative spatial statistics. In contrast, our work addresses this limitation at its source by enriching the feature representation before any downstream scoring is applied. Unlike prior rectification-based methods (e.g., BATS, LAPS), DAVIS enriches rather than constrains activation statistics, directly targeting the information loss induced by GAP.

\looseness-1
\textbf{Training-Time Regularization.} This line of research focuses on improving OOD detection by modifying the model's training objective~\cite{cider, ssn}. The resulting methods often incorporate an auxiliary dataset of outliers~\citep{OECC} or introduce regularization terms that encourage the model to produce less confident predictions on OOD samples~\cite{training_time_00, training_time_03, training_time_05, density_2022}. While effective, these approaches require a more complex training process and access to relevant outlier data, which may not always be available.

\textbf{Generative Models.} An alternative paradigm for OOD detection uses generative models to estimate the density of the ID data~\cite{density_2017,density_ood_2019_01,density_ood_2020_00,density_ood_2020_01,density_ood_2020_04,ssn,fdbd}. The intuition is that OOD samples will lie in low-density regions of the learned data manifold. However, it has been shown that deep generative models can assign high likelihoods to structurally simple OOD samples~\cite{density_ood_failed}, making them less reliable than discriminative approaches for OOD detection. $\approach$ builds on the strong empirical performance and reliability of discriminative classifiers.
\section{Method}
\label{sec: methods}

    \looseness-1 This section formally presents $\approach$. Although GAP is effective for classification, it is suboptimal for OOD detection because it typically discards cues such as peak intensity and variance that carry discriminative information between ID and OOD samples. By incorporating these properties, which are lost through simple global averaging, $\approach$ produces a more discriminative feature representation for separating ID from OOD samples, as shown in Figure~\ref{fig:density_delta_plot}.

\subsection{Dominant Activations and Variance for Increased Separation}

    \looseness-1 
    $\approach$ is a post-hoc module designed to counteract the information loss from GAP. We create a more discriminative feature representation by extracting richer statistics from the pre-pooling activation maps, before they are averaged. To formalize this, we consider a standard supervised classification setting. Let $\mathcal{X}$ denote the input space and $\mathcal{Y} = \{1, 2, \cdots, C\}$ the output label space for neural network $\theta$ trained on dataset $\mathcal{D}_{\text{in}}$. 

    For an input image $\mathbf{x}$, the network produces a set of $n$ spatial activation maps $g(\mathbf{x}) \in \mathbb{R}^{n \times k \times k}$. The conventional penultimate-layer feature vector, $h(\mathbf{x}) \in \mathbb{R}^n$, is obtained via GAP, i.e., $h(\mathbf{x}) = \texttt{Avg}(g(\mathbf{x}))$. $\approach$ replaces or augments this vector by computing more descriptive statistics from each of the $n$ channels. For each channel, in addition to the GAP-based mean $\mu(\mathbf{x})$ (equivalent to $h(\mathbf{x})$), we compute the maximum (dominant) activation $m(\mathbf{x})$ and the standard deviation $\sigma(\mathbf{x})$. We propose two formulations of the $\approach$ feature vector, $h^{\approach}(\mathbf{x})$:\footnote{We use ``$\approach$'' to denote both variants; the context clarifies any differences.}
    
    \vspace{-1mm}
    \begin{enumerate}
        \item $\approach(\mu,\sigma)$  augments the mean activation with its corresponding channel-wise standard deviation $\sigma(\mathbf{x})$, scaled by a hyperparameter $\gamma$\ as shown in Equation~\ref{eq: davis_mu_sigma}. We discuss a detailed hyperparameter selection process in Section~\ref{subsec: hyper-parameter selection} and Appendix~\ref{appendix: reproducibility}.
            \begin{small}
                \vspace{-1mm} 
                \begin{equation}
                        h^{\approach(\mu,\sigma)}(\mathbf{x}) = \mu(\mathbf{x}) + \gamma\sigma(\mathbf{x})
                        \label{eq: davis_mu_sigma}
                \end{equation}
            \end{small}
            
        \item $\approach(m)$ replaces the feature vector $h(\mathbf{x})$ by \textit{dominant activation} $m(\mathbf{x})$ as follows:
        \begin{small}
        \vspace{-1mm}
            \begin{equation}
                h^{\approach(m)}(\mathbf{x}) = m(\mathbf{x})
                \label{eq: davis_m}
            \end{equation}
        \end{small}
        \end{enumerate}
    The procedure described above  modifies the activation level of the feature vector and aims to increase the separation between ID and OOD samples.

\subsection{OOD Scoring with $h^{\approach}(\mathbf{x})$ } %

    \looseness-1 Contrary to the use of GAP feature vector, we compute OOD scores using an enriched feature representation, $h^{\approach}(\mathbf{x})$. This enriched vector preserves complementary information such as peak intensity and variance, which are otherwise discarded by GAP. Importantly, $\approach$ is model-agnostic and can be seamlessly integrated with existing downstream scoring functions that directly or indirectly uses the penultimate layer to derive the scores, including logit-based~\cite{msp}, energy-based~\cite{energy}, gradient-based~\cite{GradNorm}, and KNN~\cite{knn} distance methods. By replacing the conventional GAP feature with our enriched representation, we enhance score separability between ID and OOD samples without requiring retraining.

    For instance, $h^{\approach}(\mathbf{x})$ is passed through the original fully connected layer (with weights $\mathbf{W}$ and bias $\mathbf{b}$) to produce modified logits (Eq.~\ref{eq: logits}), which are used to compute logit-based \texttt{MSP}~\cite{msp}, energy-based \texttt{Energy}~\cite{energy}, and gradient-based \texttt{GradNorm}~\cite{GradNorm}. Specifically, \texttt{GradNorm} defines the OOD score as the L1 norm of gradients of the KL divergence between the model output and a uniform distribution. In case of distance-based \texttt{KNN}~\cite{knn} method, $h^{\approach}(\mathbf{x})$ is directly used to compute the OOD score as the Euclidean distance to the $k$ nearest ID features. Equation~\ref{eq: scoring_energy} illustrates the integration of $\approach$ with the energy-based scoring function, while other scoring formulations are detailed in Appendix~\ref{appendix: baseline methods}. Additionally, $\approach$ is compatible with existing techniques such as \texttt{ODIN}~\cite{odin}, \texttt{ReAct}~\cite{ReAct}, \texttt{DICE}~\cite{DICE}, \texttt{ASH}~\cite{ASH}, and \texttt{SCALE}~\cite{SCALE}. Brief summaries of these methods are provided in Appendix~\ref{appendix: baseline methods}.

    \begin{small}
        \begin{equation}
            f^{\approach}(\mathbf{x}) = \mathbf{W}^\top h^{\approach}(\mathbf{x}) + \mathbf{b}
            \label{eq: logits}
        \end{equation}
    \end{small}
    \vspace{-4mm}
    \begin{small}
        \begin{equation}
             S_{\texttt{Energy}}(\mathbf{x}; \theta) = - \mathbf{E}_{\theta}(\mathbf{x}) = \log{ \left( \sum_{j=1}^{C} \exp^{( f_j^{\approach}(\mathbf{x}; \theta) )} \right) }
            \label{eq: scoring_energy}
        \end{equation}
    \end{small}
    \vspace{-4mm}

\subsection{OOD Detection with {\approach}}

    The goal of OOD detection is to learn a decision boundary $G_{\lambda}(\mathbf{x} ; \theta)$  that classifies a test sample $\mathbf{x} \in \mathcal{X}$: %
    \vspace{-2mm}
    
        \begin{equation}
                G_{\lambda}(\mathbf{x} ; \theta) =
                \begin{cases}
                \mathrm{ID} & \text{if } \mathbf{x} \sim \mathcal{D}_{\text{in}} \\
                \mathrm{OOD} & \text{if } \mathbf{x} \sim \mathcal{D}_{\text{out}}
                \end{cases} \\
                =
                \begin{cases}
                    \mathrm{ID}  & \text{if } S(\mathbf{x}; \theta) \ge \lambda \\
                    \mathrm{OOD} & \text{if } S(\mathbf{x}; \theta) < \lambda
                \end{cases}
            \label{eq:decision_boundary}
        \end{equation} 
    where $S(\mathbf{x}; \theta)$ is a downstream OOD scoring function, and by convention~\cite{energy} $\lambda$ is a threshold calibrated such that 95\% of ID data ($\mathcal{D}_{\text{in}}$) is correctly classified. 

    Since $\approach$ complements existing  techniques mentioned above, we evaluate $\approach$ in conjunction with these approaches in Section~\ref{sec: ablation studies}. Subsequently, we formally characterize and explain why $\approach$ improves the separability of the scores between ID and OOD data as part of detailed study in Appendix~\ref{appendix: analysis}.

\section{Empirical Evaluation of \approach} 
\label{sec: experiments}

    We conduct a comprehensive empirical evaluation to validate the efficacy of $\approach$, testing its performance on standard benchmarks, its scalability to large-scale datasets, its robustness across diverse architectures, and its behavior in challenging near-OOD scenarios. We follow the experiment setting provided in OpenOOD benchmark~\cite{openood} and explain the experimental setup in section~\ref{subsec: experimental setup}. These empirical studies demonstrate the superior performance of $\approach$ over existing primary baselines that are reported in section~\ref{subsec: Results and Discussion} and ~\ref{subsec: near ood}.

\subsection{Experimental Setup}
\label{subsec: experimental setup}

    \textbf{Primary baselines.} 
    In our evaluation, we reproduced a representative selection of logit-, energy-, gradient-, and distance-based methods, including \texttt{MSP}, \texttt{ODIN}, \texttt{Energy}, \texttt{GradNorm}, \texttt{kNN}, \texttt{ReAct}, \texttt{DICE}, \texttt{ASH}, and \texttt{SCALE}.

    \noindent
    \textbf{Architectures.} To comprehensively evaluate $\approach$, we employ a diverse set of pre-trained backbones. For ImageNet, we use Swin-B~\cite{swin_transformer}, ConvNeXt-B~\cite{ConvNext}, EfficientNet-B0~\cite{EfficientNet}, ResNet-50~\cite{ResNet}, DenseNet-121~\cite{DenseNet}, and MobileNet-V2~\cite{MobileNet}. For CIFAR benchmarks, we use DenseNet-101, ResNet-18, ResNet-34, Wide-ResNet-28-10~\cite{wide-resnet}, and MobileNet-V2. For fair comparison, all methods share the same pre-trained backbone and are evaluated without auxiliary outlier exposure.

    Since several architectures considered here were not included in the original publications of primary baselines, we re-evaluated these methods under our experimental setup. In each case, we strictly followed the official hyperparameter selection protocols and publicly available implementations to maintain the integrity of the comparison.

    \noindent
    \textbf{Datasets.}  We follow the standardized OpenOOD benchmark~\cite{openood} and adopt its datasets, in addition to those used in seminal works~\cite{msp, MOS, energy, knn, GradNorm}. For the ImageNet benchmark, we evaluate on five commonly used far-OOD datasets: iNaturalist~\cite{iNaturalist}, SUN~\cite{sun}, Places~\cite{places365}, Textures~\cite{texture}, OpenImage-O~\cite{openood} and three near-OOD datasets: SSB-Hard~\cite{ssb-hard}, NINCO~\cite{NINCO}, ImageNet-O~\cite{max_logit}. 
    
    For CIFAR benchmarks, we use six standard far-OOD datasets: Textures~\cite{texture}, SVHN~\cite{svhn}, Places~\cite{places365}, LSUN-Crop~\cite{lsun}, LSUN-Resize~\cite{lsun}, iSUN~\cite{isun}. For near-OOD evaluation we consider CIFAR-100~\cite{cifar_dataset} and Tiny-ImageNet. 

    Importantly, our evaluation adheres to the OOD-free assumption and does not rely on a held-out OOD validation set for hyperparameter tuning. By avoiding access to OOD validation set and incorporating challenging datasets, we provide a rigorous and realistic assessment of $\approach$. 

    \noindent
    \textbf{Evaluation metrics.} We assess the effectiveness of $\approach$ by utilizing threshold-free metrics that are commonly used for evaluating OOD detection, as standardized in ~\cite{msp}. These metric includes (i) AUROC (Area Under the Receiver Operating Characteristic curve); represents the probability that a random ID sample is assigned a higher score than a random OOD sample. A value of 1.0 is perfect, while 0.5 is random (a higher value is better, $\uparrow$). (ii) FPR95 (false positive rate); is the percentage of OOD samples incorrectly classified as ID when the threshold $\lambda$ is set to correctly classify 95\% of ID samples. A lower value is better ($\downarrow$)~\cite{energy}.

\subsection{Results and Discussion}
\label{subsec: Results and Discussion}

    \begin{table}[!ht]
        \centering
        \caption{OOD detection performance on five benchmark datasets. $\boldsymbol{\downarrow}$/$\boldsymbol{\uparrow}$ denotes lower/higher is better. All values are percentages, with the best and second-best results being \textbf{highlighted} and \underline{underlined}, respectively. Results $^*$ are taken from CADRef~\cite{CADRef}}
        \resizebox{\textwidth}{!}{
        \begin{tabular}{l l cc cc cc cc cc cc}
        \toprule
        \multirow{2}{*}{\textbf{Model}} & \multirow{2}{*}{\textbf{Method}} & \multicolumn{2}{c}{\textbf{iNaturalist}} & \multicolumn{2}{c}{\textbf{SUN}} & \multicolumn{2}{c}{\textbf{Places}} & \multicolumn{2}{c}{\textbf{Textures}} & \multicolumn{2}{c}{\textbf{OpenImage-O}} & \multicolumn{2}{c}{\textbf{Average}} \\
        \cmidrule(lr){3-4} \cmidrule(lr){5-6} \cmidrule(lr){7-8} \cmidrule(lr){9-10} \cmidrule(lr){11-12} \cmidrule(lr){13-14}
        && \textbf{FPR95} $\downarrow$ & \textbf{AUROC} $\uparrow$ & \textbf{FPR95} $\downarrow$ & \textbf{AUROC} $\uparrow$ & \textbf{FPR95} $\downarrow$ & \textbf{AUROC} $\uparrow$ & \textbf{FPR95} $\downarrow$ & \textbf{AUROC} $\uparrow$ & \textbf{FPR95} $\downarrow$ & \textbf{AUROC} $\uparrow$  & \textbf{FPR95} $\downarrow$ & \textbf{AUROC} $\uparrow$  \\
        \midrule
        \multirow{17}{*}{\rotatebox{90}{\textbf{Swin-B}}}
                    &MSP                & 48.29 & 87.80 & 66.44 & 79.78 & 67.72 & 80.13 & 64.54 & 78.73 & 59.33 & 82.74 & 61.26 & 81.84 \\
                    &MaxLogit$^*$       & 55.48 & 79.86 & 67.86 & 72.23 & 69.54 & 71.97 & 61.93 & 73.98 & 65.52 & 71.08 & 64.07 & 73.82 \\
                    &ODIN               & 83.03 & 49.31 & 88.14 & 43.79 & 89.48 & 42.23 & 78.94 & 54.79 & 88.18 & 42.02 & 85.55 & 46.43 \\
                    &Energy             & 75.73 & 67.70 & 84.07 & 58.30 & 81.84 & 59.67 & 72.02 & 66.46 & 78.72 & 60.14 & 78.48 & 62.45 \\
                    &GEN$^*$            & \underline{32.94} & 92.69 & \textbf{56.61} & 85.02 & \textbf{59.95} & \underline{84.06} & 49.93 & 85.61 & 48.66 & 87.18 & \underline{49.62} & 86.91 \\
                    &ReAct              & 49.74 & 90.74 & 67.49 & 81.94 & 65.68 & 82.04 & 57.45 & 84.79 & 56.65 & 87.75 & 59.40 & 85.45 \\
                    &DICE               & 97.95 & 17.93 & 91.63 & 35.01 & 96.45 & 25.15 & 66.01 & 67.84 & 92.73 & 30.44 & 88.95 & 35.27 \\
                    &ViM$^*$            & 44.56 & 93.60 & 71.10 & 80.12 & 71.82 & 77.96 & 64.64 & 83.96 & 46.36 & 92.10 & 59.70 & 85.55 \\
                    &ASH-S              & 99.81 & 10.69 & 99.36 & 20.18 & 99.59 & 21.37 & 98.65 & 18.41 & 99.84 & 11.94 & 99.45 & 16.52 \\
                    &SCALE              & 98.93 & 24.86 & 99.07 & 26.90 & 97.56 & 27.93 & 94.54 & 38.08 & 97.76 & 24.95 & 97.57 & 28.54 \\
                    &GradNorm           & 78.22 & 86.25 & 78.53 & 82.27 & 78.15 & 79.97 & 77.87 & 76.93 & 65.83 & 85.78 & 75.72 & 82.24 \\
                    &OptFS$^*$          & 54.98 & 90.71 & 67.93 & 84.86 & 68.63 & 83.94 & 61.68 & 85.10 & 51.19 & 90.34 & 60.88 & 86.99 \\
                    &CARef$^*$          & 39.78 & 93.57 & 66.72 & 84.83 & 68.68 & 83.25 & 49.86 & 88.74 & \underline{43.52} & \underline{92.55} & 53.71 & 88.59 \\
                    &CADRef$^*$         & 37.87 & \underline{93.77} & 64.23 & \underline{85.12} & 66.71 & 83.57 & 47.38 & 89.08 & 46.63 & 92.32 & 52.56 & \underline{88.77} \\
                    &KNN                & 68.01 & 91.08 & 82.36 & 83.37 & 80.97 & 81.97 & 54.54 & 87.82 & 59.82 & 90.54 & 69.14 & 86.96 \\
        \cmidrule(lr){2-14}
        \rowcolor{gray!20} & $\approach({\mu,\sigma}) + \texttt{KNN}$ & \textbf{16.01} & \textbf{96.42} & \underline{59.23}	& \textbf{87.34}	& \underline{66.20}	& \textbf{85.08} &\textbf{17.71} & \textbf{94.42} & \textbf{26.71}	& \textbf{94.69} & \textbf{37.17} & \textbf{91.59}  \\
        \rowcolor{gray!20} & $\approach(m) + \texttt{KNN}$           & 61.66 & 90.27 & 85.13 & 80.57 & 87.87 & 78.12 & \underline{36.45} & \underline{90.71} & 62.22 & 87.75 & 66.67 & 85.48  \\
        \midrule
        \multirow{17}{*}{\rotatebox{90}{\textbf{ConvNeXt-B}}}
        & MSP               & 74.28 & 77.89 & 77.96 & 74.85 & 78.22 & 74.88 & 82.34 & 68.94 & 77.84 & 74.67 & 78.13 & 74.25 \\
        & MaxLogit$^*$      & 57.91 & 83.32 & 72.02 & 69.83 & 74.92 & 68.44 & 70.88 & 65.72 & 68.10 & 74.08 & 68.77 & 72.28 \\
        & ODIN              & 68.86 & 69.62 & 80.37 & 53.65 & 85.86 & 49.67 & 71.44 & 66.14 & 81.22 & 55.74 & 77.55 & 58.96 \\
        & Energy            & 30.71 & 94.09 & \underline{50.57} & \underline{89.24} & 51.92 & 88.89 & 70.12 & 75.63 & 48.21 & 88.82 & 50.31 & 87.33 \\
        & GEN$^*$           & \underline{25.69} & \underline{94.86} & 52.70 & 84.78 & 59.18 & 83.04 & 55.21 & 80.10 & 46.11 & 89.37 & 47.78 & 86.43 \\
        & ReAct             & 33.96 & 93.56 & 52.89 & 88.94 & 54.72 & 88.37 & 68.14 & 77.30 & 48.32 & 89.35 & 51.60 & 87.50 \\
        & DICE              & 45.76 & 90.63 & 60.65 & 85.04 & 64.94 & 83.37 & 67.66 & 75.52 & 54.89 & 86.50 & 58.78 & 84.21 \\
        & ViM$^*$           & 40.84 & 93.63 & 59.57 & 85.00 & 61.76 & 82.17 & 52.41 & 86.79 & 45.33 & 91.36 & 52.38 & 87.79 \\
        & ASH-S             & 36.88 & 91.60 & 53.01 & 87.51 & \underline{48.51} & \underline{88.33} & 79.11 & 76.68 & 54.72 & 85.66 & 54.45 & 85.96 \\
        & SCALE             & 30.34 & 93.80 & \textbf{48.71} & \textbf{89.57} & \textbf{45.64} & \textbf{90.01} & 72.54 & 78.46 & 47.65 & 88.86 & 48.97 & 88.14 \\
        & GradNorm          & 91.10 & 53.95 & 95.04 & 41.58 & 94.01 & 45.25 & 98.40 & 19.87 & 95.60 & 37.07 & 94.83 & 39.55 \\
        & OptFS$^*$         & 46.02 & 92.20 & 61.69 & 85.96 & 64.02 & 85.10 & 54.32 & 86.11 & 49.68 & 91.02 & 55.15 & 88.08 \\
        & CARef$^*$         & 29.79 & 94.81 & 58.04 & 87.37 & 63.09 & 85.40 & \underline{45.09} & \underline{90.04} & 38.82 & \textbf{93.39} & \underline{46.97} & \underline{90.20} \\
        & CADRef$^*$        & 33.08 & 94.50 & 60.16 & 86.84 & 65.09 & 84.92 & 49.17 & 89.27 & 44.85 & 92.85 & 50.47 & 89.68 \\
        & KNN               & 69.26 & 89.48 & 71.10 & 85.34 & 71.67 & 83.92 & 64.66 & 85.76 & 59.38 & 89.31 & 67.21 & 86.76  \\
        \cmidrule(lr){2-14}
        \rowcolor{gray!20} & $\approach({\mu,\sigma}) + \texttt{KNN}$ & \textbf{25.89} & \textbf{94.96} & 51.49 & 87.70 & 56.99 & 85.61 & \textbf{31.05} & \textbf{90.71} & \textbf{31.73} & \underline{92.97} & \textbf{39.43} & \textbf{90.23}  \\
        \rowcolor{gray!20} & $\approach(m) + \texttt{KNN}$            & 76.34 & 87.98 & 85.88 & 81.19 & 86.09 & 79.69 & 72.16 & 83.64 & 82.88 & 83.59 & 80.67 & 83.22  \\
        \midrule
        \multirow{11}{*}{\rotatebox{90}{\textbf{ResNet-50}}} 
                                      & MSP             & 52.83 & 88.39 & 69.11 & 81.64 & 72.06 & 80.54 & 66.26 & 80.43 & 66.97 & 83.89 & 65.45 & 82.98 \\
                                      & MaxLogit$^*$    & 50.77 & 91.14 & 60.39 & 86.43 & 66.03 & 84.03 & 54.91 & 86.38 & 57.89 & 89.13 & 58.00 & 87.42 \\
                                      & ODIN            & 41.82 & 92.25 & 57.11 & 86.77 & 64.69 & 84.12 & 47.30 & 87.82 & 59.15 & 87.54 & 54.01 & 87.70 \\
                                      & Energy          & 53.74 & 90.62 & 58.82 & 86.58 & 65.99 & 83.96 & 52.43 & 86.72 & 64.70 & 87.08 & 59.14 & 86.99 \\
                                      & GEN$^*$         & 45.76 & 92.44 & 65.54 & 85.52 & 69.24 & 83.46 & 59.24 & 85.41 & 60.44 & 89.31 & 60.04 & 87.23 \\
                                      & ReAct           & 19.56 & 96.40 & 23.95 & 94.46 & 33.48 & 91.97 & 46.40 & 90.31 & 49.78 & 89.06 & 34.64 & 92.44 \\
                                      & DICE            & 26.61 & 94.51 & 36.49 & 90.92 & 47.93 & 87.65 & 32.59 & 90.45 & 54.67 & 85.67 & 39.66 & 89.84 \\
                                      & ViM$^*$         & 71.80 & 87.42 & 81.80 & 81.07 & 83.12 & 78.39 & 14.84 & 96.83 & 58.68 & 89.30 & 62.05 & 86.60 \\
                                      & ASH-S           & 11.41 & 97.88 & 28.00 & 94.04 & 39.67 & 91.03 & 11.88 & 97.62 & 38.70 & 90.79 & 25.93 & 94.27 \\
                                      & SCALE           & \underline{10.37} & \underline{98.02} & \underline{25.78} & \underline{94.54} & \underline{36.86} & \textbf{91.96} & 14.56 & 96.75 & 36.23 & 92.30 & \underline{24.76} & \underline{94.71} \\
                                      & GradNorm        & 26.78 & 93.90 & 37.42 & 90.10 & 48.88 & 86.08 & 32.84 & 90.64 & 57.76 & 80.44 & 40.74 & 88.23 \\
                                      & OptFS$^*$       & 16.79 & 96.88 & 35.31 & 93.13 & 44.78 & 90.42 & 23.08 & 95.74 & 37.68 & 92.77 & 31.53 & 93.79 \\
                                      & CARef$^*$       & 17.46 & 96.54 & 44.89 & 89.51 & 57.64 & 85.41 & \underline{10.15} & \underline{97.94} & 37.73 & 92.57 & 33.57 & 92.39 \\
                                      & CADRef$^*$      & 16.08 & 96.90 & 39.23 & 91.26 & 51.12 & 87.80 & 12.60 & 97.14 & \textbf{32.69} & \textbf{93.93} & 30.34 & 93.41 \\
                                      & KNN             & 78.33 & 79.15 & 78.95 & 77.44 & 81.86 & 73.91 & 16.05 & 96.11 & 65.73 & 82.27 & 64.18 & 81.78 \\
    
        \cmidrule(lr){2-14}
        \rowcolor{gray!20}            & $\approach({\mu,\sigma}) + \texttt{SCALE}$ &  \textbf{9.61} & \textbf{98.10} & \textbf{24.49} & \textbf{94.62} & \textbf{36.01} & \underline{91.83} & 10.59 & 97.75 & \underline{33.30} & \underline{92.85} & \textbf{22.80} & \textbf{95.03} \\
        \rowcolor{gray!20}            & $\approach(m) + \texttt{SCALE}$            & 13.26 & 97.37 & 27.79 & 93.66 & 40.56 & 90.02 &  \textbf{9.52} & \textbf{98.11} & 34.48 & 92.50 & 25.12 & 94.33 \\
        \bottomrule
        \end{tabular}}
        \label{table: detailed_imagenet_benchmark_swin_convnext}
        \vspace{-6mm}
    \end{table}

    \begin{table}[!ht]
        \centering
        \caption{OOD detection results on ImageNet benchmarks. All values are percentages, averaged over five standard OOD datasets. Detailed results for each dataset are provided in Appendix~\ref{appendix: detailed_imagenet_benchmark}. $\boldsymbol{\downarrow}$/$\boldsymbol{\uparrow}$ denotes lower/higher is better.}
        \resizebox{0.75\textwidth}{!}{
        \begin{tabular}{l cc cc  cc}
        \toprule
        \multirow{2}{*}{\textbf{Method}} & \multicolumn{2}{c}{\textbf{EfficientNet-b0}} & \multicolumn{2}{c}{\textbf{DenseNet-121}} & \multicolumn{2}{c}{\textbf{MobileNet-v2}}\\
        \cmidrule(lr){2-3} \cmidrule(lr){4-5} \cmidrule(lr){6-7} 
        & \textbf{FPR95} $\downarrow$ & \textbf{AUROC} $\uparrow$ & \textbf{FPR95} $\downarrow$ & \textbf{AUROC} $\uparrow$ & \textbf{FPR95} $\downarrow$ & \textbf{AUROC} $\uparrow$ \\
        \midrule
        MSP          & 67.13 & 82.48 & 64.56 & 82.72 & 71.34 & 80.70 \\
        ODIN         & 68.18 & 78.41 & 51.23 & 87.39 & 58.65 & 86.73 \\
        Energy       & 80.50 & 75.60 & 53.00 & 87.50 & 61.56 & 86.27 \\
        ReAct        & 59.11 & 86.03 & 43.99 & 89.56 & 51.13 & 88.64 \\
        DICE         & 97.91 & 44.54 & 42.14 & 88.44 & 46.57 & 88.23 \\
        ASH-S        & 98.96 & 54.87 & 32.33 & 92.73 & 41.35 & 90.35 \\
        SCALE        & 98.68 & 56.31 & 36.09 & 91.83 & 36.92 & 92.04 \\
        GradNorm     & 90.60 & 54.81 & 44.86 & 86.34 & 44.78 & 88.77 \\
        KNN          & 78.17 & 74.72 & 73.68 & 71.72 & 74.84 & 70.87 \\
        \cmidrule(lr){1-7}
        \rowcolor{gray!20} $\approach({\mu,\sigma}) + \texttt{SCALE}$  & \textbf{46.99} & \textbf{88.06} & \underline{30.27} & \underline{93.08} & \underline{34.82} & \underline{92.45} \\
        \rowcolor{gray!20} $\approach(m) + \texttt{SCALE}$             & \underline{55.64} & \underline{83.65} & \textbf{29.66} & \textbf{93.17} & \textbf{33.49} & \textbf{92.61} \\
        \bottomrule
        \end{tabular}}
        \label{table: avg_imagenet_benchmark}
        \vspace{-3mm}
    \end{table}

    \textbf{ImageNet Benchmarks.} The ImageNet benchmark is substantially more challenging due to higher image resolution and a large label space of 1,000 categories. $\approach$ demonstrates a complementary effect when integrated with existing scoring functions achieving strong overall performance as shown in Table~\ref{table: detailed_imagenet_benchmark_swin_convnext}. We report the best combination of $\approach$ with primary baselines and provide other combinations in Appendix~\ref{appendix: combined methods}. 
    
    For the primary baselines we reproduced (recall section~\ref{subsec: experimental setup}), we extend our evaluation on EfficientNet-B0, DenseNet-121, and MobileNet-v2 in Table~\ref{table: avg_imagenet_benchmark}. We observe that modern architectures such as Swin-B and ConvNeXt-B perform best when combined with kNN distance-based scoring, whereas more traditional CNN backbones including ResNet-50, DenseNet-121, EfficientNet-B0, and MobileNet-V2 perform better when paired with energy-based techniques like SCALE.

    As reported in Table~\ref{table: detailed_imagenet_benchmark_swin_convnext}, $\approach(\mu, \sigma) + \texttt{KNN}$ with Swin-B reduces FPR95 by 25.09\% compared to GEN~\cite{GEN}, and by 16.05\% with ConvNeXt-B compared to CARef~\cite{CADRef}. Furthermore, $\approach(\mu, \sigma) + \texttt{SCALE}$ improves FPR95 by 7.92\% with ResNet-50 relative to SCALE.

    In Table~\ref{table: avg_imagenet_benchmark}, $\approach(\mu, \sigma) + \texttt{SCALE}$ achieves FPR95 reductions of 20.05\%, 6.37\%, and 5.69\% on EfficientNet-B0, DenseNet-121, and MobileNet-V2, respectively, compared to the best-performing existing methods. These results validate that the principles of $\approach$ generalize effectively to large-scale settings and diverse architectural families. Detailed results %
    are provided in Appendix~\ref{appendix: detailed_imagenet_benchmark}.

    \begin{table}[!ht]
        \centering
        \caption{OOD detection performance on six benchmark datasets using DenseNet-101 pre-trained on CIFAR. $\boldsymbol{\downarrow}$/$\boldsymbol{\uparrow}$ denotes lower/higher is better. All values are percentages, with the best and second-best results being \textbf{highlighted} and \underline{underlined}, respectively. Methods marked with $^*$ denote results taken from CADRef~\cite{CADRef}}
        \resizebox{\textwidth}{!}{
        \begin{tabular}{l cc cc cc cc cc cc cc}
        \toprule
         \multirow{2}{*}{\textbf{Method}} & \multicolumn{2}{c}{\textbf{SVHN}} & \multicolumn{2}{c}{\textbf{Places}} & \multicolumn{2}{c}{\textbf{iSUN}} & \multicolumn{2}{c}{\textbf{Textures}} & \multicolumn{2}{c}{\textbf{LSUN-c}} & \multicolumn{2}{c}{\textbf{LSUN-r}} & \multicolumn{2}{c}{\textbf{Average}}\\
        \cmidrule(lr){2-3} \cmidrule(lr){4-5} \cmidrule(lr){6-7} \cmidrule(lr){8-9} \cmidrule(lr){10-11} \cmidrule(lr){12-13} \cmidrule(lr){14-15}
        & \textbf{FPR95} $\downarrow$ & \textbf{AUROC} $\uparrow$ & \textbf{FPR95} $\downarrow$ & \textbf{AUROC} $\uparrow$ & \textbf{FPR95} $\downarrow$ & \textbf{AUROC} $\uparrow$ & \textbf{FPR95} $\downarrow$ & \textbf{AUROC} $\uparrow$ & \textbf{FPR95} $\downarrow$ & \textbf{AUROC} $\uparrow$ & \textbf{FPR95} $\downarrow$ & \textbf{AUROC} $\uparrow$ & \textbf{FPR95} $\downarrow$ & \textbf{AUROC} $\uparrow$ \\
        \midrule
        \multicolumn{15}{c}{\textbf{CIFAR-10}} \\
        \midrule
        MSP      & 64.76 & 88.33 & 60.19 & 88.56 & 33.34 & 95.41 & 56.60 & 90.17 & 23.41 & 96.75 & 33.88 & 95.39 & 45.36 & 92.43 \\
    MaxLogit$^*$ & 37.79 & 94.32 & 34.82 & 93.61 & 10.08 & 98.05 & 56.57 & 86.65 & 16.31 & 97.22 &  9.41 & 98.12 & 27.50 & 94.66 \\
        ODIN     & 33.09 & 94.41 & 36.68 & 92.34 &  3.22 & 99.20 & 38.49 & 91.61 &  1.84 & 99.53 &  2.89 & 99.28 & 19.37 & 96.06 \\
        Energy   & 37.91 & 93.59 & 36.38 & 92.39 &  7.83 & 98.23 & 43.85 & 90.49 &  1.95 & 99.47 &  7.34 & 98.34 & 22.54 & 95.42 \\
    GEN$^*$      & 30.75 & 95.19 & 36.30 & 93.34 & 11.93 & 97.87 & 54.00 & 88.87 & 18.29 & 96.99 & 11.29 & 97.93 & 27.09 & 95.03 \\
        ReAct    & 23.18 & 96.28 & 33.97 & 92.98 &  5.95 & 98.45 & 32.25 & 93.98 &  2.47 & 99.33 &  5.44 & 98.55 & 17.21 & 96.59 \\
        DICE     & 16.68 & 96.96 & 37.46 & 92.06 &  2.25 & 99.41 & 28.05 & 92.70 &  \textbf{0.16} & \textbf{99.94} &  2.44 & 99.35 & 14.51 & 96.74 \\
        ASH-S    & 16.20 & 97.21 & 37.79 & 92.02 &  3.91 & 98.94 & 26.40 & 94.61 &  0.84 & 99.69 &  4.15 & 98.92 & 14.88 & 96.90 \\
        SCALE    & 23.06 & 96.13 & 36.53 & 92.24 &  4.54 & 98.76 & 30.53 & 93.62 &  1.23 & 99.61 &  4.59 & 98.75 & 16.75 & 96.52 \\
        GradNorm & 22.02 & 96.19 & 47.68 & 88.65 &  4.72 & 99.03 & 27.29 & 92.21 &  \underline{0.21} & \underline{99.90} &  4.99 & 98.94 & 17.82 & 95.82 \\
    OptFS$^*$    & 24.35 & 96.01 & 40.15 & 92.33 & 10.25 & 98.00 & 32.96 & 94.29 & 18.09 & 96.99 &  9.31 & 98.11 & 22.52 & 95.95 \\
    ViM$^*$      & 8.65  & 98.45 & 54.48 & 89.80 & 4.50  & 99.12 & 20.33 & 96.18 & 15.39 & 97.36 &  3.17 & 99.28 & 17.75 & 96.70 \\
    CARef$^*$    & 4.66  & 99.11 & 41.19 & 91.59 & 6.11  & 98.77 & 16.29 & 96.79 & 9.51  & 98.22 &  5.20 & 98.93 & 13.83 & 97.23 \\
    CADRef$^*$   & \underline{4.16}  & \underline{99.17} & 32.74 & 93.72 & 4.34  & 99.13 & 17.52 & 96.71 & 7.02  & 98.66 &  3.61 & 99.23 & 11.56 & 97.77 \\
        KNN      & \textbf{1.50}  & \textbf{99.67} & 42.45 & 90.49 & 7.04  & 98.72 & 14.38 & 97.54 & 5.79  & 98.94 &  8.73 & 98.45 & 13.31 & 97.30 \\
        \cmidrule(lr){1-15}
        \rowcolor[gray]{0.9}  $\approach(\mu,\sigma) + \texttt{DICE}$	&  6.84 & 98.77 & \underline{29.76} & \underline{93.86} & \textbf{1.58} & \textbf{99.59} &  \underline{9.57} & \underline{98.25} & 0.55 & 99.86 & \textbf{1.60} & \textbf{99.57} &  \textbf{8.32} & \textbf{98.32} \\
        \rowcolor[gray]{0.9}   $\approach(m) + \texttt{DICE}$	        &  8.30 & 98.37 & \textbf{29.47} & \textbf{93.92} & \underline{1.85} & \underline{99.56} &  \textbf{7.16} & \textbf{98.69} & 1.26 & 99.72 & \underline{1.92} & \underline{99.55} &  \underline{8.33} & \underline{98.30} \\
        
        \midrule
        \multicolumn{15}{c}{\textbf{CIFAR-100}} \\
        \midrule
        MSP      & 81.38 & 75.71 & 82.62 & 74.04 & 84.12 & 68.22 & 86.95 & 68.37 & 51.82 & 87.93 & 81.34 & 69.51 & 78.04 & 73.96 \\
    MaxLogit$^*$ & 86.17 & 81.42 & 79.82 & 76.18 & 79.13 & 76.54 & 84.45 & 71.14 & 58.91 & 87.90 & 76.05 & 77.41 & 77.42 & 78.43 \\
        ODIN     & 85.94 & 80.35 & \textbf{75.59} & 77.62 & 48.03 & 89.12 & 83.37 & 67.83 & 12.78 & 97.70 & 40.28 & 91.35 & 57.67 & 84.00 \\
        Energy   & 70.99 & 86.66 & \underline{77.12} & \underline{76.94} & 64.28 & 83.92 & 83.60 & 67.47 & 11.45 & 97.89 & 56.08 & 86.84 & 60.59 & 83.29 \\
    GEN$^*$      & 78.89 & 80.97 & 83.36 & 73.88 & 85.15 & 72.00 & 83.68 & 74.26 & 70.82 & 83.72 & 84.11 & 71.51 & 81.00 & 76.06 \\
        ReAct    & 67.12 & 87.20 & 77.75 & 76.18 & 56.39 & 89.46 & 75.98 & 79.16 & 13.26 & 97.53 & 49.92 & 90.94 & 56.74 & 86.74 \\
        DICE     & 33.87 & 93.97 & 79.95 & 76.75 & 47.76 & 89.61 & 63.42 & 73.33 &  \textbf{0.79} & \textbf{99.76} & 43.65 & 91.00 & 44.91 & 87.40 \\
        ASH-S    & \textbf{10.32} & \textbf{97.99} & 85.93 & 71.95 & 39.69 & 92.04 & 35.67 & 91.76 &  5.43 & 98.98 & 42.89 & 91.30 & \underline{36.66} & 90.67 \\
        SCALE    & 16.26 & 97.05 & 78.54 & 76.97 & 43.56 & 91.21 & 45.60 & 87.23 &  3.23 & 99.30 & 42.69 & 91.02 & 38.31 & 90.46 \\
        GradNorm & 35.49 & 93.07 & 87.03 & 70.09 & 71.36 & 82.74 & 61.83 & 75.72 &  \underline{0.94} & \underline{99.75} & 68.56 & 83.65 & 54.20 & 84.17 \\
    OptFS$^*$    & 73.61 & 84.96 & 80.96 & 74.37 & 70.56 & 84.39 & 61.64 & 85.63 & 47.98 & 90.01 & 69.52 & 83.61 & 67.38 & 83.83 \\
    ViM$^*$      & 35.05 & 93.57 & 83.89 & 75.61 & \textbf{23.22} & \textbf{95.63} & \textbf{19.75} & \textbf{95.89} & 40.06 & 92.76 & \textbf{24.65} & \textbf{95.50} & 37.77 & 91.49 \\
    CARef$^*$    & 17.41 & 96.83 & 88.30 & 67.92 & 45.32 & 91.08 & \underline{25.48} & \underline{93.99} & 40.74 & 90.96 & 52.69 & 89.25 & 44.99 & 88.34 \\
    CADRef$^*$   & 18.28 & 96.69 & 78.30 & 75.91 & 42.10 & 91.59 & 28.72 & 94.13 & 27.22 & 94.70 & 47.45 & 90.26 & 40.34 & 90.55 \\
        KNN      & \underline{15.86} & \underline{96.88} & 88.36 & 66.14 & 42.98 & 89.45 & 27.11 & 94.21 & 35.82 & 89.74 & 42.90 & 89.28 & 42.17 & 87.62 \\
        \cmidrule(lr){1-15}
        \rowcolor[gray]{0.9} $\approach(\mu,\sigma) + \texttt{DICE}$ & 20.30 & 96.20 & 79.52 & \textbf{78.28} & \underline{29.69} & \underline{94.63} & 34.10 & 91.24 &  2.63 & 99.37 & \underline{32.01} & \underline{94.20} & \textbf{33.04} & \textbf{92.32} \\
        \rowcolor[gray]{0.9} $\approach(m) + \texttt{DICE}$	         & 27.97 & 94.91 & 86.13 & 76.61 & 36.01 & 93.97 & 30.32 & 93.28 &  7.58 & 98.56 & 41.90 & 93.13 & 38.32 & \underline{91.74} \\
        
        \bottomrule
        \end{tabular}}
        \label{table: detailed_cifar_densenet-101}
    \end{table}

    \begin{table}[!ht]
        \centering
        \caption{OOD detection performance on six benchmark datasets using ResNet-18, ResNet-34, Wide-ResNet, and MobileNet-v2 pre-trained on CIFAR. $\boldsymbol{\downarrow}$/$\boldsymbol{\uparrow}$ denotes lower/higher is better. All values are percentages, with the best and second-best results being \textbf{highlighted} and \underline{underlined}, respectively. \textbf{$^*$}\textit{In MobileNet-v2 on CIFAR-100, ASH is used instead of DICE as the combined method}. Detailed performance is  in Appendix~\ref{appendix: detailed_results_cifar}.}
        \resizebox{0.80\textwidth}{!}{
        \begin{tabular}{c l cc cc  cc  cc}
        \toprule
        \multirow{2}{*}{\textbf{Dataset}} & \multirow{2}{*}{\textbf{Method}} & \multicolumn{2}{c}{\textbf{ResNet-18}} & \multicolumn{2}{c}{\textbf{ResNet-34}} & \multicolumn{2}{c}{\textbf{Wide-ResNet}} & \multicolumn{2}{c}{\textbf{MobileNet-v2}}\\
         \cmidrule(lr){3-4} \cmidrule(lr){5-6} \cmidrule(lr){7-8} \cmidrule(lr){9-10}
        & & \textbf{FPR95} $\downarrow$ & \textbf{AUROC} $\uparrow$ & \textbf{FPR95} $\downarrow$ & \textbf{AUROC} $\uparrow$ & \textbf{FPR95} $\downarrow$ & \textbf{AUROC} $\uparrow$ & \textbf{FPR95} $\downarrow$ & \textbf{AUROC} $\uparrow$ \\
        \midrule
        \multirow{11}{*}{\rotatebox[origin=c]{90}{CIFAR-10}}  & MSP          & 58.43 & 91.23 & 54.86 & 91.96 & 46.41 & 92.78 & 66.16 & 88.76 \\
                                                              & ODIN         & 28.98 & 95.16 & 23.06 & 95.53 & 21.77 & 95.59 & 35.25 & 93.37 \\
                                                              & Energy       & 35.61 & 94.14 & 26.04 & 95.29 & 22.44 & 95.50 & 39.75 & 92.65 \\
                                                              & ReAct        & 30.14 & 95.15 & 26.36 & 95.37 & 28.06 & 94.43 & 38.18 & 92.71 \\
                                                              & DICE         & 30.92 & 94.69 & 23.00 & 95.84 & 22.91 & 95.27 & 36.86 & 92.88 \\
                                                              & ASH-S        & 21.83 & 96.02 & 19.57 & 96.34 & 20.04 & 96.07 & 40.26 & 92.15 \\
                                                              & SCALE        & 21.74 & 96.13 & 19.50 & 96.28 & 19.35 & 96.08 & 39.56 & 92.34 \\
                                                              & GradNorm     & 32.98 & 93.49 & 30.87 & 93.46 & 26.80 & 93.37 & 45.26 & 90.86 \\
                                                              & KNN          & 30.43 & 94.22 & 31.18 & 94.29 & 16.21 & 96.85 & 49.09 & 89.45 \\
        \cmidrule(lr){2-10}
        \rowcolor{gray!20}          & $\approach(m) + \texttt{DICE}$             & \textbf{10.46} & \textbf{97.94} & \textbf{10.67} & \textbf{97.97} &  \textbf{9.28} & \textbf{98.13} & \underline{24.78} & \underline{95.28} \\
        \rowcolor{gray!20}          & $\approach({\mu,\sigma}) + \texttt{DICE} $ & \underline{13.49} & \underline{97.54} & \underline{12.09} & \underline{97.75} &  \underline{11.55} & \underline{97.76} & \textbf{24.38} & \textbf{95.31} \\
        \midrule
        \multirow{11}{*}{\rotatebox[origin=c]{90}{CIFAR-100}} & MSP          & 80.40 & 76.16 & 79.68 & 78.08 & 80.64 & 71.32 & 83.83 & 72.78 \\
                                                              & ODIN         & 66.06 & 84.78 & 67.50 & 84.71 & 67.82 & 81.90 & 70.10 & 83.63 \\
                                                              & Energy       & 70.86 & 83.18 & 70.30 & 83.64 & 68.77 & 81.14 & 72.65 & 82.77 \\
                                                              & ReAct        & 59.43 & 87.52 & 57.87 & 86.64 & 47.38 & 90.16 & 53.57 & 87.90 \\
                                                              & DICE         & 56.90 & 85.39 & 55.17 & 86.28 & 55.02 & 84.22 & 64.78 & 82.93 \\
                                                              & ASH-S        & 54.50 & 87.47 & 54.81 & 87.88 & 44.47 & 89.19 & 51.65 & 86.37 \\
                                                              & SCALE        & 48.10 & 88.70 & 48.02 & 88.64 & 40.11 & 90.19 & 50.53 & 87.43 \\
                                                              & GradNorm     & 67.11 & 76.94 & 66.84 & 73.43 & 63.69 & 80.27 & 68.17 & 76.58 \\
                                                              & KNN          & 67.00 & 81.42 & 65.78 & 82.03 & 52.49 & 82.84 & 85.65 & 71.97 \\
        \cmidrule(lr){2-10}
        \rowcolor{gray!20}       & $\approach(m) + \texttt{DICE}^*$             & \textbf{33.38} & \textbf{92.51} & \textbf{33.91} & \textbf{92.33} & \underline{38.73}	& \underline{91.47} & \underline{46.35$^*$} & \underline{86.96$^*$} \\
        \rowcolor{gray!20}       & $\approach({\mu,\sigma}) + \texttt{DICE}^*$  & \underline{36.19} & \underline{91.54} & \underline{36.67} & \underline{91.81} & \textbf{37.95}	& \textbf{91.52} & \textbf{46.32}$^*$ & \textbf{87.23}$^*$ \\
        \bottomrule
        \end{tabular}}
        \vspace{-1.5mm}
        \label{table: avg_cifar_benchmark}
        \vspace{-2mm}
    \end{table}

    \noindent
    \textbf{CIFAR Benchmarks.} We further evaluate $\approach$ on CIFAR benchmarks. Since pre-trained DenseNet-101 is widely adopted in prior works~\cite{GradNorm, ReAct, ASH, DICE, SCALE}, we provide a comprehensive comparison on this backbone in Table~\ref{table: detailed_cifar_densenet-101}. A key observation is that no single baseline consistently outperforms others across all six OOD datasets. Nevertheless, $\approach$ combined with \texttt{DICE} achieves comparable or superior performance across datasets. On average, our method improves FPR95 by 28.03\% over CADRef~\cite{CADRef} on CIFAR-10 and by 9.87\% over ASH-S~\cite{ASH} on CIFAR-100. %

    We further extend the evaluation to ResNet-18, ResNet-34, Wide-ResNet-28-10, and MobileNet-V2, with average results reported in Table~\ref{table: avg_cifar_benchmark}. $\approach$ consistently outperforms strong primary baselines. On CIFAR-10, $\approach + \texttt{DICE}$ reduces FPR95 by 52.08\%, 45.12\%, 42.75\%, and 32.77\% using ResNet-18, ResNet-34, Wide-ResNet-28-10, and MobileNet-V2, respectively. On CIFAR-100, improvements of 30.60\%, 29.38\%, and 3.44\% are observed with ResNet-18, ResNet-34, and Wide-ResNet-28-10, while $\approach + \texttt{ASH}$ achieves a 10.25\% reduction with MobileNet-V2. Detailed results %
    are provided in Appendix~\ref{appendix: detailed_results_cifar}.

    \textbf{Discussion.} In this research study, we further investigate the performance of $\approach$ by conducting extensive experimental studies on the ImageNet and CIFAR benchmarks. The key observation is that no single method consistently outperforms all other methods across diverse datasets. However, it is noticeable that $\approach$ yields higher average performance across the evaluation benchmark and architectures used.

    We also   evaluated a unified variant using all three signals: Mean, Variance, Max. However, our empirical analysis found that this combination yielded negligible performance gains compared to the simpler $\approach(\mu, \sigma)$ or $\approach(m)$ variants. Moreover, from evaluation above, $\approach(\mu, \sigma)$ generalizes better than $\approach(m)$ across the evaluation benchmarks. Therefore, we advocate for the simpler and flexible variant $\approach(\mu, \sigma)$, which provides similar performance.

\subsection{Near-OOD Evaluation} 
\label{subsec: near ood}

    To further demonstrate $\approach$'s effectiveness, we evaluate it on the near-OOD detection task. The results using Swin-B, ConvNeXt-B, and ResNet-50 are reported in Table~\ref{table: detailed_imagenet_nearood}. $\approach + \texttt{KNN}$ reduces FPR95 by 25.52\% and 12.01\% compared to ReAct on Swin-B and ConvNeXt-B, respectively, while achieving a modest improvement of 2.56\% over SCALE on ResNet-50. These results highlight the robustness of $\approach$ in semantically challenging near-OOD settings.

    On CIFAR near-OOD evaluation, across all evaluated architectures—ResNet-18, DenseNet-101, and Wide-ResNet-28-10, our method consistently outperforms the primary baselines. For example, $\approach$ reduces FPR95 by 9.34\% with ResNet-18 and 8.40\% with Wide-ResNet-28-10, while achieving a modest improvement of 2.87\% with DenseNet-101. These results demonstrate robustness even when the semantic gap between ID and OOD data is small. Complete results are provided in Table~\ref{table: cifar_near_ood} in Appendix~\ref{appendix: near_ood_results}.

     \begin{table}[!ht]
        \centering
        \caption{Near-OOD detection performance on ImageNet-1K benchmark. $\boldsymbol{\downarrow}$ indicates lower is better, and $\boldsymbol{\uparrow}$ indicates higher is better. All values are percentages, with the best and second-best results \textbf{highlighted} and \underline{underlined}, respectively.}
        \resizebox{0.80\textwidth}{!}{
        \begin{tabular}{c l cc cc  cc  cc}
        \toprule
        \multirow{2}{*}{\textbf{Model}} & \multirow{2}{*}{\textbf{Method}} & \multicolumn{2}{c}{\textbf{SSB-Hard}} & \multicolumn{2}{c}{\textbf{NINCO}} & \multicolumn{2}{c}{\textbf{ImageNet-O}} & \multicolumn{2}{c}{\textbf{Average}}\\
         \cmidrule(lr){3-4} \cmidrule(lr){5-6} \cmidrule(lr){7-8} \cmidrule(lr){9-10}
        & & \textbf{FPR95} $\downarrow$ & \textbf{AUROC} $\uparrow$ & \textbf{FPR95} $\downarrow$ & \textbf{AUROC} $\uparrow$ & \textbf{FPR95} $\downarrow$ & \textbf{AUROC} $\uparrow$ & \textbf{FPR95} $\downarrow$ & \textbf{AUROC} $\uparrow$ \\
        \midrule
        \multirow{11}{*}{\rotatebox[origin=c]{90}{Swin-B}}    
        & MSP          & 87.66 & 57.31 & 81.12 & 64.05 & 86.35 & 55.51 & 85.04 & 58.96 \\
        & ODIN         & 91.49 & 50.28 & 87.63 & 50.53 & 93.25 & 47.17 & 90.79 & 49.33 \\
        & Energy       & 87.66 & 57.31 & 81.12 & 64.05 & 86.35 & 55.51 & 85.04 & 58.96 \\
        & ReAct        & \underline{84.38} & 68.68 & \underline{71.84} & 81.20 & 81.40 & 71.65 & 79.21 & 73.84 \\
        & DICE         & 91.20 & 49.94 & 90.35 & 37.75 & 86.65 & 56.73 & 89.40 & 48.14 \\
        & ASH-S        & 99.13 & 29.83 & 99.59 & 20.20 & 99.55 & 27.07 & 99.42 & 25.70 \\
        & SCALE        & 97.64 & 37.21 & 96.99 & 34.32 & 97.00 & 36.70 & 97.21 & 36.08 \\
        & GradNorm     & 98.17 & 34.00 & 98.20 & 25.38 & 98.75 & 28.96 & 98.37 & 29.45 \\
        & KNN          & 88.67 & \underline{70.68} & 77.46 & \underline{82.36} & 84.05 & 78.43 & 83.39 & \underline{77.16} \\        
        \cmidrule(lr){2-10}
        \rowcolor{gray!20}          & $\approach({\mu,\sigma}) + \texttt{KNN} $ & \textbf{75.42} & \textbf{73.51} & \textbf{50.56} & \textbf{87.69} & \textbf{51.00} & \textbf{85.49} & \textbf{58.99} & \textbf{82.23} \\
        \rowcolor{gray!20}          & $\approach(m) + \texttt{KNN}$             & 89.10 & 68.13 & 76.64 & 81.63 & \underline{67.60} & \underline{81.12} & \underline{77.78} & 76.96 \\
        \midrule
        \multirow{11}{*}{\rotatebox[origin=c]{90}{ConvNext-B}}& MSP          & 87.53 & 63.30 & 82.09 & 69.89 & 93.85 & 52.20 & 87.82 & 61.80 \\
                                                              & ODIN         & 85.87 & 61.10 & 81.32 & 61.46 & 92.85 & 49.51 & 86.68 & 57.36 \\
                                                              & Energy       & 83.21 & 71.73 & 67.11 & 82.72 & 92.20 & 60.82 & 80.84 & 71.75 \\
                                                              & ReAct        & 83.11 & \underline{72.02} & \underline{65.96} & \underline{83.80} & 91.60 & 62.14 & \underline{80.22} & 72.66 \\
                                                              & DICE         & 86.52 & 67.40 & 72.41 & 78.67 & 91.40 & 60.46 & 83.44 & 68.84 \\
                                                              & ASH-S        & 82.70 & 71.88 & 72.20 & 76.28 & 92.25 & 63.74 & 82.38 & 70.63 \\
                                                              & SCALE        & \underline{81.85} & \textbf{73.71} & 68.32 & 80.93 & 92.70 & 64.69 & 80.96 & 73.11 \\
                                                              & GradNorm     & 95.01 & 42.76 & 93.84 & 41.61 & 97.90 & 36.12 & 95.58 & 40.16 \\
                                                              & KNN          & 87.07 & 70.34 & 75.71 & 81.88 & \underline{89.50} & \underline{73.88} & 84.09 & \underline{75.37} \\
        \cmidrule(lr){2-10}
        \rowcolor{gray!20} & $\approach({\mu,\sigma}) + \texttt{KNN}$  & \textbf{77.33} & 70.81 & \textbf{60.28} & \textbf{84.56} & \textbf{74.15} & \textbf{76.16} & \textbf{70.59} & \textbf{77.18} \\
        \rowcolor{gray!20} & $\approach(m) + \texttt{KNN}$             & 93.83 & 65.26 & 88.58 & 78.31 & 94.00 & 71.09 & 92.14 & 71.55 \\
        \midrule
        \multirow{11}{*}{\rotatebox[origin=c]{90}{ResNet-50}}    
        & MSP          & 85.02 & 72.08 & 76.32 & 79.94 & 100.00 & 28.62 & 87.11 & 60.21 \\
        & ODIN         & 84.07 & 71.81 & 75.89 & 79.31 & 100.00 & 40.16 & 86.65 & 63.76 \\
        & Energy       & 84.42 & 72.08 & 77.58 & 79.69 & 100.00 & 41.79 & 87.33 & 64.52 \\
        & ReAct        & 78.94 & 72.81 & 71.47 & 80.01 &  97.90 & 52.41 & 82.77 & 68.41 \\
        & DICE         & 81.04 & 72.83 & 74.07 & 77.52 &  98.00 & 42.78 & 84.37 & 64.38 \\
        & ASH-S        & 80.66 & 74.36 & 64.09 & 83.22 &  89.10 & 67.45 & 77.95 & 75.01 \\
        & SCALE        & \underline{78.11} & \underline{77.36} & \underline{61.42} & \underline{85.39} &  95.10 & 59.87 & 78.21 & 74.21 \\
        & GradNorm     & 81.22 & 71.95 & 73.87 & 74.06 &  95.75 & 47.89 & 83.61 & 64.63 \\
        & KNN          & 92.77 & 59.51 & 78.48 & 76.70 &  \textbf{71.45} & \textbf{81.19} & 80.90 & 72.47 \\        
        \cmidrule(lr){2-10}
        \rowcolor{gray!20}          & $\approach({\mu,\sigma}) + \texttt{KNN} $ & \textbf{78.16} & \textbf{77.37} & \textbf{60.48} & \textbf{85.60} & 90.35 & 63.76 & \underline{76.33} & \underline{75.58} \\
        \rowcolor{gray!20}          & $\approach(m) + \texttt{KNN}$             & 80.18 & 76.58 & 63.58 & 84.56 & \underline{83.70} & \underline{67.77} & \textbf{75.82} & \textbf{76.30 }\\
        \bottomrule
        \end{tabular}}
        \vspace{-1.5mm}
        \label{table: detailed_imagenet_nearood}
        \vspace{-2mm}
    \end{table}

\subsection{Comparison with Other Baselines}
\label{subsec: comparision with other baslines}

    In the literature, several contemporary methods do not provide as extensive an evaluation as $\approach$. To position $\approach$ within overlapping experimental settings, we present a detailed comparison in Appendix~\ref{appendix: additional baselines} against NCI~\cite{NCI}, fDBD~\cite{fdbd}, and AdaScale~\cite{adascle}, strictly following their original restricted evaluation protocols. %

    Against AdaSCALE, $\approach$ consistently achieves lower FPR95 on CIFAR-10 and CIFAR-100 across multiple backbones. For example, on CIFAR-10 with Wide-ResNet-28-10, $\approach(m)+\texttt{DICE}$ reduces FPR95 from 39.12 to 15.58. On ImageNet with Swin-B, $\approach(\mu)+\texttt{KNN}$ improves FPR95 from 46.24 to 31.66, while AdaSCALE performs better on ResNet-50 by 4.52 points.

    In direct comparisons with NCI, $\approach$ yields substantial improvements, reducing FPR95 by 33.87\% on CIFAR-10 (ResNet-18) and by 22.58\% on ImageNet (ResNet-50). Similarly, compared to fDBD, $\approach$ consistently achieves large gains across overlapping settings. On CIFAR-10 (ResNet-18), $\approach(m)+\texttt{DICE}$ reduces FPR95 from 31.09\% to 13.72\%, and on ImageNet, $\approach(\mu,\sigma)+\texttt{SCALE}$ lowers FPR95 from 51.19\% to 20.17\%.

    Following prior work~\cite{energy, GradNorm, gradorth, ReAct, ASH, DICE, SCALE}, we exclude Mahalanobis distance due to its substantial computational overhead arising from covariance matrix inversion, coupled with relatively limited empirical gains. Our experiments confirm similar computational costs without consistent performance improvements. We also omit GradOrth~\cite{gradorth}, as we were unable to faithfully reproduce its results due to the lack of publicly available implementation details or an official codebase.

\subsection{Hyperparameter Selection}
\label{subsec: hyper-parameter selection}

    \looseness-1
    The hyperparameter $\boldsymbol{\gamma}$, which scales the standard deviation in Equation~\ref{eq: davis_mu_sigma}, plays a critical role in performance. Following established protocols~\cite{ReAct,DICE,SCALE,ssn}, we select $\gamma$ using a proxy OOD validation set constructed by adding pixel-wise Gaussian noise sampled from $\mathcal{N}(0, 0.2)$ to images from the ID validation set. Based on this procedure, we set $\gamma=3.0$ for all CIFAR models, $\gamma=0.5$ for traditional CNN-based ImageNet models (ResNet, DenseNet, MobileNet), and $\gamma=2.0$ for modern ImageNet architectures (Swin-B, ConvNeXt, EfficientNet-B0).

    Cross-dataset transfer of $\gamma$ (e.g., applying ImageNet-tuned values to CIFAR) still yields improvements over baseline methods, although dataset-specific tuning provides optimal performance. Complete hyperparameter settings for baseline re-evaluations are provided in Appendix~\ref{appendix: reproducibility}.

\vspace{-3mm}

\section{Discussion}
\label{sec: discussion}

    This section discusses the broader implications of $\approach$, analyzing its robustness across modern architectures, including Swin-B, ConvNeXt-B, and EfficientNet-B0, as well as its practical advantages. We further examine its computational overhead, impact on classification accuracy, and overall scope to provide a comprehensive discussion.

    \noindent
   \looseness-1
   \textbf{Recent Deep Learning Models.} In our evaluation, we consider modern architectures such as Swin-B, ConvNeXt-B, and EfficientNet-B0 due to their widespread adoption~\cite{ConvNext,ConvNext_v2,swin_transformer,EfficientNet} and strong performance relative to traditional CNN-based models (e.g., ResNet-50, DenseNet-121, and MobileNet-V2). A notable architectural difference is that modern models employ Layer Normalization and GeLU/SiLU activation functions, whereas traditional CNNs typically rely on Batch Normalization and ReLU activations. A key observation from our experiments is the performance degradation of primary baselines (e.g., kNN, ReAct, DICE, ASH, SCALE, CADRef) on these modern architectures, as shown in Tables~\ref{table: detailed_imagenet_benchmark_swin_convnext} and \ref{table: avg_imagenet_benchmark}.

    \begin{figure*}[!ht]
           \vspace{-3mm}
  \centering
        \includegraphics[width=0.9135\textwidth]{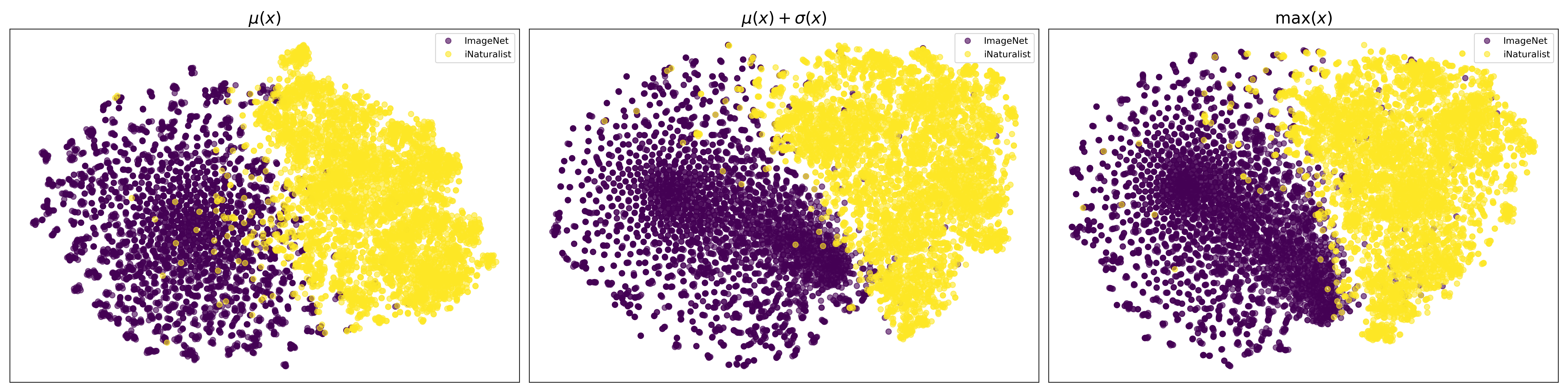}
        \vspace{-1mm}
        \caption{t-SNE distribution of statistics on Swin-B for ID (ImageNet) and OOD (iNaturalist) samples.  $\sigma(\mathbf{x}) + \mu(\mathbf{x}) $ exhibits better separation than $\mu(\mathbf{x})$ and $\max(\mathbf{x})$.}
        \label{fig: swinB_feature_plot}
        \vspace{-4mm}
    \end{figure*}
    
    We attribute this behavior to architectural differences, particularly the use of Layer Normalization  instead of Batch Normalization and alternative activation functions in Swin-B and ConvNeXt-B, as well as the design variations in EfficientNet-B0. As discussed in ReAct~\cite{ReAct}, normalization layers can distort the activation space for OOD samples due to running statistics estimated from the training distribution. Moreover, modern activation functions such as GeLU and SiLU exhibit higher overlap between ID and OOD feature distributions compared to ReLU, as illustrated in Figures~\ref{fig: swinB_feature_plot},~\ref{fig: efficientnet_b0_feature_plot}.

    In contrast, $\approach$ remains robust because it leverages more fundamental distributional statistics. Figure~\ref{fig: efficientnet_b0_feature_plot} shows that, while the mean activation $\mu(\mathbf{x})$ provides limited separability on EfficientNet-B0, both the standard deviation $\sigma(\mathbf{x})$ and maximum activation $m(\mathbf{x})$ exhibit clearer discrimination between ID and OOD samples. This robustness not only enables $\approach$ to perform reliably across modern architectures but consistently improve baselines such as kNN, ASH, and SCALE, demonstrating its general applicability. 

     \begin{figure*}[!ht]
        \centering
        \includegraphics[width=0.9135\textwidth]{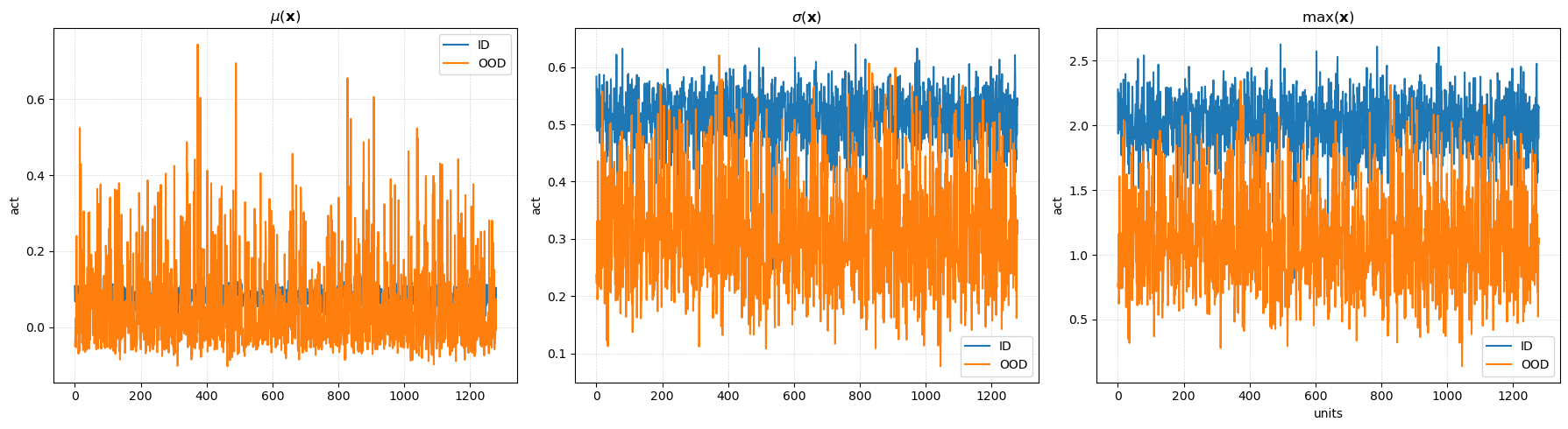}
                \vspace{-4mm}
        \caption{Feature statistics on EfficientNet-B0 for ID (ImageNet) and OOD (Texture) samples. While the mean $\mu(\mathbf{x})$ shows poor separation, the standard deviation $\sigma(\mathbf{x})$ and maximum $m(\mathbf{x})$ provide clear discrimination between ID and OOD activations.}
        \label{fig: efficientnet_b0_feature_plot}
        \vspace{-2mm}
    \end{figure*}

    \noindent
    \textbf{Activation Functions.}
    To further evaluate the robustness of $\approach$, we train ResNet-18 with alternative activation functions, including SiLU, GeLU, and Tanh, on CIFAR benchmarks. As shown in Table~\ref{table: avg_resnet18_act}, particularly on the larger CIFAR-100 dataset, existing primary baselines struggle to provide reliable OOD detection performance compared to their $\approach$-augmented counterparts. For example, on CIFAR-100, $\approach(\mu, \sigma) + \texttt{DICE}$ reduces FPR95 by 24.23\%, 26.01\%, and 35.44\% when using SiLU, GeLU, and Tanh activations, respectively. The detailed report is provided in Appendix~\ref{appendix: resnet activation results}

        \begin{table}[!ht]
        \centering
        \caption{OOD detection performance on six benchmark datasets using ResNet-18 with different activation functions: ReLu, SiLU, GeLU, Tanh pre-trained on CIFAR. $\boldsymbol{\downarrow}$/$\boldsymbol{\uparrow}$ denotes lower/higher is better. All values are percentages, with the best and second-best results being \textbf{highlighted} and \underline{underlined}, respectively.}
        \resizebox{0.90\textwidth}{!}{
        \begin{tabular}{c l cc cc  cc  cc}
        \toprule
        \multirow{2}{*}{\textbf{Dataset}} & \multirow{2}{*}{\textbf{Method}} & \multicolumn{2}{c}{\textbf{ReLU}} & \multicolumn{2}{c}{\textbf{SiLU}} & \multicolumn{2}{c}{\textbf{GeLU}} & \multicolumn{2}{c}{\textbf{TanH}}\\
        \cmidrule(lr){3-4} \cmidrule(lr){5-6} \cmidrule(lr){7-8} \cmidrule(lr){9-10}
        && \textbf{FPR95} $\downarrow$ & \textbf{AUROC} $\uparrow$ & \textbf{FPR95} $\downarrow$ & \textbf{AUROC} $\uparrow$ & \textbf{FPR95} $\downarrow$ & \textbf{AUROC} $\uparrow$ & \textbf{FPR95} $\downarrow$ & \textbf{AUROC} $\uparrow$ \\
        \midrule
        \multirow{11}{*}{\rotatebox{90}{CIFAR-10}} &MSP    & 58.43 & 91.23 & 47.27 & 93.10 & 49.18 & 92.71 & 48.57 & 92.86 \\
        &ODIN         & 28.98 & 95.16 & 19.34 & 96.40 & 18.17 & 96.61 & 19.30 & 96.42 \\
        &Energy       & 35.61 & 94.14 & 22.28 & 95.83 & 21.46 & 96.04 & 25.97 & 95.52 \\
        &ReAct        & 30.14 & 95.15 & 21.61 & 96.02 & 18.99 & 96.42 & 25.59 & 95.63 \\
        &DICE         & 30.92 & 94.69 & 19.39 & 96.36 & 16.49 & 96.76 & 19.53 & 96.18 \\
        &ASH-S        & 21.83 & 96.02 & 21.01 & 96.16 & 16.77 & 96.71 & 23.17 & 95.73 \\
        &SCALE        & 21.74 & 96.13 & 21.03 & 96.19 & 15.76 & 96.96 & 24.50 & 95.53 \\
        &GradNorm     & 32.98 & 93.49 & 27.83 & 94.43 & 26.43 & 94.79 & 32.55 & 93.98 \\
        &KNN          & 30.43 & 94.22 & 26.86 & 95.28 & 26.66 & 95.14 & 27.51 & 95.20 \\
        \cmidrule(lr){2-10}
        \rowcolor{gray!20} & $\approach(m) + \texttt{DICE}$             & \textbf{10.46} & \textbf{97.94} & \textbf{10.70}	& \textbf{97.94} & \textbf{10.57}	& \textbf{97.92} & \textbf{10.96} & \textbf{97.82} \\
        \rowcolor{gray!20} & $\approach({\mu,\sigma}) + \texttt{DICE}$  & \underline{13.49} & \underline{97.54} & \underline{11.71}	& \underline{97.81} & \underline{11.88}	& \underline{97.74} & \underline{12.51} & \underline{97.63} \\
        \midrule
        \multirow{11}{*}{\rotatebox{90}{CIFAR-100}} &MSP          & 80.40 & 76.16 & 81.32 & 73.79 & 79.90 & 76.84 & 81.57 &	73.77 \\
        &ODIN         & 66.06 & 84.78 & 67.90 & 82.52 & 65.08 & 84.32 & 66.21 &	83.24 \\
        &Energy       & 70.86 & 83.18 & 71.07 & 80.57 & 69.34 & 82.85 & 70.73 &	81.33 \\
        &ReAct        & 59.43 & 87.52 & 59.90 & 87.62 & 59.13 & 87.33 & 66.95 & 84.35 \\
        &DICE         & 56.90 & 85.39 & 55.41 & 82.70 & 58.32 & 85.46 & 59.74 &	83.86 \\
        &ASH-S        & 54.50 & 87.47 & 54.57 & 83.61 & 51.72 & 87.61 & 59.40 &	84.11 \\
        &SCALE        & 48.10 & 88.70 & 52.77 & 85.11 & 48.79 & 88.24 & 56.40 &	85.79 \\
        &GradNorm     & 67.11 & 76.94 & 72.73 & 71.34 & 72.37 & 73.11 & 73.57 &	70.16 \\
        &KNN          & 67.00 & 81.42 & 83.83 & 72.34 & 62.12 & 83.73 & 83.45 &	71.63 \\
        \cmidrule(lr){2-10}
        \rowcolor{gray!20} & $\approach(m) + \texttt{DICE}$             & \textbf{33.38} & \textbf{92.51} & \textbf{36.24} & \textbf{92.07} & \textbf{34.06} & \textbf{92.64} & \textbf{34.15} & \textbf{92.28} \\
        \rowcolor{gray!20} & $\approach({\mu,\sigma}) + \texttt{DICE}$  & \underline{36.19} & \underline{91.54} & \underline{39.98} & \underline{90.96} & \underline{36.10} & \underline{91.99} & \underline{36.41} & \underline{91.53} \\
        \bottomrule
        \end{tabular}}
        \vspace{-1.5mm}
        \label{table: avg_resnet18_act}
        \vspace{-2mm}
    \end{table}

    \noindent
    \looseness-1 
    \textbf{Overhead and Classification Accuracy.} As a post-hoc technique, $\approach$ is deployed in a two-branch pipeline: OOD detection is performed using our modified features, while the original, unmodified features are used for the final classification of any sample deemed ID. It ensures that our OOD detection improvements come at no cost to the ID accuracy. On the other hand, computational overhead is negligible; for instance, on a ResNet-50, our method increases the total GFLOPs by less than 0.1\%. A detailed ID classification accuracy using modified feature vector $h^{\approach}(\mathbf{x})$ is provided in Appendix~\ref{appendix: accuracy_analysis}.

    \noindent
    \textbf{Scope and Future Work.} In our empirical evaluation, we considered a diverse set of architectures spanning both modern and traditional deep learning models, including Swin Transformers, ConvNeXt, EfficientNet, ResNet, DenseNet, Wide-ResNet, and MobileNet. Although the principles of $\approach$ are general, its current formulation assumes a spatial aggregation operation such as global average pooling (GAP). Consequently, it is not directly applicable to early Vision Transformers (ViTs)~\cite{ViT}, which rely on a dedicated [CLS] token rather than spatial aggregation for classification. Extending $\approach$ to token-based transformer architectures is an important direction for future work.

    Our evaluation focuses on four representative scoring paradigms: logit-based, energy-based, gradient-based, and kNN distance-based methods. Energy-based scoring offers a strong performance–efficiency trade-off for post-hoc OOD detection, which motivates its inclusion as a primary baseline. While alternatives such as Mahalanobis distance~\cite{maha_distance} are studied, they incur substantial computational overhead. A promising direction for future research is to investigate whether combining $\approach$ with such computationally intensive methods yields gains that justify their additional cost.

\section{Ablation Studies}
\label{sec: ablation studies}

\vspace{-1mm}
    
    To provide a deeper understanding of $\approach$, this section summarizes three ablation studies. We present the main findings here and provide a detailed analysis in the appendix. These studies investigate our method's synergistic effect, justify our choice of statistics, and analyze its robustness to the hyperparameter $\gamma$.

    \textbf{Modular Integration with Primary Baselines.}
    To validate the complementary effect of $\approach$, we systematically evaluate baseline detectors both with and without $\approach$. The results show that $\approach$ consistently enhances strong baselines such as ASH and SCALE, across CIFAR and ImageNet benchmarks. These findings indicate that the enriched feature representation provides a stronger foundation for existing OOD detectors. A detailed breakdown is provided in Appendix~\ref{appendix: combined methods}.

    \textbf{Analysis of Alternate Statistics.} To justify our choice of statistics, we performed an ablation study using the median and Shannon entropy~\cite{shannon_entropy}. Both alternatives performed poorly, with FPR95 scores often exceeding 95\%, because they fail to produce a sufficiently distinctive and separable signal between ID and OOD samples. This analysis confirms that maximum and variance are superior choices as they produce a quantitatively stronger and more separable signal. The full results are presented in Appendix~\ref{appendix: alternate statistics}.

    \textbf{Sensitivity to Hyperparameter $\gamma$.} We analyzed the sensitivity of the $\approach(\mu, \sigma)$ method to its hyperparameter $\gamma$. Our findings indicate that the method is robust to the specific choice of $\gamma$. While performance generally improves as $\gamma$ increases from zero, the gains quickly saturate, showing that the model is not overly sensitive to precise tuning and performs well across a reasonable range of values. A detailed analysis is available in Appendix~\ref{appendix: hyper-parameter sensitiviy}.

\vspace{-3mm}

\section{Conclusion}
\label{sec: conclusion}

    \vspace{-1mm}

    In this work, we introduced $\approach$, a simple yet powerful technique that enhances OOD detection by leveraging statistical cues; specifically the channel-wise maximum and variance, that are typically discarded by GAP. Our extensive experiments demonstrate that $\approach$ is a versatile, complementary tool that significantly boosts the performance of existing techniques across diverse datasets and architectures. Notably, it shows remarkable robustness on modern models like Swin-B, ConvNext, and EfficientNet, where many conventional methods fails.

\bibliographystyle{splncs04}
\bibliography{main}

\newpage
\appendix
\section*{\centering \textbf{Appendix}}

\section{Description of Baseline Methods}
\label{appendix: baseline methods}

In resonance with existing work~\cite{energy,ReAct,DICE,ASH}, for the reader’s convenience, we summarize in detail a few common techniques for defining OOD scores that measure the degree of ID-ness on the given sample. All the methods derive the score post-hoc on neural networks trained with in-distribution data only. By convention, a higher score is indicative of being in-distribution, and vice versa.

\textbf{Softmax score } One of the earliest works on OOD detection considered using the maximum softmax
probability (MSP) to distinguish between $\mathcal{D}_{\text {in }}$ and $\mathcal{D}_{\text {out }}$~\cite{msp}. In detail, suppose the label space is \( \mathcal{Y} = \{1, 2, \cdots , C\}\). We assume the classifier $f$ is defined in terms of a feature extractor $f : \mathcal{X} \rightarrow \mathbb{R}^m$ and a linear multinomial regressor with weight matrix $W \in \mathbb{R}^{C \times m}$ and bias vector $\mathbf{b} \in \mathbb{R}^C$. The prediction probability for each class is given by :

\begin{small}
    \begin{equation}
        \mathbb{P}(y = c| \mathbf{x}) = \text{Softmax}(Wh(\mathbf{x}) + \mathbf{b})_{c}
    \end{equation}
\end{small}

The softmax score is defined as $S_{\text{MSP}}(\mathbf{x}; f) := \max_c\mathbb{P}(y = c| \mathbf{x})$.

\textbf{ODIN }~\cite{odin} This method introduced temperature scaling and input perturbation to improve the separation of MSP for ID and OOD data.  $\mathbf{\Tilde{x}}$ denotes perturbed input.

\begin{small}
    \begin{equation}
        \mathbb{P}(y = c| \mathbf{\Tilde{x}}) = \text{Softmax}[(Wh(\mathbf{\Tilde{x}}) + \mathbf{b})/T]_{c}
    \end{equation}
\end{small}

the ODIN score is defined as $S_{\text{ODIN}}(\mathbf{x}; f) := \max_c\mathbb{P}(y = c | \mathbf{\Tilde{x}})$.

\textbf{Energy score }  The energy function~\cite{energy} maps the output logit to a scalar $S_{\text{Energy}}(\mathbf{x}; f) \in \mathbb{R}$, which is relatively lower for ID data:

\begin{small}
    \begin{equation}
        S_{\text{Energy}}(\mathbf{x}; f) = -\text{Energy}(\mathbf{x}; f) = \log{ \left( \sum_{c=1}^{C} \exp( f_{c}(\mathbf{x}) ) \right) }
    \end{equation}
\end{small}

They used the \textit{negative energy score} for OOD detection, in order to align with the convention that $S(\mathbf{x}; f)$ is higher for ID data and vice versa.

\textbf{ReAct }  They perform post-hoc modification of penultimate layer of the neural network. It works by truncating the feature activations at a threshold $c$, i.e., replacing each activation with $\min(x,c)$. This limits the influence of abnormally large activations often caused by OOD inputs.The truncation threshold is set with the validation strategy in~\cite{ReAct}.Formally,
    \[
        h^{\text{ReAct}}(\mathbf{x}) = \text{ReAct}(h(\mathbf{x}); c) = \min(h(\mathbf{x}), c) \quad \text{(applied element-wise)}
    \]
The final model output becomes:
    \[
        f^{\text{ReAct}}(\mathbf{x}) = W^\top h^{\text{ReAct}}(\mathbf{x}) + \mathbf{b}
    \]
This method also uses energy score $S_{\text{Energy}}(\mathbf{x}; f^{\text{ReAct}}) \in \mathbb{R}$ for OOD detection.

\textbf{DICE}~\cite{DICE}   It is a post-hoc method to improve OOD detection by retaining only the most informative weights in the final layer of a pre-trained neural network. A \textit{contribution matrix} \( V \in \mathbb{R}^{m \times C} \) is computed, where each column is:
    \[
        \mathbf{v}_c = \mathbb{E}_{\mathbf{x} \in \mathcal{D}}[\mathbf{w}_c \odot h(\mathbf{x})]
    \]
with \( \odot \) denoting element-wise multiplication. Each entry in \( V \) quantifies the average contribution of a feature unit to class \( c \). A binary \textit{masking matrix} \( M \in \mathbb{R}^{m \times C} \) selects the top-$k$ highest-contributing weights, setting others to zero. The sparsified output is:
    \[
        f^{\text{DICE}}(\mathbf{x}; \theta) = (M \odot W)^\top h(\mathbf{x}) + \mathbf{b}
    \]
This method also uses energy score $S_{\text{Energy}}(\mathbf{x}; f^{\text{DICE}}) \in \mathbb{R}$ for OOD detection.

\textbf{ASH }~\cite{ASH}  It is also a post-hoc method that simplifies feature representations to improve OOD detection. They proposes three versions of ASH, we presented only the best performing version i.e, ASH-S.  Given an input activation vector \( h(\mathbf{x}) \) and a pruning percentile \( p \), ASH~\cite{ASH} proceeds as follows shaping the activation of penultimate layer $h(\mathbf{x})$ to get $h^{\text{ASH}}(\mathbf{x})$:

\begin{enumerate}
    \item Compute the \( p \)-th percentile threshold \( t \) of \( h(\mathbf{x}) \).
    \item Let \( s_1 = \sum h(\mathbf{x}) \), the sum of all activation values before pruning.
    \item Set all values in \( h(\mathbf{x}) \) less than \( t \) to zero.
    \item Let \( s_2 = \sum h(\mathbf{x}) \), the sum after pruning.
    \item Scale all non-zero values in \( h(\mathbf{x}) \) by \( \exp(s_1/s_2) \).
\end{enumerate}

The final model output becomes, which is then used to compute energy score $S_{\text{Energy}}(\mathbf{x}; f^{\text{ASH}}) \in \mathbb{R}$ for OOD detection :
    \[
        f^{\text{ASH}}(\mathbf{x}) = W^\top h^{\text{ASH}}(\mathbf{x}) + \mathbf{b}
    \]

\textbf{SCALE}~\cite{SCALE} It is a post-hoc method designed to enhance out-of-distribution (OOD) detection by adaptively scaling the activation of the penultimate layer $h(\mathbf{x})$ before computing the final classifier output. Given an input activation vector $h(\mathbf{x})$ and a pruning percentile $p$, SCALE~\cite{SCALE} proceeds as follows to obtain the scaled activation $h^{\text{SCALE}}(\mathbf{x})$:

\begin{enumerate}
    \item Compute the $p$-th percentile threshold $t$ of $h(\mathbf{x})$.
    \item Let $s_1 = \sum h(\mathbf{x})$, the sum of all activation values before pruning.
    \item Construct a binary mask $\mathbf{1}_{\{h(\mathbf{x}) \geq t\}}$ that keeps only the top-$p$ activations.
    \item Let $s_2 = \sum h(\mathbf{x}) \cdot \mathbf{1}_{\{h(\mathbf{x}) \geq t\}}$, the sum of the top-$p$ activations.
    \item Compute the scaling ratio $r = \tfrac{s_1}{s_2}$.
    \item Scale the original activations by $\exp(r)$:
    \[
        h^{\text{SCALE}}(\mathbf{x}) = \exp(r) \cdot h(\mathbf{x}).
    \]
\end{enumerate}

The final model output is then computed with the scaled activations, and the \emph{energy score} is used for OOD detection:
\[
    f^{\text{SCALE}}(\mathbf{x}) = W^\top h^{\text{SCALE}}(\mathbf{x}) + \mathbf{b}, 
    \quad
    S_{\text{Energy}}(\mathbf{x}; f^{\text{SCALE}}) \in \mathbb{R}.
\]

\textbf{KNN}~\cite{knn} This post-hoc, feature-space method identifies OOD samples based on their distance from ID training manifold. Let $\mathcal{H}_{\text{train}} = \{h(\mathbf{x}_i) \in \mathbb{R}^d\}_{i=1}^N$ be the set of $N$ penultimate-layer feature vectors stored from the ID training set. For a new test input $\mathbf{x}$ with feature $h(\mathbf{x})$, the kNN score is computed in three steps:

\begin{enumerate}
    \item Compute Distances: The set of Euclidean distances $\{d_i\}$ between $h(\mathbf{x})$ and all stored ID features in $\mathcal{H}_{\text{train}}$ is computed:$$d_i = \| h(\mathbf{x}) - h(\mathbf{x}_i) \|_2, \quad \forall h(\mathbf{x}_i) \in \mathcal{H}_{\text{train}}$$.

    \item Identify Neighbors: The $k$ smallest distances are identified and sorted, $d_{(1)} \leq d_{(2)} \leq \dots \leq d_{(k)}$.

    \item Calculate Score: The final kNN score is the average distance to these $k$ nearest neighbors:$$S_{\text{kNN}}(\mathbf{x}) = \frac{1}{k} \sum_{j=1}^k d_{(j)}$$
\end{enumerate}

A large score $S_{\text{kNN}}(\mathbf{x})$ indicates that the sample lies far from the ID training manifold and is therefore flagged as out-of-distribution.

\textbf{GradNorm}~\cite{GradNorm} This post-hoc, gradient-space method detects OOD samples using the magnitude of backpropagated gradients. Let $f(\mathbf{x};\theta)$ denote a pre-trained classifier and let $\mathbf{p}(\mathbf{x}) = \text{Softmax}(f(\mathbf{x}))$ be its predictive distribution over $C$ classes. GradNorm computes the gradient of the KL divergence between $\mathbf{p}(\mathbf{x})$ and the uniform distribution $\mathbf{u} = [1/C,\dots,1/C]$.

For a test input $\mathbf{x}$, the GradNorm score is computed in three steps:
\begin{enumerate}
    \item Compute KL Divergence: The KL divergence between the uniform distribution and the model prediction is computed:
        \[
            \mathcal{L}_{\text{KL}}(\mathbf{x})
            = D_{\text{KL}}\!\left(\mathbf{u} \,\|\, \mathbf{p}(\mathbf{x})\right)
            = -\frac{1}{C} \sum_{c=1}^{C} \log p_c(\mathbf{x}) + \text{const}.
        \]
    
    \item Compute Gradients: The gradient of $\mathcal{L}_{\text{KL}}(\mathbf{x})$ is backpropagated with respect to selected model parameters $\mathbf{w}$ (typically the last fully connected layer weights):
        \[
            \mathbf{g}(\mathbf{x}) 
            = \frac{\partial \mathcal{L}_{\text{KL}}(\mathbf{x})}{\partial \mathbf{w}}.
        \]
    
    \item Calculate Score: The OOD score is defined as the L1 norm of the gradient vector:
        \[
            S_{\text{GradNorm}}(\mathbf{x}) 
            = \|\mathbf{g}(\mathbf{x})\|_1.
        \]
\end{enumerate}

A larger score $S_{\text{GradNorm}}(\mathbf{x})$ indicates that the input induces stronger gradients and is therefore more likely to be in-distribution, while smaller scores suggest OOD samples.

\section{Statistical Analysis}
\label{appendix: analysis}

    In this section, we present a detailed statistical analysis of our method, $\approach$, demonstrating how it enhances the separation between in-distribution (ID) and out-of-distribution (OOD) samples. This increased separation leads to a sharper decision boundary between ID and OOD regions. Our analysis builds on key observations commonly made in prior work on OOD detection~\citep{energy, ReAct, DICE, ASH, SCALE}, which we adopt as foundational to our analysis. 

\subsection{Setup}
 
    We consider a trained neural network parameterized by \(\theta\), which encodes an input \(\mathbf{x} \in \mathbb{R}^d\) to \(n\) spatial activation maps, denoted by \( g(\mathbf{x}) \in \mathbb{R}^{n \times k \times k} \). These activation maps are then transformed into \(n\) dimensional feature vector \(h(\mathbf{x}) \in \mathbb{R}^n\) (i.e., penultimate layer) via global average pooling (GAP) as shown in Equation~\ref{eq:avg_pooling}, where \texttt{Avg} denotes the GAP operation applied independently to each of the  \(n\) activation maps in \(g(\mathbf{x})\).  

    \begin{equation}
                h(\mathbf{x}) = \texttt{Avg}\left( g(\mathbf{x}) \right)
            \label{eq:avg_pooling}
    \end{equation}

    A weight matrix \( \mathbf{W} \in \mathbb{R}^{n \times C} \) connects the feature vector \(h(\mathbf{x}) \in \mathbb{R}^n\) to the output logit \(f(\mathbf{x}) \in \mathbb{R}^C\) as shown in Equation~\ref{eq:logit} , where \( C \) is the total number of classes in \( \mathcal{Y} = \{1, 2, \cdots , C\}\). The function maps the output logit \( f(\mathbf{x}) \) to a scalar energy \( \mathbf{E}_{\theta}(\mathbf{x})\), which is relatively lower for ID data~\cite{energy} as shown in Equation~\ref{eq:score}.

    \begin{small}
        \begin{equation} 
                f(\mathbf{x}) = \mathbf{W}^\top h(\mathbf{x}) + \mathbf{b}
        \label{eq:logit}
        \end{equation}
    \end{small}
 
    \begin{small}
        \begin{equation} 
            S_{\theta}(\mathbf{x}) = -\mathbf{E}_{\theta}(\mathbf{x}) = \log{ \left( \sum_{c=1}^{C} \exp( f(\mathbf{x}; \theta) ) \right) }
        \label{eq:score}
        \end{equation}
    \end{small}

    The goal of OOD detection is to learn a decision boundary $G_{\lambda}(\mathbf{x} ; \theta)$  that classifies a test sample $\mathbf{x} \in \mathcal{X}$: 
    \begin{equation}
         G_{\lambda}(\mathbf{x} ; \theta)  = 
                \begin{cases}
                    \text{in} & \text { if } S_{\theta}(\mathbf{x}) \ge \lambda \\ 
                    \text{out} & \text { if } S_{\theta}(\mathbf{x}) < \lambda
                \end{cases}
    \end{equation}

    where a thresholding mechanism is employed to distinguish between ID and OOD samples. To align with the convention, samples with higher scores \( S_{\theta}(\mathbf{x}) \) are classified as ID while samples with lower scores are classified as OOD. By convention~\citep{energy}, the threshold \( \lambda \) is typically chosen such that a high fraction of ID, (\emph{e.g., 95\%}) is correctly classified in practice.

    As part of $\approach$,  we retrieve mean \(\mu(\mathbf{x}) \in \mathbb{R}^n \), variance \(\sigma^2(\mathbf{x}) \in \mathbb{R}^n \), and maximum \(m(\mathbf{x}) \in \mathbb{R}^n \) from the feature maps in \( g(\mathbf{x}) \). Up to this point, $\mu(\mathbf{x})$ and $h(\mathbf{x})$ are numerically equivalent. However, semantically, $\mu(\mathbf{x})$ denotes the extracted statistical features, while $h(\mathbf{x})$ refers to the penultimate layer representation. $\approach$ modifies the penultimate layer \( h(\mathbf{x}) \) of the model using \(\mu(\mathbf{x}) \), \(\sigma^2(\mathbf{x}) \), and \(m(\mathbf{x}) \) for enhanced OOD detection. In this work, we extensively explored two versions of $\approach$: 
    \begin{itemize}
        \item \( \approach(m) \)  replaces feature vector \( h(\mathbf{x}) \) by the maximum (dominant) \( m(\mathbf{x}) \) as shown in Equation~\ref{eq:DomAct_v1}.
        \begin{small}
            \begin{equation}
                    h^{\approach(m)}(\mathbf{x}) = m(\mathbf{x})
                \label{eq:DomAct_v1}
            \end{equation}
        \end{small}

        \item $\approach(\mu,\sigma)$  augments the mean activation with its corresponding channel-wise standard deviation $\sigma(\mathbf{x})$, scaled by a hyperparameter $\gamma$\ as shown in Equation~\ref{eq:DomAct_v2}.
        \begin{small}
            \begin{equation}
                     h^{\approach(\mu,\sigma)}(\mathbf{x}) = \mu(\mathbf{x}) + \gamma\sigma(\mathbf{x})
                \label{eq:DomAct_v2}
            \end{equation}
        \end{small}
    \end{itemize}

    In effect, both $\approach$ versions raise the activation level of feature vector \( h(\mathbf{x}) \). Our analysis focuses on understanding how modifying the penultimate feature representation $h(\mathbf{x})$ impacts the final OOD score. Recall that the logit vector $f(\mathbf{x})$ is a linear transformation of these features, i.e., $f(\mathbf{x}) = \mathbf{W}^\top h(\mathbf{x}) + \mathbf{b}$. Consequently, any logit-based scoring function, such as the energy score, is ultimately a function of $h(\mathbf{x})$.

\subsection{Analysis}

    In this section, we provide a statistical analysis of our method. Our analysis is grounded in a foundational observation regarding the behavior of features extracted from well-trained classifiers, consistent with prior work~\citep{energy,ReAct,ASH,SCALE}. We denote in-distribution and out-of-distribution samples as $\mathbf{x}_{\text{in}}$ and $\mathbf{x}_{\text{out}}$, respectively.

    \begin{observation}
    \label{obs:obs1}
        Given a well-trained model \( \theta \), the statistical features extracted from $\mathbf{x}_{\text{in}}$ samples consistently exhibit higher magnitudes than those from $\mathbf{x}_{\text{out}}$ samples. This holds true for the channel-wise mean \( \mu(\mathbf{x}) \), maximum \( m(\mathbf{x}) \), and standard deviation \( \sigma(\mathbf{x}) \). Formally, we state this as shown in Equations~\ref{eq:feature_stats}. More precisely, these inequalities are characteristic of the majority of individual feature dimensions as shown in Figure~\ref{fig:obs1}. 
        
        \begin{subequations}
            \begin{align}
                \mathbb{E}_{\mathbf{x} \sim \mathcal{D}_{\mathrm{in}}} [\mu(\mathbf{x})] \ge \mathbb{E}_{\mathbf{x} \sim \mathcal{D}_{\mathrm{out}}} [\mu(\mathbf{x})] \\
                \mathbb{E}_{\mathbf{x} \sim \mathcal{D}_{\mathrm{in}}} [m(\mathbf{x})] \ge \mathbb{E}_{\mathbf{x} \sim \mathcal{D}_{\mathrm{out}}} [m(\mathbf{x})] \\
                \mathbb{E}_{\mathbf{x} \sim \mathcal{D}_{\mathrm{in}}} [\sigma(\mathbf{x})] \ge \mathbb{E}_{\mathbf{x} \sim \mathcal{D}_{\mathrm{out}}} [\sigma(\mathbf{x})] 
            \end{align}
            \label{eq:feature_stats}
        \end{subequations}
    
        This fundamental property enables the network to perform both its primary classification task and OOD detection effectively. 
        
    \end{observation}

    \begin{figure*}[ht]
      \centering
      \includegraphics[width=\textwidth]{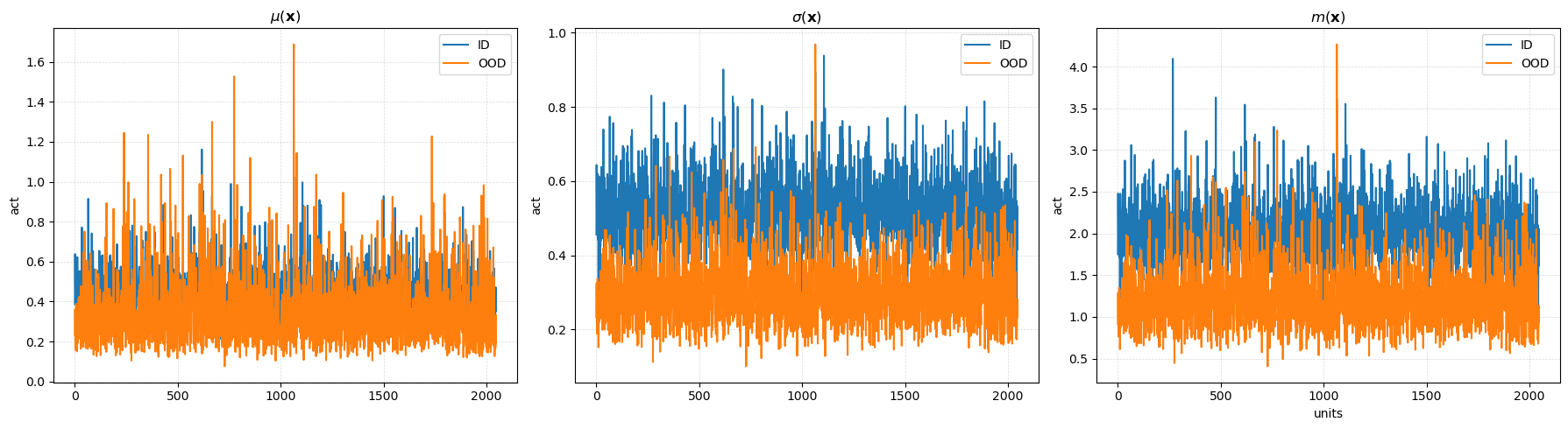}
      \caption{ \textit{Unit-wise comparison of statistical features for ID vs. OOD samples, with values averaged over the entire test set. Across a majority of feature dimensions, the mean (\( \mu(\mathbf{x}) \)), standard deviation (\( \sigma(\mathbf{x}) \)), and maximum (\( m(\mathbf{x}) \)) statistics all exhibit consistently higher values for ID samples (blue) than for OOD samples (orange). Results are shown for a ResNet-50 model with ImageNet-1K as the ID dataset and Texture as the OOD dataset. This trend holds consistently across other architectures and data combinations.}}
      \label{fig:obs1}
    \end{figure*}

    \begin{definition}
        To quantify the separation for a given feature vector $h(\mathbf{x})$, we define the separation gap $\Delta_h$, as the difference in its expected value across the ID and OOD distributions: 
            \[
                \Delta_h = \mathbb{E}_{\mathbf{x} \sim \mathcal{D}_{\mathrm{in}}}[h(\mathbf{x})] 
                           - \mathbb{E}_{\mathbf{x} \sim \mathcal{D}_{\mathrm{out}}}[h(\mathbf{x})] = \mathbb{E}[h(\mathbf{x}_{\mathrm{in}}) - h(\mathbf{x_{\mathrm{out}}})]
            \]
        Specifically, the separation gaps for the mean, maximum, and our combined mean-and-standard-deviation feature are:
        \begin{subequations}
            \begin{align}
            \Delta_\mu & := \mathbb{E}_{\mathbf{x} \sim \mathcal{D}_{\mathrm{in}}}[\mu(\mathbf{x})] 
                       - \mathbb{E}_{\mathbf{x} \sim \mathcal{D}_{\mathrm{out}}}[\mu(\mathbf{x})]  = \mathbb{E}[\mu(\mathbf{x}_{\mathrm{in}}) - \mu(\mathbf{x_{\mathrm{out}}})], \\
             \Delta_m &:= \mathbb{E}_{\mathbf{x} \sim \mathcal{D}_{\mathrm{in}}}[m(\mathbf{x})] 
                       - \mathbb{E}_{\mathbf{x} \sim \mathcal{D}_{\mathrm{out}}}[m(\mathbf{x})]  = \mathbb{E}[m(\mathbf{x}_{\mathrm{in}}) - m(\mathbf{x_{\mathrm{out}}})],  \\
             \Delta_{\mu,\sigma} &:= \mathbb{E}_{\mathbf{x} \sim \mathcal{D}_{\mathrm{in}}}[\mu(\mathbf{x}) + \sigma(\mathbf{x})] 
                                - \mathbb{E}_{\mathbf{x} \sim \mathcal{D}_{\mathrm{out}}}[\mu(\mathbf{x}) + \sigma(\mathbf{x})]  = \mathbb{E}[\mu(\mathbf{x}_{\mathrm{in}}) + \sigma(\mathbf{x}_{\mathrm{in}}) - \mu(\mathbf{x_{\mathrm{out}}}) - \sigma(\mathbf{x_{\mathrm{out}}})]
            \end{align}
        \end{subequations}
    \end{definition}

    \begin{lemma}
    \label{lemma:lemma1}
        Given Observation~\ref{obs:obs1}, the separation gap of the combined mean-and-standard-deviation feature is greater than or equal to that of the mean feature alone:
        \[
            \Delta_{\mu,\sigma} \ge \Delta_\mu
        \]
    \end{lemma}
    \begin{proof} By linearity of expectation, we can expand the definition of $\Delta_{\mu,\sigma}$ as follows:
        \begin{small}
            \[
                \begin{split}
                    &\Delta_{\mu,\sigma} - \Delta_\mu \\
                    &= \Big( \mathbb{E}_{\mathbf{x} \sim \mathcal{D}_{\mathrm{in}}}[\mu(\mathbf{x}) + \sigma(\mathbf{x})] 
                                    - \mathbb{E}_{\mathbf{x} \sim \mathcal{D}_{\mathrm{out}}}[\mu(\mathbf{x}) + \sigma(\mathbf{x})] \Big) - \Big(  \mathbb{E}_{\mathbf{x} \sim \mathcal{D}_{\mathrm{in}}}[\mu(\mathbf{x})] 
                           - \mathbb{E}_{\mathbf{x} \sim \mathcal{D}_{\mathrm{out}}}[\mu(\mathbf{x})] \Big) \\
                    &= \mathbb{E}_{\mathbf{x} \sim \mathcal{D}_{\mathrm{in}}} [\mu(\mathbf{x})] +  \mathbb{E}_{\mathbf{x} \sim \mathcal{D}_{\mathrm{in}}} [\sigma(\mathbf{x})] - \mathbb{E}_{\mathbf{x} \sim \mathcal{D}_{\mathrm{out}}} [\mu(\mathbf{x})]  - \mathbb{E}_{\mathbf{x} \sim \mathcal{D}_{\mathrm{out}}} [\sigma(\mathbf{x})] -\mathbb{E}_{\mathbf{x} \sim \mathcal{D}_{\mathrm{in}}}[\mu(\mathbf{x})] 
                           + \mathbb{E}_{\mathbf{x} \sim \mathcal{D}_{\mathrm{out}}}[\mu(\mathbf{x})] \\
                    &= \mathbb{E}_{\mathbf{x} \sim \mathcal{D}_{\mathrm{in}}} [\sigma(\mathbf{x})] - \mathbb{E}_{\mathbf{x} \sim \mathcal{D}_{\mathrm{out}}} [\sigma(\mathbf{x})] \\
                    &\ge 0 \qquad \qquad \qquad \qquad ( \text{Recall Equation}~\ref{eq:feature_stats} c \text{ of Observation}~\ref{obs:obs1})
                \end{split}
            \]
        \end{small}
        This result is empirically validated in leftmost plot of Figure~\ref{fig:obs2}, which shows that incorporating the standard deviation consistently increases the separation between ID and OOD samples.
    \end{proof}

    \begin{figure}[ht]
          \centering
          \includegraphics[width=\textwidth]{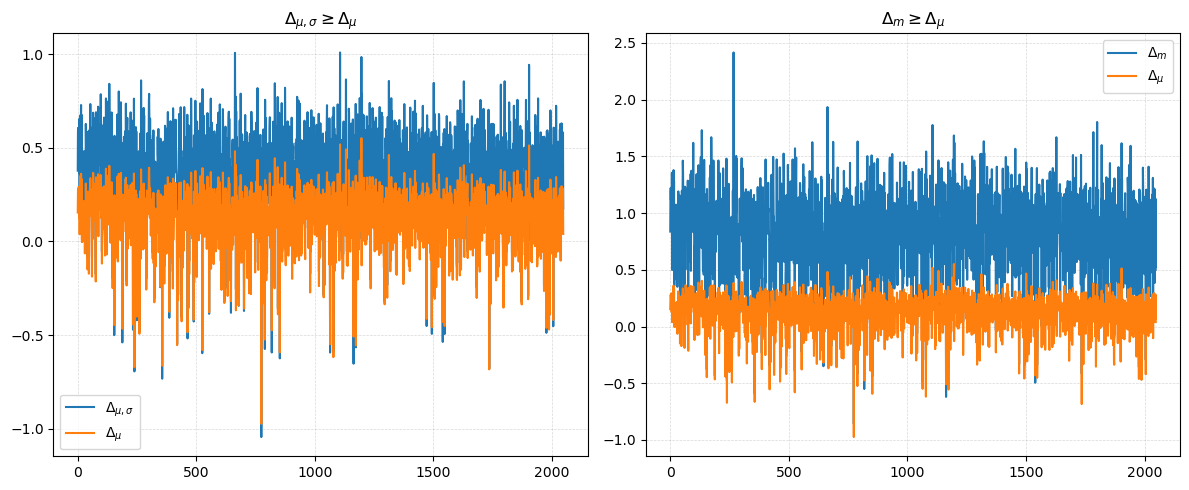}
          \caption{\textit{Comparison of the separation gap \( \Delta \) achieved by different statistical features, averaged over all test samples. Left: It demonstrate that incorporating the standard deviation \( \Delta_{\mu,\sigma} \) yields a larger separation gap than using the mean activation alone \( \Delta_\mu \). Right: It demonstrate that using the maximum activation \( \Delta_m \) yields a larger separation gap than using the mean activation \( \Delta_\mu \). Results are shown for a ResNet-50 model with ImageNet as the ID dataset and Texture as the OOD dataset. This finding holds consistently across other architectures and data combinations.}}
          \label{fig:obs2}
    \end{figure}
    
    \begin{observation}
        \label{obs:obs2} Our experiments consistently show that the maximum statistic provides a stronger separation signal than the mean statistic. We state this empirical finding as the following inequality:
        \[
            \Delta_m \ge \Delta_\mu
        \]
        This is consistent with the behavior of discriminative classifiers: ID samples are trained to elicit high-magnitude feature responses, while OOD samples tend to produce weaker, more uniform activations. This makes the maximum a more distinctive signal than the mean, which is empirically validated in right plot in Figure~\ref{fig:obs2}.
    \end{observation}

    \begin{assumption}
    \label{assumption: assumption_01}
        Our analysis adopts a key assumption from ReAct~\citep{ReAct}, a principle also leveraged by subsequent methods like SCALE~\citep{SCALE} and ASH~\citep{ASH}. To formally analyze the effect of feature modifications on the output logits $f(\mathbf{x}) = \mathbf{W}^\top h(\mathbf{x}) + \mathbf{b}$ we assume a sufficient (though not strictly necessary) condition on the classifier's final weight matrix, $\mathbf{W}$. Specifically, we assume that $\mathbf{W}^\top\mathbf{1} \ge 0$ element-wise. As noted in~\citep{ReAct}, this property is often observed empirically and can be achieved by adding a positive constant to $\mathbf{W}$ without changing the final classification decisions.
    \end{assumption}

\begin{theorem}
\label{thm:separation_transfer}
    Let $h^{\mu}(\mathbf{x})$ be the baseline feature vector (from GAP) and $h^\approach(\mathbf{x})$ be an enhanced feature vector from our method. Let $\Delta^{\mu}_{h} = \mathbb{E}[h^\mu(\mathbf{x}_{\mathrm{in}})] - \mathbb{E}[h^\mu(\mathbf{x}_{\mathrm{out}})]$  and $\Delta_h^\approach = \mathbb{E}[h^\approach(\mathbf{x}_{\mathrm{in}})] - \mathbb{E}[h^\approach(\mathbf{x}_{\mathrm{out}})]$ be the respective feature separation gap vectors. Then, under the Assumption~\ref{assumption: assumption_01}, the separation between the ID and OOD logits is also increased:
    \[
         \mathbb{E}[f^\approach(\mathbf{x}_{\mathrm{in}}) - f^\approach(\mathbf{x}_{\mathrm{out}})] \ge \mathbb{E}[f^\mu(\mathbf{x}_{\mathrm{in}}) - f^\mu(\mathbf{x}_{\mathrm{out}})]
    \]
\end{theorem}

\begin{proof} 

To derive the effect on the distribution of model output, consider output logits $f(\mathbf{x}) = \mathbf{W}^\top h(\mathbf{x}) + \mathbf{b}$ as shown in Equation~\ref{eq:logit} and assume without loss of generality that element-wise $\mathbf{W}^\top\mathbf{1} \ge 0$. This can be achieved by adding a positive constant to $\mathbf{W}$ without changing the output probabilities or classification decision (Assumption~\ref{assumption: assumption_01}).

\textit{Case 1:} For notational clarity in the following analysis, let us denote the standard feature vector (from GAP) as $h^{\mu}(\mathbf{x})$ and our enhanced feature vector as $h^{\mu + \sigma}(\mathbf{x})$. The corresponding logits are then computed as $f^{\mu}(\mathbf{x}) = \mathbf{W}^\top h^{\mu}(\mathbf{x}) + \mathbf{b}$ and  $f^{\mu+\sigma}(\mathbf{x}) = \mathbf{W}^\top h^{\mu +\sigma}(\mathbf{x}) + \mathbf{b}$, respectively. 

Let $\delta = \mathbb{E}\left[ h^{\mu +\sigma}(\mathbf{x}_{\mathrm{in}}) - h^{\mu +\sigma}(\mathbf{x}_{\mathrm{out}})  \right] - \mathbb{E}\left[ h^{\mu}(\mathbf{x}_{\mathrm{in}}) - h^{\mu}(\mathbf{x}_{\mathrm{out}})  \right] \geq 0 $ (recall Lemma~\ref{lemma:lemma1})

\[
\begin{split}
     & \mathbb{E}\left[f^{\mu+\sigma}(\mathbf{x}_{\mathrm{in}}) - f^{\mu + \sigma}(\mathbf{x}_{\mathrm{out}})\right] \\
    =& \quad \mathbb{E}\left[\mathbf{W}^\top ( h^{\mu+\sigma}(\mathbf{x}_{\mathrm{in}}) - h^{\mu+\sigma}(\mathbf{x}_{\mathrm{out}}) ) \right] \\
    =& \quad \mathbf{W}^\top \mathbb{E}\left[ h^{\mu+\sigma}(\mathbf{x}_{\mathrm{in}}) - h^{\mu+\sigma}(\mathbf{x}_{\mathrm{out}})  \right]\\
    =&  \quad \mathbf{W}^\top \Bigg( \mathbb{E}\left[ h^{\mu +\sigma}(\mathbf{x}_{\mathrm{in}}) - h^{\mu +\sigma}(\mathbf{x}_{\mathrm{out}})  \right] - \mathbb{E}\left[ h^{\mu}(\mathbf{x}_{\mathrm{in}}) - h^{\mu}(\mathbf{x}_{\mathrm{out}})  \right] + \mathbb{E}\left[ h^{\mu}(\mathbf{x}_{\mathrm{in}}) - h^{\mu}(\mathbf{x}_{\mathrm{out}})  \right]\Bigg) \\
    =& \quad \mathbf{W}^\top \Bigg( \mathbb{E}\left[ h^{\mu}(\mathbf{x}_{\mathrm{in}}) - h^{\mu}(\mathbf{x}_{\mathrm{out}})  \right] + \Delta_{\mu, \sigma} - \Delta_{\mu} \Bigg) \\
    =& \quad \mathbf{W}^\top \Bigg( \mathbb{E}\left[ h^{\mu}(\mathbf{x}_{\mathrm{in}}) - h^{\mu}(\mathbf{x}_{\mathrm{out}})  \right] + \delta\mathbf{1} \Bigg) \\
    =& \quad  \mathbb{E}\Bigg[ \mathbf{W}^\top \Big( h^{\mu}(\mathbf{x}_{\mathrm{in}}) - h^{\mu}(\mathbf{x}_{\mathrm{out}}) \Big) \Bigg] + \delta\mathbf{W}^\top\mathbf{1} \\
    \geq & \quad \mathbb{E}\left[f^{\mu}(\mathbf{x}_{\mathrm{in}}) - f^{\mu}(\mathbf{x}_{\mathrm{out}})\right]  \quad \quad \quad \quad (\because \mathbf{W}^\top\mathbf{1} \geq 0; \text{Assumption}~\ref{assumption: assumption_01})
\end{split} 
\]

\textit{Case 2:} Similar to above, for notational clarity, let us denote the standard feature vector as $h^{\mu}(\mathbf{x})$ and our dominant feature vector as $h^{m}(\mathbf{x})$. The corresponding logits are then computed as $f^{\mu}(\mathbf{x}) = \mathbf{W}^\top h^{\mu}(\mathbf{x}) + \mathbf{b}$ and  $f^{m}(\mathbf{x}) = \mathbf{W}^\top h^{m}(\mathbf{x}) + \mathbf{b}$, respectively.

Let $\delta = \mathbb{E}\left[ h^{m}(\mathbf{x}_{\mathrm{in}}) - h^{m}(\mathbf{x}_{\mathrm{out}})  \right] - \mathbb{E}\left[ h^{\mu}(\mathbf{x}_{\mathrm{in}}) - h^{\mu}(\mathbf{x}_{\mathrm{out}})  \right] \geq 0 $ (recall Observation~\ref{obs:obs2})

\[
\begin{split}
     & \mathbb{E}\left[f^{m}(\mathbf{x}_{\mathrm{in}}) - f^{m}(\mathbf{x}_{\mathrm{out}})\right] \\
    =& \quad \mathbb{E}\left[\mathbf{W}^\top ( h^{m}(\mathbf{x}_{\mathrm{in}}) - h^{m}(\mathbf{x}_{\mathrm{out}}) ) \right] \\
    =& \quad \mathbf{W}^\top \mathbb{E}\left[ h^{m}(\mathbf{x}_{\mathrm{in}}) - h^{m}(\mathbf{x}_{\mathrm{out}})  \right]\\
    =&  \quad \mathbf{W}^\top \Bigg( \mathbb{E}\left[ h^{m}(\mathbf{x}_{\mathrm{in}}) - h^{m}(\mathbf{x}_{\mathrm{out}})  \right] - \mathbb{E}\left[ h^{\mu}(\mathbf{x}_{\mathrm{in}}) - h^{\mu}(\mathbf{x}_{\mathrm{out}})  \right] + \mathbb{E}\left[ h^{\mu}(\mathbf{x}_{\mathrm{in}}) - h^{\mu}(\mathbf{x}_{\mathrm{out}})  \right]\Bigg) \\
    =& \quad \mathbf{W}^\top \Bigg( \mathbb{E}\left[ h^{\mu}(\mathbf{x}_{\mathrm{in}}) - h^{\mu}(\mathbf{x}_{\mathrm{out}})  \right] + \Delta_{m} - \Delta_{\mu} \Bigg) \\
    =& \quad \mathbf{W}^\top \Bigg( \mathbb{E}\left[ h^{\mu}(\mathbf{x}_{\mathrm{in}}) - h^{\mu}(\mathbf{x}_{\mathrm{out}})  \right] + \delta\mathbf{1} \Bigg) \\
    =& \quad  \mathbb{E}\Bigg[ \mathbf{W}^\top \Big( h^{\mu}(\mathbf{x}_{\mathrm{in}}) - h^{\mu}(\mathbf{x}_{\mathrm{out}}) \Big) \Bigg] + \delta\mathbf{W}^\top\mathbf{1} \\
    \geq & \quad \mathbb{E}\left[f^{\mu}(\mathbf{x}_{\mathrm{in}}) - f^{\mu}(\mathbf{x}_{\mathrm{out}})\right] \quad \quad \quad \quad (\because \mathbf{W}^\top\mathbf{1} \geq 0; \text{Assumption}~\ref{assumption: assumption_01})
\end{split} 
\]

Thus, our analysis demonstrates that the enhanced feature separation provided by both formulations of $\approach$ directly propagates to the logit space, resulting in a more discriminative output for OOD detection. Note that based on Assumption~\ref{assumption: assumption_01} inspired by ReAct~\citep{ReAct} the condition of $\mathbf{W}^\top\mathbf{1} \ge 0$ is sufficient but not necessary for this result to hold.

\textbf{Why $\approach$ improves the OOD scoring functions?} Our analysis demonstrates that $\approach$ improves OOD detection by amplifying the separation between the expected logit values of ID and OOD samples. For logit-based scoring functions, such as the energy score~\cite{energy}, this increased logit separation directly translates to a wider gap between the ID and OOD score distributions. This enhanced separability improves the ability to distinguish between ID and OOD, leading to better OOD detection performance. This mechanism is empirically validated in Figure~\ref{fig:density_delta_plot}, which illustrates the clearer separation in score densities after applying $\approach$.

\end{proof}

\section{Detailed OOD Detection Performance}
\label{appendix: detailed results}

\subsection{Near-OOD Evaluation}
\label{appendix: near_ood_results}

    In Table~\ref{table: cifar_near_ood}, we further evaluate $\approach$ on CIFAR-10, using CIFAR-100 and Tiny-ImageNet as near-OOD datasets. Across all evaluated architectures—ResNet-18, DenseNet-101, and Wide-ResNet-28-10, our method consistently outperforms the primary baselines. For example, $\approach$ reduces FPR95 by 9.34\% with ResNet-18 and 8.40\% with Wide-ResNet-28-10, while achieving a modest improvement of 2.87\% with DenseNet-101. These results demonstrate robustness even when the semantic gap between ID and OOD data is small. For brevity, we omit the evaluation using ResNet-34 and MobileNet-v2.

        \begin{table}[!ht]
        \centering
        \caption{Near-OOD detection CIFAR-10 evaluation using CIFAR-100 and Tiny-ImageNet is OOD dataset. The symbol $\boldsymbol{\downarrow}$ indicates lower values are better;  $\boldsymbol{\uparrow}$ indicates higher values are better.}
        \resizebox{\linewidth}{!}{
        \begin{tabular}{l l cc cc cc }
        \toprule
        \multirow{2}{*}{\textbf{Model}} & \multirow{2}{*}{\textbf{Method}} 
        & \multicolumn{2}{c}{\textbf{CIFAR-100}} & \multicolumn{2}{c}{\textbf{Tiny-ImageNet}} & \multicolumn{2}{c}{\textbf{Average}}   \\
        \cmidrule(lr){3-4} \cmidrule(lr){5-6} \cmidrule(lr){7-8} 
        && \textbf{FPR95} $\downarrow$ & \textbf{AUROC} $\uparrow$ & \textbf{FPR95} $\downarrow$ & \textbf{AUROC} $\uparrow$ & \textbf{FPR95} $\downarrow$ & \textbf{AUROC} $\uparrow$ \\
        \midrule
        \multirow{11}{*}{\rotatebox{90}{ResNet-18}} 
        & MSP          & 65.85 & 88.17 & 64.51 & 87.72 & 65.18 & 87.94   \\
        & ODIN         & 51.79 & 90.26 & 48.49 & 90.16 & 50.14 & 90.21   \\
        & Energy       & 52.32 & 90.14 & 47.74 & 90.38 & 50.03 & 90.26   \\
        & ReAct        & 52.04 & 90.42 & 47.65 & 90.48 & 49.84 & 90.45   \\
        & DICE         & 56.56 & 89.12 & 50.08 & 89.79 & 53.32 & 89.45   \\
        & ASH-S        & 51.97 & 90.17 & 48.01 & 90.34 & 49.99 & 90.26   \\
        & SCALE        & 51.74 & 90.20 & 47.89 & 90.39 & 49.81 & 90.29   \\
        & GradNorm     & 59.77 & 86.50 & 55.81 & 87.66 & 57.79 & 87.08   \\
        & KNN          & 53.84 & 89.30 & 53.14 & 88.74 & 53.49 & 89.02   \\
        \cmidrule(lr){1-8}
         \rowcolor{gray!20} & $\approach({\mu,\sigma}) + \texttt{SCALE}$   & 47.93 & 90.64 & 42.75 & 91.19 & 45.34 & 90.91   \\
        \rowcolor{gray!20}  & $\approach(m) + \texttt{SCALE}$              & 47.49 & 90.20 & 41.77 & 91.04 & 44.63 & 90.62   \\
        \midrule
        \multirow{11}{*}{\rotatebox{90}{DenseNet-101}}  
        & MSP           & 63.49 & 88.53 & 61.80 & 88.27 & 62.64 & 88.40   \\
        & ODIN          & 47.73 & 90.36 & 44.19 & 90.70 & 45.96 & 90.53   \\
        & Energy        & 48.86 & 90.29 & 43.01 & 91.19 & 45.93 & 90.74   \\
        & ReAct         & 47.60 & 90.72 & 42.95 & 90.95 & 45.27 & 90.83   \\
        & DICE          & 53.10 & 88.91 & 43.38 & 90.95 & 48.24 & 89.93   \\
        & ASH-S         & 48.64 & 90.24 & 43.45 & 91.09 & 46.05 & 90.66   \\
        & SCALE         & 48.70 & 90.24 & 43.48 & 91.09 & 46.09 & 90.67   \\
        & GradNorm      & 56.16 & 86.03 & 49.54 & 88.50 & 52.85 & 87.27   \\
        & KNN           & 49.52 & 89.96 & 45.22 & 90.45 & 47.37 & 90.20   \\
        \cmidrule(lr){1-8}
         \rowcolor{gray!20} & $\approach({\mu,\sigma}) + \texttt{SCALE}$   & 47.07 & 90.61 & 40.87 & 91.52 & 43.97 & 91.07   \\
        \rowcolor{gray!20}  & $\approach(m) + \texttt{SCALE}$              & 49.93 & 89.36 & 43.40 & 90.49 & 46.66 & 89.93   \\
        \midrule
        \multirow{11}{*}{\rotatebox{90}{Wide-ResNet}}  
        & MSP           & 63.09 & 89.16 & 61.89 & 88.50 & 62.49 & 88.83  \\
        & ODIN          & 48.23 & 90.41 & 45.89 & 90.17 & 47.06 & 90.29  \\
        & Energy        & 48.45 & 90.67 & 43.19 & 90.93 & 45.82 & 90.80  \\
        & ReAct         & 53.92 & 88.57 & 49.11 & 88.90 & 51.52 & 88.74  \\
        & DICE          & 57.78 & 88.46 & 52.18 & 89.13 & 54.98 & 88.79  \\
        & ASH-S         & 47.83 & 90.66 & 43.33 & 90.77 & 45.58 & 90.72  \\
        & SCALE         & 47.63 & 90.65 & 43.04 & 90.79 & 45.34 & 90.72  \\
        & GradNorm      & 60.67 & 84.00 & 56.90 & 85.03 & 58.78 & 84.52  \\
        & KNN           & 46.67 & 90.83 & 45.77 & 90.32 & 46.22 & 90.57  \\
        \cmidrule(lr){1-8}
         \rowcolor{gray!20} & $\approach({\mu,\sigma}) + \texttt{SCALE}$   & 44.44 & 91.32 & 38.63 & 91.70 & 41.53 & 91.51  \\
        \rowcolor{gray!20}  & $\approach(m) + \texttt{SCALE}$              & 48.92 & 90.10 & 42.12 & 90.88 & 45.52 & 90.49  \\
       
        \bottomrule
        \end{tabular}}
        \label{table: cifar_near_ood}
        \end{table}

\subsection{CIFAR Evaluation}
\label{appendix: detailed_results_cifar}

    Table~\ref{table: detailed_results_cifar10_a}, \ref{table: detailed_results_cifar10_b}, \ref{table: detailed_results_cifar100_a}, \ref{table: detailed_results_cifar10_b} report detailed OOD performance across six test datasets for ResNet-18, ResNet-34, Wide-ResNet, and MobileNet-v2 trained on CIFAR-10 and CIFAR-100. The primary baselines (MSP, ODIN, Energy, ReAct, DICE, ASH, SCALE, GradNorm, KNN) did not report results for ResNet-18, ResNet-34, Wide-ResNet, or MobileNet-v2 in their original papers. To ensure a fair comparison, we re-evaluated these methods following the hyperparameter guidelines from their respective publications.

    \begin{landscape}
    \begin{table}[ht]
    \centering
    \caption{\textit{Detailed results on six common OOD benchmark datasets: SVHN, Places365, iSUN, Textures, LSUN-crop, LSUN-resize; using pre-trained ResNet-18 and ResNet-34 on \textbf{CIFAR-10}. $\boldsymbol{\downarrow}$ indicates lower values are better and $\boldsymbol{\uparrow}$ indicates larger values are better.}}
    
    \resizebox{\linewidth}{!}{
    \begin{tabular}{ c l  cc cc cc cc cc cc cc }
        \toprule
         \multirow{2}{*}{\textbf{Model}} & \multirow{2}{*}{\textbf{Method}} & \multicolumn{2}{c}{\textbf{SVHN}} & \multicolumn{2}{c}{\textbf{Place365}} & \multicolumn{2}{c}{\textbf{iSUN}} & \multicolumn{2}{c}{\textbf{Textures}} & \multicolumn{2}{c}{\textbf{LSUN-c}}  & \multicolumn{2}{c}{\textbf{LSUN-r}} & \multicolumn{2}{c}{\textbf{Average}}  \\
        \cmidrule(lr){3-4} \cmidrule(lr){5-6} \cmidrule(lr){7-8} \cmidrule(lr){9-10} \cmidrule(lr){11-12} \cmidrule(lr){13-14} \cmidrule(lr){15-16}
        && \textbf{FPR95} $\downarrow$ & \textbf{AUROC} $\uparrow$ & \textbf{FPR95} $\downarrow$ & \textbf{AUROC} $\uparrow$ & \textbf{FPR95} $\downarrow$ & \textbf{AUROC} $\uparrow$ & \textbf{FPR95} $\downarrow$ & \textbf{AUROC} $\uparrow$ & \textbf{FPR95} $\downarrow$ & \textbf{AUROC} $\uparrow$ & \textbf{FPR95} $\downarrow$ & \textbf{AUROC} $\uparrow$ & \textbf{FPR95} $\downarrow$ & \textbf{AUROC} $\uparrow$ \\
        \midrule
        \multirow{11}{*}{\rotatebox{90}{ResNet-18}} & MSP	     & 60.39 & 92.40 & 63.49 & 88.38 & 56.59 & 91.18 & 62.71 & 90.10 & 51.87 & 93.64 & 55.53 & 91.69 & 58.43 & 91.23 \\
                                    & ODIN       & 35.96 & 94.70 & 41.11 & 92.06 & 23.36 & 96.56 & 46.74 & 91.97 &  6.66 & 98.71 & 20.04 & 96.93 & 28.98 & 95.16 \\
                                    & Energy	 & 44.32 & 94.04 & 41.31 & 91.73 & 35.46 & 94.64 & 50.39 & 91.12 &  9.77 & 98.19 & 32.41 & 95.16 & 35.61 & 94.14 \\
                                    & ReAct	     & 42.31 & 94.12 & 40.74 & 92.25 & 24.06 & 96.26 & 40.44 & 93.69 & 12.27 & 97.90 & 21.02 & 96.67 & 30.14 & 95.15 \\
                                    & DICE	     & 17.60 & 97.09 & 46.16 & 90.66 & 38.68 & 94.32 & 44.50 & 91.81 &  1.90 & 99.57 & 36.66 & 94.67 & 30.92 & 94.69 \\
                                    & ASH-S	     &  7.87 & 98.43 & 49.69 & 89.57 & 23.27 & 96.33 & 26.12 & 95.88 &  2.10 & 99.46 & 21.91 & 96.47 & 21.83 & 96.02 \\
                                    & SCALE      &  9.73 & 98.13 & 45.99 & 90.87 & 22.92 & 96.39 & 27.00 & 95.61 &  3.75 & 99.17 & 21.02 & 96.60 & 21.74 & 96.13 \\
                                    & GradNorm   & 17.98 & 96.53 & 57.08 & 86.66 & 37.88 & 93.90 & 43.79 & 90.96 &  4.33 & 99.00 & 36.83 & 93.89 & 32.98 & 93.49 \\
                                    & KNN        & 14.29 & 97.59 & 49.08 & 89.62 & 37.23 & 92.87 & 29.15 & 94.62 & 15.71 & 97.31 & 37.12 & 93.34 & 30.43 & 94.22 \\
        \cmidrule(lr){2-16}
        \rowcolor[gray]{0.9} &  $\approach(m) + \texttt{DICE}$	            &  7.95 & 98.50 & 30.55 & 93.97 &  6.70 & 98.59 &  9.66 & 98.28 & 1.24 & 99.73 &  6.63 & 98.57 & 10.46 & 97.94 \\
        \rowcolor[gray]{0.9} & $\approach(\mu,\sigma) + \texttt{DICE}$	    &  7.85 & 98.57 & 35.27 & 93.18 & 11.96 & 97.96 & 14.04 & 97.71 & 1.14 & 99.76 & 10.71 & 98.03 & 13.49 & 97.54 \\
        \midrule
        \multirow{11}{*}{\rotatebox{90}{ResNet-34}} & MSP	     & 62.20 & 91.10 & 62.76 & 88.95 & 52.52 & 92.81 & 57.93 & 90.75 & 42.06 & 95.17 & 51.72 & 92.98 & 54.86 & 91.96 \\
                                    & ODIN       & 40.33 & 92.49 & 37.63 & 92.05 & 10.29 & 98.08 & 39.06 & 92.94 &  2.58 & 99.31 &  8.44 & 98.32 & 23.06 & 95.53 \\
                                    & Energy	 & 35.44 & 93.76 & 38.15 & 92.27 & 19.90 & 96.87 & 42.52 & 92.54 &  3.38 & 99.10 & 16.86 & 97.20 & 26.04 & 95.29 \\
                                    & ReAct	     & 33.03 & 93.66 & 36.11 & 93.26 & 21.64 & 96.61 & 41.19 & 93.22 &  6.83 & 98.57 & 19.36 & 96.88 & 26.36 & 95.37 \\
                                    & DICE	     & 26.78 & 95.46 & 39.75 & 91.83 & 16.00 & 97.41 & 40.05 & 92.91 &  0.71 & 99.82 & 14.73 & 97.58 & 23.00 & 95.84 \\
                                    & ASH-S	     & 14.58 & 97.56 & 41.91 & 90.96 & 15.71 & 97.33 & 26.88 & 95.38 &  1.74 & 99.54 & 16.59 & 97.23 & 19.57 & 96.34 \\
                                    & SCALE      & 17.47 & 97.01 & 38.01 & 91.84 & 15.06 & 97.31 & 28.40 & 94.86 &  2.37 & 99.44 & 15.67 & 97.25 & 19.50 & 96.28 \\
                                    & GradNorm   & 29.13 & 94.65 & 50.13 & 88.31 & 28.06 & 95.08 & 46.56 & 88.17 &  2.72 & 99.37 & 28.64 & 95.19 & 30.87 & 93.46 \\
                                    & KNN        & 24.65 & 95.46 & 48.02 & 89.44 & 37.75 & 93.69 & 23.30 & 96.13 & 13.55 & 97.55 & 39.84 & 93.43 & 31.18 & 94.29 \\
        \cmidrule(lr){2-16}
        
        \rowcolor[gray]{0.9} &  $\approach(m) + + \texttt{DICE} $	        & 14.79 & 97.29 & 27.51 & 94.59 &  5.09 & 98.98 & 10.07 & 98.26 & 1.22 & 99.75 &  5.36 & 98.95 & 10.67 & 97.97 \\
        \rowcolor[gray]{0.9} & $\approach(\mu,\sigma) + + \texttt{DICE}$	& 15.89 & 97.25 & 30.60 & 93.97 &  6.31 & 98.78 & 12.55 & 97.88 & 0.83 & 99.82 &  6.37 & 98.78 & 12.09 & 97.75 \\
        \bottomrule
        \end{tabular}
        }
        \label{table: detailed_results_cifar10_a}
    \end{table}
    \end{landscape}

        \begin{landscape}
    \begin{table}[ht]
    \centering
    \caption{\textit{Detailed results on six common OOD benchmark datasets: SVHN, Places365, iSUN, Textures, LSUN-crop, LSUN-resize; using pre-trained ResNet-18 and ResNet-34 on \textbf{CIFAR-10}. $\boldsymbol{\downarrow}$ indicates lower values are better and $\boldsymbol{\uparrow}$ indicates larger values are better.}}
    
    \resizebox{\linewidth}{!}{
    \begin{tabular}{ c l  cc cc cc cc cc cc cc }
        \toprule
         \multirow{2}{*}{\textbf{Model}} & \multirow{2}{*}{\textbf{Method}} & \multicolumn{2}{c}{\textbf{SVHN}} & \multicolumn{2}{c}{\textbf{Place365}} & \multicolumn{2}{c}{\textbf{iSUN}} & \multicolumn{2}{c}{\textbf{Textures}} & \multicolumn{2}{c}{\textbf{LSUN-c}}  & \multicolumn{2}{c}{\textbf{LSUN-r}} & \multicolumn{2}{c}{\textbf{Average}}  \\
        \cmidrule(lr){3-4} \cmidrule(lr){5-6} \cmidrule(lr){7-8} \cmidrule(lr){9-10} \cmidrule(lr){11-12} \cmidrule(lr){13-14} \cmidrule(lr){15-16}
        && \textbf{FPR95} $\downarrow$ & \textbf{AUROC} $\uparrow$ & \textbf{FPR95} $\downarrow$ & \textbf{AUROC} $\uparrow$ & \textbf{FPR95} $\downarrow$ & \textbf{AUROC} $\uparrow$ & \textbf{FPR95} $\downarrow$ & \textbf{AUROC} $\uparrow$ & \textbf{FPR95} $\downarrow$ & \textbf{AUROC} $\uparrow$ & \textbf{FPR95} $\downarrow$ & \textbf{AUROC} $\uparrow$ & \textbf{FPR95} $\downarrow$ & \textbf{AUROC} $\uparrow$ \\
        \midrule
        \multirow{11}{*}{\rotatebox{90}{Wide-ResNet}}  & MSP	    & 36.86 & 95.14 & 60.08 & 89.09 & 47.52 & 93.16 & 59.24 & 89.41 & 29.69 & 96.25 & 45.09 & 93.62 & 46.41 & 92.78 \\
                                       & ODIN       & 24.53 & 95.83 & 38.66 & 92.10 &  9.76 & 98.25 & 47.15 & 89.61 &  3.42 & 99.24 &  7.10 & 98.54 & 21.77 & 95.59 \\
                                       & Energy	 & 22.77 & 96.24 & 37.57 & 92.06 & 12.71 & 97.75 & 48.32 & 89.67 &  3.39 & 99.19 &  9.86 & 98.07 & 22.44 & 95.50 \\
                                       & ReAct	     & 47.87 & 92.09 & 35.68 & 92.88 & 12.24 & 97.64 & 47.91 & 88.76 & 14.94 & 97.11 & 9.72 & 98.07 & 28.06 & 94.43 \\
                                       & DICE	     & 19.28 & 96.86 & 48.47 & 89.61 & 13.49 & 97.67 & 44.91 & 89.66 & 0.53 & 99.86 & 10.78 & 97.97 & 22.91 & 95.27 \\
                                       & ASH-S	     & 12.09 & 97.84 & 43.18 & 90.79 & 13.88 & 97.72 & 36.67 & 92.75 & 2.07 & 99.49 & 12.34 & 97.86 & 20.04 & 96.07 \\
                                       & SCALE      & 13.88 & 97.43 & 41.82 & 91.16 & 11.52 & 98.01 & 36.35 & 92.34 & 2.59 & 99.41 & 9.93 & 98.15 & 19.35 & 96.08 \\
                                       & GradNorm   & 18.96 & 95.92 & 56.16 & 84.99 & 20.73 & 96.26 & 44.79 & 86.73 &  0.91 & 99.70 & 19.24 & 96.63 & 26.80 & 93.37 \\
                                       & KNN        &  8.66 & 98.48 & 41.72 & 90.76 &  9.78 & 98.26 & 20.50 & 96.60 &  8.06 & 98.59 &  8.57 & 98.44 & 16.21 & 96.85 \\
        \cmidrule(lr){2-16}
        \rowcolor[gray]{0.9} & $\approach(m) + \texttt{DICE}$	        & 7.62 & 98.57 & 29.90 & 93.89 &  3.53 & 99.21 &  9.22 & 98.33 &  2.26 & 99.56 &  3.18 & 99.24 &  9.28 & 98.13 \\
        \rowcolor[gray]{0.9} & $\approach(\mu,\sigma) + \texttt{DICE}$	& 9.14 & 98.38 & 33.55 & 93.15 &  5.22 & 98.93 & 15.98 & 97.35 &  0.94 & 99.76 &  4.50 & 99.02 & 11.55 & 97.76 \\
        \midrule
        \multirow{11}{*}{\rotatebox{90}{MobileNet-v2}}  & MSP	     & 72.84 & 88.69 & 70.65 & 85.63 & 64.16 & 88.78 & 64.10 & 88.59 & 61.30 & 92.00 & 63.92 & 88.90 & 66.16 & 88.76 \\
                                        & ODIN       & 71.72 & 85.71 & 46.44 & 90.67 & 21.17 & 96.66 & 44.50 & 91.86 &  6.94 & 98.65 & 20.76 & 96.67 & 35.25 & 93.37 \\
                                        & Energy	 & 75.83 & 85.85 & 44.98 & 90.62 & 29.68 & 95.03 & 48.67 & 91.19 &  9.54 & 98.12 & 29.80 & 95.10 & 39.75 & 92.65 \\
                                        & ReAct	     & 73.06 & 85.65 & 45.30 & 90.21 & 28.00 & 95.12 & 44.08 & 92.15 & 11.31 & 97.84 & 27.31 & 95.31 & 38.18 & 92.71 \\
                                        & DICE	     & 62.13 & 87.07 & 50.33 & 89.30 & 27.68 & 95.70 & 49.57 & 90.26 &  2.18 & 99.48 & 29.27 & 95.45 & 36.86 & 92.88 \\
                                        & ASH-S	     & 46.23 & 91.48 & 61.75 & 84.40 & 41.46 & 93.36 & 43.87 & 91.50 &  7.28 & 98.68 & 40.97 & 93.46 & 40.26 & 92.15 \\
                                        & SCALE      & 56.48 & 89.83 & 59.43 & 85.26 & 34.81 & 94.33 & 42.36 & 91.91 &  9.82 & 98.23 & 34.46 & 94.50 & 39.56 & 92.34 \\
                                        & GradNorm   & 55.65 & 91.04 & 67.07 & 81.99 & 44.83 & 92.36 & 47.43 & 89.40 &  9.72 & 98.26 & 46.88 & 92.13 & 45.26 & 90.86 \\
                                        & KNN        & 58.14 & 87.70 & 53.25 & 88.14 & 46.40 & 88.96 & 55.04 & 87.97 & 39.46 & 93.62 & 42.23 & 90.27 & 49.09 & 89.45 \\
        \cmidrule(lr){2-16}
        \rowcolor[gray]{0.9} &  $\approach(m) + \texttt{DICE}$	        & 38.71 & 92.48 & 50.33 & 89.34 & 14.89 & 97.42 & 25.05 & 95.84 &  4.39 & 99.21 & 15.28 & 97.40 & 24.78 & 95.28 \\
        \rowcolor[gray]{0.9} & $\approach(\mu,\sigma) + \texttt{DICE}$	& 37.14 & 92.79 & 50.38 & 89.05 & 15.39 & 97.34 & 24.45 & 96.00 &  3.56 & 99.32 & 15.38 & 97.34 & 24.38 & 95.31 \\
        \bottomrule
        \end{tabular}
        }
        \label{table: detailed_results_cifar10_b}
    \end{table}
    \end{landscape}

    \begin{landscape}
    \begin{table}[ht]
    \centering
    \caption{\textit{Detailed results on six common OOD benchmark datasets: SVHN, Places365, iSUN, Textures, LSUN-crop, LSUN-resize. We used same model pre-trained on \textbf{CIFAR-100}. $\boldsymbol{\downarrow}$ indicates lower values are better and $\boldsymbol{\uparrow}$ indicates larger values are better.}}
    \resizebox{\linewidth}{!}{
    \begin{tabular}{ c l  cc cc cc cc cc cc cc }
        \toprule
         \multirow{2}{*}{\textbf{Model}} & \multirow{2}{*}{\textbf{Method}} & \multicolumn{2}{c}{\textbf{SVHN}} & \multicolumn{2}{c}{\textbf{Place365}} & \multicolumn{2}{c}{\textbf{iSUN}} & \multicolumn{2}{c}{\textbf{Textures}} & \multicolumn{2}{c}{\textbf{LSUN-c}}  & \multicolumn{2}{c}{\textbf{LSUN-r}} & \multicolumn{2}{c}{\textbf{Average}}  \\
        \cmidrule(lr){3-4} \cmidrule(lr){5-6} \cmidrule(lr){7-8} \cmidrule(lr){9-10} \cmidrule(lr){11-12} \cmidrule(lr){13-14} \cmidrule(lr){15-16}
        && \textbf{FPR95} $\downarrow$ & \textbf{AUROC} $\uparrow$ & \textbf{FPR95} $\downarrow$ & \textbf{AUROC} $\uparrow$ & \textbf{FPR95} $\downarrow$ & \textbf{AUROC} $\uparrow$ & \textbf{FPR95} $\downarrow$ & \textbf{AUROC} $\uparrow$ & \textbf{FPR95} $\downarrow$ & \textbf{AUROC} $\uparrow$ & \textbf{FPR95} $\downarrow$ & \textbf{AUROC} $\uparrow$ & \textbf{FPR95} $\downarrow$ & \textbf{AUROC} $\uparrow$ \\
        \midrule
        \multirow{11}{*}{\rotatebox{90}{ResNet-18}} & MSP	     & 74.26 & 83.20 & 82.49 & 75.32 & 85.58 & 70.20 & 84.89 & 74.02 & 70.79 & 82.78 & 84.36 & 71.45 & 80.40 & 76.16 \\
                                    & ODIN       & 70.30 & 88.06 & 80.14 & 77.02 & 60.26 & 86.98 & 81.56 & 76.56 & 47.73 & 91.84 & 56.35 & 88.23 & 66.06 & 84.78 \\
                                    & Energy	 & 66.64 & 89.53 & 81.23 & 76.84 & 73.67 & 82.01 & 85.30 & 75.68 & 48.01 & 91.63 & 70.30 & 83.38 & 70.86 & 83.18 \\
                                    & ReAct	     & 55.03 & 91.95 & 79.78 & 77.48 & 58.66 & 87.78 & 60.90 & 87.94 & 47.42 & 91.22 & 54.78 & 88.77 & 59.43 & 87.52 \\
                                    & DICE	     & 41.18 & 92.98 & 81.82 & 76.02 & 66.22 & 84.20 & 75.50 & 76.27 & 12.21 & 97.70 & 64.48 & 85.15 & 56.90 & 85.39 \\
                                    & ASH	     & 29.10 & 95.46 & 82.56 & 75.38 & 67.09 & 85.01 & 56.49 & 87.80 & 27.06 & 95.55 & 64.72 & 85.62 & 54.50 & 87.47 \\
                                    & SCALE      & 22.12 & 96.38 & 81.96 & 74.95 & 61.62 & 86.65 & 44.50 & 90.72 & 18.62 & 96.78 & 59.76 & 86.74 & 48.10 & 88.70 \\
                                    & GradNorm   & 57.22 & 85.93 & 88.47 & 62.15 & 81.65 & 74.94 & 79.84 & 65.96 & 15.44 & 96.94 & 80.07 & 75.73 & 67.11 & 76.94 \\
                                    & KNN        & 61.76 & 90.15 & 86.35 & 70.45 & 69.69 & 79.24 & 41.08 & 90.87 & 75.36 & 77.73 & 67.78 & 80.08 & 67.00 & 81.42 \\
        \cmidrule(lr){2-16}
        \rowcolor[gray]{0.9} &  $\approach(m) + \texttt{DICE}$	       & 10.77 & 97.91 & 80.06 & 77.59 & 33.36 & 94.13 & 31.52 & 93.52 &  7.31 & 98.53 & 37.27 & 93.39 & 33.38 & 92.51 \\
        \rowcolor[gray]{0.9} & $\approach(\mu,\sigma) + \texttt{DICE}$ & 12.40 & 97.59 & 80.46 & 76.21 & 38.77 & 92.69 & 36.74 & 92.09 &  7.78 & 98.42 & 40.98 & 92.27 & 36.19 & 91.54 \\
        \midrule
        \multirow{11}{*}{\rotatebox{90}{ResNet-34}} & MSP	     & 69.72 & 83.11 & 82.28 & 75.84 & 83.04 & 76.64 & 84.95 & 74.24 & 76.57 & 81.44 & 81.54 & 77.19 & 79.68 & 78.08 \\
                                    & ODIN       & 70.15 & 86.07 & 80.97 & 77.27 & 57.90 & 88.77 & 83.23 & 76.93 & 57.77 & 89.70 & 54.97 & 89.54 & 67.50 & 84.71 \\
                                    & Energy	 & 57.79 & 89.80 & 81.17 & 77.25 & 71.83 & 84.14 & 86.77 & 75.82 & 55.56 & 89.92 & 68.70 & 84.93 & 70.30 & 83.64 \\
                                    & ReAct	     & 30.77 & 94.47 & 77.91 & 78.11 & 64.45 & 85.01 & 62.13 & 86.27 & 47.86 & 90.72 & 64.13 & 85.24 & 57.87 & 86.64 \\
                                    & DICE	     & 25.88 & 95.11 & 80.75 & 77.18 & 65.76 & 85.44 & 74.73 & 78.11 & 18.31 & 96.55 & 65.59 & 85.30 & 55.17 & 86.28 \\
                                    & ASH	     & 23.64 & 95.92 & 82.37 & 75.92 & 60.77 & 87.53 & 59.08 & 87.60 & 41.57 & 93.06 & 61.41 & 87.24 & 54.81 & 87.88 \\
                                    & SCALE      & 13.68 & 97.51 & 80.72 & 75.87 & 57.87 & 87.14 & 46.35 & 90.27 & 28.17 & 95.07 & 61.31 & 86.01 & 48.02 & 88.64 \\
                                    & GradNorm   & 37.05 & 92.35 & 89.49 & 60.67 & 86.61 & 63.33 & 79.79 & 67.18 & 22.32 & 95.21 & 85.80 & 61.83 & 66.84 & 73.43 \\
                                    & KNN        & 38.19 & 92.57 & 87.00 & 70.69 & 66.90 & 83.87 & 53.09 & 87.74 & 82.46 & 72.83 & 67.02 & 84.50 & 65.78 & 82.03 \\
        \cmidrule(lr){2-16}
        \rowcolor[gray]{0.9} &  $\approach(m) + \texttt{DICE}$	        &  7.96 & 98.40 & 79.29 & 77.25 & 36.80 & 93.62 & 28.32 & 94.05 &  8.62 & 98.22 & 42.47 & 92.46 & 33.91 & 92.33 \\
        \rowcolor[gray]{0.9} & $\approach(\mu,\sigma) + \texttt{DICE}$	&  8.25 & 98.43 & 81.22 & 76.35 & 41.29 & 92.83 & 33.12 & 93.27 &  9.78 & 98.15 & 46.38 & 91.82 & 36.67 & 91.81 \\
        \bottomrule
        \end{tabular}}
        \label{table: detailed_results_cifar100_a}
    \end{table}
    \end{landscape}

    \begin{landscape}
    \begin{table}[ht]
    \centering
    \caption{\textit{Detailed results on six common OOD benchmark datasets: SVHN, Places365, iSUN, Textures, LSUN-crop, LSUN-resize. We used same model pre-trained on \textbf{CIFAR-100}. $\boldsymbol{\downarrow}$ indicates lower values are better and $\boldsymbol{\uparrow}$ indicates larger values are better.}}
    \resizebox{\linewidth}{!}{
    \begin{tabular}{ c l  cc cc cc cc cc cc cc }
        \toprule
         \multirow{2}{*}{\textbf{Model}} & \multirow{2}{*}{\textbf{Method}} & \multicolumn{2}{c}{\textbf{SVHN}} & \multicolumn{2}{c}{\textbf{Place365}} & \multicolumn{2}{c}{\textbf{iSUN}} & \multicolumn{2}{c}{\textbf{Textures}} & \multicolumn{2}{c}{\textbf{LSUN-c}}  & \multicolumn{2}{c}{\textbf{LSUN-r}} & \multicolumn{2}{c}{\textbf{Average}}  \\
        \cmidrule(lr){3-4} \cmidrule(lr){5-6} \cmidrule(lr){7-8} \cmidrule(lr){9-10} \cmidrule(lr){11-12} \cmidrule(lr){13-14} \cmidrule(lr){15-16}
        && \textbf{FPR95} $\downarrow$ & \textbf{AUROC} $\uparrow$ & \textbf{FPR95} $\downarrow$ & \textbf{AUROC} $\uparrow$ & \textbf{FPR95} $\downarrow$ & \textbf{AUROC} $\uparrow$ & \textbf{FPR95} $\downarrow$ & \textbf{AUROC} $\uparrow$ & \textbf{FPR95} $\downarrow$ & \textbf{AUROC} $\uparrow$ & \textbf{FPR95} $\downarrow$ & \textbf{AUROC} $\uparrow$ & \textbf{FPR95} $\downarrow$ & \textbf{AUROC} $\uparrow$ \\
        \midrule
        \multirow{11}{*}{\rotatebox{90}{Wide-ResNet}}  & MSP  & 82.69 & 75.55 & 80.17 & 75.95 & 89.65 & 58.38 & 83.83 & 73.40 & 56.50 & 88.00 & 90.98 & 56.65 & 80.64 & 71.32 \\
                                       & ODIN                 & 86.23 & 80.79 & 76.70 & 79.47 & 74.71 & 78.64 & 80.09 & 76.95 & 15.06 & 97.52 & 74.15 & 78.04 & 67.82 & 81.90 \\
                                       & Energy	           & 77.79 & 85.57 & 77.94 & 78.61 & 81.74 & 74.58 & 81.63 & 76.09 & 12.89 & 97.82 & 80.62 & 74.16 & 68.77 & 81.14 \\
                                       & ReAct	               & 67.33 & 89.16 & 75.01 & 78.90 & 38.57 & 92.94 & 46.26 & 90.82 & 16.46 & 97.01 & 40.66 & 92.16 & 47.38 & 90.16 \\
                                       & DICE	               & 41.66 & 92.91 & 80.45 & 76.92 & 74.05 & 77.45 & 57.07 & 82.28 &  1.07 & 99.75 & 75.81 & 76.00 & 55.02 & 84.22 \\
                                       & ASH-S	               & 26.08 & 95.93 & 79.99 & 76.58 & 57.50 & 87.11 & 38.90 & 91.28 &  3.97 & 99.24 & 60.40 & 85.03 & 44.47 & 89.19 \\
                                       & SCALE                & 20.24 & 96.82 & 80.17 & 75.70 & 50.48 & 89.12 & 30.46 & 93.47 &  3.30 & 99.37 & 56.03 & 86.65 & 40.11 & 90.19 \\
                                       & GradNorm             & 60.96 & 87.91 & 89.32 & 64.90 & 84.41 & 74.60 & 57.77 & 84.08 &  2.02 & 99.52 & 87.63 & 70.63 & 63.69 & 80.27 \\
                                       & KNN                  & 22.94 & 95.57 & 84.14 & 70.59 & 63.10 & 77.74 & 29.04 & 93.28 & 50.11 & 83.94 & 65.62 & 75.89 & 52.49 & 82.84 \\
        \cmidrule(lr){2-16}
        \rowcolor[gray]{0.9} &  $\approach(m) + \texttt{DICE}$	        & 21.64 & 95.92 & 85.37 & 75.36 & 42.22 & 92.84 & 23.30 & 95.00 & 8.00 & 98.31 & 51.85 & 91.40 & 38.73 & 91.47 \\
        \rowcolor[gray]{0.9} & $\approach(\mu,\sigma) + \texttt{DICE}$	& 19.77 & 96.48 & 80.97 & 77.23 & 45.51 & 91.71 & 24.52 & 94.35 & 3.26 & 99.29 & 53.68 & 90.06 & 37.95 & 91.52 \\
        \midrule
        \multirow{11}{*}{\rotatebox{90}{MobileNet-v2}} & MSP	     & 80.14 & 76.09 & 84.23 & 72.62 & 87.72 & 70.77 & 86.67 & 70.36 & 77.11 & 76.04 & 87.10 & 70.83 & 83.83 & 72.78 \\
                                        & ODIN       & 82.16 & 80.95 & 80.06 & 76.66 & 65.51 & 86.71 & 77.61 & 80.78 & 50.14 & 89.57 & 65.15 & 87.14 & 70.10 & 83.63 \\
                                        & Energy	 & 69.65 & 85.98 & 81.26 & 75.21 & 78.17 & 83.33 & 80.02 & 78.63 & 50.19 & 89.42 & 76.59 & 84.06 & 72.65 & 82.77 \\
                                        & ReAct	     & 28.41 & 95.15 & 79.46 & 74.05 & 62.45 & 87.65 & 47.70 & 90.19 & 41.25 & 92.31 & 62.16 & 88.06 & 53.57 & 87.90 \\
                                        & DICE	     & 55.62 & 87.54 & 83.25 & 74.25 & 81.23 & 79.75 & 65.02 & 81.93 & 18.17 & 96.24 & 85.40 & 77.90 & 64.78 & 82.93 \\
                                        & ASH	     & 21.90 & 96.46 & 85.12 & 69.51 & 70.46 & 82.84 & 34.80 & 92.65 & 24.14 & 95.56 & 73.46 & 81.22 & 51.65 & 86.37 \\
                                        & SCALE      & 22.36 & 96.25 & 81.46 & 73.30 & 68.52 & 84.26 & 37.62 & 91.84 & 21.43 & 96.07 & 71.77 & 82.83 & 50.53 & 87.43 \\
                                        & GradNorm   & 58.50 & 84.36 & 87.45 & 68.06 & 89.28 & 66.45 & 63.79 & 82.17 & 16.70 & 97.06 & 93.30 & 61.38 & 68.17 & 76.58 \\
                                        & KNN        & 81.82 & 77.97 & 90.65 & 63.74 & 86.88 & 76.51 & 84.57 & 72.51 & 87.82 & 59.94 & 82.16 & 81.19 & 85.65 & 71.97 \\
        \cmidrule(lr){2-16}
        \rowcolor[gray]{0.9} &  $\approach(m) + \texttt{ASH}$	       & 11.62 & 97.72 & 87.30 & 66.62 & 65.66 & 83.89 & 22.04 & 95.40 & 20.60 & 96.20 & 70.72 & 81.92 & 46.32 & 86.96 \\
        \rowcolor[gray]{0.9} & $\approach(\mu,\sigma) + \texttt{ASH}$  & 12.22 & 97.74 & 86.54 & 67.52 & 65.85 & 84.13 & 22.46 & 95.43 & 20.60 & 96.28 & 70.41 & 82.28 & 46.35 & 87.23 \\
        \bottomrule
        \end{tabular}}
        \label{table: detailed_results_cifar100_b}
    \end{table}
    \end{landscape}

\subsection{ImageNet Evaluation}
\label{appendix: detailed_imagenet_benchmark}

     Table~\ref{table: detailed_imagenet_benchmark_a} and Table~\ref{table: detailed_imagenet_benchmark_b} showcases detailed evaluation on ImageNet benchmark, using broad pre-trained model EfficientNet-B0, DenseNet-121, ResNet-50, and MobileNet-v2 for which we re-evaluated all baselines to ensure a fair comparison. Since results for DenseNet-121 and EfficientNet-b0 were not available in the original publications of primary baselines chosen (e.g., ODIN, Energy, GradNorm, KNN, ReAct, DICE, ASH, SCALE), we rigorously re-evaluated these methods ourselves. To ensure a fair and direct comparison, we carefully followed the hyperparameter selection protocols described in their respective papers.

    \begin{landscape}
    \begin{table}[ht]
    \centering
    \caption{OOD detection results on ImageNet benchmarks. All values are percentages, averaged over five OOD datasets (SUN, Places, Texture, iNaturalist, and OpenImage-O), using EfficientNet-B0 and DenseNet-121 pre-trained on ImageNet. $\boldsymbol{\downarrow}$ indicates lower is better, and $\boldsymbol{\uparrow}$ indicates higher is better.}
    \resizebox{\linewidth}{!}{
    \begin{tabular}{c l  cc cc cc cc cc cc}
        \toprule
        \multirow{2}{*}{\textbf{Model}} & \multirow{2}{*}{\textbf{Method}} & \multicolumn{2}{c}{SUN} & \multicolumn{2}{c}{Places} & \multicolumn{2}{c}{Texture} & \multicolumn{2}{c}{iNaturalist} & \multicolumn{2}{c}{OpenImage-o} & \multicolumn{2}{c}{Average}  \\
        \cmidrule(lr){3-4} \cmidrule(lr){5-6} \cmidrule(lr){7-8} \cmidrule(lr){9-10} \cmidrule(lr){11-12} \cmidrule(lr){13-14}
        && \textbf{FPR95} $\downarrow$ & \textbf{AUROC} $\uparrow$ & \textbf{FPR95} $\downarrow$ & \textbf{AUROC} $\uparrow$ & \textbf{FPR95} $\downarrow$ & \textbf{AUROC} $\uparrow$ & \textbf{FPR95} $\downarrow$ & \textbf{AUROC} $\uparrow$ & \textbf{FPR95} $\downarrow$ & \textbf{AUROC} $\uparrow$ & \textbf{FPR95} $\downarrow$ & \textbf{AUROC} $\uparrow$ \\
        \midrule
        \multirow{11}{*}{\rotatebox{90}{EfficientNet-b0}} 
                                      & MSP        & 72.56 & 80.06 & 74.07 & 79.17 & 66.83 & 81.19 & 57.42 & 87.28 & 64.79 & 84.69 & 67.13 & 82.48 \\
                                      & ODIN       & 72.57 & 75.74 & 77.87 & 73.20 & 64.96 & 79.15 & 59.01 & 83.99 & 66.51 & 79.99 & 68.18 & 78.41 \\
                                      & Energy     & 85.01 & 72.86 & 86.06 & 70.99 & 75.99 & 75.86 & 78.91 & 79.78 & 76.53 & 78.50 & 80.50 & 75.60 \\
                                      & ReAct      & 72.96 & 80.97 & 77.59 & 77.61 & 34.57 & 92.44 & 55.19 & 89.82 & 55.25 & 89.32 & 59.11 & 86.03 \\
                                      & DICE       & 98.15 & 44.65 & 99.17 & 39.44 & 93.83 & 59.29 & 99.66 & 39.80 & 98.74 & 39.52 & 97.91 & 44.54 \\
                                      & ASH-S      & 98.62 & 52.65 & 99.30 & 48.33 & 97.64 & 67.20 & 99.78 & 53.49 & 99.44 & 52.70 & 98.96 & 54.87 \\
                                      & SCALE      & 98.48 & 53.34 & 99.27 & 47.85 & 96.81 & 71.19 & 99.67 & 55.44 & 99.16 & 53.72 & 98.68 & 56.31 \\
                                      & GradNorm   & 90.24 & 58.43 & 93.65 & 52.27 & 89.56 & 54.08 & 90.12 & 57.24 & 89.42 & 52.04 & 90.60 & 54.81 \\
                                      & KNN        & 94.09 & 65.43 & 93.73 & 65.41 & 37.13 & 89.84 & 89.94 & 73.62 & 75.96 & 79.32 & 78.17 & 74.72 \\
    
        \cmidrule(lr){2-14}
        \rowcolor{gray!20}            & $\approach({\mu,\sigma}) + \texttt{SCALE}$ & 60.74 & 84.94 & 71.81 & 78.78 & 10.07 & 97.90 & 47.88 & 88.65 & 44.44 & 90.02 & 46.99 & 88.06 \\
        \rowcolor{gray!20}            & $\approach(m) + \texttt{SCALE}$            & 67.42 & 80.88 & 78.80 & 73.33 & 14.65 & 96.56 & 62.57 & 81.43 & 54.78 & 86.07 & 55.64 & 83.65 \\

        \midrule
        \multirow{11}{*}{\rotatebox{90}{DenseNet-121}}
                                      & MSP        & 67.49 & 81.41 & 69.53 & 80.95 & 67.23 & 79.18 & 49.58 & 89.05 & 68.94 & 83.04 & 64.56 & 82.72 \\
                                      & ODIN       & 54.13 & 86.33 & 60.39 & 84.14 & 50.82 & 85.81 & 32.47 & 93.66 & 58.37 & 87.02 & 51.23 & 87.39 \\
                                      & Energy     & 52.51 & 87.27 & 58.24 & 85.05 & 52.22 & 85.42 & 39.75 & 92.66 & 62.27 & 87.08 & 53.00 & 87.50 \\
                                      & ReAct      & 43.65 & 90.65 & 51.05 & 87.54 & 43.48 & 90.81 & 25.74 & 95.05 & 56.06 & 83.74 & 43.99 & 89.56 \\
                                      & DICE       & 38.75 & 89.91 & 49.29 & 86.24 & 40.85 & 88.09 & 25.78 & 94.37 & 56.05 & 83.57 & 42.14 & 88.44 \\
                                      & ASH-S      & 37.20 & 91.51 & 46.54 & 88.79 & 21.76 & 95.04 & 15.50 & 97.03 & 40.67 & 91.27 & 32.33 & 92.73 \\
                                      & SCALE      & 39.63 & 90.82 & 47.37 & 88.47 & 30.27 & 92.36 & 18.98 & 96.51 & 44.22 & 90.97 & 36.09 & 91.83 \\
                                      & GradNorm   & 41.27 & 88.67 & 52.39 & 83.98 & 43.56 & 87.63 & 26.77 & 93.40 & 60.29 & 78.01 & 44.86 & 86.34 \\
                                      & KNN        & 91.80 & 60.42 & 91.84 & 60.76 & 21.31 & 94.78 & 91.00 & 65.25 & 72.44 & 77.38 & 73.68 & 71.72 \\
    
        \cmidrule(lr){2-14}
        \rowcolor{gray!20}            & $\approach({\mu,\sigma}) + \texttt{SCALE}$ & 33.85 & 92.16 & 42.92 & 89.62 & 22.27 & 94.63 & 13.21 & 97.40 & 39.08 & 91.58 & 30.27 & 93.08 \\
        \rowcolor{gray!20}            & $\approach(m) + \texttt{SCALE}$            & 33.68 & 92.03 & 45.81 & 88.31 & 13.30 & 97.24 & 18.13 & 96.23 & 37.40 & 92.05 & 29.66 & 93.17 \\
        \bottomrule
        \end{tabular}}
        \label{table: detailed_imagenet_benchmark_a}
    \end{table}
    \end{landscape}

        \begin{landscape}
    \begin{table}[ht]
    \centering
    \caption{OOD detection results on ImageNet benchmarks. All values are percentages, averaged over five OOD datasets (SUN, Places, Texture, iNaturalist, and OpenImage-O), using ResNet-50 and MobileNet-v2 pre-trained on ImageNet. $\boldsymbol{\downarrow}$ indicates lower is better, and $\boldsymbol{\uparrow}$ indicates higher is better.}
    \resizebox{\linewidth}{!}{
    \begin{tabular}{c l  cc cc cc cc cc cc}
        \toprule
        \multirow{2}{*}{\textbf{Model}} & \multirow{2}{*}{\textbf{Method}} & \multicolumn{2}{c}{SUN} & \multicolumn{2}{c}{Places} & \multicolumn{2}{c}{Texture} & \multicolumn{2}{c}{iNaturalist} & \multicolumn{2}{c}{OpenImage-o} & \multicolumn{2}{c}{Average}  \\
        \cmidrule(lr){3-4} \cmidrule(lr){5-6} \cmidrule(lr){7-8} \cmidrule(lr){9-10} \cmidrule(lr){11-12} \cmidrule(lr){13-14}
        && \textbf{FPR95} $\downarrow$ & \textbf{AUROC} $\uparrow$ & \textbf{FPR95} $\downarrow$ & \textbf{AUROC} $\uparrow$ & \textbf{FPR95} $\downarrow$ & \textbf{AUROC} $\uparrow$ & \textbf{FPR95} $\downarrow$ & \textbf{AUROC} $\uparrow$ & \textbf{FPR95} $\downarrow$ & \textbf{AUROC} $\uparrow$ & \textbf{FPR95} $\downarrow$ & \textbf{AUROC} $\uparrow$ \\
        \midrule
        \multirow{11}{*}{\rotatebox{90}{ResNet-50}} 
                                      & MSP        & 69.11 & 81.64 & 72.06 & 80.54 & 66.26 & 80.43 & 52.83 & 88.39 & 66.97 & 83.89 & 65.45 & 82.98 \\
                                      & ODIN       & 57.11 & 86.77 & 64.69 & 84.12 & 47.30 & 87.82 & 41.82 & 92.25 & 59.15 & 87.54 & 54.01 & 87.70 \\
                                      & Energy     & 58.82 & 86.58 & 65.99 & 83.96 & 52.43 & 86.72 & 53.74 & 90.62 & 64.70 & 87.08 & 59.14 & 86.99 \\
                                      & ReAct      & 23.95 & 94.46 & 33.48 & 91.97 & 46.40 & 90.31 & 19.56 & 96.40 & 49.78 & 89.06 & 34.64 & 92.44 \\
                                      & DICE       & 36.49 & 90.92 & 47.93 & 87.65 & 32.59 & 90.45 & 26.61 & 94.51 & 54.67 & 85.67 & 39.66 & 89.84 \\
                                      & ASH-S      & 28.00 & 94.04 & 39.67 & 91.03 & 11.88 & 97.62 & 11.41 & 97.88 & 38.70 & 90.79 & 25.93 & 94.27 \\
                                      & SCALE      & 25.78 & 94.54 & 36.86 & 91.96 & 14.56 & 96.75 & 10.37 & 98.02 & 36.23 & 92.30 & 24.76 & 94.71 \\
                                      & GradNorm   & 37.42 & 90.10 & 48.88 & 86.08 & 32.84 & 90.64 & 26.78 & 93.90 & 57.76 & 80.44 & 40.74 & 88.23 \\
                                      & KNN        & 78.95 & 77.44 & 81.86 & 73.91 & 16.05 & 96.11 & 78.33 & 79.15 & 65.73 & 82.27 & 64.18 & 81.78 \\
    
        \cmidrule(lr){2-14}
        \rowcolor{gray!20}            & $\approach({\mu,\sigma}) + \texttt{SCALE}$ & 24.49 & 94.62 & 36.01 & 91.83 & 10.59 & 97.75 &  9.61 & 98.10 & 33.30 & 92.85 & 22.80 & 95.03 \\
        \rowcolor{gray!20}            & $\approach(m) + \texttt{SCALE}$            & 27.79 & 93.66 & 40.56 & 90.02 &  9.52 & 98.11 & 13.26 & 97.37 & 34.48 & 92.50 & 25.12 & 94.33 \\

        \midrule
        \multirow{11}{*}{\rotatebox{90}{MobileNet-v2}}
                                      & MSP        & 74.28 & 78.81 & 76.65 & 78.10 & 71.05 & 78.93 & 59.70 & 86.72 & 75.01 & 80.97 & 71.34 & 80.70 \\
                                      & ODIN       & 60.06 & 85.94 & 66.97 & 83.10 & 51.21 & 87.38 & 47.41 & 91.66 & 67.62 & 85.59 & 58.65 & 86.73 \\
                                      & Energy     & 59.60 & 86.16 & 66.36 & 83.15 & 54.82 & 86.57 & 55.33 & 90.37 & 71.71 & 85.10 & 61.56 & 86.27 \\
                                      & ReAct      & 52.68 & 87.21 & 59.81 & 84.04 & 40.32 & 90.96 & 42.94 & 92.75 & 59.88 & 88.25 & 51.13 & 88.64 \\
                                      & DICE       & 38.81 & 90.46 & 52.95 & 85.82 & 33.00 & 91.27 & 42.95 & 90.87 & 65.12 & 82.72 & 46.57 & 88.23 \\
                                      & ASH-S      & 43.86 & 89.98 & 58.92 & 84.72 & 13.21 & 97.10 & 39.13 & 91.96 & 51.63 & 88.00 & 41.35 & 90.35 \\
                                      & SCALE      & 38.74 & 91.64 & 53.49 & 87.34 & 14.79 & 96.65 & 30.09 & 94.46 & 47.51 & 90.11 & 36.92 & 92.04 \\
                                      & GradNorm   & 38.70 & 91.07 & 53.10 & 85.99 & 31.67 & 92.23 & 37.23 & 92.02 & 63.19 & 82.52 & 44.78 & 88.77 \\
                                      & KNN        & 93.82 & 59.41 & 94.10 & 57.72 & 19.84 & 95.34 & 93.34 & 64.82 & 73.10 & 77.07 & 74.84 & 70.87 \\
    
        \cmidrule(lr){2-14}
        \rowcolor{gray!20}            & $\approach({\mu,\sigma}) + \texttt{SCALE}$ & 37.89 & 91.82 & 52.92 & 87.24 & 11.37 & 97.60 & 28.54 & 94.76 & 43.38 & 90.83 & 34.82 & 92.45 \\
        \rowcolor{gray!20}            & $\approach(m) + \texttt{SCALE}$            & 37.47 & 91.91 & 52.72 & 86.90 &  8.85 & 98.20 & 26.76 & 95.08 & 41.63 & 90.98 & 33.49 & 92.61 \\
        \bottomrule
        \end{tabular}}
        \label{table: detailed_imagenet_benchmark_b}
    \end{table}
    \end{landscape}

\section{Comparison with Other Baselines}
\label{appendix: additional baselines}

    While in the main paper we restrict our comparison to primary baselines (i.e., MSP, ODIN, Energy, GradNorm, KNN, ReAct, DICE, ASH, SCALE), we provide a comparison of our method, $\approach$, with additional baselines AdaScale~\cite{adascle}, NCI~\cite{NCI}, and fDBD~\cite{fdbd} in this section. A comprehensive re-evaluation of NCI and fDBD across all architectures used in our study was determined to be beyond the scope of this work due to a fundamental difference in their design philosophy.

    Methods like ReAct, DICE, and our own $\approach$ are modular, post-hoc techniques that primarily modify the penultimate feature vector itself. In contrast, fDBD and NCI introduce entirely new scoring functions derived from the geometric relationship between features and the classifier's decision boundaries (fDBD) or class weight vectors (NCI). Integrating our feature-level modifications into these structurally different scoring frameworks would require significant, non-trivial engineering and could obscure a direct comparison. Therefore, for these two methods, we present a fair comparison limited to the overlapping architectures and datasets from their original publications.

\subsection{AdaSCALE: Adaptive Scaling OOD Detector}

    AdaSCALE is a post-hoc OOD detection method that replaces fixed activation-scaling strategies with an adaptive, sample-dependent mechanism. Existing approaches (ASH, SCALE, LTS) prune activations using a static percentile threshold, which cannot reliably distinguish ID from OOD data. AdaSCALE leverages the observation that OOD samples experience larger shifts in their top activated neurons under small pixel perturbations, while ID activations remain stable. It measures this activation shift (Q), adjusts it with a correction term (Co), and maps the resulting OODness score through a CDF to produce a dynamic pruning percentile. This causes ID samples to receive stronger scaling and OOD samples weaker scaling, yielding more separated energy scores and improved detection performance.

    We compare $\approach$ with AdaSCALE~\cite{adascle} in Tables~\ref{table: adascale_cifar_comparison} and \ref{table: adascale_imagenet_comparison}, strictly following the restricted dataset protocol of the original paper to ensure a fair comparison.

    On CIFAR-10 with Wide-ResNet-28-10, $\approach(m) + \texttt{DICE}$ achieves an FPR95 of 15.58 compared to 39.12 for AdaSCALE, while with DenseNet-101 it achieves 14.98 versus 40.03. A similar trend is observed on CIFAR-100, where $\approach$ consistently outperforms AdaSCALE.

    As shown in Table~\ref{table: adascale_imagenet_comparison}, this advantage also holds for Swin-B pre-trained on ImageNet, where $\approach(\mu) + \texttt{KNN}$ achieves an FPR95 of 31.66 compared to 46.24. However, for ResNet-50 on ImageNet, AdaSCALE performs better than $\approach$ by 4.52 FPR95 points.

\begin{table}[t]
        \centering
        \caption{A direct comparison of $\approach$ against the AdaSCALE baseline, using their originally reported results for \textbf{CIFAR} with a \textbf{Wide-ResNet-28-10} and \textbf{DenseNet-101} backbone. The evaluation is restricted to the SVHN, Texture, and Places365 OOD datasets to ensure a fair comparison that matches the protocol from the original AdaSCALE paper.}
        \resizebox{\linewidth}{!}{
        \begin{tabular}{l l  cc cc cc cc}
        \toprule
        \multirow{2}{*}{\textbf{Model}} & \multirow{2}{*}{\textbf{Method}} & \multicolumn{2}{c}{SVHN} & \multicolumn{2}{c}{Places365} & \multicolumn{2}{c}{Texture} & \multicolumn{2}{c}{Average}  \\
        \cmidrule(lr){3-4} \cmidrule(lr){5-6} \cmidrule(lr){7-8} \cmidrule(lr){9-10}
        && \textbf{FPR95} $\downarrow$ & \textbf{AUROC} $\uparrow$ & \textbf{FPR95} $\downarrow$ & \textbf{AUROC} $\uparrow$ & \textbf{FPR95} $\downarrow$ & \textbf{AUROC} $\uparrow$ & \textbf{FPR95} $\downarrow$ & \textbf{AUROC} $\uparrow$  \\
        \midrule
        \multicolumn{10}{c}{CIFAR-10} \\
        \midrule
        \multirow{4}{*}{\textbf{Wide-ResNet}} 
        & \texttt{AdaSCALE-A}                       & 17.84 & 95.14 & 34.57 & 92.31 & 64.96 & 88.31 & 39.12 & 91.92 \\
        & \texttt{AdaSCALE-L}                       & 18.41 & 95.10 & 37.59 & 91.97 & 62.87 & 88.67 & 39.62 & 91.91 \\
        & $\approach({\mu,\sigma}) + \texttt{DICE}$ &  9.14 & 98.38 & 33.55 & 93.15 & 15.98 & 97.35 & 19.56 & 96.29 \\
        & $\approach(m) + \texttt{DICE}$            &  \textbf{7.62} & \textbf{98.57} & \textbf{29.90} & \textbf{93.89} &  \textbf{9.22} & \textbf{98.33} & \textbf{15.58} & \textbf{96.93} \\
        \midrule
        \multirow{4}{*}{\textbf{DenseNet-101}} 
        & \texttt{AdaSCALE-A}                       & 25.04 & 94.05 & 36.77 & 91.20 & 58.28 & 87.35 & 40.03 & 90.87 \\
        & \texttt{AdaSCALE-L}                       & 26.43 & 93.87 & 37.03 & 91.25 & 58.59 & 87.19 & 40.68 & 90.77 \\
        & $\approach({\mu,\sigma}) + \texttt{DICE}$ &  \textbf{6.84} & \textbf{98.77} & 29.76 & 93.86 &  9.57 & 98.25 & 15.39 & 96.96 \\
        & $\approach(m) + \texttt{DICE}$            &  8.28 & 98.37 & \textbf{29.47} & \textbf{93.92} &  \textbf{7.18} & \textbf{98.69} & \textbf{14.98} & \textbf{96.99} \\
        \midrule
        \multicolumn{10}{c}{CIFAR-100} \\
        \midrule
        \multirow{4}{*}{\textbf{Wide-ResNet}} 
        & \texttt{AdaSCALE-A}                       & 36.79 & 89.20 & \textbf{56.48} & \textbf{81.55} & 55.93 & 81.93 & 49.73 & 84.23 \\
        & \texttt{AdaSCALE-L}                       & 32.44 & 91.02 & 57.83 & 81.51 & 50.87 & 84.14 & 47.05 & 85.56 \\
        & $\approach({\mu,\sigma}) + \texttt{DICE}$ & \textbf{19.77} & \textbf{96.48} & 80.97 & 77.23 & 24.52 & 94.35 & \textbf{41.75} & \textbf{89.35} \\
        & $\approach(m) + \texttt{DICE}$            & 21.64 & 95.92 & 85.37 & 75.36 & \textbf{23.30} & \textbf{95.00} & 43.44 & 88.76 \\
        \midrule
        \multirow{4}{*}{\textbf{DenseNet-101}} 
        & \texttt{AdaSCALE-A}                       & 46.29 & 84.31 & \textbf{61.70} & \textbf{78.86} & 71.40 & 76.59 & 59.80 & 79.92 \\
        & \texttt{AdaSCALE-L}                       & 43.97 & 85.30 & 61.97 & 78.69 & 69.31 & 77.71 & 58.42 & 80.57 \\
        & $\approach({\mu,\sigma}) + \texttt{DICE}$ & \textbf{20.30} & \textbf{96.20} & 79.52 & 78.28 & 34.10 & 91.24 & \textbf{44.64} & \textbf{88.57} \\ 
        & $\approach(m) + \texttt{DICE}$            & 27.97 & 94.91 & 86.13 & 76.61 & \textbf{30.32} & \textbf{93.28} & 48.14 & 88.27 \\
        \bottomrule
        \end{tabular}}
        \label{table: adascale_cifar_comparison}
    \end{table}

    \begin{table}[!ht]
        \centering
        \caption{A direct comparison between $\approach$ and AdaSCALE baseline using the originally reported ImageNet results with a \textbf{Swin-B} and \textbf{ResNet-50} backbone. The evaluation is restricted to iNaturalist, Places, Texture and OpenImage-O to ensure a fair comparison consistent with the protocol of the original AdaSCALE paper.}
        \resizebox{\linewidth}{!}{
        \begin{tabular}{l l  cc cc cc cc cc}
        \toprule
        \multirow{2}{*}{\textbf{Model}} & \multirow{2}{*}{\textbf{Method}} & \multicolumn{2}{c}{iNaturalist} & \multicolumn{2}{c}{Places} & \multicolumn{2}{c}{Texture} & \multicolumn{2}{c}{OpenImage-O} & \multicolumn{2}{c}{Average}  \\
        \cmidrule(lr){3-4} \cmidrule(lr){5-6} \cmidrule(lr){7-8} \cmidrule(lr){9-10} \cmidrule(lr){11-12}
        && \textbf{FPR95} $\downarrow$ & \textbf{AUROC} $\uparrow$ & \textbf{FPR95} $\downarrow$ & \textbf{AUROC} $\uparrow$ & \textbf{FPR95} $\downarrow$ & \textbf{AUROC} $\uparrow$ & \textbf{FPR95} $\downarrow$ & \textbf{AUROC} $\uparrow$ & \textbf{FPR95} $\downarrow$ & \textbf{AUROC} $\uparrow$ \\
        \midrule
        \multirow{4}{*}{\textbf{Swin-B}} 
            & \texttt{AdaSCALE-A}                        & 32.82 & 90.73 & \textbf{58.02} & 82.71 & 61.82 & 85.34 & 38.58 & 89.78 & 47.81 & 87.14 \\            
            & \texttt{AdaSCALE-L}                        & 30.95 & 91.69 & 56.32 & 83.82 & 60.17 & 86.30 & 37.52 & 90.08 & 46.24 & 87.97 \\
            & $\approach({\mu,\sigma}) + \texttt{KNN}$   & \textbf{16.01} & \textbf{96.42} & 66.20 & \textbf{85.08} & \textbf{17.71} & \textbf{94.42} & \textbf{26.71} & \textbf{94.69} & \textbf{31.66} & \textbf{92.65} \\
            & $\approach(m) + \texttt{KNN}$              & 61.66 & 90.27 & 87.87 & 78.12 & 36.45 & 90.71 & 62.22 & 87.75 & 62.05 & 86.71 \\
            \midrule
            \multirow{4}{*}{\textbf{ResNet-50}} 
            & \texttt{AdaSCALE-A}                        & \textbf{7.61} & \textbf{98.31} & \textbf{32.60} & \textbf{92.74} & 10.57 & 97.88 & \textbf{20.67} & \textbf{95.62} & \textbf{17.86} & \textbf{96.14} \\
            & \texttt{AdaSCALE-L}                        & 7.78 & 98.29 & 32.97 & 92.63 & 10.33 & 97.92 & 32.97 & 92.63 & 21.01 & 95.37 \\
            & $\approach({\mu,\sigma}) + \texttt{SCALE}$ &  9.61 & 98.10 & 36.01 & 91.83 & 10.59 & 97.75 & 33.30 & 92.85 & 22.38 & 95.13 \\
            & $\approach(m) + \texttt{SCALE}$            & 13.26 & 97.37 & 40.56 & 90.02 &  \textbf{9.52} & \textbf{98.11} & 34.48 & 92.50 & 24.46 & 94.50 \\
        \bottomrule
        \end{tabular}}
        \label{table: adascale_imagenet_comparison}
    \end{table}

\subsection{NCI: Neural Collapse Inspired OOD Detector}
 
    As shown in Table~\ref{table: NCI_cifar_comparison} for CIFAR-10 and Table~\ref{table: NCI_imagenet_comparison} for ImageNet, in a direct comparison against NCI's~\citep{NCI} reported results, our method $\approach$ demonstrates a clear and significant advantage. On CIFAR-10 with a ResNet-18 backbone, $\approach(\mu, \sigma) + \texttt{DICE}$ decisively outperforms NCI, reducing the average FPR95 by 33.87\%. This strong performance is maintained on the large-scale ImageNet benchmark, where our method reduces the FPR95 by 22.58\% on a ResNet-50.

    \begin{table}[t]
        \centering
        \caption{A direct comparison of $\approach$ against the NCI baseline, using their originally reported results for \textbf{CIFAR-10} with a \textbf{ResNet-18} backbone. The evaluation is restricted to the SVHN, Texture, and Places365 OOD datasets to ensure a fair comparison that matches the protocol from the original NCI paper.}
        \resizebox{\linewidth}{!}{
        \begin{tabular}{l  cc cc cc cc}
            \toprule
             \multirow{2}{*}{\textbf{Method}} & \multicolumn{2}{c}{SVHN} & \multicolumn{2}{c}{Places365} & \multicolumn{2}{c}{Texture} & \multicolumn{2}{c}{Average}  \\
            \cmidrule(lr){2-3} \cmidrule(lr){4-5} \cmidrule(lr){6-7} \cmidrule(lr){8-9}
            & \textbf{FPR95} $\downarrow$ & \textbf{AUROC} $\uparrow$ & \textbf{FPR95} $\downarrow$ & \textbf{AUROC} $\uparrow$ & \textbf{FPR95} $\downarrow$ & \textbf{AUROC} $\uparrow$ & \textbf{FPR95} $\downarrow$ & \textbf{AUROC} $\uparrow$ \\
            \midrule
            \texttt{NCI}                              & 28.92 & 90.81 & 34.01 & 90.74 & 26.53 & 92.18 & 29.82 & 91.24 \\
            $\approach({\mu,\sigma}) + \texttt{DICE}$ &  \textbf{7.85} & \textbf{98.57} & 35.27 & 93.18 & 14.04 & 97.71 & 19.72 & \textbf{96.49} \\ 
            $\approach(m) + \texttt{DICE}$            &  7.95 & 98.50 & \textbf{30.55} & \textbf{93.97} &  \textbf{9.66} & \textbf{98.28} & \textbf{16.72} & 96.25 \\
            \bottomrule
            \end{tabular}}
            \label{table: NCI_cifar_comparison}
    \end{table}

    \begin{table}[ht]
        \centering
        \caption{A direct comparison of $\approach$ against the NCI baseline, using their originally reported results for \textbf{ImageNet} with a \textbf{ResNet-50} backbone. The evaluation is restricted to the Texture, iNaturalist and OpenImage-O OOD datasets to ensure a fair comparison that matches the protocol from the original NCI paper.}
        \resizebox{\linewidth}{!}{
        \begin{tabular}{l cc cc cc cc}
            \toprule
             \multirow{2}{*}{\textbf{Method}} & \multicolumn{2}{c}{Texture} & \multicolumn{2}{c}{iNaturalist} & \multicolumn{2}{c}{OpenImage-O} & \multicolumn{2}{c}{Average}  \\
            \cmidrule(lr){2-3} \cmidrule(lr){4-5} \cmidrule(lr){6-7} \cmidrule(lr){8-9}
            & \textbf{FPR95} $\downarrow$ & \textbf{AUROC} $\uparrow$ & \textbf{FPR95} $\downarrow$ & \textbf{AUROC} $\uparrow$ & \textbf{FPR95} $\downarrow$ & \textbf{AUROC} $\uparrow$ & \textbf{FPR95} $\downarrow$ & \textbf{AUROC} $\uparrow$ \\
            \midrule
            \texttt{NCI}                                & 23.79 & 96.63 & 14.31 & 96.95 & \textbf{30.98} & \textbf{92.98} & 23.03 & 95.52 \\
            $\approach({\mu,\sigma}) + \texttt{SCALE}$  & 10.59 & 97.75 &  \textbf{9.61} & \textbf{98.10} & 33.30 & 92.85 & \textbf{17.83} & \textbf{96.23} \\    
            $\approach(m) + \texttt{SCALE}$             &  \textbf{9.52} & \textbf{98.11} & 13.26 & 97.37 & 34.48 & 92.50 & 19.09 & 95.99 \\
            \bottomrule
            \end{tabular}}
            \label{table: NCI_imagenet_comparison}
    \end{table}

\subsection{fDBD: Fast Decision Boundary OOD Detector}

    As shown in Tables \ref{table: fDBD_cifar_comparison} and \ref{table: fDBD_imagenet_comparison}, our method, $\approach$, demonstrates a decisive and substantial performance advantage over the fDBD's reported results in all comparable, overlapping settings. The strength of $\approach$ is most apparent when it is composed with existing techniques, creating a powerful synergistic effect that dramatically improves OOD detection. On CIFAR-10, this combination is particularly effective. Using a ResNet-18, $\approach(m) + \texttt{DICE}$ slashes the average FPR95 by 55.9\% relative to fDBD (from 31.09\% down to 13.72\%). The gains are even more pronounced on a DenseNet-101, where $\approach(\mu,\sigma) + \texttt{DICE}$ achieves an FPR95 reduction of 38.20\%.

        \begin{table}[ht]
            \centering
            \caption{Direct comparison of $\approach$ against the fDBD baseline on \textbf{CIFAR-10}. To ensure a fair comparison, the evaluation is restricted to the four OOD datasets reported in the original fDBD paper: SVHN, Places365, iSUN, and Texture.}
            \resizebox{\linewidth}{!}{
            \begin{tabular}{l l  cc cc cc cc cc}
                \toprule
                \multirow{2}{*}{\textbf{Model}} & \multirow{2}{*}{\textbf{Method}} & \multicolumn{2}{c}{SVHN} & \multicolumn{2}{c}{Places365} & \multicolumn{2}{c}{iSUN} & \multicolumn{2}{c}{Texture} & \multicolumn{2}{c}{Average}  \\
                \cmidrule(lr){3-4} \cmidrule(lr){5-6} \cmidrule(lr){7-8} \cmidrule(lr){9-10} \cmidrule(lr){11-12}
                && \textbf{FPR95} $\downarrow$ & \textbf{AUROC} $\uparrow$ & \textbf{FPR95} $\downarrow$ & \textbf{AUROC} $\uparrow$ & \textbf{FPR95} $\downarrow$ & \textbf{AUROC} $\uparrow$ & \textbf{FPR95} $\downarrow$ & \textbf{AUROC} $\uparrow$ & \textbf{FPR95} $\downarrow$ & \textbf{AUROC} $\uparrow$ \\
                \midrule
                \multirow{2}{*}{\textbf{ResNet-18}} & \texttt{fDBD}                         & 22.58 & 96.07 & 46.59 & 90.40 & 23.96 & 95.85 & 31.24 & 94.48 & 31.09 & 94.20 \\
                                                & $\approach({\mu,\sigma}) + \texttt{DICE}$ &  \textbf{7.85} & \textbf{98.57} & 35.27 & 93.18 & 11.96 & 97.96 & 14.04 & 97.71 & 17.78 & 96.86 \\
                                                & $\approach(m) + \texttt{DICE}$            &  7.95 & 98.50 & \textbf{30.55} & \textbf{93.97} &  \textbf{6.70} & \textbf{98.59} &  \textbf{9.66} & \textbf{98.28} & \textbf{13.72} & \textbf{97.34} \\
                \midrule
                \multirow{2}{*}{\textbf{DenseNet-101}} & \texttt{fDBD}                      &  5.89 & 98.67 & 39.52 & 91.53 & 5.90 & 98.75 & 22.75 & 95.81 & 18.52 & 96.19 \\ 
                                                & $\approach({\mu,\sigma}) + \texttt{DICE}$ &  \textbf{6.84} & \textbf{98.77} & 29.76 & 93.86 & \textbf{1.58} & \textbf{99.59} &  9.57 & 98.25 & \textbf{11.44} & 97.62 \\ 
                                                & $\approach(m) + \texttt{DICE}$            &  8.30 & 98.37 & \textbf{29.47} & \textbf{93.92} & 1.85 & 99.56 &  \textbf{7.16} & \textbf{98.69} & 11.70 & \textbf{97.64} \\
                \bottomrule
                \end{tabular}}
                \label{table: fDBD_cifar_comparison}
        \end{table}

    As demonstrated in Table~\ref{table: fDBD_imagenet_comparison}, this commanding performance extends to the large-scale ImageNet benchmark. While fDBD struggles with a high average FPR95 of 51.19\%, our $\approach(\mu,\sigma) + \texttt{SCALE}$ achieves an FPR95 of just 20.17\%,  a massive 56.40\% relative reduction. These results validate that by first enriching the feature representation with more discriminative statistics, $\approach$ enables subsequent methods to operate far more effectively, establishing a new state-of-the-art over the fDBD baseline.

        \begin{table}[ht]
            \centering
            \caption{This table presents a direct comparison of $\approach$ against the fDBD baseline on \textbf{ImageNet} using a \textbf{ResNet-50} backbone. To ensure a fair comparison, the evaluation is restricted to the iNaturalist and Texture OOD datasets, matching the protocol in the original fDBD paper.}
            \resizebox{\linewidth}{!}{
            \begin{tabular}{l l  cc cc cc cc cc}
                \toprule
                \multirow{2}{*}{\textbf{Method}} & \multicolumn{2}{c}{SUN} & \multicolumn{2}{c}{Places365} & \multicolumn{2}{c}{Texture} & \multicolumn{2}{c}{iNaturalist} & \multicolumn{2}{c}{Average}  \\
                \cmidrule(lr){2-3} \cmidrule(lr){4-5} \cmidrule(lr){6-7} \cmidrule(lr){8-9} \cmidrule(lr){10-11}
                & \textbf{FPR95} $\downarrow$ & \textbf{AUROC} $\uparrow$ & \textbf{FPR95} $\downarrow$ & \textbf{AUROC} $\uparrow$ & \textbf{FPR95} $\downarrow$ & \textbf{AUROC} $\uparrow$ & \textbf{FPR95} $\downarrow$ & \textbf{AUROC} $\uparrow$ & \textbf{FPR95} $\downarrow$ & \textbf{AUROC} $\uparrow$ \\
                \midrule
                 \texttt{fDBD}                             & 60.60 & 86.97 & 66.40 & 84.27 & 37.50 & 92.12 & 40.24 & 93.67 & 51.19 & 89.26 \\
                $\approach({\mu,\sigma}) + \texttt{SCALE}$ & \textbf{24.49} & \textbf{94.62} & \textbf{36.01} & \textbf{91.83} & \textbf{10.59} & \textbf{97.75} &  \textbf{9.61} & \textbf{98.10} & \textbf{20.17} & \textbf{95.58} \\  
                $\approach(m) + \texttt{SCALE}$            & 27.79 & 93.66 & 40.56 & 90.02 &  9.52 & 98.11 & 13.26 & 97.37 & 22.78 & 94.79 \\
                \bottomrule
                \end{tabular}}
                \label{table: fDBD_imagenet_comparison}
        \end{table}

\subsection{Additional Baselines}

    Additionally, we conduct a large-scale benchmark comparison on ImageNet-1k against plethora of existing literature using both ResNet-50 and MobileNet-v2. As shown in Table~\ref{table: other_baselines_imagenet_benchmark}, we compare our method against 19 existing baselines for ResNet-50~\cite{msp, odin, G-odin, maha_distance, knn, gradorth, GradNorm, NNGuide, ViM, fdbd, BATS, energy, ReAct, DICE, ASH, SCALE} and 15 baselines for MobileNet-v2~\cite{msp, odin, maha_distance, gradorth, GradNorm, NNGuide, ViM, BATS, energy, ReAct, DICE, ASH, SCALE}, with all competitors' results taken directly from their original publications (those were not reproduced). This comprehensive evaluation demonstrates that $\approach$ achieves competitive and consistent performance compared to all prior post-hoc methods on this challenging benchmark. Importantly, no one existing methods is universally superior, even in traditional CNN-based backbone. 

    \begin{table*}[!ht]
        \centering
        \caption{Detailed Comparison with existing OOD detection methods on the ImageNet-1k benchmark, using ResNet-50 and MobileNet-v2. Methods marked with $^*$ were reproduced by us; results for all other methods are taken from their original publications. The symbol $\boldsymbol{\downarrow}$ indicates lower values are better; $\boldsymbol{\uparrow}$ indicates larger values are better.}
        \resizebox{\textwidth}{!}{
        \begin{tabular}{l l  cc cc cc cc cc}
            \toprule
            \multirow{2}{*}{\textbf{Model}} & \multirow{2}{*}{\textbf{Method}} & \multicolumn{2}{c}{SUN} & \multicolumn{2}{c}{Places} & \multicolumn{2}{c}{Texture} & \multicolumn{2}{c}{iNaturalist} & \multicolumn{2}{c}{Average}  \\
            \cmidrule(lr){3-4} \cmidrule(lr){5-6} \cmidrule(lr){7-8} \cmidrule(lr){9-10} \cmidrule(lr){11-12}
            && \textbf{FPR95} $\downarrow$ & \textbf{AUROC} $\uparrow$ & \textbf{FPR95} $\downarrow$ & \textbf{AUROC} $\uparrow$ & \textbf{FPR95} $\downarrow$ & \textbf{AUROC} $\uparrow$ & \textbf{FPR95} $\downarrow$ & \textbf{AUROC} $\uparrow$ & \textbf{FPR95} $\downarrow$ & \textbf{AUROC} $\uparrow$ \\
            \midrule
            \multirow{19}{*}{\rotatebox{90}{ResNet-50}} & MSP$^*$~\cite{msp}                & 68.58 & 81.75 & 71.57 & 80.63 & 66.13 & 80.46 & 52.77 & 88.42 & 64.76 & 82.82 \\
                                        & ODIN$^*$~\cite{odin}              & 60.15 & 84.59 & 67.89 & 81.78 & 50.23 & 85.62 & 47.66 & 89.66 & 56.48 & 85.41 \\
                                        & GODIN~\cite{G-odin}               & 60.83 & 85.60 & 63.70 & 83.81 & 77.85 & 73.27 & 61.91 & 85.40 & 66.07 & 82.02 \\
                                    & Mahalanobis~\cite{maha_distance}      & 68.36 & 84.35 & 73.32 & 81.46 & 16.05 & 94.96 & 39.90 & 93.76 & 49.41 & 88.63 \\
                                    & KNN$^*$~\cite{knn}                    & 78.95 & 77.44 & 81.86 & 73.91 & 16.05 & 96.11 & 78.33 & 79.15 & 63.80 & 81.65 \\
                                        & GradOrth~\cite{gradorth}          & 19.61 & 95.76 & 33.67 & 91.78 & 11.19 & 98.06 & 11.04 & 98.00 & 18.57 & 96.31 \\
                                        & GradNorm$^*$~\cite{GradNorm}      & 37.42 & 90.10 & 48.88 & 86.08 & 32.84 & 90.64 & 26.78 & 93.9 & 36.48 & 90.18 \\
                                        & NN-Guide~\cite{NNGuide}           & 31.62 & 91.66 & 38.88 & 90.12 & 24.93 & 91.52 & 12.02 & 97.47 & 26.86 & 92.69 \\
                                        & ViM~\cite{ViM}                    & 43.10 & 89.39 & 52.86 & 86.61 & 17.18 & 93.58 & 20.34 & 96.24 & 33.37 & 91.45 \\
                                        & fDBD~\cite{fdbd}                  & 60.60 & 86.97 & 66.40 & 84.27 & 37.50 & 92.12 & 40.24 & 93.67 & 51.19 & 89.26 \\
                                        & BATS~\cite{BATS}                  & 22.62 & 95.33 & 34.34 & 91.83 & 38.90 & 92.27 & 12.57 & 97.67 & 27.11 & 94.20 \\
                                        & LAPS~\cite{laps}                  & 15.81 & 96.18 & 24.71 & 93.64 & 41.49 & 91.81 & 12.72 & 97.50 & 23.68 & 94.78 \\
                                        & Energy$^*$~\cite{energy}          & 58.28 & 86.73 & 65.40 & 84.13 & 52.29 & 86.73 & 53.95 & 90.59 & 57.48 & 87.05 \\
                                        & ReAct$^*$~\cite{ReAct}            & 23.68 & 94.44 & 33.33 & 91.96 & 46.33 & 90.30 & 19.73 & 96.37 & 30.77 & 93.27 \\
                                        & DICE$^*$~\cite{DICE}              & 36.11 & 91.01 & 47.62 & 87.76 & 32.38 & 90.48 & 26.48 & 94.53 & 35.65 & 90.94 \\
                                        & ASH-S$^*$~\cite{ASH}                & 28.01 & 94.02 & 39.84 & 90.98 & 11.95 & 97.60 & 11.52 & 97.87 & 22.83 & 95.12 \\
                                        & SCALE$^*$~\cite{SCALE}            & 25.78 & 94.54 & 36.86 & 91.96 & 14.56 & 96.75 & 10.37 & 98.02 & 21.89 & 95.32 \\
            \cmidrule(lr){2-12}
            \rowcolor[gray]{0.9}        & \textbf{$\approach(\mu, \sigma)+\texttt{SCALE}$}  & 24.49 & 94.62 & 36.01 & 91.83 & 10.59 & 97.75 & 9.61 & 98.10 & 20.18 & 95.58 \\
            \rowcolor[gray]{0.9}        & \textbf{$\approach(m)+\texttt{SCALE}$}            & 27.79 & 93.66 & 40.56 & 90.02 & 9.52 & 98.11 & 13.26 & 97.37 & 22.78 & 94.79 \\
            \midrule
            \multirow{16}{*}{\rotatebox{90}{MobileNet-v2}} & MSP$^*$~\cite{msp}                 & 74.20 & 78.88 & 76.89 & 78.14 & 70.99 & 78.95 & 59.86 & 86.72 & 70.49 & 80.67 \\
                                        & ODIN$^*$~\cite{odin}                  & 54.07 & 85.88 & 57.36 & 84.71 & 49.96 & 85.03 & 55.39 & 87.62 & 54.20 & 85.81 \\
                                       & Mahalanobis~\cite{maha_distance}       & 54.79 & 86.33 & 53.77 & 83.69 & 88.72 & 37.28 & 62.04 & 82.37 & 64.83 & 72.40 \\
                                       & KNN$^*$~\cite{knn}                     & 93.82 & 59.41 & 94.1 & 57.72 & 19.84 & 95.34 & 93.34 & 64.82 & 75.28 & 69.32 \\
                                        & GradOrth~\cite{gradorth}              & 30.82 & 93.18 & 40.27 & 89.12 & 12.69 & 97.52 & 26.81 & 93.17 & 27.65 & 93.25 \\
                                        & GradNorm$^*$~\cite{GradNorm}          & 38.7 & 91.07 & 53.1 & 85.99 & 31.67 & 92.23 & 37.23 & 92.02 & 40.18 & 90.33 \\
                                        & NN-Guide~\cite{NNGuide}               & 79.57 & 76.10 & 81.87 & 74.23 & 38.78 & 89.32 & 68.24 & 82.07 & 67.12 & 80.43 \\
                                        & ViM~\cite{ViM}                        & 88.67 & 66.37 & 92.16 & 62.43 & 40.71 & 89.59 & 86.86 & 69.57 & 77.10 & 71.99 \\
                                        & BATS~\cite{BATS}                      & 41.68 & 90.21 & 52.43 & 86.26 & 38.69 & 90.76 & 31.56 & 94.33 & 41.09 & 90.39 \\
                                        & LAPS~\cite{laps}                      & 30.07 & 92.98 & 39.70 & 90.10 & 51.37 & 88.29 & 18.82 & 96.76 & 34.99 & 92.03 \\
                                        & Energy$^*$~\cite{energy}              & 59.36 & 86.24 & 66.27 & 83.21 & 54.54 & 86.58 & 55.31 & 90.34 & 58.87 & 86.59 \\
                                        & ReAct$^*$~\cite{ReAct}                & 52.46 & 87.26 & 59.89 & 84.07 & 40.25 & 90.96 & 43.05 & 92.72 & 48.91 & 88.75 \\
                                        & DICE$^*$~\cite{DICE}                  & 37.84 & 90.81 & 52.35 & 86.17 & 32.57 & 91.46 & 41.53 & 91.30 & 41.07 & 89.94 \\
                                        & ASH-S$^*$~\cite{ASH}                    & 43.63 & 90.02 & 58.85 & 84.73 & 13.12 & 97.10 & 39.13 & 91.94 & 38.68 & 90.95 \\
                                        & SCALE$^*$~\cite{SCALE}                & 38.74 & 91.64 & 53.49 & 87.34 & 14.79 & 96.65 & 30.09 & 94.46 & 34.28 & 92.52 \\
            \cmidrule(lr){2-12}
            \rowcolor[gray]{0.9}        & \textbf{$\approach(\mu, \sigma)+\texttt{SCALE}$}      & 37.89 & 91.82 & 52.92 & 87.24 & 11.37 & 97.60 & 28.54 & 94.76 & 32.68 & 92.86 \\
            \rowcolor[gray]{0.9}        & \textbf{$\approach(m)+\texttt{SCALE}$}                & 37.47 & 91.91 & 52.72 & 86.90 &  8.85 & 98.20 & 26.76 & 95.08 & 31.45 & 93.02 \\
            \bottomrule
            \end{tabular}
            }
            \label{table: other_baselines_imagenet_benchmark}
        \end{table*}

\section{Analysis of Activation on OOD Detection}
\label{appendix: resnet activation results}

    This section details our investigation into the performance degradation of standard OOD baselines on alternate activation functions: SiLU, GeLU, and TanH. As eluded in main paper, we attribute this failure to the fundamental difference between the sparse feature maps produced by the ReLU activation function versus the dense maps produced by SiLU, GeLU, and TanH. For brevity, we perform analysis using SiLU activation that is used in EfficientNet-B0. 

    The ReLU function, $\text{ReLU}(\mathbf{x}) = \max(0,\mathbf{x})$, creates sparse activations by forcing all negative inputs to zero. This sparsity is an implicit assumption for pruning-based methods like ASH and SCALE. In contrast, the SiLU function, $\text{SiLU}(\mathbf{x}) = \mathbf{x} \cdot \sigma(\mathbf{x})$ where $\sigma(\mathbf{x})$ is the sigmoid, is non-sparsifying and produces dense feature maps, altering the statistical landscape on which these methods were designed to operate.

    \begin{figure}[!ht]
      \centering
      \includegraphics[width=\textwidth]{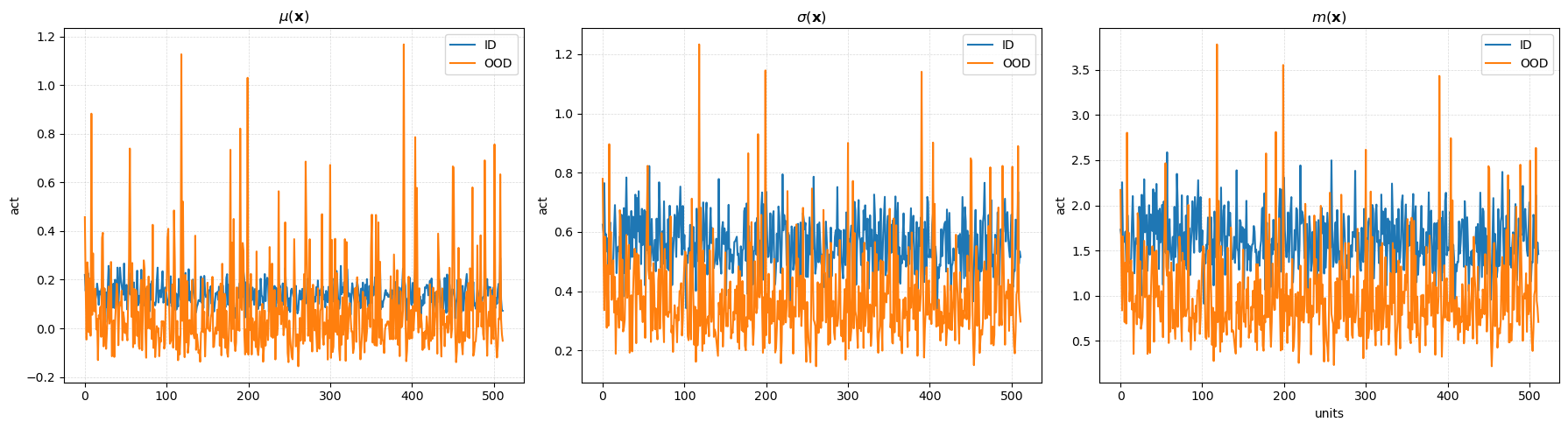}
      \caption{ \textit{Feature statistics for ID (CIFAR-100) vs. OOD (Texture) samples on an ResNet18-SiLU backbone. While mean $\mu(\mathbf{x})$ show poor separation, both the standard deviation $\sigma(\mathbf{x})$ and maximum $m(\mathbf{x})$ statistics maintain a clear separation between ID and OOD activations.}}
      \label{fig: resnet18_silu_features}
    \end{figure}

    To isolate this effect, we evaluated ResNet-18 with SiLU, GeLU, and TanH instead of ReLU on the CIFAR benchmarks. This experiment revealed two key findings as shown in Figure~\ref{fig: resnet18_silu_features}

    \begin{itemize}
        \item High Overlap of Mean Features: The standard mean features extracted using GAP from the ResNet18-SiLU model show a high degree of overlap between ID and OOD samples. This poor separability is the primary reason for the failure of baselines that rely solely on these features.
    
        \item Enhanced Separation with $\approach$: In contrast, the feature representations from our method, $\approach(m)$ and $\approach(\mu, \sigma)$, successfully establish a clear and discriminative boundary between the ID and OOD distributions. This demonstrates that the maximum and standard deviation statistics are robust cues even in a dense activation landscape.
        
    \end{itemize}

    These qualitative findings are confirmed by the quantitative results in Table~\ref{table: ResNet-18_SiLU_CIFAR-10_combination}, ~\ref{table: ResNet-18_SiLU_CIFAR-100_combination}, ~\ref{table: ResNet-18_GeLU_CIFAR-10_combination}, ~\ref{table: ResNet-18_GeLU_CIFAR-100_combination}, ~\ref{table: ResNet-18_TanH_CIFAR-10_combination}, ~\ref{table: ResNet-18_TanH_CIFAR-100_combination}. There, we show that standard baselines on alternate activation function perform poorly, specially on complex dataset CIFAR-100 compared to standard ResNet18 in Table~\ref{table: detailed_results_cifar10_a}, but are critically rescued when combined with $\approach$. This analysis strongly supports our hypothesis that the dense nature of alternate activation function challenges conventional OOD methods, and it highlights the robustness and generality of the statistical cues leveraged by our approach.

\begin{landscape}
    \begin{table}[!htbp]
    \centering
    \caption{Detailed results of post-hoc methods combined with $\approach$ using \textbf{ResNet-18} with the \textbf{SiLU} activation function instead of \textbf{ReLU}, trained on the CIFAR dataset. $\boldsymbol{\uparrow}$ indicates that higher values are better; $\boldsymbol{\downarrow}$ indicates that lower values are better. The symbols denote the statistics used: $\mu$ (mean), $\sigma$ (standard deviation), and $m$ (maximum).}
    \resizebox{\linewidth}{!}{
    \begin{tabular}{c l  cc cc cc cc cc cc cc }
    \toprule
    \multirow{2}{*}{\textbf{Dataset}} & \multirow{2}{*}{\textbf{Combined Method}} & \multicolumn{2}{c}{\textbf{SVHN}} & \multicolumn{2}{c}{\textbf{Place365}} & \multicolumn{2}{c}{\textbf{iSUN}} & \multicolumn{2}{c}{\textbf{Textures}} & \multicolumn{2}{c}{\textbf{LSUN-c}}  & \multicolumn{2}{c}{\textbf{LSUN-r}} & \multicolumn{2}{c}{\textbf{Average}}  \\
    \cmidrule(lr){3-4} \cmidrule(lr){5-6} \cmidrule(lr){7-8} \cmidrule(lr){9-10} \cmidrule(lr){11-12} \cmidrule(lr){13-14} \cmidrule(lr){15-16}
    && \textbf{FPR95} $\downarrow$ & \textbf{AUROC} $\uparrow$ & \textbf{FPR95} $\downarrow$ & \textbf{AUROC} $\uparrow$ & \textbf{FPR95} $\downarrow$ & \textbf{AUROC} $\uparrow$ & \textbf{FPR95} $\downarrow$ & \textbf{AUROC} $\uparrow$ & \textbf{FPR95} $\downarrow$ & \textbf{AUROC} $\uparrow$ & \textbf{FPR95} $\downarrow$ & \textbf{AUROC} $\uparrow$ & \textbf{FPR95} $\downarrow$ & \textbf{AUROC} $\uparrow$ \\
    \midrule

    \multirow{27}{*}{CIFAR-10}  
    & MSP                        & 56.15 & 92.93 & 62.45 & 88.72 & 41.70 & 94.49 & 58.03 & 91.02 & 26.15 & 96.60 & 39.12 & 94.79 & 47.27 & 93.10 \\
    & + $\approach(m)$           & 54.29 & 85.57 & 65.92 & 79.62 & 49.12 & 89.53 & 52.64 & 86.94 & 31.24 & 93.75 & 47.55 & 89.76 & 50.13 & 87.53 \\
    & + $\approach(\mu, \sigma)$ & 53.47 & 85.33 & 66.64 & 77.71 & 48.62 & 88.38 & 52.71 & 85.59 & 29.52 & 93.75 & 47.41 & 88.87 & 49.73 & 86.61 \\
    \cmidrule(lr){2-16}

    & ODIN                       & 26.52 & 95.89 & 39.60 & 91.87 & 6.64 & 98.59 & 35.55 & 93.89 & 1.41 & 99.53 & 6.33 & 98.65 & 19.34 & 96.40 \\
    & + $\approach(m)$           & 25.47 & 95.83 & 40.46 & 91.74 & 7.76 & 98.25 & 35.60 & 93.83 & 1.80 & 99.31 & 7.01 & 98.34 & 19.68 & 96.22 \\
    & + $\approach(\mu, \sigma)$ & 25.56 & 95.91 & 41.01 & 91.72 & 7.50 & 98.35 & 35.85 & 93.84 & 1.74 & 99.39 & 6.82 & 98.43 & 19.75 & 96.27 \\
    \cmidrule(lr){2-16}

    & Energy                     & 32.13 & 95.20 & 40.20 & 91.53 & 10.39 & 97.91 & 39.57 & 92.97 & 1.93 & 99.34 & 9.45 & 98.03 & 22.28 & 95.83 \\
    & + $\approach(m)$           & 23.88 & 96.03 & 35.55 & 93.30 & 7.20 & 98.50 & 14.79 & 97.50 & 3.15 & 99.27 & 7.88 & 98.43 & 15.41 & 97.17 \\
    & + $\approach(\mu, \sigma)$ & 23.90 & 96.14 & 38.24 & 92.82 & 8.25 & 98.39 & 16.76 & 97.23 & 2.49 & 99.39 & 8.57 & 98.36 & 16.37 & 97.05 \\
    \cmidrule(lr){2-16}

    & GradNorm                   & 39.84 & 93.96 & 45.39 & 90.54 & 16.06 & 97.28 & 49.49 & 88.06 & 2.15 & 99.26 & 14.03 & 97.48 & 27.83 & 94.43 \\
    & + $\approach(m)$           & 13.48 & 97.51 & 40.32 & 90.89 & 9.21 & 98.28 & 12.23 & 97.64 & 0.45 & 99.90 & 10.53 & 98.01 & 14.37 & 97.04 \\
    & + $\approach(\mu, \sigma)$ & 13.40 & 97.52 & 41.66 & 90.81 & 10.04 & 98.15 & 13.35 & 97.44 & 0.14 & 99.93 & 11.19 & 97.90 & 14.96 & 96.96 \\
    \cmidrule(lr){2-16}

    & KNN                        & 31.04 & 95.32 & 48.66 & 89.66 & 22.41 & 96.43 & 33.35 & 94.48 & 4.64 & 99.11 & 21.08 & 96.69 & 26.86 & 95.28 \\
    & + $\approach(m)$           & 44.92 & 92.88 & 38.51 & 92.38 & 21.49 & 96.48 & 27.16 & 95.43 & 6.59 & 98.86 & 18.68 & 96.86 & 26.23 & 95.48 \\
    & + $\approach(\mu, \sigma)$ & 32.52 & 94.83 & 39.33 & 92.05 & 17.89 & 96.97 & 22.43 & 96.14 & 5.19 & 99.10 & 16.00 & 97.25 & 22.23 & 96.06 \\
    \cmidrule(lr){2-16}

    & ReAct                      & 31.97 & 95.25 & 39.12 & 91.96 & 10.03 & 97.95 & 37.18 & 93.68 & 1.71 & 99.36 & 9.08 & 98.08 & 21.51 & 96.05 \\
    & + $\approach(m)$           & 24.58 & 95.93 & 36.02 & 93.13 & 7.37 & 98.48 & 15.32 & 97.41 & 3.38 & 99.25 & 8.18 & 98.41 & 15.81 & 97.10 \\
    & + $\approach(\mu, \sigma)$ & 23.66 & 96.09 & 38.01 & 92.72 & 8.09 & 98.38 & 17.00 & 97.16 & 2.59 & 99.38 & 8.51 & 98.36 & 16.31 & 97.01 \\
    \cmidrule(lr){2-16}

    & DICE                       & 23.87 & 96.29 & 43.98 & 90.89 & 9.19 & 98.22 & 29.06 & 94.70 & 0.20 & 99.92 & 10.04 & 98.12 & 19.39 & 96.36 \\
    & + $\approach(m)$           & 14.99 & 97.42 & 30.73 & 93.87 & 4.48 & 99.05 & 8.03 & 98.53 & 0.54 & 99.84 & 5.40 & 98.91 & 10.70 & 97.94 \\
    & + $\approach(\mu, \sigma)$ & 15.10 & 97.46 & 33.85 & 93.40 & 5.30 & 98.94 & 9.86 & 98.35 & 0.35 & 99.89 & 5.80 & 98.82 & 11.71 & 97.81 \\
    \cmidrule(lr){2-16}

    & ASH-S                      & 24.10 & 96.08 & 46.30 & 90.63 & 12.24 & 97.76 & 29.88 & 94.88 & 0.13 & 99.94 & 13.39 & 97.66 & 21.01 & 96.16 \\
    & + $\approach(m)$           & 19.73 & 96.50 & 33.94 & 93.39 & 7.64 & 98.45 & 11.10 & 98.00 & 1.26 & 99.63 & 8.73 & 98.32 & 13.73 & 97.38 \\
    & + $\approach(\mu, \sigma)$ & 19.15 & 96.68 & 36.41 & 92.97 & 8.92 & 98.32 & 12.22 & 97.88 & 0.83 & 99.71 & 9.42 & 98.22 & 14.49 & 97.30 \\
    \cmidrule(lr){2-16}

    & SCALE                       & 21.51 & 96.45 & 44.14 & 91.24 & 13.64 & 97.59 & 32.93 & 94.35 & 0.08 & 99.94 & 13.88 & 97.56 & 21.03 & 96.19  \\
    & + $\approach(m) $           & 19.04 & 96.60 & 33.45 & 93.45 & 7.20 & 98.49 & 11.72 & 97.90 & 1.01 & 99.66 & 8.04 & 98.38 & 13.41 & 97.41 \\
    & + $\approach({\mu,\sigma})$ & 19.04 & 96.66 & 36.78 & 92.99 & 8.80 & 98.31 & 13.48 & 97.74 & 0.70 & 99.72 & 9.28 & 98.21 & 14.68 & 97.27 \\
    \bottomrule
    \end{tabular}}
    \label{table: ResNet-18_SiLU_CIFAR-10_combination}
    \end{table}
\end{landscape}

\begin{landscape}
    \begin{table}[!htbp]
        \centering
        \caption{Detailed results of post-hoc methods combined with $\approach$ using \textbf{ResNet-18} with the \textbf{SiLU} activation function instead of \textbf{ReLU}, trained on the CIFAR dataset. $\boldsymbol{\uparrow}$ indicates that higher values are better; $\boldsymbol{\downarrow}$ indicates that lower values are better. The symbols denote the statistics used: $\mu$ (mean), $\sigma$ (standard deviation), and $m$ (maximum).}
        \resizebox{\linewidth}{!}{
        \begin{tabular}{c l  cc cc cc cc cc cc cc }
        \toprule
        \multirow{2}{*}{\textbf{Dataset}} & \multirow{2}{*}{\textbf{Combined Method}} & \multicolumn{2}{c}{\textbf{SVHN}} & \multicolumn{2}{c}{\textbf{Place365}} & \multicolumn{2}{c}{\textbf{iSUN}} & \multicolumn{2}{c}{\textbf{Textures}} & \multicolumn{2}{c}{\textbf{LSUN-c}}  & \multicolumn{2}{c}{\textbf{LSUN-r}} & \multicolumn{2}{c}{\textbf{Average}}  \\
        \cmidrule(lr){3-4} \cmidrule(lr){5-6} \cmidrule(lr){7-8} \cmidrule(lr){9-10} \cmidrule(lr){11-12} \cmidrule(lr){13-14} \cmidrule(lr){15-16}
        && \textbf{FPR95} $\downarrow$ & \textbf{AUROC} $\uparrow$ & \textbf{FPR95} $\downarrow$ & \textbf{AUROC} $\uparrow$ & \textbf{FPR95} $\downarrow$ & \textbf{AUROC} $\uparrow$ & \textbf{FPR95} $\downarrow$ & \textbf{AUROC} $\uparrow$ & \textbf{FPR95} $\downarrow$ & \textbf{AUROC} $\uparrow$ & \textbf{FPR95} $\downarrow$ & \textbf{AUROC} $\uparrow$ & \textbf{FPR95} $\downarrow$ & \textbf{AUROC} $\uparrow$ \\
        \midrule
            \multirow{24}{*}{CIFAR-100}
            & MSP                        & 77.86 & 79.19 & 83.80 & 74.05 & 91.19 & 63.67 & 85.76 & 73.19 & 59.92 & 87.22 & 89.41 & 65.42 & 81.32 & 73.79 \\
            & + $\approach(m)$           & 80.46 & 82.22 & 88.45 & 70.49 & 90.27 & 66.73 & 83.95 & 76.98 & 76.11 & 83.41 & 90.84 & 67.03 & 85.01 & 74.48 \\
            & + $\approach(\mu, \sigma)$ & 79.12 & 82.26 & 87.15 & 69.95 & 91.24 & 63.77 & 83.92 & 75.93 & 74.55 & 83.71 & 90.77 & 64.41 & 84.46 & 73.34 \\
            \cmidrule(lr){2-16}
            
            & ODIN                       & 68.44 & 88.58 & 81.12 & 75.81 & 80.93 & 76.74 & 83.46 & 76.82 & 17.55 & 97.00 & 75.89 & 80.15 & 67.90 & 82.52 \\
            & + $\approach(m)$           & 63.88 & 89.67 & 80.31 & 76.13 & 83.00 & 76.10 & 81.31 & 77.75 & 20.49 & 96.51 & 78.92 & 79.34 & 67.99 & 82.58 \\
            & + $\approach(\mu, \sigma)$ & 63.93 & 89.69 & 80.60 & 76.02 & 82.22 & 76.39 & 81.65 & 77.62 & 19.72 & 96.70 & 78.11 & 79.63 & 67.70 & 82.68 \\
            \cmidrule(lr){2-16}
            
            & Energy                     & 70.16 & 88.22 & 81.62 & 75.39 & 87.89 & 71.90 & 83.94 & 75.59 & 18.58 & 96.74 & 84.22 & 75.60 & 71.07 & 80.57 \\
            & + $\approach(m)$           & 25.26 & 95.92 & 77.90 & 78.86 & 69.87 & 86.94 & 38.67 & 92.50 & 11.95 & 97.97 & 69.65 & 87.19 & 48.88 & 89.90 \\
            & + $\approach(\mu, \sigma)$ & 25.38 & 95.97 & 78.64 & 77.80 & 76.25 & 83.98 & 45.50 & 91.10 & 10.54 & 98.19 & 75.38 & 84.73 & 51.95 & 88.63 \\
            \cmidrule(lr){2-16}
            
            & GradNorm                   & 74.71 & 83.17 & 84.84 & 67.20 & 92.67 & 56.34 & 88.81 & 60.44 & 5.43 & 98.77 & 89.90 & 62.14 & 72.73 & 71.34 \\
            & + $\approach(m)$           & 21.87 & 95.65 & 93.01 & 62.14 & 74.70 & 85.27 & 33.67 & 91.64 & 9.30 & 98.29 & 79.93 & 84.14 & 52.08 & 86.19 \\
            & + $\approach(\mu, \sigma)$ & 20.94 & 95.89 & 93.45 & 59.75 & 80.80 & 82.56 & 36.84 & 90.80 & 6.79 & 98.65 & 83.55 & 81.57 & 53.73 & 84.87 \\
            \cmidrule(lr){2-16}
            
            & KNN                        & 80.61 & 84.30 & 85.04 & 69.47 & 93.00 & 63.10 & 82.89 & 76.37 & 70.35 & 74.81 & 91.08 & 65.98 & 83.83 & 72.34 \\
            & + $\approach(m)$           & 11.10 & 97.82 & 81.97 & 74.56 & 59.07 & 87.79 & 24.15 & 94.82 & 23.22 & 93.51 & 62.79 & 87.50 & 43.72 & 89.33 \\
            & + $\approach(\mu, \sigma)$ & 11.78 & 97.76 & 82.63 & 73.36 & 66.16 & 83.77 & 27.68 & 94.17 & 31.88 & 90.62 & 69.33 & 83.70 & 48.24 & 87.23 \\
            \cmidrule(lr){2-16}
            
            & ReAct                      & 54.12 & 90.84 & 80.64 & 76.61 & 66.98 & 88.00 & 66.72 & 86.45 & 27.90 & 94.96 & 63.06 & 88.84 & 59.90 & 87.62 \\
            & + $\approach(m)$           & 23.81 & 96.03 & 77.76 & 78.79 & 68.00 & 87.57 & 38.30 & 92.61 & 12.51 & 97.86 & 67.60 & 87.52 & 48.00 & 90.06 \\
            & + $\approach(\mu, \sigma)$ & 22.80 & 96.29 & 79.00 & 77.97 & 72.74 & 86.03 & 43.69 & 91.76 & 11.38 & 98.08 & 72.16 & 86.20 & 50.29 & 89.39 \\
            \cmidrule(lr){2-16}
            
            & DICE                       & 19.78 & 96.53 & 84.56 & 72.27 & 84.66 & 71.12 & 59.17 & 82.42 & 1.98 & 99.58 & 82.31 & 74.30 & 55.41 & 82.70 \\
            & + $\approach(m)$           & 12.52 & 97.44 & 82.47 & 76.44 & 50.05 & 91.72 & 23.28 & 94.97 & 3.46 & 99.16 & 52.83 & 91.29 & 37.44 & 91.84 \\
            & + $\approach(\mu, \sigma)$ & 10.68 & 97.87 & 81.90 & 75.76 & 59.23 & 89.12 & 25.99 & 94.39 & 2.21 & 99.45 & 59.88 & 89.16 & 39.98 & 90.96 \\
            \cmidrule(lr){2-16}
            
            & ASH-S                      & 13.60 & 97.39 & 85.86 & 70.56 & 86.41 & 73.43 & 53.72 & 86.43 & 1.96 & 99.56 & 85.89 & 74.29 & 54.57 & 83.61 \\
            & + $\approach(m)$           & 8.18 & 98.41 & 81.23 & 75.79 & 49.45 & 91.44 & 20.20 & 95.87 & 5.58 & 98.90 & 54.73 & 90.31 & 36.56 & 91.79 \\
            & + $\approach(\mu, \sigma)$ & 7.14 & 98.64 & 81.11 & 75.19 & 54.30 & 90.30 & 21.05 & 95.80 & 5.02 & 99.02 & 58.78 & 89.22 & 37.90 & 91.36 \\
            \cmidrule(lr){2-16}
            
            & SCALE                       & 14.28 & 97.36 & 85.50 & 71.18 & 83.92 & 76.62 & 52.38 & 87.33 & 1.86 & 99.57 & 82.56 & 77.55 & 53.42 & 84.94 \\
            & + $\approach(m)$            & 9.47 & 98.16 & 83.17 & 75.36 & 56.46 & 90.05 & 22.59 & 95.39 & 6.14 & 98.78 & 63.04 & 88.62 & 40.14 & 91.06 \\
            & + $\approach({\mu,\sigma})$ & 7.94 & 98.44 & 82.81 & 75.10 & 60.86 & 88.86 & 22.54 & 95.43 & 5.36 & 98.96 & 65.95 & 87.44 & 40.91 & 90.70 \\
        \bottomrule
        \end{tabular}}
        \label{table: ResNet-18_SiLU_CIFAR-100_combination}
    \end{table}
\end{landscape}

\begin{landscape}
    \begin{table}[!htbp]
        \centering
        \caption{Detailed results of post-hoc methods combined with $\approach$ using \textbf{ResNet-18} with the \textbf{GeLU} activation function instead of \textbf{ReLU}, trained on the CIFAR dataset. $\boldsymbol{\uparrow}$ indicates that higher values are better; $\boldsymbol{\downarrow}$ indicates that lower values are better. The symbols denote the statistics used: $\mu$ (mean), $\sigma$ (standard deviation), and $m$ (maximum).}
        \resizebox{\linewidth}{!}{
        \begin{tabular}{c l  cc cc cc cc cc cc cc }
        \toprule
        \multirow{2}{*}{\textbf{Dataset}} & \multirow{2}{*}{\textbf{Combined Method}} & \multicolumn{2}{c}{\textbf{SVHN}} & \multicolumn{2}{c}{\textbf{Place365}} & \multicolumn{2}{c}{\textbf{iSUN}} & \multicolumn{2}{c}{\textbf{Textures}} & \multicolumn{2}{c}{\textbf{LSUN-c}}  & \multicolumn{2}{c}{\textbf{LSUN-r}} & \multicolumn{2}{c}{\textbf{Average}}  \\
        \cmidrule(lr){3-4} \cmidrule(lr){5-6} \cmidrule(lr){7-8} \cmidrule(lr){9-10} \cmidrule(lr){11-12} \cmidrule(lr){13-14} \cmidrule(lr){15-16}
        && \textbf{FPR95} $\downarrow$ & \textbf{AUROC} $\uparrow$ & \textbf{FPR95} $\downarrow$ & \textbf{AUROC} $\uparrow$ & \textbf{FPR95} $\downarrow$ & \textbf{AUROC} $\uparrow$ & \textbf{FPR95} $\downarrow$ & \textbf{AUROC} $\uparrow$ & \textbf{FPR95} $\downarrow$ & \textbf{AUROC} $\uparrow$ & \textbf{FPR95} $\downarrow$ & \textbf{AUROC} $\uparrow$ & \textbf{FPR95} $\downarrow$ & \textbf{AUROC} $\uparrow$ \\
        \midrule
            \multirow{24}{*}{CIFAR-10}
            & MSP                        & 50.39 & 93.14 & 61.96 & 88.84 & 50.92 & 93.16 & 55.20 & 91.42 & 26.45 & 96.47 & 50.17 & 93.24 & 49.18 & 92.71 \\
            & + $\approach(m)$           & 52.78 & 88.33 & 67.64 & 79.64 & 56.22 & 88.30 & 53.99 & 87.43 & 34.41 & 93.66 & 55.69 & 88.65 & 53.46 & 87.67 \\
            & + $\approach(\mu, \sigma)$ & 51.87 & 92.81 & 67.41 & 86.63 & 57.01 & 92.50 & 55.55 & 91.51 & 31.12 & 95.94 & 56.23 & 92.42 & 53.20 & 91.97 \\
            \cmidrule(lr){2-16}
            
            & ODIN                       & 23.13 & 96.34 & 36.43 & 92.41 &  8.32 & 98.32 & 31.37 & 94.67 & 1.89 & 99.50 & 7.86 & 98.39 & 18.17 & 96.61 \\
            & + $\approach(m)$           & 23.55 & 96.25 & 38.39 & 92.21 & 10.44 & 97.97 & 32.80 & 94.51 & 2.40 & 99.33 & 9.30 & 98.06 & 19.48 & 96.39 \\
            & + $\approach(\mu, \sigma)$ & 23.54 & 96.28 & 38.32 & 92.23 &  9.78 & 98.08 & 32.29 & 94.55 & 2.26 & 99.38 & 8.87 & 98.16 & 19.18 & 96.45 \\
            \cmidrule(lr){2-16}
            
            & Energy                     & 25.86 & 95.92 & 37.20 & 92.06 & 14.97 & 97.38 & 35.00 & 93.97 & 2.42 & 99.35 & 13.29 & 97.53 & 21.46 & 96.04 \\
            & + $\approach(m)$           & 17.87 & 96.89 & 38.30 & 92.75 & 10.24 & 98.14 & 14.56 & 97.55 & 3.54 & 99.22 & 10.11 & 98.13 & 15.77 & 97.11 \\
            & + $\approach(\mu, \sigma)$ & 16.92 & 97.07 & 39.17 & 92.27 & 11.05 & 97.99 & 16.70 & 97.27 & 2.59 & 99.36 & 10.98 & 97.97 & 16.23 & 96.99 \\
            \cmidrule(lr){2-16}
            
            & GradNorm                   & 31.62 & 93.93 & 43.38 & 90.62 & 19.17 & 96.93 & 43.01 & 91.10 & 3.25 & 99.14 & 18.16 & 97.00 & 26.43 & 94.79 \\
            & + $\approach(m)$           &  7.54 & 98.53 & 46.45 & 89.94 &  6.30 & 98.72 &  9.04 & 98.24 & 0.60 & 99.83 &  7.45 & 98.56 & 12.90 & 97.31 \\
            & + $\approach(\mu, \sigma)$ &  7.88 & 98.50 & 48.54 & 89.59 &  7.15 & 98.61 & 10.41 & 98.04 & 0.33 & 99.89 &  8.73 & 98.38 & 13.84 & 97.17 \\
            \cmidrule(lr){2-16}
            
            & KNN                        & 12.71 & 97.58 & 43.21 & 90.47 & 35.81 & 94.35 & 31.79 & 94.4 & 2.89 & 99.34 & 33.56 & 94.73 & 26.66 & 95.14 \\
            & + $\approach(m)$           & 16.23 & 97.25 & 38.8 & 92.25 & 22.03 & 96.38 & 22.43 & 95.85 & 4.62 & 99.17 & 20.57 & 96.58 & 20.78 & 96.25 \\
            & + $\approach(\mu, \sigma)$ & 13.49 & 97.59 & 38.06 & 92.06 & 23.29 & 96.13 & 23.16 & 95.74 & 4.54 & 99.16 & 21.06 & 96.37 & 20.6 & 96.17 \\
            \cmidrule(lr){2-16}
            
            & ReAct                      & 23.34 & 96.30 & 35.31 & 92.63 & 12.39 & 97.61 & 29.27 & 94.89 & 2.80 & 99.28 & 10.85 & 97.83 & 18.99 & 96.42 \\
            & + $\approach(m)$           & 19.13 & 96.70 & 39.67 & 92.23 & 11.52 & 97.98 & 16.19 & 97.36 & 3.96 & 99.15 & 11.23 & 97.98 & 16.95 & 96.90 \\
            & + $\approach(\mu, \sigma)$ & 17.52 & 96.97 & 40.27 & 91.89 & 12.04 & 97.87 & 17.39 & 97.12 & 2.91 & 99.32 & 11.69 & 97.87 & 16.97 & 96.84 \\
            \cmidrule(lr){2-16}
            
            & DICE                       & 14.95 & 97.56 & 41.83 & 90.85 & 8.62 & 98.36 & 23.23 & 95.78 & 0.31 & 99.91 & 10.00 & 98.13 & 16.49 & 96.76\\
            & + $\approach(m)$           &  8.81 & 98.31 & 35.53 & 92.91 & 4.95 & 98.99 &  7.27 & 98.67 & 0.99 & 99.77 &  5.87 & 98.86 & 10.57 & 97.92\\
            & + $\approach(\mu, \sigma)$ &  8.80 & 98.39 & 38.67 & 92.32 & 6.39 & 98.82 &  9.24 & 98.43 & 0.76 & 99.84 &  7.44 & 98.66 & 11.88 & 97.74\\
            \cmidrule(lr){2-16}
            
            & ASH-S                      & 11.31 & 98.05 & 47.16 & 89.65 & 9.13 & 98.29 & 21.77 & 96.29 & 0.27 & 99.92 & 10.99 & 98.04 & 16.77 & 96.71 \\
            & + $\approach(m)$           &  8.16 & 98.44 & 40.97 & 91.67 & 8.21 & 98.46 &  9.06 & 98.34 & 1.28 & 99.69 &  8.76 & 98.36 & 12.74 & 97.49 \\
            & + $\approach(\mu, \sigma)$ &  7.17 & 98.64 & 42.48 & 91.10 & 8.86 & 98.40 &  9.86 & 98.24 & 0.98 & 99.77 &  9.18 & 98.30 & 13.09 & 97.41 \\
            \cmidrule(lr){2-16}
            
            & SCALE                      & 9.53 & 98.32 & 45.92 & 90.26 & 8.61 & 98.40 & 20.21 & 96.63 & 0.28 & 99.93 & 10.00 & 98.21 & 15.76 & 96.96 \\
            & + $\approach(m)$           & 7.85 & 98.51 & 42.30 & 91.35 & 7.16 & 98.61 &  9.52 & 98.30 & 1.06 & 99.73 &  8.26 & 98.50 & 12.69 & 97.50 \\
            & + $\approach(\mu, \sigma)$ & 7.05 & 98.62 & 44.12 & 90.81 & 8.12 & 98.49 & 10.21 & 98.13 & 0.84 & 99.78 &  8.80 & 98.38 & 13.19 & 97.37 \\
        \bottomrule
        \end{tabular}}
        \label{table: ResNet-18_GeLU_CIFAR-10_combination}
    \end{table}
\end{landscape}

\begin{landscape}
    \begin{table}[!htbp]
        \centering
        \caption{Detailed results of post-hoc methods combined with $\approach$ using \textbf{ResNet-18} with the \textbf{GeLU} activation function instead of \textbf{ReLU}, trained on the CIFAR dataset. $\boldsymbol{\uparrow}$ indicates that higher values are better; $\boldsymbol{\downarrow}$ indicates that lower values are better. The symbols denote the statistics used: $\mu$ (mean), $\sigma$ (standard deviation), and $m$ (maximum).}
        \resizebox{\linewidth}{!}{
        \begin{tabular}{c l  cc cc cc cc cc cc cc }
        \toprule
        \multirow{2}{*}{\textbf{Dataset}} & \multirow{2}{*}{\textbf{Combined Method}} & \multicolumn{2}{c}{\textbf{SVHN}} & \multicolumn{2}{c}{\textbf{Place365}} & \multicolumn{2}{c}{\textbf{iSUN}} & \multicolumn{2}{c}{\textbf{Textures}} & \multicolumn{2}{c}{\textbf{LSUN-c}}  & \multicolumn{2}{c}{\textbf{LSUN-r}} & \multicolumn{2}{c}{\textbf{Average}}  \\
        \cmidrule(lr){3-4} \cmidrule(lr){5-6} \cmidrule(lr){7-8} \cmidrule(lr){9-10} \cmidrule(lr){11-12} \cmidrule(lr){13-14} \cmidrule(lr){15-16}
        && \textbf{FPR95} $\downarrow$ & \textbf{AUROC} $\uparrow$ & \textbf{FPR95} $\downarrow$ & \textbf{AUROC} $\uparrow$ & \textbf{FPR95} $\downarrow$ & \textbf{AUROC} $\uparrow$ & \textbf{FPR95} $\downarrow$ & \textbf{AUROC} $\uparrow$ & \textbf{FPR95} $\downarrow$ & \textbf{AUROC} $\uparrow$ & \textbf{FPR95} $\downarrow$ & \textbf{AUROC} $\uparrow$ & \textbf{FPR95} $\downarrow$ & \textbf{AUROC} $\uparrow$ \\
        \midrule
            \multirow{24}{*}{CIFAR-100}
            & MSP                        & 78.62 & 79.06 & 81.69 & 75.92 & 84.16 & 73.23 & 84.38 & 74.94 & 66.96 & 84.56 & 83.59 & 73.34 & 79.9 & 76.84 \\
            & + $\approach(m)$           & 81.77 & 81.91 & 87.51 & 73.27 & 86.85 & 73.49 & 85.12 & 77.90 & 80.28 & 81.81 & 87.97 & 72.51 & 84.92 & 76.82 \\
            & + $\approach(\mu, \sigma)$ & 81.07 & 81.10 & 85.23 & 74.62 & 86.30 & 73.61 & 84.98 & 77.09 & 76.13 & 83.08 & 86.87 & 73.10 & 83.43 & 77.10 \\
            \cmidrule(lr){2-16}
            
            & ODIN                       & 76.63 & 84.20 & 75.95 & 79.12 & 63.44 & 85.05 & 83.17 & 77.11 & 30.34 & 94.69 & 60.94 & 85.72 & 65.08 & 84.32 \\
            & + $\approach(m)$           & 73.18 & 84.88 & 75.75 & 79.18 & 65.55 & 84.51 & 82.52 & 77.58 & 33.33 & 94.21 & 63.40 & 85.15 & 65.62 & 84.25 \\
            & + $\approach(\mu, \sigma)$ & 73.16 & 84.79 & 75.60 & 79.19 & 64.26 & 84.83 & 82.02 & 77.59 & 31.66 & 94.47 & 62.03 & 85.47 & 64.79 & 84.39 \\
            \cmidrule(lr){2-16}
            
            & Energy                     & 74.79 & 85.11 & 77.58 & 78.51 & 73.87 & 81.18 & 85.83 & 76.06 & 32.66 & 94.19 & 71.33 & 82.03 & 69.34 & 82.85 \\
            & + $\approach(m)$           & 39.80 & 93.86 & 74.48 & 79.88 & 56.84 & 88.89 & 49.41 & 90.85 & 26.70 & 95.47 & 59.41 & 88.07 & 51.11 & 89.50 \\
            & + $\approach(\mu, \sigma)$ & 44.76 & 93.16 & 76.21 & 78.94 & 61.28 & 87.57 & 55.71 & 89.46 & 26.89 & 95.40 & 62.89 & 86.94 & 54.62 & 88.58 \\
            \cmidrule(lr){2-16}
            
            & GradNorm                   & 83.48 & 70.22 & 84.86 & 67.56 & 84.71 & 71.84 & 90.46 & 59.33 &  6.19 & 98.76 & 84.54 & 70.94 & 72.37 & 73.11 \\
            & + $\approach(m)$           & 51.13 & 90.98 & 95.22 & 58.24 & 78.12 & 81.94 & 44.52 & 87.52 & 18.82 & 96.87 & 85.28 & 78.77 & 62.18 & 82.39 \\
            & + $\approach(\mu, \sigma)$ & 61.23 & 88.79 & 95.70 & 56.37 & 81.71 & 79.26 & 51.88 & 84.95 & 16.53 & 97.26 & 87.67 & 75.99 & 65.79 & 80.44 \\
            \cmidrule(lr){2-16}
            
            & KNN                        & 58.46 & 89.78 & 78.34 & 75.19 & 59.98 & 83.86 & 57.39 & 87.28 & 61.03 & 81.10 & 57.49 & 85.18 & 62.12 & 83.73 \\
            & + $\approach(m)$           & 27.92 & 95.10 & 77.89 & 77.51 & 42.02 & 90.91 & 28.62 & 94.18 & 44.33 & 88.11 & 42.50 & 91.18 & 43.88 & 89.50 \\
            & + $\approach(\mu, \sigma)$ & 29.13 & 94.91 & 77.52 & 76.88 & 43.72 & 89.98 & 29.57 & 93.94 & 50.41 & 86.15 & 43.96 & 90.42 & 45.72 & 88.71 \\
            \cmidrule(lr){2-16}
            
            & ReAct                      & 53.99 & 91.05 & 75.44 & 79.48 & 61.92 & 86.70 & 62.89 & 87.43 & 40.12 & 92.50 & 60.44 & 86.84 & 59.13 & 87.33 \\
            & + $\approach(m)$           & 37.69 & 94.01 & 75.32 & 79.83 & 58.20 & 88.29 & 48.71 & 90.72 & 30.72 & 94.51 & 60.67 & 87.41 & 51.88 & 89.13 \\
            & + $\approach(\mu, \sigma)$ & 39.50 & 93.76 & 76.06 & 79.51 & 59.92 & 87.77 & 51.65 & 90.21 & 30.35 & 94.64 & 62.55 & 87.00 & 53.34 & 88.81 \\
            \cmidrule(lr){2-16}
            
            & DICE                       & 60.69 & 87.90 & 77.47 & 78.15 & 65.79 & 84.57 & 76.76 & 78.21 & 3.72 & 99.23 & 65.52 & 84.68 & 58.32 & 85.46 \\
            & + $\approach(m)$           & 17.29 & 96.79 & 77.85 & 79.15 & 36.28 & 93.76 & 25.73 & 94.80 & 4.70 & 99.04 & 42.54 & 92.30 & 34.06 & 92.64 \\
            & + $\approach(\mu, \sigma)$ & 19.65 & 96.36 & 76.64 & 78.57 & 40.40 & 92.61 & 30.48 & 93.86 & 3.76 & 99.21 & 45.66 & 91.36 & 36.10 & 91.99 \\
            \cmidrule(lr){2-16}
            
            & ASH-S                      & 41.09 & 93.22 & 78.03 & 77.23 & 62.52 & 85.13 & 60.82 & 86.32 & 4.50 & 99.11 & 63.38 & 84.62 & 51.72 & 87.61 \\
            & + $\approach(m)$           & 10.47 & 97.97 & 77.00 & 77.86 & 42.64 & 91.57 & 23.17 & 95.47 & 8.56 & 98.42 & 47.85 & 90.26 & 34.95 & 91.93 \\
            & + $\approach(\mu, \sigma)$ & 11.67 & 97.86 & 77.15 & 77.59 & 45.14 & 90.84 & 25.05 & 95.17 & 8.91 & 98.42 & 49.04 & 89.83 & 36.16 & 91.62 \\
            \cmidrule(lr){2-16}
            
            & SCALE                      & 34.69 & 94.34 & 78.92 & 76.39 & 60.09 & 85.95 & 51.83 & 88.85 & 4.40 & 99.14 & 62.80 & 84.79 & 48.79 & 88.24 \\
            & + $\approach(m)$           & 11.35 & 97.82 & 82.25 & 75.24 & 48.31 & 89.84 & 24.34 & 94.69 & 9.45 & 98.24 & 56.41 & 87.72 & 38.69 & 90.59 \\
            & + $\approach(\mu, \sigma)$ & 11.44 & 97.83 & 81.38 & 75.53 & 48.81 & 89.58 & 24.34 & 94.69 & 9.08 & 98.31 & 55.45 & 87.65 & 38.42 & 90.60 \\
        \bottomrule
        \end{tabular}}
        \label{table: ResNet-18_GeLU_CIFAR-100_combination}
    \end{table}
\end{landscape}

\begin{landscape}
    \begin{table}[!htbp]
        \centering
        \caption{Detailed results of post-hoc methods combined with $\approach$ using \textbf{ResNet-18} with the \textbf{TanH} activation function instead of \textbf{ReLU}, trained on the CIFAR dataset. $\boldsymbol{\uparrow}$ indicates that higher values are better; $\boldsymbol{\downarrow}$ indicates that lower values are better. The symbols denote the statistics used: $\mu$ (mean), $\sigma$ (standard deviation), and $m$ (maximum).}
        \resizebox{\linewidth}{!}{
        \begin{tabular}{c l  cc cc cc cc cc cc cc }
        \toprule
        \multirow{2}{*}{\textbf{Dataset}} & \multirow{2}{*}{\textbf{Combined Method}} & \multicolumn{2}{c}{\textbf{SVHN}} & \multicolumn{2}{c}{\textbf{Place365}} & \multicolumn{2}{c}{\textbf{iSUN}} & \multicolumn{2}{c}{\textbf{Textures}} & \multicolumn{2}{c}{\textbf{LSUN-c}}  & \multicolumn{2}{c}{\textbf{LSUN-r}} & \multicolumn{2}{c}{\textbf{Average}}  \\
        \cmidrule(lr){3-4} \cmidrule(lr){5-6} \cmidrule(lr){7-8} \cmidrule(lr){9-10} \cmidrule(lr){11-12} \cmidrule(lr){13-14} \cmidrule(lr){15-16}
        && \textbf{FPR95} $\downarrow$ & \textbf{AUROC} $\uparrow$ & \textbf{FPR95} $\downarrow$ & \textbf{AUROC} $\uparrow$ & \textbf{FPR95} $\downarrow$ & \textbf{AUROC} $\uparrow$ & \textbf{FPR95} $\downarrow$ & \textbf{AUROC} $\uparrow$ & \textbf{FPR95} $\downarrow$ & \textbf{AUROC} $\uparrow$ & \textbf{FPR95} $\downarrow$ & \textbf{AUROC} $\uparrow$ & \textbf{FPR95} $\downarrow$ & \textbf{AUROC} $\uparrow$ \\
        \midrule
            \multirow{24}{*}{CIFAR-10}
            & MSP                        & 53.49 & 93.72 & 61.64 & 88.73 & 49.78 & 93.26 & 58.19 & 90.73 & 19.17 & 97.33 & 49.12 & 93.40 & 48.57 & 92.86 \\
            & + $\approach(m)$           & 43.20 & 92.18 & 64.55 & 83.09 & 56.37 & 89.17 & 53.85 & 88.57 & 28.87 & 94.98 & 54.52 & 89.81 & 50.23 & 89.64 \\
            & + $\approach(\mu, \sigma)$ & 41.65 & 94.40 & 64.23 & 87.51 & 53.60 & 92.48 & 53.83 & 91.45 & 23.07 & 96.69 & 52.59 & 92.63 & 48.16 & 92.53 \\
            \cmidrule(lr){2-16}
            
            & ODIN                       & 19.57 & 96.86 & 36.54 & 92.47 & 12.83 & 97.80 & 35.83 & 93.70 & 0.90 & 99.68 & 10.13 & 98.04 & 19.30 & 96.42 \\
            & + $\approach(m)$           & 20.67 & 96.69 & 38.16 & 92.32 & 15.31 & 97.43 & 36.10 & 93.67 & 1.26 & 99.55 & 12.69 & 97.68 & 20.70 & 96.22 \\
            & + $\approach(\mu, \sigma)$ & 20.02 & 96.77 & 37.78 & 92.34 & 14.39 & 97.56 & 35.80 & 93.70 & 1.09 & 99.60 & 11.67 & 97.80 & 20.12 & 96.30 \\
            \cmidrule(lr){2-16}
            
            & Energy                     & 29.09 & 95.84 & 37.49 & 92.25 & 24.64 & 96.32 & 42.41 & 92.42 & 1.38 & 99.50 & 20.81 & 96.77 & 25.97 & 95.52 \\
            & + $\approach(m)$           & 16.55 & 97.12 & 37.24 & 93.03 & 12.73 & 97.81 & 19.47 & 96.94 & 4.10 & 99.20 & 11.89 & 97.89 & 17.00 & 97.00 \\
            & + $\approach(\mu, \sigma)$ & 14.29 & 97.43 & 37.45 & 92.57 & 12.63 & 97.71 & 19.47 & 96.74 & 2.61 & 99.41 & 11.77 & 97.83 & 16.37 & 96.95 \\
            \cmidrule(lr){2-16}
            
            & GradNorm                   & 34.70 & 94.91 & 43.81 & 91.25 & 33.60 & 94.86 & 50.94 & 88.27 & 2.71 & 99.12 & 29.51 & 95.48 & 32.55 & 93.98 \\
            & + $\approach(m)$           &  7.29 & 98.60 & 42.53 & 90.86 & 11.19 & 97.94 & 11.33 & 97.82 & 0.49 & 99.84 & 14.37 & 97.42 & 14.53 & 97.08 \\
            & + $\approach(\mu, \sigma)$ &  6.14 & 98.78 & 46.15 & 90.23 & 13.65 & 97.51 & 13.97 & 97.37 & 0.21 & 99.90 & 16.90 & 97.04 & 16.17 & 96.80 \\
            \cmidrule(lr){2-16}
            
            & KNN                        & 16.20 & 97.10 & 41.11 & 91.36 & 36.67 & 94.47 & 37.11 & 93.69 & 1.83 & 99.55 & 32.14 & 95.05 & 27.51 & 95.2 \\
            & + $\approach(m)$           & 32.90 & 94.79 & 40.34 & 92.10 & 28.68 & 95.21 & 39.89 & 92.95 & 7.11 & 98.72 & 22.87 & 96.09 & 28.63 & 94.98 \\
            & + $\approach(\mu, \sigma)$ & 19.48 & 96.68 & 38.55 & 92.25 & 23.93 & 95.99 & 31.81 & 94.60 & 4.41 & 99.25 & 19.64 & 96.57 & 22.97 & 95.89 \\
            \cmidrule(lr){2-16}
            
            & ReAct                      & 29.34 & 95.83 & 37.00 & 92.41 & 23.99 & 96.39 & 41.70 & 92.86 & 1.44 & 99.49 & 20.07 & 96.82 & 25.59 & 95.63 \\
            & + $\approach(m)$           & 16.44 & 97.06 & 37.23 & 92.75 & 13.10 & 97.70 & 19.66 & 96.79 & 4.19 & 99.16 & 12.15 & 97.79 & 17.13 & 96.88\\
            & + $\approach(\mu, \sigma)$ & 14.39 & 97.41 & 38.21 & 92.32 & 13.36 & 97.62 & 20.27 & 96.62 & 2.69 & 99.39 & 12.33 & 97.74 & 16.87 & 96.85 \\
            \cmidrule(lr){2-16}
            
            & DICE                       & 13.35 & 97.73 & 43.25 & 90.76 & 14.55 & 97.29 & 31.61 & 94.02 & 0.13 & 99.96 & 14.28 & 97.31 & 19.53 & 96.18 \\
            & + $\approach(m)$           &  7.22 & 98.62 & 34.82 & 92.92 &  6.48 & 98.68 &  8.78 & 98.39 & 0.77 & 99.81 &  7.68 & 98.48 & 10.96 & 97.82 \\
            & + $\approach(\mu, \sigma)$ &  6.37 & 98.80 & 38.80 & 92.27 &  8.37 & 98.46 & 11.97 & 98.06 & 0.41 & 99.88 &  9.15 & 98.31 & 12.51 & 97.63 \\
            \cmidrule(lr){2-16}
            
            & ASH-S                      & 14.25 & 97.57 & 49.08 & 89.60 & 23.55 & 96.22 & 34.04 & 93.92 & 0.08 & 99.96 & 23.28 & 96.19 & 24.05 & 95.58 \\
            & + $\approach(m)$           & 13.57 & 97.61 & 37.91 & 92.90 & 13.14 & 97.77 & 16.65 & 97.35 & 3.07 & 99.37 & 12.86 & 97.80 & 16.20 & 97.14 \\
            & + $\approach(\mu, \sigma)$ & 11.48 & 97.88 & 38.56 & 92.44 & 13.59 & 97.63 & 17.46 & 97.12 & 2.02 & 99.52 & 12.87 & 97.70 & 16.00 & 97.05 \\
            \cmidrule(lr){2-16}
            
            & SCALE                      & 15.22 & 97.37 & 46.92 & 90.28 & 24.95 & 96.00 & 35.98 & 93.51 & 0.09 & 99.96 & 23.84 & 96.06 & 24.50 & 95.53 \\
            & + $\approach(m)$           & 11.05 & 98.01 & 37.99 & 92.68 & 13.69 & 97.68 & 14.88 & 97.51 & 2.09 & 99.52 & 13.26 & 97.67 & 15.49 & 97.18 \\
            & + $\approach(\mu, \sigma)$ &  9.76 & 98.19 & 39.75 & 92.12 & 15.09 & 97.49 & 15.78 & 97.35 & 1.31 & 99.63 & 14.55 & 97.50 & 16.04 & 97.05 \\
        \bottomrule
        \end{tabular}}
        \label{table: ResNet-18_TanH_CIFAR-10_combination}
    \end{table}
\end{landscape}

\begin{landscape}
    \begin{table}[!htbp]
        \centering
        \caption{Detailed results of post-hoc methods combined with $\approach$ using \textbf{ResNet-18} with the \textbf{TanH} activation function instead of \textbf{ReLU}, trained on the CIFAR dataset. $\boldsymbol{\uparrow}$ indicates that higher values are better; $\boldsymbol{\downarrow}$ indicates that lower values are better. The symbols denote the statistics used: $\mu$ (mean), $\sigma$ (standard deviation), and $m$ (maximum).}
        \resizebox{\linewidth}{!}{
        \begin{tabular}{c l  cc cc cc cc cc cc cc }
        \toprule
        \multirow{2}{*}{\textbf{Dataset}} & \multirow{2}{*}{\textbf{Combined Method}} & \multicolumn{2}{c}{\textbf{SVHN}} & \multicolumn{2}{c}{\textbf{Place365}} & \multicolumn{2}{c}{\textbf{iSUN}} & \multicolumn{2}{c}{\textbf{Textures}} & \multicolumn{2}{c}{\textbf{LSUN-c}}  & \multicolumn{2}{c}{\textbf{LSUN-r}} & \multicolumn{2}{c}{\textbf{Average}}  \\
        \cmidrule(lr){3-4} \cmidrule(lr){5-6} \cmidrule(lr){7-8} \cmidrule(lr){9-10} \cmidrule(lr){11-12} \cmidrule(lr){13-14} \cmidrule(lr){15-16}
        && \textbf{FPR95} $\downarrow$ & \textbf{AUROC} $\uparrow$ & \textbf{FPR95} $\downarrow$ & \textbf{AUROC} $\uparrow$ & \textbf{FPR95} $\downarrow$ & \textbf{AUROC} $\uparrow$ & \textbf{FPR95} $\downarrow$ & \textbf{AUROC} $\uparrow$ & \textbf{FPR95} $\downarrow$ & \textbf{AUROC} $\uparrow$ & \textbf{FPR95} $\downarrow$ & \textbf{AUROC} $\uparrow$ & \textbf{FPR95} $\downarrow$ & \textbf{AUROC} $\uparrow$ \\
        \midrule
            \multirow{24}{*}{CIFAR-100}
            & MSP                        & 87.05 & 69.51 & 82.31 & 74.45 & 88.00 & 69.17 & 84.47 & 73.04 & 59.67 & 86.64 & 87.92 & 69.80 & 81.57 & 73.77\\
            & + $\approach(m)$           & 85.05 & 78.24 & 88.26 & 71.15 & 88.24 & 72.21 & 84.82 & 75.83 & 77.24 & 82.00 & 88.32 & 72.33 & 85.32 & 75.29\\
            & + $\approach(\mu, \sigma)$ & 84.97 & 76.85 & 86.16 & 72.88 & 88.50 & 71.48 & 83.53 & 75.29 & 70.09 & 84.12 & 87.96 & 71.91 & 83.54 & 75.42\\
            \cmidrule(lr){2-16}
            
            & ODIN                       & 85.98 & 77.18 & 77.28 & 78.75 & 68.37 & 84.02 & 82.66 & 77.37 & 17.27 & 97.19 & 65.69 & 84.94 & 66.21 & 83.24\\
            & + $\approach(m)$           & 83.22 & 79.16 & 78.11 & 78.74 & 72.36 & 83.08 & 82.38 & 77.62 & 22.03 & 96.50 & 70.27 & 83.93 & 68.06 & 83.17\\
            & + $\approach(\mu, \sigma)$ & 83.36 & 78.97 & 78.27 & 78.72 & 71.17 & 83.52 & 81.84 & 77.70 & 20.87 & 96.75 & 69.08 & 84.36 & 67.43 & 83.34\\
            \cmidrule(lr){2-16}
            
            & Energy                     & 82.96 & 78.79 & 79.05 & 77.46 & 79.25 & 79.29 & 86.19 & 75.27 & 19.44 & 96.73 & 77.50 & 80.44 & 70.73 & 81.33 \\
            & + $\approach(m)$           & 48.04 & 92.36 & 77.03 & 79.10 & 51.16 & 90.82 & 41.29 & 91.10 & 15.32 & 97.42 & 52.94 & 90.28 & 47.63 & 90.18 \\
            & + $\approach(\mu, \sigma)$ & 52.05 & 91.70 & 77.59 & 78.19 & 57.19 & 89.15 & 46.95 & 89.61 & 13.58 & 97.73 & 58.71 & 88.77 & 51.01 & 89.19 \\
            \cmidrule(lr){2-16}
            
            & GradNorm                   & 87.31 & 67.27 & 82.41 & 69.68 & 88.90 & 63.84 & 90.98 & 55.90 & 3.23 & 99.17 & 88.58 & 65.11 & 73.57 & 70.16\\
            & + $\approach(m)$           & 28.99 & 94.78 & 92.88 & 61.29 & 62.14 & 87.77 & 33.72 & 91.86 & 6.04 & 98.81 & 70.03 & 85.68 & 48.97 & 86.70\\
            & + $\approach(\mu, \sigma)$ & 34.47 & 93.98 & 93.66 & 58.40 & 69.47 & 84.67 & 39.31 & 90.18 & 4.14 & 99.14 & 76.60 & 82.20 & 52.94 & 84.76\\
            \cmidrule(lr){2-16}
            
            & KNN                        & 90.20 & 71.76 & 83.04 & 72.52 & 87.26 & 68.00 & 83.46 & 76.37 & 69.18 & 72.71 & 87.58 & 68.44 & 83.45 & 71.63\\
            & + $\approach(m)$           & 28.92 & 94.90 & 84.66 & 72.70 & 60.27 & 84.74 & 26.33 & 94.53 & 20.60 & 94.10 & 66.13 & 83.65 & 47.82 & 87.44\\
            & + $\approach(\mu, \sigma)$ & 38.59 & 92.76 & 82.74 & 72.80 & 63.08 & 81.85 & 29.29 & 93.74 & 30.99 & 90.29 & 67.75 & 81.03 & 52.07 & 85.41\\
            \cmidrule(lr){2-16}
            
            & ReAct                      & 79.97 & 82.33 & 78.65 & 77.59 & 73.64 & 83.33 & 78.44 & 82.16 & 20.19 & 96.52 & 70.80 & 84.20 & 66.95 & 84.35\\
            & + $\approach(m)$           & 48.28 & 92.11 & 76.71 & 78.98 & 51.37 & 90.42 & 41.19 & 91.08 & 16.19 & 97.22 & 53.47 & 89.93 & 47.87 & 89.96\\
            & + $\approach(\mu, \sigma)$ & 52.28 & 91.65 & 77.77 & 78.28 & 57.65 & 89.06 & 46.56 & 89.95 & 14.70 & 97.56 & 58.84 & 88.75 & 51.30 & 89.21\\
            \cmidrule(lr){2-16}
            
            & DICE                       & 55.30 & 87.99 & 83.98 & 73.83 & 74.80 & 81.24 & 66.67 & 79.64 & 1.01 & 99.72 & 76.67 & 80.71 & 59.74 & 83.86\\
            & + $\approach(m)$           & 14.92 & 97.15 & 82.58 & 76.28 & 36.68 & 93.69 & 24.82 & 94.69 & 2.53 & 99.35 & 43.35 & 92.54 & 34.15 & 92.28\\
            & + $\approach(\mu, \sigma)$ & 15.70 & 97.04 & 82.08 & 75.64 & 42.48 & 92.24 & 28.10 & 93.66 & 1.47 & 99.59 & 48.63 & 91.00 & 36.41 & 91.53\\
            \cmidrule(lr){2-16}
            
            & ASH-S                      & 52.40 & 89.78 & 83.15 & 73.83 & 75.43 & 80.23 & 66.01 & 81.85 & 1.29 & 99.70 & 78.15 & 79.26 & 59.40 & 84.11 \\
            & + $\approach(m)$           & 20.50 & 96.47 & 78.10 & 77.64 & 40.16 & 92.70 & 24.96 & 94.85 & 6.94 & 98.72 & 44.92 & 91.68 & 35.93 & 92.01 \\
            & + $\approach(\mu, \sigma)$ & 24.53 & 96.06 & 79.26 & 76.96 & 46.52 & 91.46 & 29.34 & 93.98 & 6.69 & 98.81 & 49.97 & 90.55 & 39.39 & 91.30 \\
            \cmidrule(lr){2-16}
            
            & SCALE                      & 47.73 & 91.52 & 86.23 & 71.03 & 72.32 & 83.64 & 55.98 & 86.49 & 1.05 & 99.74 & 75.08 & 82.30 & 56.40 & 85.79 \\
            & + $\approach(m)$           & 12.76 & 97.62 & 85.50 & 72.25 & 39.48 & 92.88 & 22.64 & 95.25 & 4.89 & 99.02 & 47.47 & 91.51 & 35.46 & 91.42 \\
            & + $\approach(\mu, \sigma)$ & 12.30 & 97.71 & 85.11 & 72.28 & 39.62 & 92.84 & 22.87 & 95.29 & 4.63 & 99.10 & 46.39 & 91.55 & 35.15 & 91.46 \\
        \bottomrule
        \end{tabular}}
        \label{table: ResNet-18_TanH_CIFAR-100_combination}
    \end{table}
\end{landscape}

\section{ Accuracy and Computational Overhead}
\label{appendix: accuracy_analysis}

    \textbf{Accuracy.}
    Directly using the enriched features $h^{\approach}(\mathbf{x})$ for classification may slightly degrade in-distribution (ID) accuracy. As shown in Table~\ref{table: cifar_accuracy} for CIFAR-trained models, Table~\ref{table: imagenet_accuracy} for traditional CNN-based ImageNet models, and Table~\ref{table: imagenet_swin_convnext_accuracy} for Swin-B and ConvNeXt-B, the accuracy drop is marginal.  and not statistically significant.
    
    As a post-hoc method, $\approach$ is deployed in a standard two-branch pipeline, where the original, unmodified features are always used for final classification of samples identified as ID, thereby preserving the base model’s accuracy without incurring additional cost.

     \noindent
     \textbf{Computational Overhead.}
     Computing OOD scores with $\approach$ modifies the original \(h(\mathbf{x}) \) (as detailed in Section~\ref{sec: methods}) and introduces negligible overhead -- only a single matrix multiplication for the OOD head, which is trivial compared to the GFLOPs of the main CNN backbone. For instance, in ResNet-50 on ImageNet, the backbone requires ~5.42 GFLOPs, while the OOD head adds only ~0.004 GFLOPs, i.e., less than 0.1\% of the total cost.

     \begin{table}[!ht]
        \centering
        \caption{ID classification accuracy (\%) under $\approach$ for CIFAR datasets using ResNet-18, ResNet-34, DenseNet-101, Wide-ResNet-28-10, and MobileNet-v2 architectures.}
        \resizebox{\linewidth}{!}{
        \begin{tabular}{ l l c| c| c| c| c}
    
        \toprule
        \textbf{Dataset} & \textbf{Method} & \textbf{ResNet-18} & \textbf{ResNet-34} & \textbf{DenseNet-101}  & \textbf{Wide-ResNet} & \textbf{MobileNet-v2} \\
        \midrule
        \multirow{3}{*}{CIFAR-10} 
        & \texttt{Standard}                          & 93.89 & 93.96 & 93.61 & 95.13 & 92.52  \\
        & $\approach(m) + \texttt{DICE}$             & 93.23 & 93.39 & 92.53 & 94.46 & 89.21  \\
        & $\approach({\mu,\sigma}) + \texttt{DICE} $ & 93.37 & 93.68 & 93.36 & 94.21 & 88.65  \\
        \midrule
        \multirow{3}{*}{CIFAR-100} 
        & \texttt{Standard}                          & 75.20 & 75.70 & 74.47 & 77.96 & 72.69 \\
        & $\approach(m) + \texttt{DICE}$             & 71.50 & 73.15 & 62.23 & 75.23 & 69.54 \\
        & $\approach({\mu,\sigma}) + \texttt{DICE} $ & 72.50 & 73.81 & 68.47 & 74.31 & 70.50 \\
        \bottomrule
        \end{tabular}}
        \label{table: cifar_accuracy}
    \end{table}

    \begin{table}[!ht]
        \centering
         \caption{ID classification accuracy (\%) under $\approach$ for ImageNet datasets using Swin-B and ConvNeXt-B architectures.}
        \begin{tabular}{l c|c }
        \toprule
        \textbf{Method} & \textbf{Swin-B} & \textbf{ConvNeXt-B}      \\
        \midrule
        \texttt{Standard}                          & 83.07 & 83.48   \\
        $\approach(m) + \texttt{KNN}$              & 82.52 & 83.21   \\
        $\approach({\mu,\sigma}) + \texttt{KNN}$   & 79.38 & 74.60   \\
        \bottomrule
        \end{tabular}
        \label{table: imagenet_swin_convnext_accuracy}
    \end{table}

     \begin{table}[!ht]
        \centering
         \caption{ID classification accuracy (\%) under $\approach$ for ImageNet datasets using DenseNet-121, ResNet-50, MobileNet-v2, and EfficientNet-b0 architectures.}
        \begin{tabular}{l c|c|c|c}
        \toprule
        \textbf{Method}  & \textbf{DenseNet-121} & \textbf{ResNet-50} & \textbf{MobileNet-v2}  & \textbf{EfficientNet-b0}\\
        \midrule
        \texttt{Standard}                           & 74.44 & 76.15 & 71.87 & 77.67  \\
        $\approach(m) + \texttt{SCALE}$             & 74.34 & 76.15 & 71.75 & 77.50  \\
        $\approach({\mu,\sigma}) + \texttt{SCALE}$  & 71.94 & 73.48 & 69.18 & 74.87  \\
        \bottomrule
        \end{tabular}
        \label{table: imagenet_accuracy}
    \end{table}

\section{Reproducibility Statement}
\label{appendix: reproducibility}

We are committed to ensuring the reproducibility of our research. To this end, we provide detailed information regarding our code, experimental setup, hyperparameter selection, and computational environment.

\noindent
\textbf{Code and Data Availability.} 
The complete source code for our method, \approach, along with the scripts used to run all experiments and generate figures, will be made publicly available on GitHub\footnote{https://github.com/epsilon-2007/DAVIS}. We will also provide the model weights for our trained CIFAR models. All datasets used in this work (CIFAR-10, CIFAR-100, ImageNet-1k, and all OOD benchmarks) are publicly available and were used without modification, following the standard preprocessing steps described in their original publications and common benchmarks.

\noindent
\textbf{Experimental Setup.}
\begin{itemize}
    \item \textbf{CIFAR Benchmarks:} For fair comparison, our primary models include DenseNet-101, ResNet-18, ResNet-34, Wide-ResNet and MobileNet-v2.
    Following established protocols~\cite{GradNorm,ReAct,DICE,ASH,SCALE}, all models were trained from scratch for 100 epochs using SGD with a momentum of 0.9, a weight decay of 0.0001, and a batch size of 64. The learning rate was initialized at 0.1 and decayed by a factor of 10 at epochs 50, 75, and 90.
    
    \item \textbf{ImageNet Benchmark:} For our large-scale experiments, we used the official pre-trained models provided by PyTorch for Swin-B, ConvNeXt-B, EfficientNet-b0, DenseNet-121, ResNet-50, and MobileNet-v2. No fine-tuning was performed.
\end{itemize}

\noindent
\textbf{Hyperparameter Details.}
The hyperparameter $\boldsymbol{\gamma}$, which scales the standard deviation in Equation~\ref{eq: davis_mu_sigma}, plays a critical role in performance. Following established protocols~\cite{ReAct,DICE,SCALE,ssn}, we select $\gamma$ using a proxy OOD validation set constructed by adding pixel-wise Gaussian noise sampled from $\mathcal{N}(0, 0.2)$ to images from the ID validation set. Based on this procedure, we set $\gamma=3.0$ for all CIFAR models, $\gamma=0.5$ for traditional CNN-based ImageNet models (ResNet, DenseNet, MobileNet), and $\gamma=2.0$ for modern ImageNet architectures (Swin-B, ConvNeXt, EfficientNet-B0).

For all baseline methods (KNN, ODIN, ReAct, DICE, ASH, SCALE), we strictly followed the hyperparameter selection protocols described in their respective papers. When re-evaluating these baselines on new architectures not present in their original work, we performed a hyperparameter search using the same validation procedure they described. Key hyperparameters for these methods are summarized below:

\begin{itemize}
    \item \textbf{ODIN:} We adopted the optimal hyperparameter values reported in the original publication. Accordingly, we set the temperature to $T = 1000$, with a noise magnitude $\epsilon$ of 0.004 for CIFAR and 0.0015 for ImageNet.
    
    \item \textbf{ReAct:} The clipping percentile $p$ was selected from $\{85, 90, 95\}$. While we found $p=90$ to be optimal for the standalone ReAct baseline, consistent with the original paper, the optimal value shifted to $p=95$ when ReAct was combined with our $\approach$.
    
    \item \textbf{DICE:} We selected the sparsity ratio $p$ from $\{70, 75, 80, 85, 90, 95\}$. Our validation process consistently identified $p=70\%$ as the optimal value.
    
    \item \textbf{ASH:} The pruning percentile $p$ was selected from $\{80, 85, 90\}$. The optimal value was found to be dependent on the dataset and architecture. We report the specific optimal value for each major setting to ensure the strongest and fairest possible comparison.
    
    \begin{itemize}
        \item For \textbf{ImageNet}, the optimal value was consistently $p=90$ for most architectures, with the exception of EfficientNet-b0, which required a less aggressive pruning of $p=50$.
        
       \item For \textbf{CIFAR-10}, the optimal values were $p=90$ for DenseNet, $p=80$ for ResNet and Wide-ResNet models, and $p=70$ for MobileNet-v2. These values held for both the standalone baseline and when combined with $\approach$.
        
        \item For \textbf{CIFAR-100}, the optimal value for the ResNet and Wide-ResNet models was consistently $p=80$. For other architectures, we observed an interaction effect: the optimal percentile for DenseNet shifted from $p=90$ (baseline) to $p=80$ (with $\approach$), and for MobileNet-v2, it shifted from $p=90$ to $p=85$.
    \end{itemize}

    \item \textbf{SCALE:} For the SCALE baseline, the pruning percentile p was set to a fixed value of $p=85$ across all experiments. We adopted this value directly from the original SCALE paper~\cite{SCALE} to ensure our re-implementation was consistent with the authors' reported optimal setting, providing a fair comparison.

    \item \textbf{KNN:} Following the original recommendation in KNN~\cite{knn} we used $k=50$ for all of our experiments for fair comparisons. 
    
\end{itemize}

\noindent
\textbf{Computational Environment.} 
All CIFAR model training and OOD detection experiments were conducted on an Apple M2 Max system with 96 GB of RAM. The experiments were implemented in Python using PyTorch (v2.1) and the Torchvision library.

\section{Modular Integration with Primary Baselines}
\label{appendix: combined methods}

    Our method $\approach$ is designed not as a replacement for existing techniques, but as a complementary module that potantially enhances them. It acts as a feature pre-processing step, enriching the standard penultimate layer representation with more discriminative statistics (e.g., maximum and variance) that subsequent methods like \texttt{MSP}~\cite{msp}, \texttt{ODIN}~\cite{odin}, \texttt{Energy}~\cite{energy}, \texttt{GradNorm}~\cite{GradNorm}, \texttt{KNN}~\cite{knn}, \texttt{ReAct}~\cite{ReAct}, \texttt{DICE}~\cite{DICE}, \texttt{ASH}~\cite{ASH}, and \texttt{SCALE}~\cite{SCALE} can then leverage more effectively. So, instead of replacing existing techniques, $\approach$ introduces an additional degree of freedom to modify the penultimate layer in a way that works in tandem with them. 

    \textbf{Note on Compatibility with Scoring Functions.} $\approach$ is compatible with any downstream scoring functions: we primarily focused on MSP, Energy, GradNorm and KNN. In traditional CNN based architectures (ResNet, DenseNet, MobileNet, Wide-ResNet): $\approach$ attains best performance when applied on top of energy based OOD detection techniques like DICE, ASH and SCALE, while on modern deep learning models (Swin-B, ConvNeXt-B) $\approach$ attains best performance when applied on top of KNN OOD detection method.

    Our method, $\approach$, transforms the feature space to create a richer, more separable distribution of logits that is highly beneficial for the holistic energy score. A side effect of this transformation, however, is that the magnitude of the single dominant logit for ID samples can be suppressed. This suppression is the primary reason for the  largely unchanged performance observed when combining $\approach$ with MSP and ODIN as shown in Table~\ref{table: imagenet_combination_swin}, ~\ref{table: imagenet_combination_convnext}, and ~\ref{table: imagenet_combination_resnet}.

    This finding is consistent with the ID accuracy drop analyzed in Appendix~\ref{appendix: accuracy_analysis}, which is also governed by the dominant logit. Therefore, the full potential of our enriched feature representation is best realized by holistic scoring functions like the energy score, rather than those dependent on a single "winner-takes-all" logit.

    \textbf{CIFAR.} To validate complementary effect, we conduct a systematic evaluation, comparing the performance of each primary baseline with and without the application of our $\approach$ module. The results, presented for CIFAR-10 (Tables~\ref{table: ResNet-18_CIFAR-10_combination} and \ref{table: DenseNet-101_CIFAR-10_combination}) and CIFAR-100 (Tables~\ref{table: ResNet-18_CIFAR-100_combination} and \ref{table: DenseNet-101_CIFAR-100_combination}), demonstrate consistent and significant performance gains across all architectures. Notably, integrating $\approach$ boosts the performance of strong baselines, confirming that our enriched feature representation provides a more robust foundation for OOD detection. For brevity, we present only modular integration with DenseNet-101 and ResNet-18.

    \textbf{ImageNet.} For brevity, we present only modular integration using Swin-B, ConvNeXt-B and ResNet-50 in Table~\ref{table: imagenet_combination_swin}, ~\ref{table: imagenet_combination_convnext}, and ~\ref{table: imagenet_combination_resnet} respectively. In large-scale ImageNet-1K dataset , $\approach$ struggle to consistently improve all the primary baseline considered. However, $\approach$ does improve the best performing primary baselines: KNN for Swin-B and ConvNeXt, SCALE for ResNet-50.

    \begin{landscape}
        \begin{table}[!htbp]
        \centering
        \caption{Detailed results of post-hoc methods combined with $\approach$ using \textbf{ResNet-18}  pre-trained on \textbf{CIFAR-10}. $\boldsymbol{\uparrow}$ indicates higher is better; $\boldsymbol{\downarrow}$ indicates lower is better. The symbols denote the statistic used: $\mu$ (mean), $\sigma$ (std. deviation), $m$ (maximum)}        \resizebox{\linewidth}{!}{
        \begin{tabular}{c l  cc cc cc cc cc cc cc }
            \toprule
             \multirow{2}{*}{\textbf{Model}} & \multirow{2}{*}{\textbf{Combined Method}} & \multicolumn{2}{c}{\textbf{SVHN}} & \multicolumn{2}{c}{\textbf{Place365}} & \multicolumn{2}{c}{\textbf{iSUN}} & \multicolumn{2}{c}{\textbf{Textures}} & \multicolumn{2}{c}{\textbf{LSUN-c}}  & \multicolumn{2}{c}{\textbf{LSUN-r}} & \multicolumn{2}{c}{\textbf{Average}}  \\
            \cmidrule(lr){3-4} \cmidrule(lr){5-6} \cmidrule(lr){7-8} \cmidrule(lr){9-10} \cmidrule(lr){11-12} \cmidrule(lr){13-14} \cmidrule(lr){15-16}
            && \textbf{FPR95} $\downarrow$ & \textbf{AUROC} $\uparrow$ & \textbf{FPR95} $\downarrow$ & \textbf{AUROC} $\uparrow$ & \textbf{FPR95} $\downarrow$ & \textbf{AUROC} $\uparrow$ & \textbf{FPR95} $\downarrow$ & \textbf{AUROC} $\uparrow$ & \textbf{FPR95} $\downarrow$ & \textbf{AUROC} $\uparrow$ & \textbf{FPR95} $\downarrow$ & \textbf{AUROC} $\uparrow$ & \textbf{FPR95} $\downarrow$ & \textbf{AUROC} $\uparrow$ \\
            \midrule
            \multirow{27}{*}{ResNet-18} & MSP	                      & 60.39 & 92.40 & 63.49 & 88.38 & 56.59 & 91.18 & 62.71 & 90.10 & 51.87 & 93.64 & 55.53 & 91.69 & 58.43 & 91.23 \\
                                        & + $\approach(m)$            & 51.03 & 89.90 & 64.30 & 81.73 & 56.30 & 87.06 & 55.87 & 87.04 & 43.94 & 92.10 & 54.37 & 88.06 & 54.30 & 87.65 \\
                                        & + $\approach(\mu, \sigma)$  & 51.90 & 89.00 & 64.90 & 78.91 & 57.94 & 83.35 & 57.61 & 83.90 & 46.20 & 90.69 & 56.29 & 84.64 & 55.81 & 85.08 \\
            \cmidrule(lr){2-16}
                                        & ODIN	                      & 35.96 & 94.70 & 41.11 & 92.06 & 23.36 & 96.56 & 46.74 & 91.97 &  6.66 & 98.71 & 20.04 & 96.93 & 28.98 & 95.16 \\
                                        & + $\approach(m)$            & 35.44 & 94.69 & 40.75 & 92.20 & 24.90 & 96.31 & 46.61 & 92.02 &  7.03 & 98.60 & 21.87 & 96.67 & 29.43 & 95.08 \\
                                        & + $\approach(\mu, \sigma)$  & 35.80 & 94.68 & 40.88 & 92.22 & 24.36 & 96.42 & 46.51 & 92.04 &  6.75 & 98.68 & 21.40 & 96.77 & 29.28 & 95.13 \\
            \cmidrule(lr){2-16}
                                        & Energy                      & 44.32 & 94.04 & 41.31 & 91.73 & 35.46 & 94.64 & 50.39 & 91.12 & 9.77 & 98.19 & 32.41 & 95.16 & 35.61 & 94.14 \\
                                        & + $\approach(m)$            & 19.81 & 96.29 & 32.32 & 93.73 & 15.79 & 97.37 & 21.90 & 96.44 & 5.83 & 98.73 & 13.63 & 97.61 & 18.21 & 96.69 \\
                                        & + $\approach(\mu, \sigma)$  & 19.83 & 96.47 & 36.11 & 93.16 & 20.48 & 96.90 & 26.77 & 95.89 & 5.62 & 98.79 & 17.75 & 97.21 & 21.09 & 96.00 \\ 
            \cmidrule(lr){2-16}
                                        & GradNorm                    & 17.98 & 96.53 & 57.08 & 86.66 & 37.88 & 93.90 & 43.79 & 90.96 &  4.33 & 99.00 & 36.83 & 93.89 & 32.98 & 93.49 \\
                                        & + $\approach(m)$            &  5.14 & 99.00 & 48.27 & 89.98 & 11.84 & 97.73 &  9.91 & 98.01 &  0.64 & 99.81 & 13.79 & 97.37 & 14.93 & 96.98 \\
                                        & + $\approach(\mu, \sigma)$  &  5.92 & 98.90 & 53.09 & 88.92 & 16.47 & 97.04 & 13.21 & 97.52 &  0.41 & 99.84 & 19.25 & 96.64 & 18.06 & 96.48 \\
            \cmidrule(lr){2-16}
                                        & KNN 	                      & 14.29 & 97.59 & 49.08 & 89.62 & 37.23 & 92.87 & 29.15 & 94.62 & 15.71 & 97.31 & 37.12 & 93.34 & 30.43 & 94.22 \\
                                        & + $\approach(m)$            & 17.73 & 96.94 & 46.18 & 90.93 & 35.01 & 93.4 & 32.52 & 94.18 & 11.16 & 98.13 & 31.68 & 94.45 & 29.05 & 94.67 \\
                                        & + $\approach(\mu, \sigma)$  & 14.36 & 97.48 & 45.79 & 90.62 & 33.46 & 93.61 & 29.33 & 94.71 & 11.55 & 97.91 & 30.83 & 94.43 & 27.55 & 94.79 \\
            \cmidrule(lr){2-16}
                                        & ReAct	                      & 42.31 & 94.12 & 40.74 & 92.25 & 24.06 & 96.26 & 40.44 & 93.69 & 12.27 & 97.90 & 21.02 & 96.67 & 30.14 & 95.15 \\
                                        & + $\approach(m)$            & 21.54 & 95.93 & 32.84 & 93.65 & 16.26 & 97.27 & 23.07 & 96.21 &  6.67 & 98.61 & 14.07 & 97.54 & 19.07 & 96.53 \\
                                        & + $\approach(\mu, \sigma)$  & 21.70 & 96.10 & 36.12 & 93.18 & 20.04 & 96.90 & 26.93 & 95.74 &  6.43 & 98.67 & 17.24 & 97.25 & 21.40 & 96.31 \\
            \cmidrule(lr){2-16}
                                        & DICE	                      & 17.60 & 97.09 & 46.16 & 90.66 & 38.68 & 94.32 & 44.50 & 91.81 & 1.90 & 99.57 & 36.66 & 94.67 & 30.92 & 94.69 \\
                                        & + $\approach(m)$            &  7.95 & 98.50 & 30.55 & 93.97 &  6.70 & 98.59 &  9.66 & 98.28 & 1.24 & 99.73 &  6.63 & 98.57 & 10.46 & 97.94 \\
                                        & + $\approach(\mu, \sigma)$  &  7.85 & 98.57 & 35.27 & 93.18 & 11.96 & 97.96 & 14.04 & 97.71 & 1.14 & 99.76 & 10.71 & 98.03 & 13.49 & 97.54 \\
            \cmidrule(lr){2-16}
                                        & ReAct+DICE                  & 12.25 & 97.85 & 45.77 & 91.27 & 18.12 & 96.92 & 26.90 & 95.60 & 1.35 & 99.65 & 16.93 & 97.09 & 20.22 & 96.40 \\
                                        & + $\approach(m)$            &  8.23 & 98.47 & 30.40 & 94.06 &  6.49 & 98.63 &  9.52 & 98.29 & 1.33 & 99.73 &  6.34 & 98.62 & 10.38 & 97.97 \\
                                        & + $\approach(\mu, \sigma)$  &  8.31 & 98.49 & 34.88 & 93.34 & 10.29 & 98.14 & 13.01 & 97.82 & 1.30 & 99.75 &  9.49 & 98.22 & 12.88 & 97.62 \\
            \cmidrule(lr){2-16}
                                        & ASH-S	                      & 7.87 & 98.43 & 49.69 & 89.57 & 23.27 & 96.33 & 26.12 & 95.88 & 2.10 & 99.46 & 21.91 & 96.47 & 21.83 & 96.02 \\
                                        & + $\approach(m)$            & 9.25 & 98.37 & 41.26 & 91.73 & 10.87 & 98.01 & 11.60 & 97.95 & 1.94 & 99.54 & 10.46 & 98.02 & 14.23 & 97.27 \\
                                        & + $\approach(\mu, \sigma)$  & 7.31 & 98.68 & 43.72 & 91.00 & 12.58 & 97.81 & 12.16 & 97.86 & 1.70 & 99.60 & 12.23 & 97.85 & 14.95 & 97.13 \\
            \cmidrule(lr){2-16}
                                        & SCALE                       &  9.73 & 98.13 & 45.99 & 90.87 & 22.92 & 96.39 & 27.00 & 95.61 & 3.75 & 99.17 & 21.02 & 96.60 & 21.74 & 96.13  \\
                                        & + $\approach(m) $           & 10.21 & 98.17 & 37.09 & 92.92 &  9.74 & 98.17 & 12.77 & 97.83 & 2.86 & 99.39 &  9.71 & 98.19 & 13.73 & 97.45 \\
                                        & + $\approach({\mu,\sigma})$ &  8.88 & 98.43 & 39.84 & 92.38 & 12.36 & 97.93 & 14.11 & 97.65 & 2.71 & 99.43 & 11.62 & 97.99 & 14.92 & 97.30 \\
            \bottomrule
            \end{tabular}
            }
            \label{table: ResNet-18_CIFAR-10_combination}
        \end{table}
    \end{landscape}

    \begin{landscape}
        \begin{table}[!htbp]
        \centering
        \caption{\textit{Detailed results of post-hoc methods combined with $\approach$ using \textbf{DenseNet-101}  pre-trained on \textbf{CIFAR-10}. $\boldsymbol{\uparrow}$ indicates higher is better; $\boldsymbol{\downarrow}$ indicates lower is better. The symbols denote the statistic used: $\mu$ (mean), $\sigma$ (std. deviation), $m$ (maximum)}}
        \resizebox{\linewidth}{!}{
        \begin{tabular}{c l  cc cc cc cc cc cc cc }
            \toprule
             \multirow{2}{*}{\textbf{Model}} & \multirow{2}{*}{\textbf{Combined Method}} & \multicolumn{2}{c}{\textbf{SVHN}} & \multicolumn{2}{c}{\textbf{Place365}} & \multicolumn{2}{c}{\textbf{iSUN}} & \multicolumn{2}{c}{\textbf{Textures}} & \multicolumn{2}{c}{\textbf{LSUN-c}}  & \multicolumn{2}{c}{\textbf{LSUN-r}} & \multicolumn{2}{c}{\textbf{Average}}  \\
            \cmidrule(lr){3-4} \cmidrule(lr){5-6} \cmidrule(lr){7-8} \cmidrule(lr){9-10} \cmidrule(lr){11-12} \cmidrule(lr){13-14} \cmidrule(lr){15-16}
            && \textbf{FPR95} $\downarrow$ & \textbf{AUROC} $\uparrow$ & \textbf{FPR95} $\downarrow$ & \textbf{AUROC} $\uparrow$ & \textbf{FPR95} $\downarrow$ & \textbf{AUROC} $\uparrow$ & \textbf{FPR95} $\downarrow$ & \textbf{AUROC} $\uparrow$ & \textbf{FPR95} $\downarrow$ & \textbf{AUROC} $\uparrow$ & \textbf{FPR95} $\downarrow$ & \textbf{AUROC} $\uparrow$ & \textbf{FPR95} $\downarrow$ & \textbf{AUROC} $\uparrow$ \\
            \midrule
            \multirow{27}{*}{DenseNet-101} & MSP	                  & 64.76 & 88.33 & 60.19 & 88.56 & 33.34 & 95.41 & 56.60 & 90.17 & 23.41 & 96.75 & 33.88 & 95.39 & 45.36 & 92.43 \\
                                        & + $\approach(m)$            & 65.33 & 85.68 & 67.63 & 83.28 & 51.76 & 91.67 & 61.38 & 87.95 & 50.8 & 91.17 & 53.27 & 91.43 & 58.36 & 88.53 \\
                                        & + $\approach(\mu, \sigma)$  & 63.11 & 85.23 & 65.59 & 82.43 & 45.08 & 92.18 & 58.1 & 87.51 & 42.91 & 92.61 & 46.61 & 91.93 & 53.57 & 88.65 \\
            \cmidrule(lr){2-16}
                                        & ODIN	                      & 33.09 & 94.41 & 36.68 & 92.34 & 3.22 & 99.20 & 38.49 & 91.61 & 1.84 & 99.53 & 2.89 & 99.28 & 19.37 & 96.06 \\
                                        & + $\approach(m)$            & 32.06 & 94.64 & 36.37 & 92.67 & 4.01 & 98.83 & 37.96 & 91.88 & 2.18 & 99.33 & 3.66 & 98.91 & 19.37 & 96.05 \\
                                        & + $\approach(\mu, \sigma)$  & 34.26 & 94.34 & 37.20 & 92.51 & 3.52 & 99.05 & 38.03 & 91.80 & 1.99 & 99.45 & 3.14 & 99.13 & 19.69 & 96.04 \\
            \cmidrule(lr){2-16}
                                        & Energy                      & 37.91 & 93.59 & 36.38 & 92.39 & 7.83 & 98.23 & 43.85 & 90.49 & 1.95 & 99.47 & 7.34 & 98.34 & 22.54 & 95.42 \\
                                        & + $\approach(m)$            & 30.50 & 94.50 & 33.04 & 93.28 & 7.81 & 98.29 & 25.16 & 95.85 & 9.95 & 98.19 & 7.32 & 98.37 & 18.96 & 96.41 \\
                                        & + $\approach(\mu, \sigma)$  & 24.75 & 95.93 & 32.08 & 93.43 & 5.77 & 98.64 & 25.23 & 95.84 & 5.13 & 98.94 & 5.63 & 98.70 & 16.43 & 96.91 \\ 
            \cmidrule(lr){2-16}
                                        & GradNorm                    & 22.02 & 96.19 & 47.68 & 88.65 &  4.72 & 99.03 & 27.29 & 92.21 & 0.21 & 99.90 & 4.99 & 98.94 & 17.82 & 95.82 \\
                                        & + $\approach(m)$            &  6.43 & 98.63 & 53.90 & 87.67 &  6.31 & 98.75 & 10.02 & 97.80 & 1.12 & 99.70 & 6.99 & 98.63 & 14.13 & 96.86 \\
                                        & + $\approach(\mu, \sigma)$  &  6.75 & 98.65 & 54.39 & 87.91 &  5.43 & 98.95 & 13.53 & 97.15 & 0.28 & 99.86 & 5.78 & 98.83 & 14.36 & 96.89 \\
            \cmidrule(lr){2-16}
                                        & KNN 	                      & 1.50 & 99.67 & 42.45 & 90.49 & 7.04 & 98.72 & 14.38 & 97.54 & 5.79 & 98.94 & 8.73 & 98.45 & 13.31 & 97.30 \\
                                        & + $\approach(m)$            & 3.73 & 99.22 & 45.27 & 90.91 & 5.18 & 98.98 & 12.41 & 97.76 & 5.61 & 99.00 & 5.60 & 98.97 & 12.97 & 97.47 \\
                                        & + $\approach(\mu, \sigma)$  & 2.49 & 99.52 & 39.90 & 91.60 & 4.55 & 99.15 & 11.88 & 97.96 & 5.04 & 99.09 & 4.57 & 99.12 & 11.40 & 97.74 \\                            
            \cmidrule(lr){2-16}
                                        & ReAct	                      & 23.18 & 96.28 & 33.97 & 92.98 & 5.95 & 98.45 & 32.25 & 93.98 &  2.47 & 99.33 & 5.44 & 98.55 & 17.21 & 96.59 \\
                                        & + $\approach(m)$            & 32.67 & 93.43 & 35.41 & 92.58 & 9.48 & 98.00 & 27.84 & 95.18 & 12.60 & 97.66 & 9.34 & 98.07 & 21.22 & 95.82 \\
                                        & + $\approach(\mu, \sigma)$  & 23.78 & 95.59 & 34.68 & 92.90 & 7.37 & 98.37 & 26.12 & 95.46 &  7.26 & 98.60 & 7.30 & 98.43 & 17.75 & 96.56 \\
            \cmidrule(lr){2-16}
                                        & DICE	                      & 16.68 & 96.96 & 37.46 & 92.06 & 2.25 & 99.41 & 28.05 & 92.70 & 0.16 & 99.94 & 2.44 & 99.35 & 14.51 & 96.74 \\
                                        & + $\approach(m)$            &  8.30 & 98.37 & 29.47 & 93.92 & 1.85 & 99.56 &  7.16 & 98.69 & 1.26 & 99.72 & 1.92 & 99.55 &  8.33 & 98.30 \\
                                        & + $\approach(\mu, \sigma)$  &  6.84 & 98.77 & 29.76 & 93.86 & 1.58 & 99.59 &  9.57 & 98.25 & 0.55 & 99.86 & 1.60 & 99.57 &  8.32 & 98.32 \\
            \cmidrule(lr){2-16}
                                        & ReAct+DICE                  & 4.63 & 99.02 & 36.10 & 92.93 & 1.80 & 99.51 & 17.32 & 96.76 & 0.12 & 99.95 & 1.93 & 99.47 & 10.32 & 97.94 \\
                                        & + $\approach(m)$            & 7.86 & 98.39 & 29.27 & 93.87 & 1.73 & 99.58 &  5.90 & 98.85 & 1.19 & 99.73 & 1.84 & 99.56 &  7.96 & 98.33 \\
                                        & + $\approach(\mu, \sigma)$  & 5.90 & 98.87 & 29.45 & 93.92 & 1.52 & 99.59 &  7.41 & 98.62 & 0.55 & 99.86 & 1.85 & 99.57 &  7.78 & 98.41 \\
            \cmidrule(lr){2-16}
                                        & ASH	                      & 16.20 & 97.21 & 37.79 & 92.02 & 3.91 & 98.94 & 26.40 & 94.61 & 0.84 & 99.69 & 4.15 & 98.92 & 14.88 & 96.90 \\
                                        & + $\approach(m)$            & 12.48 & 97.63 & 34.90 & 93.03 & 4.63 & 98.90 & 11.76 & 97.94 & 4.59 & 99.08 & 5.07 & 98.84 & 12.24 & 97.57 \\
                                        & + $\approach(\mu, \sigma)$  &  9.78 & 98.20 & 34.54 & 93.10 & 3.59 & 99.10 & 12.45 & 97.87 & 2.25 & 99.44 & 4.06 & 99.08 & 11.11 & 97.80 \\
            \cmidrule(lr){2-16}
                                        & Scale                       & 23.06 & 96.13 & 36.53 & 92.24 & 4.54 & 98.76 & 30.53 & 93.62 & 1.23 & 99.61 & 4.59 & 98.75 & 16.75 & 96.52 \\
                                        &  + $\approach(m)$           & 20.38 & 96.25 & 33.11 & 93.27 & 5.86 & 98.63 & 15.80 & 97.27 & 6.57 & 98.72 & 6.07 & 98.61 & 14.63 & 97.13 \\
                                        & + $\approach({\mu,\sigma})$ & 14.93 & 97.30 & 31.94 & 93.46 & 4.04 & 98.95 & 15.74 & 97.31 & 3.28 & 99.26 & 4.28 & 98.95 & 12.37 & 97.54 \\             
            \bottomrule
            \end{tabular}}
            \label{table: DenseNet-101_CIFAR-10_combination}
        \end{table}
    \end{landscape}

        \begin{landscape}
        \begin{table}[!htbp]
        \centering
        \caption{Detailed results of post-hoc methods combined with $\approach$ using \textbf{ResNet-18} pre-trained on \textbf{CIFAR-100}. $\boldsymbol{\uparrow}$ indicates higher is better; $\boldsymbol{\downarrow}$ indicates lower is better. The symbols denote the statistic used: $\mu$ (mean), $\sigma$ (std. deviation), $m$ (maximum)}
        \resizebox{\linewidth}{!}{
        \begin{tabular}{c l  cc cc cc cc cc cc cc }
            \toprule
             \multirow{2}{*}{\textbf{Model}} & \multirow{2}{*}{\textbf{Combined Method}} & \multicolumn{2}{c}{\textbf{SVHN}} & \multicolumn{2}{c}{\textbf{Place365}} & \multicolumn{2}{c}{\textbf{iSUN}} & \multicolumn{2}{c}{\textbf{Textures}} & \multicolumn{2}{c}{\textbf{LSUN-c}}  & \multicolumn{2}{c}{\textbf{LSUN-r}} & \multicolumn{2}{c}{\textbf{Average}}  \\
            \cmidrule(lr){3-4} \cmidrule(lr){5-6} \cmidrule(lr){7-8} \cmidrule(lr){9-10} \cmidrule(lr){11-12} \cmidrule(lr){13-14} \cmidrule(lr){15-16}
            && \textbf{FPR95} $\downarrow$ & \textbf{AUROC} $\uparrow$ & \textbf{FPR95} $\downarrow$ & \textbf{AUROC} $\uparrow$ & \textbf{FPR95} $\downarrow$ & \textbf{AUROC} $\uparrow$ & \textbf{FPR95} $\downarrow$ & \textbf{AUROC} $\uparrow$ & \textbf{FPR95} $\downarrow$ & \textbf{AUROC} $\uparrow$ & \textbf{FPR95} $\downarrow$ & \textbf{AUROC} $\uparrow$ & \textbf{FPR95} $\downarrow$ & \textbf{AUROC} $\uparrow$ \\
            \midrule
            \multirow{27}{*}{ResNet-18} & MSP	                      & 74.26 & 83.20 & 82.49 & 75.32 & 85.58 & 70.20 & 84.89 & 74.02 & 70.79 & 82.78 & 84.36 & 71.45 & 80.40 & 76.16 \\
                                        & + $\approach(m)$            & 79.01 & 83.56 & 86.57 & 73.79 & 87.79 & 72.86 & 83.39 & 77.95 & 79.77 & 81.23 & 87.53 & 73.03 & 84.01 & 77.07 \\
                                        & + $\approach(\mu, \sigma)$  & 78.90 & 82.88 & 86.06 & 72.47 & 87.68 & 70.35 & 83.07 & 76.64 & 78.68 & 80.12 & 87.38 & 70.78 & 83.63 & 75.54 \\
            \cmidrule(lr){2-16}
                                        & ODIN	                      & 70.30 & 88.06 & 80.14 & 77.02 & 60.26 & 86.98 & 81.56 & 76.56 & 47.73 & 91.84 & 56.35 & 88.23 & 66.06 & 84.78 \\
                                        & + $\approach(m)$            & 70.33 & 88.25 & 80.32 & 77.37 & 66.08 & 85.56 & 81.70 & 77.02 & 50.98 & 91.37 & 62.87 & 86.65 & 68.71 & 84.37 \\
                                        & + $\approach(\mu, \sigma)$  & 69.39 & 88.41 & 80.54 & 77.33 & 64.80 & 85.84 & 81.37 & 77.04 & 49.73 & 91.60 & 61.27 & 86.94 & 67.85 & 84.53 \\                            
            \cmidrule(lr){2-16}
                                        & Energy                      & 66.64 & 89.53 & 81.23 & 76.84 & 73.67 & 82.01 & 85.30 & 75.68 & 48.01 & 91.63 & 70.30 & 83.38 & 70.86 & 83.18 \\
                                        & + $\approach(m)$            & 38.40 & 94.25 & 77.62 & 78.98 & 56.43 & 89.99 & 55.73 & 89.14 & 35.90 & 94.17 & 55.87 & 89.95 & 53.32 & 89.41 \\
                                        & + $\approach(\mu, \sigma)$  & 42.01 & 93.75 & 78.99 & 77.58 & 61.13 & 88.60 & 59.89 & 87.89 & 38.56 & 93.48 & 59.76 & 88.86 & 56.72 & 88.36 \\
            \cmidrule(lr){2-16}
                                        & GradNorm                    & 57.22 & 85.93 & 88.47 & 62.15 & 81.65 & 74.94 & 79.84 & 65.96 & 15.44 & 96.94 & 80.07 & 75.73 & 67.11 & 76.94 \\
                                        & + $\approach(m)$            & 25.36 & 95.05 & 94.20 & 57.46 & 68.18 & 85.25 & 42.59 & 88.22 & 16.94 & 96.94 & 74.09 & 83.83 & 53.56 & 84.46 \\
                                        & + $\approach(\mu, \sigma)$  & 33.50 & 93.52 & 95.02 & 54.35 & 72.93 & 82.55 & 50.80 & 85.21 & 18.58 & 96.62 & 77.47 & 81.59 & 58.05 & 82.31 \\
            \cmidrule(lr){2-16}
                                        & KNN 	                      & 61.76 & 90.15 & 86.35 & 70.45 & 69.69 & 79.24 & 41.08 & 90.87 & 75.36 & 77.73 & 67.78 & 80.08 & 67.00 & 81.42 \\
                                        & + $\approach(m)$            & 45.01 & 92.43 & 86.42 & 71.69 & 63.06 & 84.56 & 34.02 & 92.75 & 56.70 & 83.88 & 63.25 & 84.34 & 58.08 & 84.94 \\
                                        & + $\approach(\mu, \sigma)$  & 46.92 & 92.26 & 85.36 & 71.79 & 63.90 & 83.46 & 34.02 & 92.75 & 59.00 & 83.13 & 64.42 & 83.39 & 58.94 & 84.46 \\                               
            \cmidrule(lr){2-16}
                                        & ReAct	                      & 55.03 & 91.95 & 79.78 & 77.48 & 58.66 & 87.78 & 60.90 & 87.94 & 47.42 & 91.22 & 54.78 & 88.77 & 59.43 & 87.52 \\
                                        & + $\approach(m)$            & 42.98 & 93.55 & 77.00 & 78.94 & 55.73 & 89.57 & 54.18 & 89.53 & 40.68 & 92.67 & 54.12 & 89.84 & 54.12 & 89.02 \\
                                        & + $\approach(\mu, \sigma)$  & 44.43 & 93.37 & 78.24 & 78.18 & 57.83 & 88.94 & 56.26 & 89.06 & 42.53 & 92.19 & 55.52 & 89.5 & 55.8 & 88.54 \\
            \cmidrule(lr){2-16}
                                        & DICE	                      & 41.18 & 92.98 & 81.82 & 76.02 & 66.22 & 84.20 & 75.50 & 76.27 & 12.21 & 97.70 & 64.48 & 85.15 & 56.90 & 85.39 \\
                                        & + $\approach(m)$            & 10.77 & 97.91 & 80.06 & 77.59 & 33.36 & 94.13 & 31.52 & 93.52 &  7.31 & 98.53 & 37.27 & 93.39 & 33.38 & 92.51 \\
                                        & + $\approach(\mu, \sigma)$  & 12.40 & 97.59 & 80.46 & 76.21 & 38.77 & 92.69 & 36.74 & 92.09 &  7.78 & 98.42 & 40.98 & 92.27 & 36.19 & 91.54 \\
            \cmidrule(lr){2-16}
                                        & ReAct+DICE                  & 36.18 & 93.65 & 86.43 & 71.92 & 59.57 & 88.43 & 48.46 & 87.59 & 11.78 & 97.59 & 60.22 & 88.45 & 50.44 & 87.94 \\
                                        & + $\approach(m)$            & 11.37 & 97.76 & 82.96 & 75.31 & 35.09 & 93.96 & 29.26 & 93.39 & 8.53 & 98.24 & 38.86 & 93.25 & 34.34 & 91.99 \\
                                        & + $\approach(\mu, \sigma)$  & 13.03 & 97.49 & 82.67 & 74.91 & 38.04 & 93.08 & 31.99 & 92.61 & 8.91 & 98.13 & 40.32 & 92.68 & 35.83 & 91.48 \\
            \cmidrule(lr){2-16}
                                        & ASH	                      & 29.10 & 95.46 & 82.56 & 75.38 & 67.09 & 85.01 & 56.49 & 87.80 & 27.06 & 95.55 & 64.72 & 85.62 & 54.50 & 87.47 \\
                                        & + $\approach(m)$            & 12.41 & 97.87 & 81.30 & 75.93 & 46.38 & 91.18 & 26.58 & 94.94 & 18.44 & 96.99 & 48.47 & 90.45 & 38.93 & 91.23 \\
                                        & + $\approach(\mu, \sigma)$  & 13.44 & 97.70 & 81.85 & 75.09 & 49.06 & 90.36 & 28.55 & 94.52 & 19.93 & 96.69 & 50.46 & 89.97 & 40.55 & 90.72 \\
            \cmidrule(lr){2-16}
                                        & SCALE                       & 22.12 & 96.38 & 81.96 & 74.95 & 61.62 & 86.65 & 44.50 & 90.72 & 18.62 & 96.78 & 59.76 & 86.74 & 48.10 & 88.70 \\
                                        & + $\approach(m)$            & 10.40 & 98.10 & 82.86 & 74.60 & 45.08 & 91.15 & 24.10 & 95.04 & 14.09 & 97.49 & 48.37 & 89.96 & 37.48 & 91.06 \\
                                        & + $\approach({\mu,\sigma})$ & 11.43 & 97.94 & 83.24 & 74.00 & 47.37 & 90.57 & 25.64 & 94.78 & 15.35 & 97.28 & 49.16 & 89.69 & 38.70 & 90.71 \\
            \bottomrule
            \end{tabular}}
            \label{table: ResNet-18_CIFAR-100_combination}
        \end{table}
    \end{landscape}

    \begin{landscape}
        \begin{table}[ht]
        \centering
        \caption{Detailed results of post-hoc methods combined with $\approach$ \textbf{DenseNet-101} pre-trained on \textbf{CIFAR-100}. $\boldsymbol{\uparrow}$ indicates higher is better; $\boldsymbol{\downarrow}$ indicates lower is better. The symbols denote the statistic used: $\mu$ (mean), $\sigma$ (std. deviation), $m$ (maximum)}
        \resizebox{\linewidth}{!}{
        \begin{tabular}{c l  cc cc cc cc cc cc cc }
            \toprule
             \multirow{2}{*}{\textbf{Model}} & \multirow{2}{*}{\textbf{Combined Method}} & \multicolumn{2}{c}{\textbf{SVHN}} & \multicolumn{2}{c}{\textbf{Place365}} & \multicolumn{2}{c}{\textbf{iSUN}} & \multicolumn{2}{c}{\textbf{Textures}} & \multicolumn{2}{c}{\textbf{LSUN-c}}  & \multicolumn{2}{c}{\textbf{LSUN-r}} & \multicolumn{2}{c}{\textbf{Average}}  \\
            \cmidrule(lr){3-4} \cmidrule(lr){5-6} \cmidrule(lr){7-8} \cmidrule(lr){9-10} \cmidrule(lr){11-12} \cmidrule(lr){13-14} \cmidrule(lr){15-16}
            && \textbf{FPR95} $\downarrow$ & \textbf{AUROC} $\uparrow$ & \textbf{FPR95} $\downarrow$ & \textbf{AUROC} $\uparrow$ & \textbf{FPR95} $\downarrow$ & \textbf{AUROC} $\uparrow$ & \textbf{FPR95} $\downarrow$ & \textbf{AUROC} $\uparrow$ & \textbf{FPR95} $\downarrow$ & \textbf{AUROC} $\uparrow$ & \textbf{FPR95} $\downarrow$ & \textbf{AUROC} $\uparrow$ & \textbf{FPR95} $\downarrow$ & \textbf{AUROC} $\uparrow$ \\
            \midrule
            \multirow{27}{*}{DenseNet-101} & MSP	                  & 81.38 & 75.71 & 82.62 & 74.04 & 84.12 & 68.22 & 86.95 & 68.37 & 51.82 & 87.93 & 81.34 & 69.51 & 78.04 & 73.96 \\
                                        & + $\approach(m)$            & 84.27 & 77.18 & 88.63 & 69.53 & 86.77 & 72.39 & 86.15 & 72.76 & 80.46 & 79.07 & 86.99 & 71.73 & 85.54 & 73.78 \\
                                        & + $\approach(\mu, \sigma)$  & 83.27 & 78.70 & 87.90 & 70.83 & 87.57 & 69.56 & 85.85 & 72.53 & 77.17 & 80.93 & 88.06 & 69.25 & 84.97 & 73.63 \\
            \cmidrule(lr){2-16}
                                        & ODIN	                      & 85.94 & 80.35 & 75.59 & 77.62 & 48.03 & 89.12 & 83.37 & 67.83 & 12.78 & 97.70 & 40.28 & 91.35 & 57.67 & 84.00 \\
                                        & + $\approach(m)$            & 75.87 & 85.39 & 75.76 & 78.73 & 59.23 & 86.19 & 82.59 & 69.69 & 17.62 & 97.08 & 52.52 & 88.40 & 60.60 & 84.25 \\
                                        & + $\approach(\mu, \sigma)$  & 77.99 & 83.62 & 76.25 & 78.32 & 55.94 & 87.15 & 82.46 & 69.02 & 15.78 & 97.32 & 49.49 & 89.29 & 59.65 & 84.12 \\                            
            \cmidrule(lr){2-16}
                                        & Energy                      & 70.99 & 86.66 & 77.12 & 76.94 & 64.28 & 83.92 & 83.60 & 67.47 & 11.45 & 97.89 & 56.08 & 86.84 & 60.59 & 83.29 \\
                                        & + $\approach(m)$            & 52.32 & 89.41 & 73.49 & 79.62 & 35.76 & 93.37 & 49.88 & 88.07 & 18.58 & 96.44 & 37.27 & 92.98 & 44.55 & 89.98 \\
                                        & + $\approach(\mu, \sigma)$  & 45.23 & 91.91 & 72.11 & 79.78 & 42.21 & 92.16 & 55.85 & 85.47 & 13.92 & 97.53 & 40.52 & 92.45 & 44.97 & 89.88 \\ 
            \cmidrule(lr){2-16}
                                        & GradNorm                    & 35.49 & 93.07 & 87.03 & 70.09 & 71.36 & 82.74 & 61.83 & 75.72 &  0.94 & 99.75 & 68.56 & 83.65 & 54.20 & 84.17 \\
                                        & + $\approach(m)$            & 29.26 & 94.37 & 92.87 & 65.81 & 59.78 & 88.27 & 31.93 & 92.70 & 12.00 & 97.83 & 68.87 & 86.69 & 49.12 & 87.61 \\
                                        & + $\approach(\mu, \sigma)$  & 29.44 & 94.53 & 92.70 & 64.87 & 62.21 & 87.33 & 36.81 & 91.13 &  6.36 & 98.72 & 69.54 & 86.00 & 49.51 & 87.10 \\
            \cmidrule(lr){2-16}
                                        & KNN 	                      & 15.86 & 96.88 & 88.36 & 66.14 & 42.98 & 89.45 & 27.11 & 94.21 & 35.82 & 89.74 & 42.90 & 89.28 & 42.17 & 87.62 \\
                                        & + $\approach(m)$            & 12.24 & 97.74 & 91.03 & 65.18 & 36.69 & 93.44 & 26.01 & 93.71 & 20.19 & 96.31 & 43.81 & 92.40 & 38.33 & 89.80 \\
                                        & + $\approach(\mu, \sigma)$  & 10.78 & 97.84 & 88.61 & 67.45 & 30.98 & 93.88 & 23.09 & 94.70 & 23.33 & 95.03 & 34.62 & 93.33 & 35.23 & 90.37 \\                               
            \cmidrule(lr){2-16}
                                        & ReAct	                      & 67.12 & 87.20 & 77.75 & 76.18 & 56.39 & 89.46 & 75.98 & 79.16 & 13.26 & 97.53 & 49.92 & 90.94 & 56.74 & 86.74 \\
                                        & + $\approach(m)$            & 52.11 & 89.87 & 75.65 & 78.41 & 37.90 & 93.24 & 51.42 & 87.49 & 22.32 & 95.63 & 39.01 & 93.03 & 46.40 & 89.61 \\
                                        & + $\approach(\mu, \sigma)$  & 44.85 & 92.25 & 74.50 & 78.63 & 43.84 & 92.30 & 55.18 & 86.60 & 17.47 & 96.83 & 41.97 & 92.74 & 46.30 & 89.89 \\
            \cmidrule(lr){2-16}
                                        & DICE	                      & 33.87 & 93.97 & 79.95 & 76.75 & 47.76 & 89.61 & 63.42 & 73.33 &  0.79 & 99.76 & 43.65 & 91.00 & 44.91 & 87.40 \\
                                        & + $\approach(m)$            & 27.97 & 94.91 & 86.13 & 76.61 & 36.01 & 93.97 & 30.32 & 93.28 &  7.58 & 98.56 & 41.90 & 93.13 & 38.32 & 91.74 \\
                                        & + $\approach(\mu, \sigma)$  & 20.30 & 96.20 & 79.52 & 78.28 & 29.69 & 94.63 & 34.10 & 91.24 &  2.63 & 99.37 & 32.01 & 94.20 & 33.04 & 92.32 \\
            \cmidrule(lr){2-16}
                                        & ReAct+DICE                  & 28.01 & 95.38 & 83.56 & 74.73 & 37.02 & 93.92 & 46.93 & 86.09 & 0.68 & 99.79 & 37.20 & 93.93 & 38.90 & 90.64 \\
                                        & + $\approach(m)$            & 23.75 & 95.30 & 86.53 & 74.82 & 36.01 & 93.77 & 28.30 & 93.55 & 7.59 & 98.53 & 41.77 & 92.98 & 37.32 & 91.49 \\
                                        & + $\approach(\mu, \sigma)$  & 17.65 & 96.60 & 82.54 & 76.47 & 29.84 & 94.80 & 30.57 & 92.97 & 2.73 & 99.36 & 32.66 & 94.38 & 32.66 & 92.43 \\
            \cmidrule(lr){2-16}
                                        & ASH	                      & 10.32 & 97.99 & 85.93 & 71.95 & 39.69 & 92.04 & 35.67 & 91.76 &  5.43 & 98.98 & 42.89 & 91.30 & 36.66 & 90.67 \\
                                        & + $\approach(m)$            & 17.78 & 96.87 & 81.38 & 76.97 & 34.35 & 93.71 & 24.22 & 95.11 & 10.20 & 98.18 & 38.96 & 92.90 & 34.48 & 92.29 \\
                                        & + $\approach(\mu, \sigma)$  & 12.12 & 97.72 & 79.01 & 77.28 & 33.61 & 93.56 & 26.33 & 94.75 &  6.26 & 98.81 & 37.11 & 92.96 & 32.41 & 92.51 \\
            \cmidrule(lr){2-16}
                                        & Scale                       & 16.26 & 97.05 & 78.54 & 76.97 & 43.56 & 91.21 & 45.60 & 87.23 &  3.23 & 99.30 & 42.69 & 91.02 & 38.31 & 90.46 \\
                                        & + $\approach(m)$            & 19.45 & 96.40 & 81.18 & 76.81 & 36.21 & 93.26 & 25.11 & 94.71 & 11.07 & 97.97 & 44.32 & 91.81 & 36.22 & 91.83 \\
                                        & + $\approach({\mu,\sigma})$ & 13.46 & 97.53 & 78.75 & 77.82 & 32.68 & 93.77 & 24.29 & 95.06 &  6.09 & 98.85 & 38.36 & 92.74 & 32.27 & 92.63 \\            
            \bottomrule
            \end{tabular}}
            \label{table: DenseNet-101_CIFAR-100_combination}
        \end{table}
    \end{landscape}

    \begin{landscape}
    \begin{table}[!ht]
    \centering
     \caption{Detailed results of post-hoc methods combined with $\approach$ using \textbf{Swin-B} pre-trained on ImageNet-1K. $\boldsymbol{\uparrow}$ indicates higher is better; $\boldsymbol{\downarrow}$ indicates lower is better. The symbols denote the statistic used: $\mu$ (mean), $\sigma$ (std. deviation), $m$ (maximum).}
    \resizebox{0.85\linewidth}{!}{
    \begin{tabular}{ l  cc cc cc cc cc cc }
        \toprule
         \multirow{2}{*}{\textbf{Method}} & \multicolumn{2}{c}{\textbf{SUN}} & \multicolumn{2}{c}{\textbf{Place365}} & \multicolumn{2}{c}{\textbf{Textures}} & \multicolumn{2}{c}{\textbf{iNaturalist}} & \multicolumn{2}{c}{\textbf{OpenImage-O}} & \multicolumn{2}{c}{\textbf{Average}}  \\
        \cmidrule(lr){2-3} \cmidrule(lr){4-5} \cmidrule(lr){6-7} \cmidrule(lr){8-9} \cmidrule(lr){10-11} \cmidrule(lr){12-13}
        & \textbf{FPR95} $\downarrow$ & \textbf{AUROC} $\uparrow$ & \textbf{FPR95} $\downarrow$ & \textbf{AUROC} $\uparrow$ & \textbf{FPR95} $\downarrow$ & \textbf{AUROC} $\uparrow$ & \textbf{FPR95} $\downarrow$ & \textbf{AUROC} $\uparrow$ & \textbf{FPR95} $\downarrow$ & \textbf{AUROC} $\uparrow$ & \textbf{FPR95} $\downarrow$ & \textbf{AUROC} $\uparrow$ \\
        \midrule
        MSP                                      & 66.44 & 79.78 & 67.72 & 80.13 & 64.54 & 78.73 & 48.29 & 87.80 & 59.33 & 82.74 & 61.26 & 81.84 \\
         + $\approach(m)$                        & 77.97 & 76.93 & 81.08 & 75.06 & 74.70 & 78.05 & 68.93 & 81.66 & 74.72 & 78.59 & 75.48 & 78.06 \\
         + $\approach(\mu, \sigma)$              & 69.49 & 78.01 & 72.20 & 76.32 & 64.77 & 78.19 & 55.48 & 83.74 & 65.21 & 78.78 & 65.43 & 79.01 \\
        \cmidrule(lr){1-13}
        ODIN                                     & 88.14 & 43.79 & 89.48 & 42.23 & 78.94 & 54.79 & 83.03 & 49.31 & 88.18 & 42.02 & 85.55 & 46.43 \\
         + $\approach(m)$                        & 76.68 & 58.39 & 77.57 & 58.62 & 72.25 & 62.09 & 64.56 & 68.54 & 76.00 & 56.63 & 73.41 & 60.85 \\
         + $\approach(\mu, \sigma)$              & 80.25 & 51.87 & 82.21 & 51.00 & 74.59 & 58.07 & 69.14 & 61.94 & 81.56 & 48.23 & 77.55 & 54.22  \\
        \cmidrule(lr){1-13}
         Energy                                  & 84.07 & 58.30 & 81.84 & 59.67 & 72.02 & 66.46 & 75.73 & 67.70 & 78.72 & 60.14 & 78.48 & 62.45 \\
         + $\approach(m)$                        & 93.79 & 60.53 & 95.33 & 56.70 & 84.98 & 73.34 & 95.74 & 61.12 & 95.08 & 57.96 & 92.98 & 61.93 \\
         + $\approach(\mu, \sigma)$              & 85.39 & 56.92 & 83.66 & 57.56 & 72.11 & 66.34 & 80.23 & 63.12 & 81.91 & 56.20 & 80.66 & 60.03 \\
        \cmidrule(lr){1-13}
         GradNorm                                 & 78.53 & 82.27 & 78.15 & 79.97 & 77.87 & 76.93 & 78.22 & 86.25 & 65.83 & 85.78 & 75.72 & 82.24 \\
          + $\approach(m)$                        & 99.40 & 21.56 & 99.79 & 21.15 & 94.26 & 37.95 & 99.96 & 16.51 & 99.69 & 17.53 & 98.62 & 22.94 \\
          + $\approach(\mu, \sigma)$              & 99.27 & 21.14 & 99.01 & 21.76 & 93.71 & 35.51 & 99.92 & 16.52 & 99.33 & 17.09 & 98.25 & 22.41 \\
        \cmidrule(lr){1-13}
          KNN                                     & 82.36 & \underline{83.37} & 80.97 & 81.97 & 54.54 & 87.82 & 68.01 & \underline{91.08} & 59.82 & \underline{90.54} & 69.14 & \underline{86.96} \\
          + $\approach(m)$                        & 85.13 & 80.57 & 87.87 & 78.12 & \underline{36.45} & \underline{90.71} & 61.66 & 90.27 & 62.22 & 87.75 & 66.67 & 85.48 \\
           + $\approach(\mu, \sigma)$             & \underline{59.23}	& \textbf{87.34}	& \underline{66.20}	& \textbf{85.08} &\textbf{17.71} & \textbf{94.42} & \textbf{16.01} & \textbf{96.42} & \textbf{26.71}	& \textbf{94.69} & \textbf{37.17} & \textbf{91.59} \\
        \cmidrule(lr){1-13}
          ReAct                                    & \underline{67.49} & 81.94 & \textbf{65.68} & \underline{82.04} & 57.45 & 84.79 & \underline{49.74} & 90.74 & \underline{56.65} & 87.75 & \underline{59.40} & 85.45 \\
           + $\approach(m)$                        & 83.19 & 74.00 & 82.50 & 71.25 & 84.65 & 77.81 & 82.77 & 79.45 & 81.34 & 77.30 & 82.89 & 75.96 \\
           + $\approach(\mu, \sigma)$              & 84.29 & 81.42 & 79.56 & 81.50 & 90.04 & 78.40 & 83.76 & 87.17 & 75.33 & 86.72 & 82.59 & 83.04 \\
        \cmidrule(lr){1-13}
          DICE                                    & 91.63 & 35.01 & 96.45 & 25.15 & 66.01 & 67.84 & 97.95 & 17.93 & 92.73 & 30.44 & 88.95 & 35.27 \\
          + $\approach(m)$                        & 87.34 & 65.21 & 90.29 & 60.53 & 84.31 & 71.87 & 93.96 & 68.63 & 85.99 & 71.96 & 88.38 & 67.64 \\
           + $\approach(\mu, \sigma)$             & 86.15 & 50.32 & 91.63 & 41.82 & 66.72 & 74.64 & 95.35 & 36.33 & 86.16 & 52.42 & 85.20 & 51.11 \\
        \cmidrule(lr){1-13}
          ASH-S                                     & 99.36 & 20.18 & 99.59 & 21.37 & 98.65 & 18.41 & 99.81 & 10.69 & 99.84 & 11.94 & 99.45 & 16.52 \\
           + $\approach(m)$                         & 96.46 & 41.11 & 97.82 & 38.44 & 66.40 & 73.73 & 97.51 & 41.85 & 94.90 & 43.11 & 90.62 & 47.65 \\
            + $\approach(\mu, \sigma)$              & 93.91 & 48.81 & 96.28 & 44.86 & 57.78 & 80.39 & 92.51 & 55.97 & 92.05 & 50.98 & 86.51 & 56.20 \\
        \cmidrule(lr){1-13}                            
          SCALE                                    & 99.07 & 26.90 & 97.56 & 27.93 & 94.54 & 38.08 & 98.93 & 24.86 & 97.76 & 24.95 & 97.57 & 28.54 \\
           + $\approach(m)$                        & 97.23 & 36.04 & 98.58 & 33.71 & 66.81 & 71.39 & 97.93 & 39.20 & 95.67 & 39.18 & 91.24 & 43.90 \\
           + $\approach(\mu, \sigma)$              & 91.91 & 49.28 & 94.60 & 46.13 & 50.43 & 82.99 & 86.69 & 60.81 & 87.91 & 55.57 & 82.31 & 58.95 \\                           
        \bottomrule
        \end{tabular}}
        \label{table: imagenet_combination_swin}
    \end{table}
    \end{landscape}

    \begin{landscape}
    \begin{table}[!ht]
    \centering
     \caption{Detailed results of post-hoc methods combined with $\approach$ using \textbf{ConvNeXt-B} pre-trained on ImageNet-1K. $\boldsymbol{\uparrow}$ indicates higher is better; $\boldsymbol{\downarrow}$ indicates lower is better. The symbols denote the statistic used: $\mu$ (mean), $\sigma$ (std. deviation), $m$ (maximum).}
    \resizebox{0.85\linewidth}{!}{
    \begin{tabular}{ l  cc cc cc cc cc cc }
        \toprule
         \multirow{2}{*}{\textbf{Method}} & \multicolumn{2}{c}{\textbf{SUN}} & \multicolumn{2}{c}{\textbf{Place365}} & \multicolumn{2}{c}{\textbf{Textures}} & \multicolumn{2}{c}{\textbf{iNaturalist}} & \multicolumn{2}{c}{\textbf{OpenImage-O}} & \multicolumn{2}{c}{\textbf{Average}}  \\
        \cmidrule(lr){2-3} \cmidrule(lr){4-5} \cmidrule(lr){6-7} \cmidrule(lr){8-9} \cmidrule(lr){10-11} \cmidrule(lr){12-13}
        & \textbf{FPR95} $\downarrow$ & \textbf{AUROC} $\uparrow$ & \textbf{FPR95} $\downarrow$ & \textbf{AUROC} $\uparrow$ & \textbf{FPR95} $\downarrow$ & \textbf{AUROC} $\uparrow$ & \textbf{FPR95} $\downarrow$ & \textbf{AUROC} $\uparrow$ & \textbf{FPR95} $\downarrow$ & \textbf{AUROC} $\uparrow$ & \textbf{FPR95} $\downarrow$ & \textbf{AUROC} $\uparrow$ \\
        \midrule
        MSP                                      & 77.96 & 74.85 & 78.22 & 74.88 & 82.34 & 68.94 & 74.28 & 77.89 & 77.84 & 74.67 & 78.13 & 74.25 \\
         + $\approach(m)$                        & 85.68 & 65.46 & 86.91 & 64.02 & 85.94 & 65.59 & 84.59 & 66.68 & 84.13 & 67.13 & 85.45 & 65.78 \\
         + $\approach(\mu, \sigma)$              & 79.21 & 75.21 & 79.67 & 74.99 & 82.77 & 69.47 & 74.74 & 79.23 & 78.74 & 75.48 & 79.03 & 74.87 \\
        \cmidrule(lr){1-13}
        ODIN                                     & 80.37 & 53.65 & 85.86 & 49.67 & 71.44 & 66.14 & 68.86 & 69.62 & 81.22 & 55.74 & 77.55 & 58.96 \\
         + $\approach(m)$                        & 66.54 & 72.52 & 71.51 & 70.19 & 62.93 & 76.21 & 53.98 & 82.96 & 65.92 & 73.74 & 64.18 & 75.12 \\
         + $\approach(\mu, \sigma)$              & 69.75 & 65.80 & 74.92 & 62.78 & 64.82 & 71.97 & 56.61 & 79.54 & 71.82 & 65.94 & 67.58 & 69.21 \\
        \cmidrule(lr){1-13}
         Energy                                  & \underline{50.57} & \underline{89.24} & 51.92 & 88.89 & 70.12 & 75.63 & 30.71 & 94.09 & 48.21 & 88.82 & 50.31 & 87.33 \\
         + $\approach(m)$                        & 77.18 & 71.93 & 80.99 & 66.91 & 73.95 & 79.48 & 73.59 & 70.73 & 69.86 & 76.15 & 75.11 & 73.04 \\
         + $\approach(\mu, \sigma)$              & 51.32 & 89.09 & 52.67 & 88.79 & 69.75 & 76.50 & 31.62 & 93.91 & 48.48 & 88.82 & 50.77 & 87.42 \\
        \cmidrule(lr){1-13}
         GradNorm                                 & 95.04 & 41.58 & 94.01 & 45.25 & 98.40 & 19.87 & 91.10 & 53.95 & 95.60 & 37.07 & 94.83 & 39.55 \\
          + $\approach(m)$                        & 98.85 & 36.69 & 99.32 & 34.24 & 92.29 & 53.05 & 99.93 & 33.44 & 99.29 & 34.77 & 97.94 & 38.44 \\
          + $\approach(\mu, \sigma)$              & 98.28 & 36.96 & 98.96 & 35.02 & 92.80 & 50.95 & 99.72 & 37.04 & 98.72 & 35.80 & 97.70 & 39.15 \\
        \cmidrule(lr){1-13}
          KNN                                     & 71.10 & 85.34 & 71.67 & 83.92 & \underline{64.66} & \underline{85.76} & 69.26 & 89.48 & 59.38 & 89.31 & 67.21 & 86.76  \\
          + $\approach(m)$                        & 85.88 & 81.19 & 86.09 & 79.69 & 72.16 & 83.64 & 76.34 & 87.98 & 82.88 & 83.59 & 80.67 & 83.22  \\
           + $\approach(\mu, \sigma)$             & 51.49 & 87.70 & 56.99 & 85.61 & \textbf{31.05} & \textbf{90.71} & \textbf{25.89} & \textbf{94.96} & \textbf{31.73} & \textbf{92.97} & \textbf{39.43} & \textbf{90.23} \\
        \cmidrule(lr){1-13}
          ReAct                                   & 52.89 & 88.94 & 54.72 & 88.37 & 68.14 & 77.30 & 33.96 & 93.56 & 48.32 & \underline{89.35} & 51.60 & 87.50 \\
           + $\approach(m)$                       & 83.98 & 72.40 & 85.93 & 68.07 & 81.76 & 77.47 & 81.13 & 71.81 & 74.95 & 77.47 & 81.55 & 73.44 \\
           + $\approach(\mu, \sigma)$             & 62.36 & 86.56 & 62.16 & 86.67 & 75.21 & 72.01 & 43.29 & 92.01 & 56.07 & 87.42 & 59.82 & 84.94 \\
        \cmidrule(lr){1-13}
          DICE                                    & 60.65 & 85.04 & 64.94 & 83.37 & 67.66 & 75.52 & 45.76 & 90.63 & 54.89 & 86.50 & 58.78 & 84.21 \\
          + $\approach(m)$                        & 79.48 & 70.75 & 85.20 & 64.89 & 77.57 & 76.52 & 78.49 & 70.12 & 73.46 & 75.33 & 78.84 & 71.52 \\
           + $\approach(\mu, \sigma)$             & 62.44 & 84.53 & 66.49 & 82.79 & 70.69 & 73.78 & 49.83 & 90.05 & 57.60 & 85.92 & 61.41 & 83.41 \\
        \cmidrule(lr){1-13}
          ASH                                     & 53.01 & 87.51 & \underline{48.51} & \underline{88.33} & 79.11 & 76.68 & 36.88 & 91.60 & 54.72 & 85.66 & 54.45 & 85.96 \\
           + $\approach(m)$                       & 98.30 & 39.80 & 98.99 & 36.72 & 80.92 & 65.85 & 99.69 & 35.13 & 97.94 & 42.08 & 95.17 & 43.92 \\
            + $\approach(\mu, \sigma)$            & 94.53 & 50.09 & 96.10 & 46.19 & 68.83 & 72.03 & 94.48 & 50.22 & 91.12 & 50.16 & 89.01 & 53.74 \\
        \cmidrule(lr){1-13}                            
          SCALE                                   & \textbf{48.71} & \textbf{89.57} & \textbf{45.64} & \textbf{90.01} & 72.54 & 78.46 & \underline{30.34} & \underline{93.80} & \underline{47.65} & 88.86 & \underline{48.97} & \underline{88.14} \\
           + $\approach(m)$                       & 95.31 & 52.56 & 96.87 & 47.33 & 72.20 & 74.94 & 96.90 & 49.97 & 92.75 & 56.59 & 90.80 & 56.28 \\
           + $\approach(\mu, \sigma)$             & 83.44 & 62.69 & 87.27 & 59.18 & 56.37 & 78.19 & 77.71 & 68.00 & 77.36 & 62.82 & 76.43 & 66.18 \\                           
        \bottomrule
        \end{tabular}}
        \label{table: imagenet_combination_convnext}
    \end{table}
    \end{landscape}

    \begin{landscape}
    \begin{table}[!ht]
    \centering
     \caption{Detailed results of post-hoc methods combined with $\approach$ using \textbf{ResNet-50} pre-trained on ImageNet-1K. $\boldsymbol{\uparrow}$ indicates higher is better; $\boldsymbol{\downarrow}$ indicates lower is better. The symbols denote the statistic used: $\mu$ (mean), $\sigma$ (std. deviation), $m$ (maximum).}
    \resizebox{0.85\linewidth}{!}{
    \begin{tabular}{ l  cc cc cc cc cc cc }
        \toprule
         \multirow{2}{*}{\textbf{Method}} & \multicolumn{2}{c}{\textbf{SUN}} & \multicolumn{2}{c}{\textbf{Place365}} & \multicolumn{2}{c}{\textbf{Textures}} & \multicolumn{2}{c}{\textbf{iNaturalist}} & \multicolumn{2}{c}{\textbf{OpenImage-O}} & \multicolumn{2}{c}{\textbf{Average}}  \\
        \cmidrule(lr){2-3} \cmidrule(lr){4-5} \cmidrule(lr){6-7} \cmidrule(lr){8-9} \cmidrule(lr){10-11} \cmidrule(lr){12-13}
        & \textbf{FPR95} $\downarrow$ & \textbf{AUROC} $\uparrow$ & \textbf{FPR95} $\downarrow$ & \textbf{AUROC} $\uparrow$ & \textbf{FPR95} $\downarrow$ & \textbf{AUROC} $\uparrow$ & \textbf{FPR95} $\downarrow$ & \textbf{AUROC} $\uparrow$ & \textbf{FPR95} $\downarrow$ & \textbf{AUROC} $\uparrow$ & \textbf{FPR95} $\downarrow$ & \textbf{AUROC} $\uparrow$ \\
        \midrule
        MSP                                      & 69.11 & 81.64 & 72.06 & 80.54 & 66.26 & 80.43 & 52.83 & 88.39 & 66.97 & 83.89 & 65.45 & 82.98 \\
         + $\approach(m)$                        & 81.56 & 75.85 & 83.35 & 74.42 & 71.26 & 80.95 & 75.63 & 79.93 & 76.89 & 79.65 & 77.74 & 78.16 \\
         + $\approach(\mu, \sigma)$              & 72.96 & 79.77 & 75.22 & 78.59 & 64.49 & 81.44 & 58.92 & 85.81 & 67.98 & 82.79 & 67.91 & 81.68 \\
        \cmidrule(lr){1-13}
        ODIN                                     & 57.11 & 86.77 & 64.69 & 84.12 & 47.30 & 87.82 & 41.82 & 92.25 & 59.15 & 87.54 & 54.01 & 87.70 \\
         + $\approach(m)$                        & 58.10 & 86.94 & 65.19 & 84.48 & 49.88 & 87.66 & 45.37 & 91.92 & 61.19 & 87.62 & 55.95 & 87.72 \\
         + $\approach(\mu, \sigma)$              & 57.27 & 86.75 & 64.96 & 84.12 & 47.75 & 87.74 & 42.39 & 92.19 & 59.75 & 87.49 & 54.42 & 87.66 \\
        \cmidrule(lr){1-13}
         Energy                                  & 58.82 & 86.58 & 65.99 & 83.96 & 52.43 & 86.72 & 53.74 & 90.62 & 64.70 & 87.08 & 59.14 & 86.99 \\
         + $\approach(m)$                        & 53.12 & 85.49 & 63.58 & 80.54 & 20.94 & 95.09 & 36.11 & 91.29 & 40.83 & 90.56 & 42.92 & 88.59 \\
         + $\approach(\mu, \sigma)$              & 54.71 & 86.47 & 64.06 & 82.55 & 21.13 & 95.32 & 34.50 & 92.95 & 41.21 & 91.32 & 43.12 & 89.72 \\
        \cmidrule(lr){1-13}
         GradNorm                                 & 37.42 & 90.10 & 48.88 & 86.08 & 32.84 & 90.64 & 26.78 & 93.9 & 57.76 & 80.44 & 40.74 & 88.23 \\
          + $\approach(m)$                        & 58.87 & 78.88 & 71.83 & 70.67 & 19.47 & 94.61 & 41.84 & 88.48 & 59.86 & 77.39 & 50.37 & 82.01 \\
          + $\approach(\mu, \sigma)$              & 46.08 & 86.44 & 58.35 & 80.76 & 23.72 & 93.85 & 31.31 & 92.39 & 58.09 & 79.05 & 43.51 & 86.50 \\
        \cmidrule(lr){1-13}
          KNN                                      & 78.95 & 77.44 & 81.86 & 73.91 & 16.05 & 96.11 & 78.33 & 79.15 & 65.73 & 82.27 & 64.18 & 81.78 \\
          + $\approach(m)$                         & 54.28 & 86.75 & 61.99 & 82.79 & 18.60 & 96.41 & 42.36 & 91.23 & 66.56 & 83.12 & 48.76 & 88.06 \\
           + $\approach(\mu, \sigma)$              & 51.30 & 87.97 & 60.48 & 83.81 & 20.30 & 95.99 & 44.62 & 91.36 & 64.67 & 84.41 & 48.27 & 88.71 \\
        \cmidrule(lr){1-13}
          ReAct                                    & 23.95 & 94.46 & 33.48 & 91.97 & 46.4 & 90.31 & 19.56 & 96.4 & 49.78 & 89.06 & 34.64 & 92.44 \\
           + $\approach(m)$                        & 43.97 & 89.54 & 53.47 & 85.96 & 26.81 & 93.91 & 30.74 & 93.64 & 42.57 & 90.67 & 39.51 & 90.74 \\
           + $\approach(\mu, \sigma)$              & 34.14 & 92.45 & 44.47 & 89.64 & 34.11 & 92.88 & 21.69 & 95.91 & 43.74 & 91.30 & 35.63 & 92.44 \\
        \cmidrule(lr){1-13}
          DICE                                    & 36.49 & 90.92 & 47.93 & 87.65 & 32.59 & 90.45 & 26.61 & 94.51 & 54.67 & 85.67 & 39.66 & 89.84 \\
          + $\approach(m)$                        & 49.67 & 84.39 & 63.99 & 77.73 & 15.87 & 96.21 & 31.88 & 91.98 & 49.12 & 85.14 & 42.11 & 87.09 \\
           + $\approach(\mu, \sigma)$             & 37.38 & 90.32 & 50.34 & 86.42 & 19.75 & 94.92 & 21.76 & 95.33 & 46.10 & 88.05 & 35.07 & 91.01 \\
        \cmidrule(lr){1-13}
          ASH                                     & 28.00 & 94.04 & 39.67 & 91.03 & 11.88 & 97.62 & 11.41 & 97.88 & 38.70 & 90.79 & 25.93 & 94.27 \\
           + $\approach(m)$                       & 29.89 & 93.61 & 43.08 & 89.89 &  \textbf{9.22} & \textbf{98.26} & 15.25 & 97.21 & 39.73 & 91.08 & 27.43 & 94.01 \\
            + $\approach(\mu, \sigma)$            & 27.67 & 94.02 & 40.01 & 90.73 &  \underline{9.91} & \underline{98.07} & 11.69 & 97.81 & 37.48 & 90.92 & 25.35 & 94.31 \\
        \cmidrule(lr){1-13}                            
          SCALE                                   & 25.78 & 94.54 & 36.86 & 91.96 & 14.56 & 96.75 & \underline{10.37} & \underline{98.02} & 36.23 & 92.30 & \underline{24.76} & \underline{94.71} \\
           + $\approach(m)$                        & \underline{13.26} & \underline{97.37} & \underline{27.79} & \underline{93.66} & 40.56 & 90.02 &  \textbf{9.52} & \textbf{98.11} & \underline{34.48} & \underline{92.50} & 25.12 & 94.33 \\
           + $\approach(\mu, \sigma)$              &  \textbf{9.61} & \textbf{98.10} & \textbf{24.49} & \textbf{94.62} & 36.01 & 91.83 & 10.59 & 97.75 & \textbf{33.30} & \textbf{92.85} & \textbf{22.80} & \textbf{95.03} \\                          
        \bottomrule
        \end{tabular}}
        \label{table: imagenet_combination_resnet}
    \end{table}
    \end{landscape}

\section{Alternate Statistics: Median and Shannon Entropy}
\label{appendix: alternate statistics}

To justify our choice of mean, variance, and maximum statistics, we conducted an ablation study exploring two alternatives: the median and Shannon entropy. While an activation map encodes many statistical properties, for our framework to be effective, the chosen statistic must produce a distinctive signature for ID versus OOD samples. For the purpose of this study, we limit our experiment to SUN, Place, Texture and iNaturalist (ImageNet benchmark)

\textbf{Median.} We begin by extracting the \textit{median} from each activation map of \( g(\mathbf{x}) \in \mathbb{R}^{n \times k \times k} \), transforming it into an \(n\)-dimensional feature vector \( h(\mathbf{x}) \in \mathbb{R}^n \) using global median pooling, as defined in Equation~\ref{eq: median_pooling}:

\begin{equation}
    h(\mathbf{x}) = \texttt{median}\left( g(\mathbf{x}) \right)
    \label{eq: median_pooling}
\end{equation}

Here, \texttt{median} denotes a global median pooling operation applied independently to each of the \( n \) activation maps in \( g(\mathbf{x}) \). 

\textbf{Shannon Entropy}. In addition to the median, we compute the \textit{Shannon entropy} for each activation map. For the \( i \)-th channel activation \( g_i(\mathbf{x}) \in \mathbb{R}^{k \times k} \), the entropy is computed as shown in Equation~\ref{eq: entropy}. To do so, we first flatten \( g_i(\mathbf{x}) \) into a vector of length \( k^2 \), and normalize it to define a discrete probability distribution \( p_{ij} \), as described in Equation~\ref{eq: prob}. By collecting the entropy values across all channels, we obtain the final feature representation \( h(\mathbf{x}) \in \mathbb{R}^n \), as defined in Equation~\ref{eq: entropy_pooling} .

\begin{small}
    \begin{equation}
        p_{ij} = \frac{g_i(\mathbf{x})_j}{\sum_{l=1}^{k^2} g_i(\mathbf{x})_l}, \quad j = 1, \dots, k^2
        \label{eq: prob}
    \end{equation}  
\end{small}
\begin{small}
    \begin{equation}
        \texttt{entropy}_i(\mathbf{x}) = -\sum_{j=1}^{k^2} p_{ij} \log p_{ij}
        \label{eq: entropy}
\end{equation}  
\end{small}
\begin{small}
    \begin{equation}
        h(\mathbf{x}) = \texttt{entropy}\left( g(\mathbf{x}) \right) = 
        \left[ \texttt{entropy}_1(\mathbf{x}), \dots, \texttt{entropy}_n(\mathbf{x}) \right]^\top
        \label{eq: entropy_pooling}
    \end{equation}
\end{small}

\begin{figure}[ht]
  \centering
  \includegraphics[width=\textwidth]{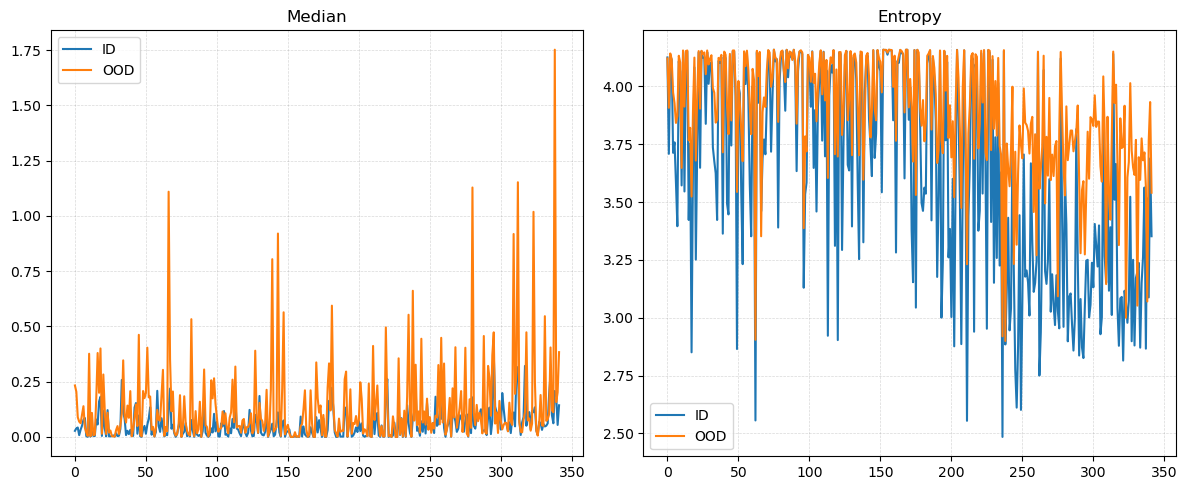}
  \caption{ \textit{Illustration of penultimate layer features derived from median (\textbf{left}) and entropy (\textbf{right}) statistics. For both measures, out-of-distribution (OOD) samples (Texture) exhibit consistently higher values than in-distribution (ID) samples (CIFAR-100), creating an "inverted separation" that is challenging for standard OOD scoring. (Model: DenseNet-101).}}
  \label{fig: median_entropy}
\end{figure}

The performance of these alternate statistics is presented in Table~\ref{table: ablation_median_entropy_imagenet} (ImageNet),  Table~\ref{table: ablation_median_entropy_cifar-10} and ~\ref{table: ablation_median_entropy_cifar-100} (CIFAR). The results are unambiguous: representations built from the maximum and mean/variance statistics are overwhelmingly superior to those from the median and entropy. The entropy-based features, in particular, perform poorly, with FPR95 scores often approaching 100\%, rendering them ineffective for OOD detection. The median-based features also struggle significantly, with FPR95 scores consistently above 80\% on ImageNet.

The reason for this poor performance is illustrated in the distributions shown in Figure~\ref{fig: median_entropy}. Both median and entropy produce feature representations where OOD samples consistently yield higher-magnitude values than ID samples. This creates an "inverted separation" that is fundamentally at odds with standard scoring functions (like energy), which assume higher scores correspond to ID samples. This confirms our central hypothesis: effective OOD detection requires statistics that specifically amplify the characteristic signals of ID samples relative to OOD samples, a property that the maximum and variance possess, but the median and entropy do not in this context.

\begin{table}[ht]
    \centering
    \caption{Performance of various pre-pooling statistics on the ImageNet-1K benchmark. To isolate the effect of each statistic, results are based solely on the energy score without post-hoc methods applied. $\boldsymbol{\downarrow}$ indicates that lower values are better, while $\boldsymbol{\uparrow}$ indicates that higher values are better.}
    \resizebox{\linewidth}{!}{
    \begin{tabular}{ l l cc cc cc cc cc}
    \toprule
    \multirow{2}{*}{\textbf{Model}} & \multirow{2}{*}{\textbf{Statistics}} & \multicolumn{2}{c}{\textbf{SUN}} & \multicolumn{2}{c}{\textbf{Place}} & \multicolumn{2}{c}{\textbf{Texture}} & \multicolumn{2}{c}{\textbf{iNaturalist}} & \multicolumn{2}{c}{\textbf{Average}}  \\
    \cmidrule(lr){3-4} \cmidrule(lr){5-6} \cmidrule(lr){7-8} \cmidrule(lr){9-10} \cmidrule(lr){11-12}

    && \textbf{FPR95} $\downarrow$ & \textbf{AUROC} $\uparrow$ & \textbf{FPR95} $\downarrow$ & \textbf{AUROC} $\uparrow$ & \textbf{FPR95} $\downarrow$ & \textbf{AUROC} $\uparrow$ & \textbf{FPR95} $\downarrow$ & \textbf{AUROC} $\uparrow$ & \textbf{FPR95} $\downarrow$ & \textbf{AUROC} $\uparrow$ \\
    \midrule
    \multirow{6}{*}{DenseNet-121}  & $\mu(\mathbf{x})$                          & 52.51 & 87.27 & 58.24 & 85.05 & 52.22 & 85.42 & 39.75 & 92.66 & 50.68 & 87.60 \\
                                   & $\mu(\mathbf{x}) + 2.0\sigma(\mathbf{x})$  & 48.40 & 88.64 & 58.38 & 84.94 & 25.37 & 94.14 & 30.31 & 93.79 & 40.62 & 90.38 \\
                                   & $m(\mathbf{x})$                            & 48.55 & 88.03 & 59.36 & 83.53 & 27.54 & 93.62 & 33.74 & 92.20 & 42.30 & 89.34 \\
                                   & $median(\mathbf{x})$                       & 91.02 & 73.85 & 88.85 & 73.68 & 87.43 & 67.32 & 90.3 & 78.44 & 89.4 & 73.32 \\
                                   & $entropy(\mathbf{x})$                      & 99.84 & 22.68 & 99.82 & 25.15 & 99.65 & 39.69 & 99.94 & 20.19 & 99.81 & 26.93 \\
    \midrule
    \multirow{6}{*}{ResNet-50}     & $\mu(\mathbf{x})$                          & 58.82 & 86.58 & 65.99 & 83.96 & 52.43 & 86.72 & 53.74 & 90.62 & 57.74 & 86.97 \\
                                   & $\mu(\mathbf{x}) + 2.0\sigma(\mathbf{x})$  & 54.71 & 86.47 & 64.06 & 82.55 & 21.13 & 95.32 & 34.50 & 92.95 & 43.60 & 89.33 \\
                                   & $m(\mathbf{x})$                            & 53.12 & 85.49 & 63.58 & 80.54 & 20.94 & 95.09 & 36.11 & 91.29 & 43.44 & 88.10 \\
                                   & $median(\mathbf{x})$                       & 82.74 & 79.96 & 83.14 & 78.55 & 82.66 & 76.52 & 86.80 & 82.26 & 83.83 & 79.32 \\
                                   & $entropy(\mathbf{x})$                      & 99.84 & 27.69 & 99.65 & 30.39 & 99.36 & 43.69 & 99.69 & 36.11 & 99.64 & 34.47 \\
    \midrule
    \multirow{6}{*}{MobileNet-v2}  & $\mu(\mathbf{x})$                          & 59.60 & 86.16 & 66.36 & 83.15 & 54.82 & 86.57 & 55.33 & 90.37 & 59.03 & 86.56 \\
                                   & $\mu(\mathbf{x}) + 2.0\sigma(\mathbf{x})$  & 55.47 & 87.15 & 64.82 & 83.18 & 25.78 & 94.33 & 38.94 & 92.98 & 46.25 & 89.41 \\
                                   & $m(\mathbf{x})$                            & 53.64 & 87.26 & 63.61 & 82.93 & 25.87 & 94.03 & 40.25 & 92.22 & 45.84 & 89.11 \\
                                   & $median(\mathbf{x})$                       & 85.12 & 78.08 & 85.02 & 76.46 & 83.88 & 74.62 & 86.35 & 80.98 & 85.09 & 77.53 \\ 
                                   & $entropy(\mathbf{x})$                      & 99.61 & 19.59 & 99.12 & 25.53 & 99.77 & 18.35 & 99.58 & 27.55 & 99.52 & 22.75 \\
    \midrule
    \multirow{6}{*}{EfficientNet-b0} & $\mu(\mathbf{x})$                          & 85.01 & 72.86 & 86.06 & 70.99 & 75.99  & 75.86 & 78.91  & 79.78 & 81.49 & 74.87 \\
                                     & $\mu(\mathbf{x}) + 2.0\sigma(\mathbf{x})$  & 61.85 & 84.67 & 71.31 & 79.01 & 12.18  & 97.52 & 48.25  & 88.35 & 48.40 & 87.39 \\
                                     & $m(\mathbf{x})$                            & 69.42 & 80.74 & 79.09 & 73.67 & 15.82  & 96.38 & 63.68  & 80.63 & 57.00 & 82.85 \\                                
                                     & $median(\mathbf{x})$                       & 99.92 & 17.72 & 99.71 & 22.15 & 100.00 & 3 .80 & 100.00 & 10.63 & 99.91 & 13.57 \\
                                     & $entropy(\mathbf{x})$                      & 97.25 & 47.42 & 97.96 & 41.37 & 81.35  & 76.51 & 98.14  & 38.82 & 93.67 & 51.03 \\                            
    \bottomrule
    \end{tabular}
    }
    \label{table: ablation_median_entropy_imagenet}
\end{table}

\begin{landscape}
\begin{table}[ht]
    \centering
     \caption{Performance of various pre-pooling statistics on the CIFAR-10. To isolate the effect of each statistic, results are based solely on the energy score without any additional post-hoc methods applied. $\boldsymbol{\downarrow}$ indicates that lower values are better, while $\boldsymbol{\uparrow}$ indicates that higher values are better.}
    \resizebox{\linewidth}{!}{
    \begin{tabular}{ c l cc cc cc cc cc cc cc}
    \toprule
     \multirow{2}{*}{\textbf{Model}} & \multirow{2}{*}{\textbf{Statistics}} & \multicolumn{2}{c}{\textbf{SVHN}} & \multicolumn{2}{c}{\textbf{Place365}} & \multicolumn{2}{c}{\textbf{iSUN}} & \multicolumn{2}{c}{\textbf{Textures}} & \multicolumn{2}{c}{\textbf{LSUN-c}}  & \multicolumn{2}{c}{\textbf{LSUN-r}} & \multicolumn{2}{c}{\textbf{Average}}  \\
    \cmidrule(lr){3-4} 
    \cmidrule(lr){5-6} 
    \cmidrule(lr){7-8} 
    \cmidrule(lr){9-10} 
    \cmidrule(lr){11-12} 
    \cmidrule(lr){13-14} 
    \cmidrule(lr){15-16}
    && \textbf{FPR95} $\downarrow$ & \textbf{AUROC} $\uparrow$ & \textbf{FPR95} $\downarrow$ & \textbf{AUROC} $\uparrow$ & \textbf{FPR95} $\downarrow$ & \textbf{AUROC} $\uparrow$ & \textbf{FPR95} $\downarrow$ & \textbf{AUROC} $\uparrow$ & \textbf{FPR95} $\downarrow$ & \textbf{AUROC} $\uparrow$ & \textbf{FPR95} $\downarrow$ & \textbf{AUROC} $\uparrow$ & \textbf{FPR95} $\downarrow$ & \textbf{AUROC} $\uparrow$ \\
    \midrule
     \multirow{6}{*}{ResNet-18}
     & $\mu(\mathbf{x})$                             & 44.32 & 94.04 & 41.31 & 91.73 & 35.46 & 94.64 & 50.39 & 91.12 &  9.77 & 98.19 & 32.41 & 95.16 & 35.61 & 94.14 \\
     & $\mu(\mathbf{x}) + 3.0\sigma(\mathbf{x})$     & 19.83 & 96.47 & 36.11 & 93.16 & 20.48 & 96.90 & 26.77 & 95.89 &  5.62 & 98.79 & 17.75 & 97.21 & 21.09 & 96.40 \\
     & $m(\mathbf{x})$                               & 19.81 & 96.29 & 32.32 & 93.73 & 15.79 & 97.37 & 21.90 & 96.44 &  5.83 & 98.73 & 13.63 & 97.61 & 18.21 & 96.69 \\
     & $median(\mathbf{x})$                          & 99.57 & 31.77 & 88.60 & 63.39 & 96.20 & 50.98 & 99.22 & 31.82 & 95.09 & 65.21 & 95.68 & 51.78 & 95.73 & 49.16 \\
     & $entropy(\mathbf{x})$                         & 98.01 & 78.03 & 98.99 & 52.54 & 99.99 & 43.91 & 99.52 & 65.07 & 99.90 & 62.83 & 99.98 & 41.74 & 99.40 & 57.35 \\
    \cmidrule(lr){2-16}
    \multirow{6}{*}{ResNet-34} 
    & $\mu(\mathbf{x})$                              & 35.44 & 93.76 & 38.15 & 92.27 & 19.90 & 96.87 & 42.52 & 92.54 &  3.38 & 99.10 & 16.86 & 97.20 & 26.04 & 95.29 \\
    & $\mu(\mathbf{x}) + 3.0\sigma(\mathbf{x})$      & 28.96 & 94.90 & 32.66 & 93.69 & 11.31 & 97.88 & 23.07 & 96.39 &  3.47 & 99.22 & 10.35 & 98.04 & 18.30 & 96.69 \\
    & $m(\mathbf{x})$                                & 29.29 & 94.59 & 29.59 & 94.21 & 10.17 & 98.03 & 20.23 & 96.73 &  4.00 & 99.10 &  9.45 & 98.16 & 17.12 & 96.81 \\
    & $median(\mathbf{x})$                           & 99.80 & 24.42 & 96.12 & 55.41 & 99.24 & 48.47 & 99.86 & 27.82 & 97.35 & 58.12 & 98.91 & 51.98 & 98.55 & 44.37 \\
    & $entropy(\mathbf{x})$                          & 80.63 & 89.37 & 97.05 & 62.50 & 99.73 & 61.00 & 96.52 & 79.21 & 98.92 & 78.69 & 99.66 & 60.31 & 95.42 & 71.85 \\
    \cmidrule(lr){2-16}
    \multirow{6}{*}{DenseNet-101} 
    & $\mu(\mathbf{x})$                             & 37.91 & 93.59 & 36.38 & 92.39 &  7.83 & 98.23 & 43.85 & 90.49 &  1.95 & 99.47 &  7.34 & 98.34 & 22.54 & 95.42 \\
    & $\mu(\mathbf{x}) + 3.0\sigma(\mathbf{x})$     & 24.75 & 95.93 & 32.08 & 93.43 &  5.77 & 98.64 & 25.23 & 95.84 &  5.13 & 98.94 &  5.63 & 98.70 & 16.43 & 96.91 \\
    & $m(\mathbf{x})$                               & 30.50 & 94.50 & 33.04 & 93.28 &  7.81 & 98.29 & 25.16 & 95.85 &  9.95 & 98.19 &  7.32 & 98.37 & 18.96 & 96.41 \\
    & $median(\mathbf{x})$                          & 96.80 & 78.31 & 98.70 & 55.92 & 99.78 & 72.66 & 98.55 & 70.95 & 99.49 & 67.91 & 99.76 & 72.17 & 98.85 & 69.65 \\
    & $entropy(\mathbf{x})$                         & 92.28 & 69.08 & 83.00 & 75.15 & 87.27 & 79.73 & 95.92 & 58.62 & 17.72 & 96.86 & 87.34 & 79.75 & 77.26 & 76.53 \\
    \cmidrule(lr){2-16}
    \multirow{6}{*}{MobileNet-v2} 
    & $\mu(\mathbf{x})$                             & 75.83 & 85.85 & 44.98 & 90.62 & 29.68 & 95.03 & 48.67 & 91.19 &  9.54 & 98.12 & 29.80 & 95.10 & 39.75 & 92.65 \\
    & $\mu(\mathbf{x}) + 3.0\sigma(\mathbf{x})$     & 62.07 & 89.18 & 43.87 & 90.86 & 22.38 & 96.40 & 33.24 & 94.71 & 14.21 & 97.50 & 21.32 & 96.57 & 32.85 & 94.20 \\
    & $m(\mathbf{x})$                               & 69.61 & 87.56 & 44.33 & 91.00 & 23.33 & 96.26 & 37.87 & 94.11 & 17.62 & 97.02 & 22.48 & 96.40 & 35.87 & 93.73 \\
    & $median(\mathbf{x})$                          & 89.92 & 80.06 & 60.84 & 85.74 & 54.69 & 89.94 & 66.37 & 84.49 & 12.42 & 97.68 & 54.48 & 89.74 & 56.45 & 87.94 \\
    & $entropy(\mathbf{x})$                         & 99.75 & 80.37 & 99.68 & 55.79 & 99.94 & 63.55 & 99.75 & 71.94 & 99.56 & 72.44 & 99.95 & 62.49 & 99.77 & 67.76 \\
    \bottomrule                                                              
    \end{tabular}}
    \label{table: ablation_median_entropy_cifar-10}
\end{table}
\end{landscape}

\begin{landscape}
\begin{table}[ht]
    \centering
     \caption{Performance of various pre-pooling statistics on CIFAR-100 benchmarks. To isolate the effect of each statistic, results are based solely on the energy score without post-hoc methods applied. $\boldsymbol{\downarrow}$ indicates that lower values are better, while $\boldsymbol{\uparrow}$ indicates that higher values are better.}
    \resizebox{\linewidth}{!}{
    \begin{tabular}{ c l cc cc cc cc cc cc cc}
    \toprule
     \multirow{2}{*}{\textbf{Model}} & \multirow{2}{*}{\textbf{Statistics}} & \multicolumn{2}{c}{\textbf{SVHN}} & \multicolumn{2}{c}{\textbf{Place365}} & \multicolumn{2}{c}{\textbf{iSUN}} & \multicolumn{2}{c}{\textbf{Textures}} & \multicolumn{2}{c}{\textbf{LSUN-c}}  & \multicolumn{2}{c}{\textbf{LSUN-r}} & \multicolumn{2}{c}{\textbf{Average}}  \\
    \cmidrule(lr){3-4} 
    \cmidrule(lr){5-6} 
    \cmidrule(lr){7-8} 
    \cmidrule(lr){9-10} 
    \cmidrule(lr){11-12} 
    \cmidrule(lr){13-14} 
    \cmidrule(lr){15-16}
    && \textbf{FPR95} $\downarrow$ & \textbf{AUROC} $\uparrow$ & \textbf{FPR95} $\downarrow$ & \textbf{AUROC} $\uparrow$ & \textbf{FPR95} $\downarrow$ & \textbf{AUROC} $\uparrow$ & \textbf{FPR95} $\downarrow$ & \textbf{AUROC} $\uparrow$ & \textbf{FPR95} $\downarrow$ & \textbf{AUROC} $\uparrow$ & \textbf{FPR95} $\downarrow$ & \textbf{AUROC} $\uparrow$ & \textbf{FPR95} $\downarrow$ & \textbf{AUROC} $\uparrow$ \\
    \midrule
     \multirow{6}{*}{ResNet-18} 
     & $\mu(\mathbf{x})$                         & 66.64 & 89.53 & 81.23 & 76.84 & 73.67 & 82.01 & 85.30 & 75.68 & 48.01 & 91.63 & 70.30 & 83.38 & 70.86 & 83.18 \\
     & $\mu(\mathbf{x}) + 3.0\sigma(\mathbf{x})$ & 42.01 & 93.75 & 78.99 & 77.58 & 61.13 & 88.60 & 59.89 & 87.89 & 38.56 & 93.48 & 59.76 & 88.86 & 56.72 & 88.36 \\
     & $\mu(\mathbf{x}) + 4.0\sigma(\mathbf{x})$ & 40.43 & 93.92 & 78.88 & 77.54 & 59.96 & 88.98 & 57.50 & 88.55 & 38.04 & 93.55 & 58.75 & 89.18 & 55.59 & 88.62 \\
     & $median(\mathbf{x})$                      & 92.21 & 76.58 & 86.71 & 72.47 & 88.20 & 68.61 & 96.58 & 56.13 & 74.53 & 84.11 & 85.97 & 71.06 & 87.37 & 71.50 \\
     & $entropy(\mathbf{x})$                     & 99.68 & 22.01 & 97.29 & 50.70 & 99.46 & 49.21 & 98.37 & 38.77 & 99.57 & 32.48 & 99.34 & 50.35 & 98.95 & 40.59 \\
    \cmidrule(lr){2-16}
    \multirow{6}{*}{ResNet-34} 
    & $\mu(\mathbf{x})$                         & 57.79 & 89.80 & 81.17 & 77.25 & 71.83 & 84.14 & 86.77 & 75.82 & 55.56 & 89.92 & 68.70 & 84.93 & 70.30 & 83.64 \\
    & $\mu(\mathbf{x}) + 3.0\sigma(\mathbf{x})$ & 31.07 & 95.00 & 80.04 & 77.53 & 59.46 & 89.50 & 59.75 & 88.31 & 41.21 & 93.00 & 60.05 & 89.32 & 55.26 & 88.78 \\
    & $m(\mathbf{x})$                           & 29.56 & 95.19 & 77.68 & 78.81 & 55.37 & 90.51 & 56.29 & 89.14 & 38.37 & 93.37 & 56.27 & 90.23 & 52.26 & 89.54 \\
    & $median(\mathbf{x})$                      & 99.80 & 24.42 & 96.12 & 55.41 & 99.24 & 48.47 & 99.86 & 27.82 & 97.35 & 58.12 & 98.91 & 51.98 & 98.55 & 44.37 \\
    & $entropy(\mathbf{x})$                     & 80.63 & 89.37 & 97.05 & 62.50 & 99.73 & 61.00 & 96.52 & 79.21 & 98.92 & 78.69 & 99.66 & 60.31 & 95.42 & 71.85 \\
    \cmidrule(lr){2-16}
    \multirow{6}{*}{DenseNet-101} 
    & $\mu(\mathbf{x})$                         & 70.99 & 86.66 & 77.12 & 76.94 & 64.28 & 83.92 & 83.60 & 67.47 & 11.45 & 97.89 & 56.08 & 86.84 & 60.59 & 83.29 \\
    & $\mu(\mathbf{x}) + 3.0\sigma(\mathbf{x})$ & 45.23 & 91.91 & 72.11 & 79.78 & 42.21 & 92.16 & 55.85 & 85.47 & 13.92 & 97.53 & 40.52 & 92.45 & 44.97 & 89.88 \\
    & $m(\mathbf{x})$                           & 52.32 & 89.41 & 73.49 & 79.62 & 35.76 & 93.37 & 49.88 & 88.07 & 18.58 & 96.44 & 37.27 & 92.98 & 44.55 & 89.98 \\
    & $median(\mathbf{x})$                      & 99.42 & 41.70 & 94.57 & 57.14 & 98.23 & 48.13 & 99.66 & 28.11 & 63.70 & 85.71 & 96.48 & 55.05 & 92.01 & 52.64 \\
    & $entropy(\mathbf{x})$                     & 98.87 & 48.54 & 97.38 & 47.45 & 97.54 & 56.47 & 96.56 & 56.33 & 99.15 & 48.05 & 96.62 & 62.24 & 97.69 & 53.18 \\
    \cmidrule(lr){2-16}
    \multirow{6}{*}{MobileNet-v2} 
    & $\mu(\mathbf{x})$                         & 69.65 & 85.98 & 81.26 & 75.21 & 78.17 & 83.33 & 80.02 & 78.63 & 50.19 & 89.42 & 76.59 & 84.06 & 72.65 & 82.77 \\
    & $\mu(\mathbf{x}) + 3.0\sigma(\mathbf{x})$ & 40.00 & 92.73 & 81.64 & 73.23 & 64.99 & 86.36 & 42.70 & 91.25 & 34.67 & 93.75 & 66.30 & 86.04 & 55.05 & 87.23 \\
    & $m(\mathbf{x})$                           & 41.94 & 92.00 & 81.49 & 73.31 & 64.52 & 86.40 & 44.22 & 90.68 & 34.22 & 93.72 & 65.30 & 86.10 & 55.28 & 87.04 \\
    & $median(\mathbf{x})$                      & 96.12 & 73.46 & 84.46 & 73.88 & 91.78 & 75.12 & 94.52 & 62.93 & 71.15 & 81.93 & 90.61 & 76.70 & 88.11 & 74.00 \\
    & $entropy(\mathbf{x})$                     & 99.81 & 30.47 & 99.13 & 33.71 & 99.94 & 33.94 & 99.73 & 36.10 & 99.86 & 40.63 & 99.95 & 32.66 & 99.74 & 34.59 \\ 
    \bottomrule                                                              
    \end{tabular}
    }
    \label{table: ablation_median_entropy_cifar-100}
\end{table}
\end{landscape}

\section{Analysis of the Hyperparameter $\gamma$}
\label{appendix: hyper-parameter sensitiviy}

In this section, we conduct two ablation studies. We begin by evaluating the OOD detection performance using different statistical feature representations, including the maximum \( m(\mathbf{x}) \), the mean \( \mu(\mathbf{x}) \), and a combination of \( \mu(\mathbf{x}) \) and the standard deviation \( \sigma(\mathbf{x}) \). This is followed by an analysis using the median value within each activation map, as well as the entropy computed over individual activation maps. For the purpose of this study, we limit our experiment to SUN, Place, Texture and iNaturalist (ImageNet benchmark)

We evaluate OOD detection performance using different formulations of \( \approach \) on the CIFAR and ImageNet benchmark, as formulated in Equation~\ref{eq:davis_formulation}. In particular for CIFAR benchmark, we vary the parameter \( \gamma \) over the set \(\{1.0, 2.0, 3.0, 4.0\}\) and compare the results using the energy score, as summarized in Table~\ref{table:ablation_cifar_10} and Table~\ref{table:ablation_cifar_100}. Similarly for ImageNet benchmark, we vary the parameter \( \gamma \) over the set \(\{0.5, 1.0, 1.5, 2.0\}\) and compare the results using the energy score, as summarized in Table~\ref{table:ablation_imagenet}.  Notably, we do not incorporate any existing techniques in this experiment; instead, we directly report the FPR95 and AUROC scores using the standard energy-based score.

\begin{subequations}
        \begin{align}
            h(\mathbf{x}) &= \mu(\mathbf{x}) \\
            h(\mathbf{x}) &= \mu(\mathbf{x}) + \gamma\sigma(\mathbf{x}) \\
            h(\mathbf{x}) &= m(\mathbf{x})
        \end{align}
        \label{eq:davis_formulation}
\end{subequations}

From Table~\ref{table:ablation_imagenet}, Table~\ref{table:ablation_cifar_10} and ~\ref{table:ablation_cifar_100}, we observe that as \( \gamma \) increases, the improvement in OOD detection performance gradually diminishes. When \( \gamma \) becomes very large, the performance saturates, leading to marginal gains or stagnation. Additionally, as \( \gamma \) increases, the behavior of the score function tends to converge to that of using \( m(\mathbf{x}) \) alone.
 
\begin{landscape}
\begin{table}[ht]
    \centering
     \caption{Ablation study of applying $\approach$ under different formulations on ImageNet-1K. The results are based solely on energy scores, with $\approach$ not combined with any other post-hoc methods. $\boldsymbol{\downarrow}$ indicates lower values are better and $\boldsymbol{\uparrow}$ indicates larger values are better.}
    \resizebox{0.85\linewidth}{!}{
    \begin{tabular}{ l l cc cc cc cc cc}
    \toprule
    \multirow{2}{*}{\textbf{Model}} & \multirow{2}{*}{$\approach$} & \multicolumn{2}{c}{\textbf{SUN}} & \multicolumn{2}{c}{\textbf{Place}} & \multicolumn{2}{c}{\textbf{Texture}} & \multicolumn{2}{c}{\textbf{iNaturalist}} & \multicolumn{2}{c}{\textbf{Average}}  \\
    \cmidrule(lr){3-4} \cmidrule(lr){5-6} \cmidrule(lr){7-8} \cmidrule(lr){9-10} \cmidrule(lr){11-12}

    && \textbf{FPR95} $\downarrow$ & \textbf{AUROC} $\uparrow$ & \textbf{FPR95} $\downarrow$ & \textbf{AUROC} $\uparrow$ & \textbf{FPR95} $\downarrow$ & \textbf{AUROC} $\uparrow$ & \textbf{FPR95} $\downarrow$ & \textbf{AUROC} $\uparrow$ & \textbf{FPR95} $\downarrow$ & \textbf{AUROC} $\uparrow$ \\
    \midrule
    \multirow{6}{*}{DenseNet-121}  & $\mu(\mathbf{x})$                          & 52.51 & 87.27 & 58.24 & 85.05 & 52.22 & 85.42 & 39.75 & 92.66 & 50.68 & 87.60 \\
                                   & $\mu(\mathbf{x}) + 0.5\sigma(\mathbf{x})$  & 48.33 & 88.64 & 56.95 & 85.80 & 37.50 & 90.73 & 31.69 & 94.14 & 43.62 & 89.83 \\
                                   & $\mu(\mathbf{x}) + 1.0\sigma(\mathbf{x})$  & 48.35 & 88.81 & 57.28 & 85.58 & 31.05 & 92.67 & 30.51 & 94.17 & 41.80 & 90.30 \\
                                   & $\mu(\mathbf{x}) + 1.5\sigma(\mathbf{x})$  & 48.31 & 88.75 & 57.99 & 85.24 & 27.38 & 93.61 & 30.40 & 93.99 & 41.02 & 90.40 \\
                                   & $\mu(\mathbf{x}) + 2.0\sigma(\mathbf{x})$  & 48.40 & 88.64 & 58.38 & 84.94 & 25.37 & 94.14 & 30.31 & 93.79 & 40.62 & 90.38 \\
                                   & $m(\mathbf{x})$                            & 48.55 & 88.03 & 59.36 & 83.53 & 27.54 & 93.62 & 33.74 & 92.20 & 42.30 & 89.34 \\
    \midrule
    \multirow{6}{*}{ResNet-50}     & $\mu(\mathbf{x})$                          & 58.82 & 86.58 & 65.99 & 83.96 & 52.43 & 86.72 & 53.74 & 90.62 & 57.74 & 86.97 \\
                                   & $\mu(\mathbf{x}) + 0.5\sigma(\mathbf{x})$  & 54.55 & 87.36 & 63.27 & 84.26 & 34.40 & 91.81 & 39.60 & 92.80 & 47.95 & 89.06 \\
                                   & $\mu(\mathbf{x}) + 1.0\sigma(\mathbf{x})$  & 53.82 & 87.18 & 63.12 & 83.72 & 27.36 & 93.80 & 35.92 & 93.14 & 45.05 & 89.46 \\
                                   & $\mu(\mathbf{x}) + 1.5\sigma(\mathbf{x})$  & 54.41 & 86.83 & 63.69 & 83.10 & 23.39 & 94.77 & 35.01 & 93.10 & 44.12 & 89.45 \\
                                   & $\mu(\mathbf{x}) + 2.0\sigma(\mathbf{x})$  & 54.71 & 86.47 & 64.06 & 82.55 & 21.13 & 95.32 & 34.50 & 92.95 & 43.60 & 89.33 \\
                                   & $m(\mathbf{x})$                            & 53.12 & 85.49 & 63.58 & 80.54 & 20.94 & 95.09 & 36.11 & 91.29 & 43.44 & 88.10 \\ 
    \midrule
    \multirow{6}{*}{MobileNet-v2}  & $\mu(\mathbf{x})$                          & 59.60 & 86.16 & 66.36 & 83.15 & 54.82 & 86.57 & 55.33 & 90.37 & 59.03 & 86.56 \\
                                   & $\mu(\mathbf{x}) + 0.5\sigma(\mathbf{x})$  & 56.37 & 87.07 & 64.55 & 83.69 & 38.76 & 90.99 & 44.87 & 92.34 & 51.14 & 88.53  \\
                                   & $\mu(\mathbf{x}) + 1.0\sigma(\mathbf{x})$  & 55.44 & 87.23 & 64.27 & 83.59 & 31.40 & 92.82 & 41.02 & 92.84 & 48.03 & 89.12 \\
                                   & $\mu(\mathbf{x}) + 1.5\sigma(\mathbf{x})$  & 55.23 & 87.22 & 64.45 & 83.39 & 27.70 & 93.77 & 39.68 & 92.97 & 46.76 & 89.33 \\
                                   & $\mu(\mathbf{x}) + 2.0\sigma(\mathbf{x})$  & 55.47 & 87.15 & 64.82 & 83.18 & 25.78 & 94.33 & 38.94 & 92.98 & 46.25 & 89.41 \\
                                   & $m(\mathbf{x})$                            & 53.64 & 87.26 & 63.61 & 82.93 & 25.87 & 94.03 & 40.25 & 92.22 & 45.84 & 89.11 \\
    \midrule
    \multirow{6}{*}{EfficientNet-b0} & $\mu(\mathbf{x})$                        & 85.01 & 72.86 & 86.06 & 70.99 & 75.99 & 75.86 & 78.91 & 79.78 & 81.49 & 74.87 \\
                                   & $\mu(\mathbf{x}) + 0.5\sigma(\mathbf{x})$  & 62.63 & 83.98 & 70.06 & 79.49 & 24.01 & 95.43 & 47.51 & 89.36 & 51.05 & 87.06 \\
                                   & $\mu(\mathbf{x}) + 1.0\sigma(\mathbf{x})$  & 61.07 & 84.67 & 70.09 & 79.51 & 16.01 & 96.97 & 46.69 & 89.09 & 48.47 & 87.56 \\
                                   & $\mu(\mathbf{x}) + 1.5\sigma(\mathbf{x})$  & 61.56 & 84.72 & 70.87 & 79.24 & 13.55 & 97.37 & 47.65 & 88.67 & 48.41 & 87.50 \\
                                   & $\mu(\mathbf{x}) + 2.0\sigma(\mathbf{x})$  & 61.85 & 84.67 & 71.31 & 79.01 & 12.18 & 97.52 & 48.25 & 88.35 & 48.40 & 87.39 \\
                                   & $m(\mathbf{x})$                            & 69.42 & 80.74 & 79.09 & 73.67 & 15.82 & 96.38 & 63.68 & 80.63 & 57.00 & 82.85 \\                                                            
    \bottomrule
    \end{tabular}
    }
    \label{table:ablation_imagenet}
\end{table}

\end{landscape}

\begin{landscape}
\begin{table}[ht]
    \centering
    \caption{Ablation study of applying $\approach$ under different formulations on CIFAR-10. The results are based solely on energy scores, with $\approach$ not combined with any other post-hoc methods. $\boldsymbol{\downarrow}$ indicates lower values are better and $\boldsymbol{\uparrow}$ indicates larger values are better.}
    \resizebox{\linewidth}{!}{
    \begin{tabular}{l c l cc cc cc cc cc cc cc}
    \toprule
     \multirow{2}{*}{\textbf{Dataset}} & \multirow{2}{*}{\textbf{$\approach$}} & \multirow{2}{*}{\textbf{Model}} & \multicolumn{2}{c}{\textbf{SVHN}} & \multicolumn{2}{c}{\textbf{Place365}} & \multicolumn{2}{c}{\textbf{iSUN}} & \multicolumn{2}{c}{\textbf{Textures}} & \multicolumn{2}{c}{\textbf{LSUN-c}}  & \multicolumn{2}{c}{\textbf{LSUN-r}} & \multicolumn{2}{c}{\textbf{Average}}  \\
    \cmidrule(lr){4-5} \cmidrule(lr){6-7} \cmidrule(lr){8-9} \cmidrule(lr){10-11} \cmidrule(lr){12-13} \cmidrule(lr){14-15} \cmidrule(lr){16-17}
    &&& \textbf{FPR95} $\downarrow$ & \textbf{AUROC} $\uparrow$ & \textbf{FPR95} $\downarrow$ & \textbf{AUROC} $\uparrow$ & \textbf{FPR95} $\downarrow$ & \textbf{AUROC} $\uparrow$ & \textbf{FPR95} $\downarrow$ & \textbf{AUROC} $\uparrow$ & \textbf{FPR95} $\downarrow$ & \textbf{AUROC} $\uparrow$ & \textbf{FPR95} $\downarrow$ & \textbf{AUROC} $\uparrow$ & \textbf{FPR95} $\downarrow$ & \textbf{AUROC} $\uparrow$ \\
    \midrule
    \multirow{24}{*}{CIFAR-10} & \multirow{3}{*}{ResNet-18} & $\mu(\mathbf{x})$     & 44.32 & 94.04 & 41.31 & 91.73 & 35.46 & 94.64 & 50.39 & 91.12 & 9.77 & 98.19 & 32.41 & 95.16 & 35.61 & 94.14 \\
                                   && $\mu(\mathbf{x}) + 1.0\sigma(\mathbf{x})$     & 23.17 & 96.07 & 37.39 & 92.89 & 24.13 & 96.45 & 33.60 & 94.96 & 6.35 & 98.69 & 21.25 & 96.81 & 24.32 & 95.98 \\
                                   && $\mu(\mathbf{x}) + 2.0\sigma(\mathbf{x})$     & 20.51 & 96.36 & 36.60 & 93.08 & 21.74 & 96.77 & 28.74 & 95.62 & 5.80 & 98.77 & 18.78 & 97.09 & 22.03 & 96.28 \\
                                   && $\mu(\mathbf{x}) + 3.0\sigma(\mathbf{x})$     & 19.83 & 96.47 & 36.11 & 93.16 & 20.48 & 96.90 & 26.77 & 95.89 & 5.62 & 98.79 & 17.75 & 97.21 & 21.09 & 96.40 \\
                                   && $\mu(\mathbf{x}) + 4.0\sigma(\mathbf{x})$     & 19.23 & 96.52 & 35.59 & 93.20 & 19.82 & 96.97 & 25.60 & 96.03 & 5.53 & 98.81 & 17.08 & 97.27 & 20.48 & 96.47 \\
                                   && $m(\mathbf{x})$                               & 19.81 & 96.29 & 32.32 & 93.73 & 15.79 & 97.37 & 21.90 & 96.44 & 5.83 & 98.73 & 13.63 & 97.61 & 18.21 & 96.69 \\
    \cmidrule(lr){2-17}
                               & \multirow{3}{*}{ResNet-34} & $\mu(\mathbf{x})$     & 35.44 & 93.76 & 38.15 & 92.27 & 19.90 & 96.87 & 42.52 & 92.54 & 3.38 & 99.10 & 16.86 & 97.20 & 26.04 & 95.29 \\
                               && $\mu(\mathbf{x}) + 1.0\sigma(\mathbf{x})$         & 29.12 & 94.81 & 33.66 & 93.45 & 12.81 & 97.70 & 27.71 & 95.66 & 3.28 & 99.22 & 11.35 & 97.89 & 19.65 & 96.45 \\
                               && $\mu(\mathbf{x}) + 2.0\sigma(\mathbf{x})$         & 28.70 & 94.88 & 32.81 & 93.62 & 11.50 & 97.83 & 24.41 & 96.18 & 3.39 & 99.22 & 10.55 & 97.99 & 18.56 & 96.62 \\
                               && $\mu(\mathbf{x}) + 3.0\sigma(\mathbf{x})$         & 28.96 & 94.90 & 32.66 & 93.69 & 11.31 & 97.88 & 23.07 & 96.39 & 3.47 & 99.22 & 10.35 & 98.04 & 18.30 & 96.69 \\
                               && $\mu(\mathbf{x}) + 4.0\sigma(\mathbf{x})$         & 28.88 & 94.91 & 32.29 & 93.72 & 10.96 & 97.91 & 22.34 & 96.50 & 3.47 & 99.21 & 10.15 & 98.06 & 18.01 & 96.72 \\
                               && $m(\mathbf{x})$                                   & 29.29 & 94.59 & 29.59 & 94.21 & 10.17 & 98.03 & 20.23 & 96.73 & 4.00 & 99.10 &  9.45 & 98.16 & 17.12 & 96.81 \\
    \cmidrule(lr){2-17}
                                & \multirow{3}{*}{DenseNet-101} & $\mu(\mathbf{x})$                         & 37.91 & 93.59 & 36.38 & 92.39 & 7.83 & 98.23 & 43.85 & 90.49 & 1.95 & 99.47 & 7.34 & 98.34 & 22.54 & 95.42 \\
                                                           && $\mu(\mathbf{x}) + 1.0\sigma(\mathbf{x})$     & 27.94 & 95.62 & 32.66 & 93.32 & 5.52 & 98.63 & 29.96 & 94.64 & 3.68 & 99.17 & 5.43 & 98.70 & 17.53 & 96.68 \\
                                                           && $\mu(\mathbf{x}) + 2.0\sigma(\mathbf{x})$     & 25.73 & 95.87 & 32.50 & 93.42 & 5.65 & 98.65 & 26.79 & 95.49 & 4.66 & 99.02 & 5.44 & 98.71 & 16.79 & 96.86 \\
                                                           && $\mu(\mathbf{x}) + 3.0\sigma(\mathbf{x})$     & 24.75 & 95.93 & 32.08 & 93.43 & 5.77 & 98.64 & 25.23 & 95.84 & 5.13 & 98.94 & 5.63 & 98.70 & 16.43 & 96.91 \\
                                                           && $\mu(\mathbf{x}) + 4.0\sigma(\mathbf{x})$     & 24.20 & 95.96 & 31.95 & 93.44 & 5.80 & 98.63 & 24.04 & 96.03 & 5.47 & 98.88 & 5.64 & 98.69 & 16.18 & 96.94 \\
                                                           && $m(\mathbf{x})$                               & 30.50 & 94.50 & 33.04 & 93.28 & 7.81 & 98.29 & 25.16 & 95.85 & 9.95 & 98.19 & 7.32 & 98.37 & 18.96 & 96.41 \\
    \cmidrule(lr){2-17}
                                & \multirow{3}{*}{MobileNet-v2} & $\mu(\mathbf{x})$                         & 75.83 & 85.85 & 44.98 & 90.62 & 29.68 & 95.03 & 48.67 & 91.19 &  9.54 & 98.12 & 29.80 & 95.10 & 39.75 & 92.65 \\
                                                           && $\mu(\mathbf{x}) + 1.0\sigma(\mathbf{x})$     & 66.48 & 88.29 & 43.56 & 90.97 & 24.53 & 96.08 & 38.67 & 93.72 & 11.88 & 97.81 & 23.22 & 96.22 & 34.72 & 93.85 \\
                                                           && $\mu(\mathbf{x}) + 2.0\sigma(\mathbf{x})$     & 63.81 & 88.91 & 43.94 & 90.92 & 23.16 & 96.31 & 35.46 & 94.40 & 13.38 & 97.62 & 21.98 & 96.47 & 33.62 & 94.11 \\
                                                           && $\mu(\mathbf{x}) + 3.0\sigma(\mathbf{x})$     & 62.07 & 89.18 & 43.87 & 90.86 & 22.38 & 96.40 & 33.24 & 94.71 & 14.21 & 97.50 & 21.32 & 96.57 & 32.85 & 94.20 \\
                                                           && $\mu(\mathbf{x}) + 4.0\sigma(\mathbf{x})$     & 60.98 & 89.32 & 43.59 & 90.81 & 21.82 & 96.44 & 31.79 & 94.88 & 14.55 & 97.42 & 20.56 & 96.62 & 32.21 & 94.25 \\
                                                           && $m(\mathbf{x})$                               & 69.61 & 87.56 & 44.33 & 91.00 & 23.33 & 96.26 & 37.87 & 94.11 & 17.62 & 97.02 & 22.48 & 96.40 & 35.87 & 93.73 \\
    \bottomrule
    \end{tabular}
    }
    \label{table:ablation_cifar_10}
\end{table}
\end{landscape}

\begin{landscape}
\begin{table}[ht]
    \centering
    \caption{Ablation study of applying $\approach$ under different formulations on CIFAR-100. The results are based solely on energy scores, with $\approach$ not combined with any other post-hoc methods. $\boldsymbol{\downarrow}$ indicates lower values are better and $\boldsymbol{\uparrow}$ indicates larger values are better.}
    \resizebox{\linewidth}{!}{
    \begin{tabular}{l c l cc cc cc cc cc cc cc}
    \toprule
     \multirow{2}{*}{\textbf{Dataset}} & \multirow{2}{*}{\textbf{$\approach$}} & \multirow{2}{*}{\textbf{Model}} & \multicolumn{2}{c}{\textbf{SVHN}} & \multicolumn{2}{c}{\textbf{Place365}} & \multicolumn{2}{c}{\textbf{iSUN}} & \multicolumn{2}{c}{\textbf{Textures}} & \multicolumn{2}{c}{\textbf{LSUN-c}}  & \multicolumn{2}{c}{\textbf{LSUN-r}} & \multicolumn{2}{c}{\textbf{Average}}  \\
    \cmidrule(lr){4-5} \cmidrule(lr){6-7} \cmidrule(lr){8-9} \cmidrule(lr){10-11} \cmidrule(lr){12-13} \cmidrule(lr){14-15} \cmidrule(lr){16-17}
    &&& \textbf{FPR95} $\downarrow$ & \textbf{AUROC} $\uparrow$ & \textbf{FPR95} $\downarrow$ & \textbf{AUROC} $\uparrow$ & \textbf{FPR95} $\downarrow$ & \textbf{AUROC} $\uparrow$ & \textbf{FPR95} $\downarrow$ & \textbf{AUROC} $\uparrow$ & \textbf{FPR95} $\downarrow$ & \textbf{AUROC} $\uparrow$ & \textbf{FPR95} $\downarrow$ & \textbf{AUROC} $\uparrow$ & \textbf{FPR95} $\downarrow$ & \textbf{AUROC} $\uparrow$ \\
    \midrule
    \multirow{24}{*}{CIFAR-100} & \multirow{3}{*}{ResNet-18} & $\mu(\mathbf{x})$                        & 66.64 & 89.53 & 81.23 & 76.84 & 73.67 & 82.01 & 85.30 & 75.68 & 48.01 & 91.63 & 70.30 & 83.38 & 70.86 & 83.18 \\
                                                           && $\mu(\mathbf{x}) + 1.0\sigma(\mathbf{x})$ & 50.36 & 92.68 & 79.76 & 77.60 & 66.90 & 86.55 & 70.83 & 84.19 & 41.42 & 93.01 & 64.24 & 87.16 & 62.25 & 86.86 \\
                                                           && $\mu(\mathbf{x}) + 2.0\sigma(\mathbf{x})$ & 44.41 & 93.44 & 79.20 & 77.62 & 62.91 & 87.94 & 63.87 & 86.71 & 39.10 & 93.34 & 61.13 & 88.32 & 58.44 & 87.89 \\
                                                           && $\mu(\mathbf{x}) + 3.0\sigma(\mathbf{x})$ & 42.01 & 93.75 & 78.99 & 77.58 & 61.13 & 88.60 & 59.89 & 87.89 & 38.56 & 93.48 & 59.76 & 88.86 & 56.72 & 88.36 \\
                                                           && $\mu(\mathbf{x}) + 4.0\sigma(\mathbf{x})$ & 40.43 & 93.92 & 78.88 & 77.54 & 59.96 & 88.98 & 57.50 & 88.55 & 38.04 & 93.55 & 58.75 & 89.18 & 55.59 & 88.62 \\
                                                           && $m(\mathbf{x})$                           & 38.40 & 94.25 & 77.62 & 78.98 & 56.43 & 89.99 & 55.73 & 89.14 & 35.90 & 94.17 & 55.87 & 89.95 & 53.32 & 89.41 \\
    \cmidrule(lr){2-17}
                                & \multirow{3}{*}{ResNet-34} & $\mu(\mathbf{x})$                         & 57.79 & 89.80 & 81.17 & 77.25 & 71.83 & 84.14 & 86.77 & 75.82 & 55.56 & 89.92 & 68.70 & 84.93 & 70.30 & 83.64 \\
                                                            && $\mu(\mathbf{x}) + 1.0\sigma(\mathbf{x})$ & 38.34 & 93.68 & 80.15 & 77.72 & 63.63 & 87.91 & 71.26 & 84.56 & 45.92 & 92.16 & 62.71 & 88.03 & 60.33 & 87.34 \\
                                                            && $\mu(\mathbf{x}) + 2.0\sigma(\mathbf{x})$ & 33.40 & 94.61 & 79.93 & 77.63 & 61.00 & 88.99 & 64.01 & 87.13 & 42.73 & 92.74 & 60.97 & 88.91 & 57.01 & 88.34 \\
                                                            && $\mu(\mathbf{x}) + 3.0\sigma(\mathbf{x})$ & 31.07 & 95.00 & 80.04 & 77.53 & 59.46 & 89.50 & 59.75 & 88.31 & 41.21 & 93.00 & 60.05 & 89.32 & 55.26 & 88.78 \\
                                                            && $\mu(\mathbf{x}) + 4.0\sigma(\mathbf{x})$ & 30.17 & 95.21 & 80.44 & 77.45 & 58.88 & 89.78 & 57.50 & 88.99 & 40.66 & 93.14 & 59.67 & 89.56 & 54.55 & 89.02 \\
                                                           && $m(\mathbf{x})$                            & 29.56 & 95.19 & 77.68 & 78.81 & 55.37 & 90.51 & 56.29 & 89.14 & 38.37 & 93.37 & 56.27 & 90.23 & 52.26 & 89.54 \\
    \cmidrule(lr){2-17}
                                & \multirow{3}{*}{DenseNet-101} & $\mu(\mathbf{x})$                     & 70.99 & 86.66 & 77.12 & 76.94 & 64.28 & 83.92 & 83.60 & 67.47 & 11.45 & 97.89 & 56.08 & 86.84 & 60.59 & 83.29 \\
                                                           && $\mu(\mathbf{x}) + 1.0\sigma(\mathbf{x})$ & 48.54 & 91.64 & 72.99 & 79.57 & 48.89 & 90.36 & 65.82 & 80.52 & 12.65 & 97.80 & 44.65 & 91.29 & 48.92 & 88.53 \\
                                                           && $\mu(\mathbf{x}) + 2.0\sigma(\mathbf{x})$ & 45.38 & 91.93 & 71.75 & 79.77 & 43.78 & 91.65 & 58.81 & 83.94 & 13.18 & 97.63 & 41.41 & 92.13 & 45.72 & 89.51 \\
                                                           && $\mu(\mathbf{x}) + 3.0\sigma(\mathbf{x})$ & 45.23 & 91.91 & 72.11 & 79.78 & 42.21 & 92.16 & 55.85 & 85.47 & 13.92 & 97.53 & 40.52 & 92.45 & 44.97 & 89.88 \\
                                                           && $\mu(\mathbf{x}) + 4.0\sigma(\mathbf{x})$ & 45.67 & 91.86 & 72.40 & 79.76 & 41.30 & 92.43 & 54.45 & 86.33 & 14.39 & 97.46 & 40.31 & 92.61 & 44.75 & 90.07 \\
                                                           && $m(\mathbf{x})$                           & 52.32 & 89.41 & 73.49 & 79.62 & 35.76 & 93.37 & 49.88 & 88.07 & 18.58 & 96.44 & 37.27 & 92.98 & 44.55 & 89.98 \\
    \cmidrule(lr){2-17}
                                & \multirow{3}{*}{MobileNet-v2} & $\mu(\mathbf{x})$                     & 69.65 & 85.98 & 81.26 & 75.21 & 78.17 & 83.33 & 80.02 & 78.63 & 50.19 & 89.42 & 76.59 & 84.06 & 72.65 & 82.77 \\
                                                           && $\mu(\mathbf{x}) + 1.0\sigma(\mathbf{x})$ & 47.36 & 91.05 & 80.49 & 74.45 & 68.34 & 85.74 & 56.91 & 87.66 & 38.63 & 92.57 & 68.52 & 85.76 & 60.04 & 86.20 \\
                                                           && $\mu(\mathbf{x}) + 2.0\sigma(\mathbf{x})$ & 42.01 & 92.25 & 81.14 & 73.72 & 66.13 & 86.22 & 47.16 & 90.17 & 35.84 & 93.39 & 67.00 & 86.00 & 56.55 & 86.96 \\
                                                           && $\mu(\mathbf{x}) + 3.0\sigma(\mathbf{x})$ & 40.00 & 92.73 & 81.64 & 73.23 & 64.99 & 86.36 & 42.70 & 91.25 & 34.67 & 93.75 & 66.30 & 86.04 & 55.05 & 87.23 \\
                                                           && $\mu(\mathbf{x}) + 4.0\sigma(\mathbf{x})$ & 38.28 & 92.99 & 81.77 & 72.90 & 64.03 & 86.42 & 40.00 & 91.85 & 33.51 & 93.93 & 65.60 & 86.02 & 53.87 & 87.35 \\
                                                           && $m(\mathbf{x})$                           & 41.94 & 92.00 & 81.49 & 73.31 & 64.52 & 86.40 & 44.22 & 90.68 & 34.22 & 93.72 & 65.30 & 86.10 & 55.28 & 87.04 \\
    \bottomrule
    \end{tabular}
    }
    \label{table:ablation_cifar_100}
\end{table}
\end{landscape}

\end{document}